\documentclass[twoside,11pt]{article}

\usepackage{blindtext}

\usepackage{jmlr2e}

\usepackage[utf8]{inputenc} 
\usepackage[T1]{fontenc}    
\usepackage{hyperref}       
\usepackage{url}            
\usepackage{booktabs}       
\usepackage{amsfonts}       
\usepackage{nicefrac}       
\usepackage{microtype}      
\usepackage{xcolor}         

\usepackage{graphicx}
\usepackage{amsmath}
\usepackage{pifont}
\usepackage{bm}
\usepackage{multirow}
\usepackage{algorithm}
\usepackage{algpseudocode}
\usepackage{amssymb}
\usepackage{amsfonts}
\usepackage{pbox}
\usepackage{cleveref}
\usepackage[font={footnotesize}]{subcaption}
\usepackage[font={footnotesize}]{caption}
\usepackage{rotating}
\usepackage{colortbl}
\usepackage{wrapfig}
\usepackage{accents}

\newcommand{\norm}[1]{\left\| #1\right\|}

\newcommand{\hs}{\hat{s}}

\newtheorem{prop}{Proposition}

\newtheorem{defi}{Definition}
\newtheorem{rmk}{Remark}

\definecolor{darkblue}{RGB}{0,139,139}
\definecolor{darkgreen}{RGB}{34,139,34}
\definecolor{darkred}{RGB}{139,0,0}

\usepackage{lastpage}
\jmlrheading{24}{2024}{1-\pageref{LastPage}}{9/24; Revised x/xx}{x/xx}{21-0000}{Weiye Zhao and Feihan Li}

\ShortHeadings{Absolute State-wise Constrained Policy Optimization}{Zhao$^*$, Li$^*$, Sun, Yang, Chen, Wei and Liu}
\firstpageno{1}

\begin{document}

\title{Absolute State-wise Constrained Policy Optimization: High-Probability State-wise Constraints Satisfaction}

\author{\name Weiye Zhao\thanks{Equal Contribution} \email weiyezha@andrew.cmu.edu 
       \AND
       \name Feihan Li$^*$ \email feihanl@andrew.cmu.edu
       \AND
       \name Yifan Sun \email yifansu2@andrew.cmu.edu 
       \AND
       \name Yujie Yang \email yujieyan@andrew.cmu.edu 
       \AND
       \name Rui Chen \email ruic3@andrew.cmu.edu 
       \AND
       \name Tianhao Wei \email twei2@andrew.cmu.edu 
       \AND
       \name Changliu Liu \email cliu6@andrew.cmu.edu \\
       \addr Robotics Institute\\
       Carnegie Mellon University\\
       Pittsburgh, PA 15213, USA
       }

\editor{My editor}

\maketitle

\begin{abstract}
Enforcing state-wise safety constraints is critical for the application of reinforcement learning (RL) in real-world problems, such as autonomous driving and robot manipulation.
However, existing safe RL methods only enforce state-wise constraints in expectation or enforce hard state-wise constraints with strong assumptions. The former does not exclude the probability of safety violations, while the latter is impractical. Our insight is that although it is intractable to guarantee hard state-wise constraints in a model-free setting, we can enforce state-wise safety with high probability while excluding strong assumptions. To accomplish the goal, we propose Absolute State-wise Constrained Policy Optimization (ASCPO), a novel general-purpose policy search algorithm that guarantees high-probability state-wise constraint satisfaction for stochastic systems.
We demonstrate the effectiveness of our approach by training neural network policies for extensive robot locomotion tasks, where the agent must adhere to various state-wise safety constraints. Our results show that ASCPO significantly outperforms existing methods in handling state-wise constraints across challenging continuous control tasks, highlighting its potential for real-world applications.
\end{abstract}

\begin{keywords}
  State-wise Constraint, Constrained Policy Optimization, High-probability Constraints Satisfaction, Trust Region, Safe Reinforcement Learning
\end{keywords}

\addtocontents{toc}{\protect\setcounter{tocdepth}{0}}

\section{Introduction}

In the reinforcement learning (RL) literature, safe RL is a specific branch that considers constraint satisfaction in addition to reward maximization \citep{achiam2017constrained}.
The explicit enforcement of constraint satisfaction is often critical to training RL agents for real-world deployments~\cite{he2023hierarchical, noren2021safe, zhao2020experimental}.
That is because even if RL agents are penalized for violating safety constraints, they may still trade safety for overall higher rewards~\cite{zhao2019stochastic}, which is unacceptable in many practical scenarios.
For example, a self-driving car should never get too close to a pedestrian to save travel time.
Safe RL aims to address such problems by treating the satisfaction of safety constraints separately from the maximization of task objectives~\cite{wei2024meta, zhao2024absolute}.
To achieve that goal, it is important to properly model the safety constraints.
Traditional RL builds on Markov Decision Processes (MDPs), which produce trajectories of agent states and actions as the result of agent-environment interactions.
Under the framework, safety constraints can be defined in various forms.
In general, constraints can be defined in a cumulative sense \citep{ray2019benchmarking, stooke2020responsive, achiam2017constrained, yang2020projection} that the total safety violation from all states should be bounded.
Constraints can also be defined state-wise, such that safety is ensured at each time step either in expectation \citep{pham2018optlayer, amani2022doubly, zhao2024statewise} or 
almost surely under assumptions of system dynamics knowledge~\citep{berkenkamp2017safembrl, fisac2018general, wachi2020safe, shi2023near, wachi2024safe, zhao2023safety}.

\begin{figure}[t]
    \centering
    \raisebox{-\height}{\includegraphics[width=0.9\linewidth]{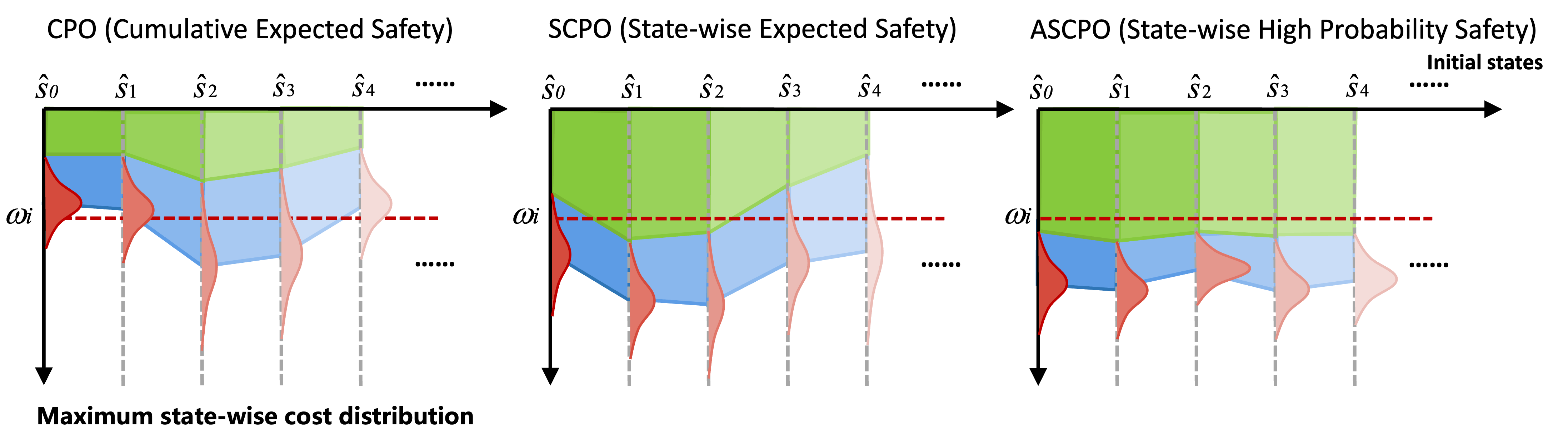}}

    \captionsetup{width=\linewidth}
    \caption{Explanation of ASCPO principles. For simplicity, the distribution is assumed to be Gaussian. Green and blue represent the maximum and expectation of maximum state-wise cost distribution respectively. Y-axis represents the distribution of maximum state-wise cost samples (introduced in \Cref{sec: preliminaries}) across different initial states. ASCPO is designed to constrain the maximum state-wise cost under safety threshold ($w_i$) while SCPO only focuses on constraining the expectation of state-wise cost and CPO~\citep{achiam2017constrained} lacks the capability to ensure state-wise safety.} 
    \label{fig: cpo vs scpo vs ascpo}
\end{figure}

Ideally, safety constraints should be strictly satisfied at every time step~\cite{he2023autocost, zhao2023learn}.
However, strictly satisfying safety constraints in a model-free setting is challenging, especially when the system has dynamic limits.
Hence, in this paper, we consider a slightly relaxed yet still strong safety condition where constraints are guaranteed for \textit{each state} with a configurable \textit{high probability}~\citep{wachi2024survey}.
Notably, such a definition poses significantly stricter safety conditions than cumulative safety and is much more difficult to solve.
The most recent advancement in that direction is state-wise constrained policy optimization (SCPO)~\citep{zhao2024statewise}, which enforces state-wise constraints.
However, SCPO only considers constraint satisfaction in expectation without considering the variance of violations.
As a result, safety constraints might still be violated at a non-trivial rate if the violations are long-tail distributed. 

In this paper, we propose absolute state-wise constrained policy optimization (ASCPO) to guarantee state-wise safety constraints with a high probability.
In addition to the \textit{expected} constraint violation, which SCPO considers, ASCPO also considers the \textit{variance} to constrain the probability upper bound of violations within a user-specified threshold.
The main idea of ASCPO is illustrated in \cref{fig: cpo vs scpo vs ascpo}.
Through experiments, we demonstrate that ASCPO can achieve high rewards in high-dimensional tasks with state-of-the-art low safety violations. Our code is available on Github.\footnote{\url{https://github.com/intelligent-control-lab/Absolute-State-wise-Constrained-Policy-Optimization}}
Our contribution is summarized below:
\begin{itemize}
    \item To the best of the authors' knowledge, the proposed approach is the first policy optimization method to ensure high-probability satisfaction of state-wise safety constraints without assuming knowledge of the underlying dynamics. 
\end{itemize}

\section{Related Work}

In safe reinforcement learning (RL), there are two main types of safety specifications: cumulative safety and state-wise safety. Here we primarily discuss state-wise safe RL methods. Interested readers can refer to survey papers~\citep{liu2021policy,gu2022review,brunke2022safe,wachi2024survey} for more comprehensive discussions.
Specifically, state-wise safety (instantaneous safety) aims to constrain the instantaneous cost at every step. In a stochastic environment, strictly satisfying a state-wise safety constraint at every step is impractical because the instantaneous cost is a random variable under a possibly unbounded distribution.
Therefore, existing methods either constrain the instantaneous cost at each step in expectation or ensure a high probability of satisfying instantaneous constraints at all steps.

\paragraph{Constraint in expectation}
To enforce expected state-wise constraints, OptLayer~\citep{pham2018optlayer} uses a safety layer that modifies potentially unsafe actions generated by a reward-maximizing policy to constraint-satisfying ones. 
~\citeauthor{zhao2024statewise} introduce state-wise constrained policy optimization (SCPO) that guarantees to constrain the maximum violation along a trajectory in expectation based on a novel Maximum MDP framework.
Similar idea of constraining the maximum violation can also be found in Hamilton-Jacobi (HJ) reachability analysis~(\citeauthor{bansal2017hamilton}), which computes a value function to represent the maximum violation in the future.
The HJ reachability value function is widely used to learn safety-oriented policies or build constraints for policy optimization in safe RL~(\citeauthor{fisac2019bridging,yu2022reachability}).
Although these methods take state-wise safety into consideration, their constraint in expectation formulations still fall short of controlling the distributions of instantaneous safety violations.

\paragraph{High-probability Constraint}
 Existing \textbf{safe exploration literature} primarily focuses on addressing high-probability and even surely state-wise constraint satisfaction~\citep{zhao2023state}, including (i) structural solutions and (ii) end-to-end solutions. Structural solutions construct a hierarchical safe agent with an upper layer generating reference actions and a lower layer performing safe action projection at every time step. To achieve this, prior knowledge about system dynamics is required~\cite{wei2022persistently, zhao2020contact, zhao2022provably, chen2024safety, chen2024real, li2023learning}, such as analytical dynamics~\citep{cheng2019end,shao2021reachability,fisac2018general,ferlez2020shieldnn}, black-box dynamics~\citep{zhao2021model, zhao2024implicit}, learned dynamics~\citep{dalal2018safe,zhang2022evaluating, bharadhwaj2020conservative, chow2019lyapunov,thananjeyan2021recovery}, or the Lipschitzness bound of dynamics~\citep{zhao2023probabilistic}. Similarly, end-to-end solutions ensure safe actions by considering the uncertainty of learned dynamics at every time step, where high-probability state-wise safety guarantees stem from the knowledge of the Lipschitzness bound of dynamics~\citep{berkenkamp2017safembrl, wachi2018safe, wachi2024safe} and constraints~\citep{ wachi2020safe} or evaluable black-box system dynamics~\citep{wachi2020safe}.

In contrast, our method directly ensures \textbf{ high-probability state-wise constraint satisfaction without assumptions on knowledge of the underlying dynamics}.

\section{Problem Formulation}
\subsection{State-wise Constrained Markov Decision Process}
\label{sec: preliminaries}
This paper studies the persistent satisfaction of cost constraints \textbf{at every step} for episodic tasks, in the framework of \textit{State-wise Constrained Markov Decision Process} (SCMDP)~\citep{zhao2023state} for finite horizon.
An finite horizon MDP finishes within $H \in \mathbb{N}$ steps, and is specified by a tuple $(\mathcal{S}, \mathcal{A}, \gamma, R, P, \mu)$, where $\mathcal{S}$ is the state space, and $\mathcal{A}$ is the control space, $R: \mathcal{S} \times \mathcal{A} \mapsto \mathbb{R}$ is the reward function, $ 0 \leq \gamma < 1$ is the discount factor, $\mu : \mathcal{S} \mapsto \mathbb{R}$
 is the initial state distribution, and $P: \mathcal{S} \times \mathcal{A} \times \mathcal{S} \mapsto \mathbb{R}$
 is the transition probability function.
$P(s'|s,a)$ is the probability of transitioning to state $s'$ given that the previous state was $s$ and the agent took action $a$ at state $s$. Building upon MDP, SCMDP introduces a set of cost functions, $C_1, C_2, \cdots, C_m$, where $C_i : \mathcal{S} \times \mathcal{A} \times \mathcal{S} \rightarrow \mathbb{R}$ maps the state action transition tuple into a cost value.
A stationary policy $\pi: \mathcal{S} \mapsto \mathcal{P}(\mathcal{A})$ is a map from states to a probability distribution over actions, with $\pi(a|s)$ denoting the probability of selecting action $a$ in state $s$. We denote the set of all stationary policies by $\Pi$. 

The goal for SCMDP is to learn a policy $\pi$ that maximizes a performance measure
$\mathcal{J}_0(\pi)\doteq \mathbb{E}_{\tau \sim \pi}\left[\sum_{t=0}^H \gamma^t R(s_t, a_t, s_{t+1})\right]$, so that the cost for every state action transition satisfies a constraint. Formally,
\begin{align}
    \label{eq: fundamental problem}
    \underset{\pi}{\textbf{max}}~\mathcal{J}_0(\pi), ~\textbf{s.t.}~ \forall i, \forall(s_t, a_t, s_{t+1}) \sim \tau, C_i(s_t, a_t, s_{t+1}) \leq w_i
\end{align}
where $\tau = [s_0, a_0, s_1, \cdots]$ and $\tau \sim \pi$, $w_i \in \mathbb{R}$. Here $\tau \sim \pi$ is shorthand for that the distribution over trajectories depends on $\pi: s_0 \sim \mu, a_t \sim \pi(\cdot | s_t), s_{t+1} \sim P(\cdot|s_t, a_t)$.

\paragraph{Restricting Maximum Cost}

\begin{figure}[t]
    \raisebox{-\height}{\includegraphics[width=\linewidth]{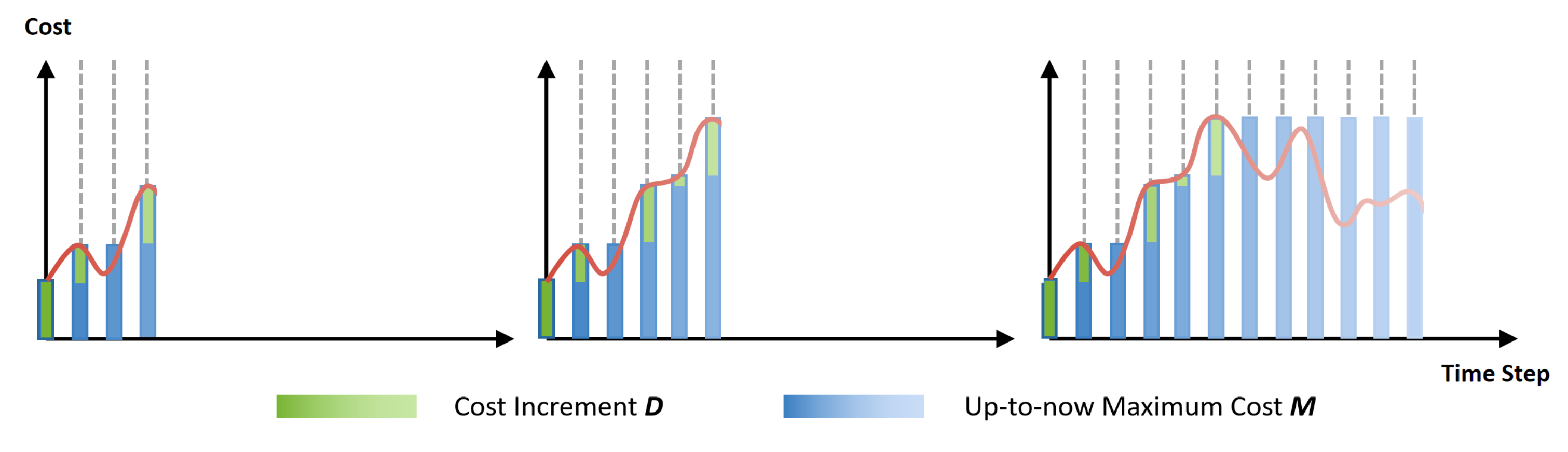}}
    \caption{Intuition of the maximum state-wise cost. The evolution of the maximum state-wise cost across a single episode is shown in the blue bars. The red curves represent the state-wise cost, while the green bars indicate the maximum state-wise cost increment at each step.
     } 
    \label{fig:MMDP}
\end{figure}
It is noteworthy that for \eqref{eq: fundamental problem}, each state-action transition pair introduces a constraint, leading to a computational complexity that increases nearly cubically as the MDP horizon ($H$) grows~\citep{zhao2024statewise}. Instead of directly constraining the cost of each possible state-action transition, \textbf{it is easier to constrain the maximum state-wise cost along the trajectory}. To efficiently compute the maximum state-wise cost, we follow \textit{ maximum Markov decision process} (MMDP)~\citep{zhao2024statewise} to introduce 
(i) a set of up-to-now maximum state-wise costs $\textbf{M} \doteq [M_1, M_2, \cdots, M_m]$ where $M_i \in \mathcal{M}_i \subset \mathbb{R}$, and (ii) a set of \textit{cost increment} 
functions, $D_1, D_2, \cdots, D_m$, where $D_i:(\mathcal{S}, \mathcal{M}_i) \times \mathcal{A} \times \mathcal{S} \mapsto [0, \mathbb{R}^+]$ maps the augmented state action transition tuple into a nonnegative cost increment.
We define the augmented state ${\hat s} = (s, \textbf{M}) \in (\mathcal{S},\mathcal{M}^m) \doteq \hat{\mathcal{S}}$, where $\hat{\mathcal{S}}$ is the augmented state space with $\mathcal{M}^m = (\mathcal{M}_1, \mathcal{M}_2, \cdots, \mathcal{M}_m)$.
Formally, 
\begin{align}
    D_i\big({\hat s}_t, a_t, {\hat s}_{t+1}\big) = \max\{C_i(s_t, a_t, s_{t+1}) - {M_{it}}, 0\} , i \in 1,\cdots ,m.
\end{align}
Setting $D_i\big({\hat s}_0, a_0, {\hat s}_{1}\big) = C_i(s_0,a_0,s_1)$, we have $M_{it} = \sum_{k=0}^{t-1} D_i\big({\hat s}_k, a_k, {\hat s}_{k+1}\big)$ for $t \geq 1 $. 
Hence, we define the \textit{maximum state-wise cost} performance sample for $\pi$ as: 
\begin{align}
\label{eq: sum of max cost}
    {\mathcal{D}_i}_{\pi}(\hs_0) =   
    \sum_{t=0}^{H} D_i\big({\hat s}_t, a_t, {\hat s}_{t+1}\big),
\end{align}
where the state action sequence $\hat\tau = [a_0, {\hat s}_1, \dots] \sim \pi$ starts with an initial state $\hs_0$, which follows initial state distribution $\mu$. The intuition of MMDP is illustrated in \Cref{fig:MMDP}. With \eqref{eq: sum of max cost}, \eqref{eq: fundamental problem} can be rewritten as:
\begin{align}
    \label{eq: fundamental problem scpo}
    \underset{\pi}{\textbf{max}} ~\mathcal{J}(\pi),~\textbf{s.t.}~ \forall i, ~{\mathcal{D}_i}_{\pi}(\hs_0) \leq w_i,
\end{align}
where $\mathcal{J}(\pi) = \mathbb{E}_{\tau \sim \pi}\left[\sum_{t=0}^H \gamma^t R({\hat s}_t, a_t, {\hat s}_{t+1})\right]$ and $R({\hat s},a,{\hat s}')\doteq R(s, a, s')$.

\begin{rmk}
  We would like to highlight the core differences between \eqref{eq: fundamental problem scpo} and Constrained Markov Decision Processes (CMDP). Although \eqref{eq: fundamental problem scpo} may appear similar to CMDP due to the inclusion of cost increments, the nature of the constraints is fundamentally different. Specifically, the constraints in \eqref{eq: fundamental problem scpo} take the form of a non-discounted summation over a finite horizon, whereas CMDP considers a discounted summation over an infinite horizon. Additionally, \eqref{eq: fundamental problem scpo} restricts the constraint satisfaction for each individual performance sample,
  whereas CMDP only requires constraint satisfaction for expected performance. Consequently, conventional techniques or theories used to solve CMDP are not applicable here.
\end{rmk}

With $R(\tau)$ being the discounted return of a trajectory with infinite horizon, we define the on-policy value function as $V_\pi({\hat s}) \doteq \mathbb{E}_{\tau \sim \pi}[R(\tau) | {\hat s}_0 = {\hat s}]$, the on-policy action-value function as $Q_\pi({\hat s},a) \doteq \mathbb{E}_{\tau \sim \pi}[R(\tau) | {\hat s}_0 = {\hat s}, a_0=a]$, and the advantage function as $A_\pi({\hat s},a) \doteq Q_\pi({\hat s}, a) - V_\pi({\hat s})$. 
Lastly, we define on-policy value functions, action-value functions, and advantage functions for the cost increments in analogy to $V_\pi$, $Q_\pi$, and $A_\pi$. Similarly, we can define $V^H_\pi$, $Q^H_\pi$, and $A^H_\pi$ for trajectory with $H$ horizon. With $D_i$ replacing $R$, respectively. 
We denote those by $V^H_{[D_i]\pi}$, $Q^H_{[D_i]\pi}$ and $A^H_{[D_i]\pi}$.

The idea of restricting the maximum cost in a trajectory is widely used in safe RL to ensure state-wise constraint satisfaction~\citep{fisac2019bridging,yu2022reachability}.
In addition to MMDP formulation, another method to achieve this goal is HJ reachability analysis~\citep{bansal2017hamilton}, which computes a value function to represent the maximum cost along the trajectory starting from the current state:
\begin{equation}
    F_{i\pi}(s_0)=\max_{t=0,1,\dots,H} C_i(s_t,a_t,s_{t+1}),
\end{equation}
where $F_{i\pi}$ is the HJ reachability value function of the $i$-th cost under policy $\pi$.
Compared with MMDP, HJ reachability analysis directly computes the maximum cost using the maximum operator, while MMDP transforms the maximum computation into a summation by introducing cost increment functions.
Mathematically, $F_{i\pi}$ in HJ reachability analysis equals $\mathcal{D}_{i\pi}$ in MMDP at the initial state of a trajectory.
Therefore, we can equivalently replace $\mathcal{D}_{i\pi}(\hat{s}_0)$ with $F_{i\pi}(s_0)$ in the constraint of problem \eqref{eq: fundamental problem scpo}.
However, the advantage of using MMDP formulation is that the summation operator in $\mathcal{D}_{i\pi}$ enables us to use existing theoretical tools to derive bounds on the maximum state-wise cost, which is crucial for constraint satisfaction guarantee~\citep{zhao2024statewise}.
In contrast, deriving similar cost bounds for Hamilton-Jacobi (HJ) reachability analysis is challenging due to the maximum operator in $F_{i\pi}$. So far, this max and the associated error bounds can only be evaluated using traditional HJ methods, and no learning-based HJ method is capable of simultaneously evaluating the max and the associated error bounds.
To the best of our knowledge, no existing HJ-reachability-based safe RL methods have established such bounds.

\subsection{Upper Probability Bound of Maximum State-wise Cost}
\label{sec: absolute bound}
Restricting every possible maximum state-wise cost performance sample in \eqref{eq: fundamental problem scpo} is impractical
since ${\mathcal{D}_i}_{\pi}(\hs_0)$ is a continuous random variable under a possibly unbounded distribution. The best we can do is to restrict the upper probability bound~\citep{pishro2014introduction} of the maximum state-wise cost performance with high confidence, defined as:

\begin{defi}[Upper Probability Bound of Constraint Satisfaction]
\label{def: abs bound}

    Given a tuple $(\mathcal{B} \in \mathbb{R}, p \in \mathbb{R}^+)$, $\mathcal{B}$ is defined as the upper probability bound with confidence ${p}$. Mathematically:
    \begin{align}
        Pr\big(\mathcal{D}_{i\pi}(\hs_0)\leq \mathcal{B}\big) \geq p~.
    \end{align}
\end{defi}

\begin{prop}
\label{prop: absolute bound definition}
For an unknown distribution of random variable $\mathcal{D}_{i\pi}(\hs_0)$, denote \\
$\mathcal{E}_{[D_i]}(\pi), \mathcal{V}_{[D_i]}(\pi)$ as the expectation and variance of the distribution, i.e. \\$\mathcal{E}_{[D_i]}(\pi) = \mathbb{E}_{\hs_0 \sim \mu, \hat\tau \sim \pi}[\mathcal{D}_{i\pi}(\hs_0)]$, $\mathcal{V}_{[D_i]}(\pi) = \mathbb{V}ar_{\hs_0 \sim \mu, \hat\tau \sim \pi}[\mathcal{D}_{i\pi}(\hs_0)]$.
$\mathcal{B}_{[D_i]k}(\pi) \doteq \mathcal{E}_{[D_i]}(\pi)+k\mathcal{V}_{[D_i]}(\pi)$ is guaranteed to be a upper probability bound of $\mathcal{D}_{i\pi}(\hs_0)$ in \Cref{def: abs bound} with confidence $p_k^\psi \doteq 1-\frac{1}{k^2 \psi+1} \in (0,1)$. Here $k$ is the probability factor ($k \geq 0$, $k \in \mathbb{R}$) and $\psi = \mathcal{V}_{min} \in \mathbb{R}^+$, where $\mathcal{V}_{min}$ is the minima of $\mathcal{V}_{[D_i]}(\pi)$.
\end{prop}

\begin{rmk}
    \Cref{prop: absolute bound definition} is proved in \Cref{sec: proof of prop absolute bound}. \Cref{prop: absolute bound definition} shows that more than $p_k^\psi$ of the samples from the distribution of $\mathcal{D}_{i\pi}(\hs_0)$ will be smaller than the bound $\mathcal{B}_{[D_i]k}(\pi)$. Given a positive constant $\psi$, we can make $p_k^\psi \rightarrow 1$ by setting a large enough $k$, so that $\mathcal{B}_{[D_i]k}(\pi)$ represents the upper probability bound of $\mathcal{D}_{i\pi}(\hs_0)$ with with a confidence level close to 1. 
\end{rmk}

\subsection{Policy Optimization Problem}
\label{sec:amdp}
In this paper, we focus on restricting the upper probability bound of maximum state-wise cost performance in SCMDP.
In accordance with \Cref{def: abs bound}, the overarching objective
is to identify a policy $\pi$ that effectively maximizes the performance measure and ensures $\mathcal{B}_k(\pi) < w$.  Mathematically,
\begin{align}
\label{eq: apo optimization original}
     \underset{\pi}{~\textbf{max}}~ \mathcal{J}(\pi),~\textbf{s.t.}~ \forall i, ~{\mathcal{E}}_{[D_i]}(\pi)+k{\mathcal{V}}_{[D_i]}(\pi)\leq w_i.
\end{align}
\begin{figure}[t]
    \raisebox{-\height}{\includegraphics[width=\linewidth]{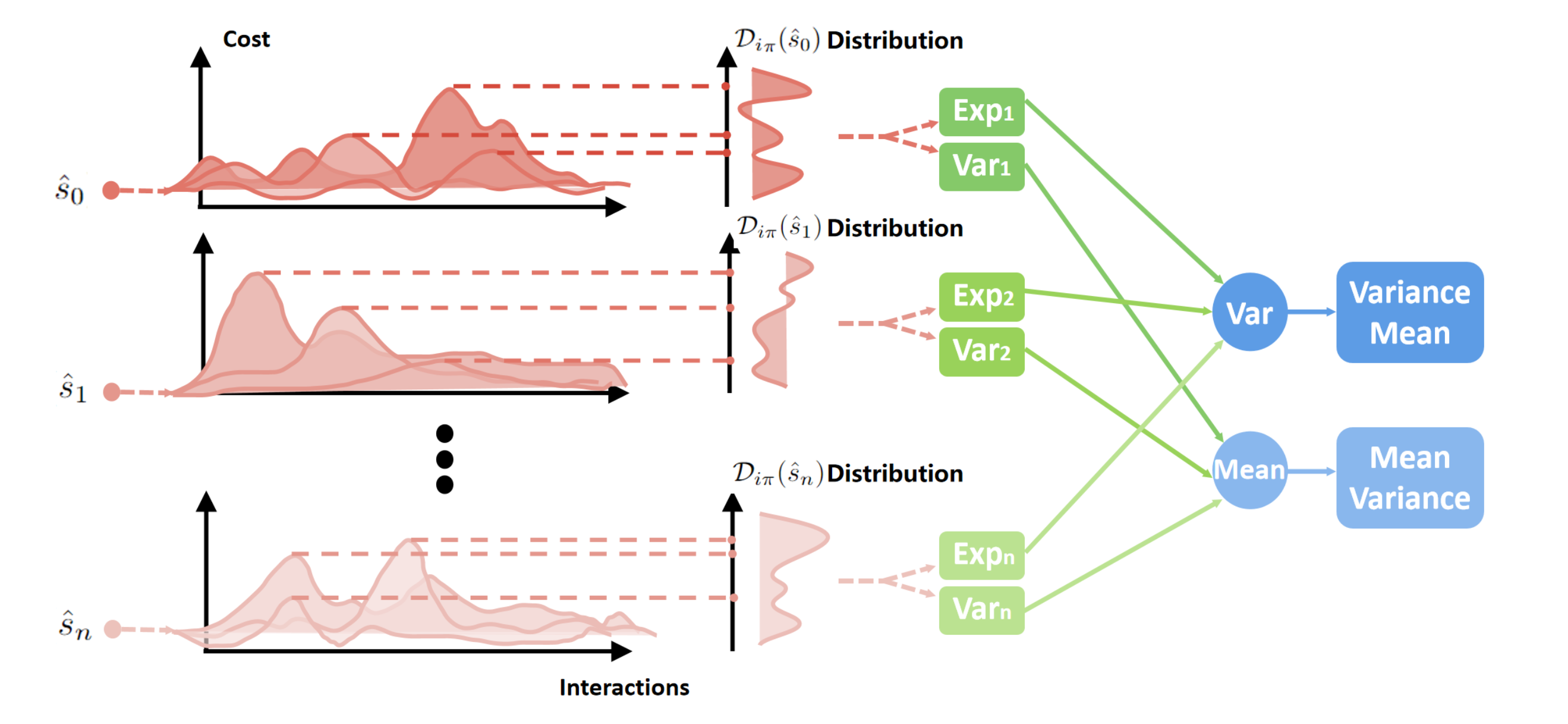}}
    \caption{Explanation of MV and VM. Since maximum state-wise cost from different start states belong to a mixture of one-dimensional distributions, the variance of maximum state-wise cost can be deconstructed into two components: MeanVariance (MV) and VarianceMean (VM).
     } 
    \label{fig:mv and vm}
\end{figure}

\section{Absolute State-wise Constrained Policy Optimization}


\begin{table}[H]
    \centering
    \caption{Notation Table of ASCPO \textbf{I}}
    \begin{tabular}{p{2cm} p{13cm}}
        \toprule
        \textbf{Notation} & \textbf{Description} \\
        \midrule
        $\mathcal{S}$ & State space \\
        $\mathcal{A}$ & Action space \\
        $\mathcal{M}$ & Up-to-now maximum state-wise cost space: $\mathcal{M} \subset \mathbb{R}$\\
        $\hat{\mathcal{S}}$ & Augmented state space\\
        $\gamma$ & Discount factor: $ 0 \leq \gamma < 1$ \\
        $R$ & Reward function: $\mathcal{S} \times \mathcal{A} \mapsto \mathbb{R}$ \\
        $P$ & Transition probability function: $\mathcal{S} \times \mathcal{A} \times \mathcal{S} \mapsto \mathbb{R}$ \\
        $\mu$ & Initial state distribution: $\mathcal{S} \mapsto \mathbb{R}$ \\
        $C$ & Cost function: $\mathcal{S} \times \mathcal{A} \times \mathcal{S} \rightarrow \mathbb{R}$\\
        $D$ & Cost increment function: $(\mathcal{S}, \mathcal{M}) \times \mathcal{A} \times \mathcal{S} \mapsto [0, \mathbb{R}^+]$\\
        $\mathcal{P}(\mathcal{A})$ & Probability distribution over actions \\
        $\pi$ & Stationary policy: $\mathcal{S} \mapsto \mathcal{P}(\mathcal{A})$ \\
        $\Pi$ & Set of all stationary policies \\
        $t$ & Time step along the trajectory \\
        $m$ & Number of the constraints \\
        $i$ & Index of the constraints \\
        $k$ & Index of the policy update iterations \\
        $w$ & Maximum cost for constraints \\
        $s_t$  & State at step $t$:  $s\in\mathcal{S}$  \\
        $a_t$  & Action at step $t$:  $a\in\mathcal{A}$ \\
        $M_t$  & Up-to-now maximum state-wise cost at step $t$:  $M\in\mathcal{M}$ \\
        $\hat{s}_t$  & Augmented state at step $t$:  $\hat{s}\in\hat{\mathcal{S}}$  \\
        $\tau$  & Trajectory: a sequence of action and state  \\
        $H / h$  & Horizon of a trajectory  \\
        $\pi(a|s)$  & Probability of selecting action $a$ in state $s$ \\
        $\pi(\cdot | s)$ & Probability distribution of all action in state $s$\\
        $P(\cdot|s, a)$ & Probability distribution of all next state in state $s$ with action $a$\\
        $\mathcal{J}(\pi)$ & Expectation performance of policy $\pi$\\
        $\mathcal{D}_{\pi}(\hs_0)$ & Maximum state-wise cost performance sample of policy $\pi$ \\
        $R(\tau)$ & Discounted return of a trajectory with infinite horizon \\
        $V_\pi$ & Value function with infinite horizon of policy $\pi$\\
        $Q_\pi$ & Action-value function with infinite horizon of policy $\pi$\\
        $A_\pi$ & Advantage function with infinite horizon of policy $\pi$\\
        $V^H_\pi$ & Value function with $H$ horizon of policy $\pi$\\
        $Q^H_\pi$ & Action-value function with $H$ horizon of policy $\pi$\\
        $A^H_\pi$ & Advantage function with $H$ horizon of policy $\pi$\\
        
        \bottomrule
    \end{tabular}
\end{table}

\begin{table}[H]
    \centering
    \caption{Notation Table of ASCPO \textbf{II}}
    \begin{tabular}{p{2cm} p{13cm}}
        \toprule
        $V^H_{[D]\pi}$ & Value function of cost increment $D$ with $H$ horizon of policy $\pi$\\
        $Q^H_{[D]\pi}$ & Action-value function of cost increment $D$ with $H$ horizon of policy $\pi$\\
        $A^H_{[D]\pi}$ & Advantage function of cost increment $D$ with $H$ horizon of policy $\pi$\\
        $F_\pi$ & HJ reachability value functio\\
        $\mathcal{B}$ & Upper probability bound\\
        ${p}$ & Confidence of the probability bound\\
        $k$ & Probability factor: $k \geq 0$, $k \in \mathbb{R}$\\
        $\mathcal{V}_{min}$ & The minima of $\mathcal{D}_{\pi}(\hs_0)$\\
        $\psi$ & $\mathcal{V}_{min} \in \mathbb{R}^+$ \\
        $\mathcal{E}_{[D]}(\pi)$ & Expectation of the distribution of $\mathcal{D}_{\pi}(\hs_0)$\\
        $\mathcal{V}_{[D]}(\pi)$ & Variance of the distribution of $\mathcal{D}_{\pi}(\hs_0)$\\
        $\mathcal{B}_{[D]k}(\pi)$ & Upper probability bound of $\mathcal{D}_{\pi}(\hs_0)$ with confidence $p_k^\psi$\\
        $\mathcal{J}^l_{\pi, \pi_j}$ & Surrogate function for policy update to bound $\mathcal{J}(\pi)$ from below \\
        $\mathcal{E}^{u}_{[D]\pi, \pi_j}$ & Surrogate function for policy update to bound $\mathcal{E}_{[D]}(\pi)$ from above \\
        $\overline{MV}_{[D]\pi,\pi_j}$ & Upper bound of expected variance of the maximum state-wise cost\\
        $\overline{VM}_{[D]\pi,\pi_j}$ & Upper bound of the variance of the expected maximum state-wise cost \\
        $\mathcal{D}_{KL}(\pi \| \pi_j)[{\hat s}]$ & KL divergence between two policies $(\pi, \pi_j)$ at state $\hat s$\\
        $d_{\pi}$ & Discounted future state distribution of policy $\pi$ \\
        $\bar d_{\pi}$ & Non-discounted future state distribution of policy $\pi$ \\
        $\epsilon_{[D]}^{\pi}$ & Maximum expected advantage of policy $\pi$\\
        ${\mathcal{R}}_\pi(\hs_0)$ & Discounted return starts at state $s_0$ with infinite horizon of policy $\pi$\\
        $V_\pi(\hs_0)$ & Value of state $\hs_0$ with infinite horizon of policy $\pi$\\
        ${\mathcal{R}}_\pi^H\hs_0)$ & Discounted return starts at state $s_0$ with $H$ horizon of policy $\pi$\\
        $V_\pi^H(\hs_0)$ & Value of state $\hs_0$ with $H$ horizon of policy $\pi$\\
        $\mathcal{E}(\pi)$ & Expectation of the distribution of  ${\mathcal{R}}_\pi(\hs_0)$ of policy $\pi$\\
        $\mathcal{V}(\pi)$ & Variance of the distribution of  ${\mathcal{R}}_\pi(\hs_0)$ of policy $\pi$\\
        $\mathcal{B}_{k}(\pi, \gamma)$ & Upper probability bound of ${\mathcal{R}}_\pi(\hs_0)$\\
        $MV_\pi$ & \textbf{MeanVariance} of policy ${\pi}$\\
        $VM_\pi$ & \textbf{VarianceMean} of policy ${\pi}$\\
        $\omega_\pi^h(\hs)$ & Variance of the state-action value function ${Q^h_\pi}$ at state ${\hs}$ with $h$ horizon\\
        $\omega_{[D]\pi}^h(\hs)$ & Variance of the state-action value function $Q^h_{[D]\pi}$ at state ${\hs}$ with $h$ horizon\\
        $\bm \Omega_\pi^h$ & The vector of $\omega_\pi^h(\hs)$\\
        $\bm \Omega_{[D]\pi}^{h}$ & The vector of $\omega_{[D]\pi}^h(\hs)$\\
        $\xi$ & Action probability ratio\\
        
        \bottomrule
    \end{tabular}
\end{table}
To optimize \eqref{eq: apo optimization original}, we need to evaluate the objective and constraints under an unknown $\pi$. While the exact computations of $\mathcal{J}(\pi)$, ${\mathcal{E}}_{[D_i]}(\pi)$, and ${\mathcal{V}}_{[D_i]}(\pi)$ are infeasible before the actual rollout, we can alternatively find surrogate functions for the objective and constraints of \eqref{eq: apo optimization original} such that (i) they provide a tight lower bound for the objective and a tight upper bound for the constraints, and (ii) they can be easily estimated from samples collected from the most recent policy.

Therefore, we introduce (i) $\mathcal{J}^l_{\pi, \pi_j}$ as a surrogate function to bound $\mathcal{J}(\pi)$ from below, (ii) $~~\mathcal{E}^{u}_{[D_i]\pi, \pi_j}$ as a surrogate function to bound $\mathcal{E}_{[D_i]}(\pi)$ from above, and (iii) \\$\left(\overline{MV}_{[D_i]\pi,\pi_j} + \overline{VM}_{[D_i]\pi,\pi_j} \right)$ as a surrogate function to bound $\mathcal{V}_{[D_i]}(\pi)$ from above, in the $(j+1)$-th iteration. All three surrogate functions could be estimated using samples from $\pi$. Notice that the upper bound of $\mathcal{V}_{[D_i]}(\pi)$ involves two terms, where $\overline{MV}_{[D_i]\pi,\pi_j}$ reflects the upper bound of expected variance of the maximum state-wise cost over different start states. $\overline{VM}_{[D_i]\pi,\pi_j}$ reflects the upper bound of variance of the expected maximum state-wise cost of different start states. The detailed interpretations are shown in \cref{fig:mv and vm}.

\subsection{Surrogate Functions for Objective and Constraints}
\label{sec: surro func obj and cons}
In the following discussion, we derive the surrogate functions $\mathcal{J}^l_{\pi, \pi_j}$, $~~\mathcal{E}^{u}_{[D_i]\pi, \pi_j}$, $\overline{MV}_{[D_i]\pi,\pi_j}$, and $\overline{VM}_{[D_i]\pi,\pi_j}$.
\paragraph{Lower Bound for Objective}
To bound the objective below, we directly follow the policy performance bound introduced by \citeauthor{achiam2017constrained}, which provides a tight lower bound for the objective. Mathematically, 
\begin{align}
    &\mathcal{J}^l_{\pi, \pi_j}\doteq\mathcal{J}(\pi_j)+\frac{1}{1-\gamma} \underset{\substack{\substack{\hat s \sim d^{\pi_j}\\a\sim {\pi}}}}{\mathbb{E}} \bigg[ A^H_{\pi_j}(\hat s,a) - \frac{2\gamma \epsilon^{\pi}}{1-\gamma} \sqrt{\frac 12 \mathcal{D}_{KL}({\pi} \| \pi_j)[\hat s]} \bigg] .
\end{align}
where $\mathcal{D}_{KL}(\pi \| \pi_j)[{\hat s}]$ is the KL divergence between two policies $(\pi, \pi_j)$ at state $\hat s$ and $d_{\pi_j} \doteq (1-\gamma)\sum\limits_{t=0}^H\gamma^t P({\hat s}_t={\hat s}|{\pi_j})$ is the discounted future state distribution.

\paragraph{Upper Bound of Maximum State-wise Cost Expectation}
Different from $\mathcal{J}(\pi)$, ${\mathcal{E}}_{[D_i]}(\pi)$ takes the form of a non-discounted summation over a finite horizon. Here we follow the tight upper bound of ${\mathcal{E}}_{[D_i]}(\pi)$ introduced by SCPO from \citeauthor{zhao2024statewise}, 
\begin{align}
    \label{eq: value upper bound}
    &\mathcal{E}^{u}_{[D_i]\pi, \pi_j}\doteq\mathcal{E}_{[D_i]}(\pi) +  \underset{\substack{{\hat s} \sim \bar d_{\pi_j} \\ a\sim \pi}}{\mathbb{E}}\Bigg[  A_{[D_i]\pi_j}^{H}({\hat s},a) + 2(H+1)\epsilon_{[D_i]}^{\pi} \sqrt{\frac{1}{2} \mathbb{E}_{{\hat s} \sim \bar d_{\pi_j}}[\mathcal{D}_{KL}( \pi \| \pi_j)[{\hat s}]]}\Bigg] .
\end{align}
where $\epsilon_{[D_i]}^{\pi} \doteq \underset{\hs}{\mathbf{max}}|\underset{a\sim\pi}{\mathbb{E}}[A^H_{[D_i]\pi_j}(\hs,a)]|$ is the maximum expected advantage and $\bar d_{\pi_j} \doteq \sum\limits_{t=0}^H P({\hat s}_t={\hat s}|{\pi_j})$ is the non-discounted future state distribution.

\begin{prop} 
For any policies $\pi', \pi$, 
the following bound holds:
\begin{align}
  \mathcal{E}_{[D_i]}(\pi') \leq 
  \mathcal{E}^{u}_{[D_i]\pi, \pi'}
\end{align}
\label{lem: upper bound of mean}
\end{prop}
\begin{proof}
The proof of \Cref{lem: upper bound of mean} follows \cite{zhao2024statewise}.
\end{proof}

\paragraph{Upper Bound of Maximum State-wise Cost Variance}
To understand maximum state-wise cost variance $\mathcal{V}_{[D_i](\pi)}$, we begin with
establishing a general version of \textbf{performance variance}, where we regard the cost increment function $D_i$ 
as a broader reward function $R$ with a discount factor $\gamma$. Note that this use of $R$ is a symbolic overload and differs from the previously defined reward. This redefinition aims to simplify the following discussion.

First we define ${\mathcal{R}}_\pi(\hs_0) = \sum_{t=0}^\infty \gamma^t R(\hs_t, a_t, \hs_{t+1})$ as infinite-horizon discounted return starts at state $s_0$ and define expected return ${V_\pi(\hs_0) = \underset{\hat \tau \sim \pi}{ \mathbb{E}}\big[\mathcal{R}_\pi(\hs_0)\big]}$ as the value of state $\hs_0$. Notice that for finite horizon MDP, ${\mathcal{R}}_\pi(\hs_0) = \mathcal{R}_\pi^H(\hs_0)$, where ${\mathcal{R}}_\pi(\hs_0) = \sum_{t=0}^H \gamma^t R(\hs_t, a_t, \hs_{t+1})$. Then for all trajectories $\hat \tau \sim \pi$ start from state ${\hs_0 \sim \mu}$, the expectation and variance of  $\mathcal{R}_\pi(\hs_0)$ can be respectively defined as  $\mathcal{E}(\pi)$ and $\mathcal{V}(\pi)$. Following \Cref{def: abs bound}, we define $\mathcal{B}_{k}(\pi, \gamma) \doteq \mathcal{E}(\pi)+k\mathcal{V}(\pi)$ as the upper probability bound of ${\mathcal{R}}_\pi(\hs_0)$ with discount term $\gamma$, and we treat $\mathcal{B}_{k}(\pi) = \mathcal{B}_{k}(\pi, 1)$. Formally:

\begin{align}
    \mathcal{E}(\pi) & = \underset{\substack{\hs_0 \sim \mu \\ {\hat \tau}\sim {\pi}}}{\mathbb{E}}\bigg[{\mathcal{R}}_\pi(\hs_0)\bigg] 
    = \underset{\substack{\hs_0 \sim \mu}}{\mathbb{E}} \bigg[V_\pi(\hs_0)\bigg] \\  
    \mathcal{V}(\pi) &= \underset{\substack{\hs_0 \sim \mu \\ {\hat \tau}\sim {\pi}}}{\mathbb{E}} \bigg[\big(\mathcal{R}_\pi(\hs_0) - \mathcal{E}(\pi)\big)^2\bigg] \label{eq:variance interpretation} \\ \nonumber 
    &= \underset{\substack{\hs_0 \sim \mu}}{\mathbb{E}}\bigg[\underset{\substack{{\hat \tau}\sim {\pi}}}{\mathbb{V}ar}\big[\mathcal{R}_\pi(\hs_0)\big] +\bigg[\underset{{\hat \tau} \sim \pi}{ \mathbb{E}}\big[\mathcal{R}_\pi(\hs_0)\big]\bigg]^2\bigg] - \mathcal{E}(\pi)^2  \\ \nonumber 
    &= \underset{\substack{\hs_0 \sim \mu}}{\mathbb{E}}\bigg[\underset{\substack{{\hat \tau}\sim {\pi}}}{\mathbb{V}ar}\big[\mathcal{R}_\pi(\hs_0)\big] + V_\pi(\hs_0)^2\bigg] - \mathcal{E}(\pi)^2  \\ \nonumber 
    &= \underset{\substack{\hs_0 \sim \mu}}{\mathbb{E}}\bigg[\underset{\substack{{\hat \tau}\sim {\pi}}}{\mathbb{V}ar}\big[\mathcal{R}_\pi(\hs_0)\big]\bigg] + \underset{\substack{\hs_0 \sim \mu}}{\mathbb{E}}\bigg[V_\pi(\hs_0)^2\bigg] - \mathcal{E}(\pi)^2  \\ \nonumber 
    &= {\underset{\substack{\hs_0 \sim \mu}}{\mathbb{E}}\bigg[\underset{\substack{{\hat \tau}\sim {\pi}}}{\mathbb{V}ar}\big[\mathcal{R}_\pi(\hs_0)\big]\bigg]} + \underset{\hs_0 \sim \mu}{\mathbb{V}ar} [V_\pi(\hs_0)] \\
    &= \underbrace{\underset{\substack{\hs_0 \sim \mu}}{\mathbb{E}}\bigg[\underset{\substack{{\hat \tau}\sim {\pi}}}{\mathbb{V}ar}\big[\mathcal{R}_\pi^H(\hs_0)\big]\bigg]}_{MeanVariance} + \underbrace{\underset{\hs_0 \sim \mu}{\mathbb{V}ar} [V_\pi^H(\hs_0)]}_{VarianceMean}
\end{align}
Note that for the derivation of ${\mathcal{V}(\pi)}$, we treat the return of all trajectories as a mixture of one-dimensional distributions. Each distribution consists of the returns of trajectories from the same start state. The variance can then be divided into two parts: \\
1. \textbf{MeanVariance} reflects the expected variance of the return over different start states. \\
2. \textbf{VarianceMean} reflects the variance of the average return of different start states. 

Subsequently, we can derive the bound of \textbf{MeanVariance} and \textbf{VarianceMean} with the following propositions, where the proofs of \Cref{lem: bound of MV} and \Cref{lem: bound of VM} are summarized in \Cref{proof: MeanVariance Bound} and \Cref{proof: VarianceMean Bound}, respectively. The analysis of \Cref{lem: bound of MV} leverages the performance variance expression for finite horizon Markov Decision Processes (MDPs) and applies divergence analysis to establish the difference for each term in the expression between two policies. \Cref{lem: bound of VM} is analyzed by explicitly breaking down the individual terms within $VarianceMean$, which are combinations of the value function, and then examining the value function differences between the two policies.

Additionally, \Cref{lem: bound of MV} and \Cref{lem: bound of VM} are results with discount term $\gamma$, and we will get to non-discount result in \Cref{sec: ascpo optimization main theory}.

\begin{prop}[Bound of MeanVariance]
\label{lem: bound of MV}
Denote \textbf{MeanVariance} of policy ${\pi}$ as $MV_\pi = \underset{\substack{\hs_0 \sim \mu}}{\mathbb{E}}\bigg[\underset{\substack{{\hat \tau}\sim {\pi}}}{\mathbb{V}ar}\big[\mathcal{R}_\pi^H(\hs_0)\big]\bigg]$. Given two policies $\pi', \pi$, the following bound holds:
\begin{align}
    |MV_{\pi'} - MV_{\pi}| & \leq   \|\mu^\top\|_\infty \sum_{h=1}^H \bigg\{ \gamma^{2(H-h)} \underset{\hs}{\mathbf{max}}\bigg|\underset{\substack{\\a\sim\pi\\\hs'\sim P}}{\mathbb{E}}\left[\left(\frac{\pi^{\prime}(a|\hs)}{\pi(a|\hs)}-1\right) A_{\pi}^h(\hs,a,\hs')^2\right] \\ \nonumber 
 &~~~~~+ 2\underset{\substack{a \sim \pi \\ \hs' \sim P}}{\mathbb{E}}\left[\left(\frac{\pi^{\prime}(a|\hs)}{\pi(a|\hs)}\right)A_{\pi}^h(\hs,a,\hs')\right]|K^h(\hs,a,\hs')|_{max} + |K^h(\hs,a,\hs')|_{max}^2\bigg| \\ \nonumber 
 &~~~~~+ 2\gamma^{2(H-h)} \|\bm \Omega_\pi^{h}\|_\infty \bigg\}
\end{align}
where $A^h_\pi(\hs,a) = \mathbb{E}_{\hs' \sim P}[A^h_\pi(\hs,a,\hs')] \doteq Q^h_\pi(\hs, a) - V^h_\pi(\hs)$, $\bm \Omega_\pi^h = \begin{bmatrix}
    \omega_\pi^h(\hs^1) \\
    \omega_\pi^h(\hs^2) \\
    \vdots
\end{bmatrix}$ and $\omega_\pi^h(\hs) = \underset{\substack{a \sim \pi \\ \hs' \sim P}}{\mathbb{E}} \big[Q_\pi^h(\hs,a,\hs')^2\big] - V_{\pi}^h(\hs)^2$ is defined as the variance of the state-action value function ${Q^h_\pi}$ at state ${\hs}$ for $h$-horizon MDP.
\begin{align}
    |K^h(\hs,a,\hs')|_{max} &= \left|L^h(\hs,a,\hs')\right| + \frac{4\epsilon(\gamma - \gamma^h)}{(1-\gamma)^2}\mathcal{D}_{KL}^{max}(\pi'\|\pi) \\ \nonumber
    L^h(\hs,a,\hs')&= \gamma\underset{\substack{\hs_0 = \hs' \\{\hat \tau} \sim \pi}}{\mathbb{E}}\bigg[\sum_{t=0}^{h-2} \gamma^t A_{\pi',\pi}^{h-1-t}(\hs_t)\bigg] - \underset{\substack{\hs_0 = \hs \\{\hat \tau} \sim \pi}}{\mathbb{E}}\bigg[\sum_{t=0}^{h-1} \gamma^tA_{\pi',\pi}^{h-t}(\hs_t)\bigg]\\ \nonumber
    \epsilon &= \underset{\hs, a, h}{\mathbf{max}}|A_\pi^h(\hs,a)|
\end{align}
where $A^h_{\pi',\pi}(\hs) =  \underset{a \sim \pi'}{\mathbb{E}}\bigg[A^h_{\pi}(\hs,a)\bigg]$.
\end{prop}

\begin{prop}[Bound of VarianceMean]
\label{lem: bound of VM}
Denote \textbf{VarianceMean} of policy ${\pi}$ as $VM_\pi = \underset{\hs_0 \sim \mu}{\mathbb{V}ar} [V^H_\pi(\hs_0)]$. Given two policies $\pi', \pi$, the \textbf{VarianceMean} of $\pi'$ can be bounded by:
\begin{align}
    VM_{\pi'}  &\leq \underset{\hs_0 \sim \mu}{\mathbb{E}} [(V^H_\pi(\hs_0))^2] + \|\mu^T\|_\infty\underset{\hs}{\mathbf{max}}\bigg||\eta(\hs)|_{max}^2+2|V^H_\pi(\hs)|\cdot|\eta(\hs)|_{max}\bigg| \\ \nonumber
    &- \left(\mathbf{min}\left\{\mathbf{max}\left\{0,\ \mathcal{E}^l_{\pi^{\prime}, \pi}\right\}, \mathcal{E}^u_{\pi^{\prime}, \pi}\right\}\right)^2
\end{align}
where
\begin{align}
    |\eta(\hs)|_{max} &= \left|L(\hs)\right| + \frac{2\epsilon(\gamma - \gamma^H)}{(1-\gamma)^2}\mathcal{D}^{max}_{KL}(\pi'\|\pi) \\ \nonumber
    L(\hs) &= \underset{\substack{\hs_0 = \hs \\{\hat \tau} \sim \pi}}{\mathbb{E}}\bigg[\sum_{t=0}^{H-1} \gamma^t\bar A^{H-t}_{\pi',\pi}(\hs_t)\bigg] \\ \nonumber
    \mathcal{E}^l_{\pi^{\prime}, \pi}&=\mathcal{E}(\pi) + \underset{\substack{\hs \sim \overline{d}_\pi \\ a\sim {\pi'}}}{\mathbb{E}} \bigg[ A^H_\pi(\hs,a) - 2(H+1)\epsilon^{\pi'} \sqrt{\frac 12 \mathcal{D}_{KL}({\pi'} \| \pi)[\hs]} \bigg] \\ \nonumber
    \mathcal{E}^u_{\pi^{\prime}, \pi}&=\mathcal{E}(\pi) + \underset{\substack{\hs \sim \overline{d}_\pi \\ a\sim {\pi'}}}{\mathbb{E}} \bigg[ A^H_\pi(\hs,a) + 2(H+1)\epsilon^{\pi'} \sqrt{\frac 12 \mathcal{D}_{KL}({\pi'} \| \pi)[\hs]} \bigg]~~~~~~.
\end{align}
\end{prop}

With \Cref{lem: bound of MV} and \Cref{lem: bound of VM}, given existing policy $\pi$, we can define the upper bounds of $MV$ and $VM$ for unkonwn policy $\pi'$ as:

\begin{align}
\label{eq: MV and VM bound original definitions}
    &~~~~{MV}_{\pi',\pi} = MV_{\pi} + \|\mu^\top\|_\infty \sum_{h=1}^H \bigg( \gamma^{2(H-h)} \underset{\hs}{\mathbf{max}}\bigg|\underset{\substack{\\a\sim\pi\\\hs'\sim P}}{\mathbb{E}}\left[\left(\frac{\pi^{\prime}(a|\hs)}{\pi(a|\hs)}-1\right) A_{\pi}^h(\hs,a,\hs')^2\right] \\ \nonumber 
    &~~~~~~~~~~~~~~ + 2\underset{\substack{a \sim \pi \\ \hs' \sim P}}{\mathbb{E}}\left[\left(\frac{\pi^{\prime}(a|\hs)}{\pi(a|\hs)}\right)A_{\pi}^h(\hs,a,\hs')\right]|K^h(\hs,a,\hs')|_{max} + |K^h(\hs,a,\hs')|_{max}^2\bigg|\\ \nonumber 
    &~~~~~~~~~~~~~~ + 2\gamma^{2(H-h)} \|\bm \Omega_\pi^{h}\|_\infty \bigg) \\ \nonumber
    &~~~~{VM}_{\pi',\pi} = \underset{\hs_0 \sim \mu}{\mathbb{E}} [(V^H_\pi(\hs_0))^2] + \|\mu^T\|_\infty\underset{\hs}{\mathbf{max}}\bigg||\eta(\hs)|_{max}^2+2|V^H_\pi(\hs)|\cdot|\eta(\hs)|_{max}\bigg| \\ \nonumber
    &~~~~~~~~~~~~~~ -\left(\mathbf{min}\left\{\mathbf{max}\left\{0,\ \mathcal{E}^l_{\pi^{\prime}, \pi}\right\}, \mathcal{E}^u_{\pi^{\prime}, \pi}\right\}\right)^2~~~~~~.
\end{align}

By treating $R$ as the cost increment function $D_i$ and letting $\gamma \rightarrow 1^-$ (shown in proof of \Cref{theo: high prob stasfication}) from \eqref{eq: MV and VM bound original definitions}, we effectively obtain the upper bound of  $\mathcal{V}_{[D_i](\pi)}$ as $\left(\overline{MV}_{[D_i]\pi,\pi_j} + \overline{VM}_{[D_i]\pi,\pi_j} \right)$, where 
\begin{align}
    \label{eq: mv def}
     &\overline{MV}_{[D_i]\pi,\pi_j} 
     \doteq
     MV_{[D_i]\pi_j} + \\
     &~~\textcolor{darkred}{\|\mu^\top\|_\infty} \sum_{h=1}^H \bigg( \textcolor{darkgreen}{\underset{\hs}{\mathbf{max}}}\bigg|\underset{\substack{\\a\sim\pi_j\\\hs'\sim P}}{\mathbb{E}}\left[\left(\xi_j-1\right) (A_{i,j}^h)^2 + 2\xi_j A_{i,j}^h |\overline{K}_{i}^h| + |\overline{K}_{i}^h|^2\right]\bigg| + 2 \|\bm \Omega_{i,j}^{h}\|_\infty \bigg) \nonumber\\ 
\label{eq: vm def}
     &\overline{VM}_{[D_i]\pi,\pi_j} 
     \doteq
     \underset{\hs_0 \sim \mu}{\mathbb{E}} [(V^H_{[D_i]\pi_j}(\hs_0))^2] \\ \nonumber
     &~~~~~+ \textcolor{darkred}{\|\mu^\top\|_\infty}\textcolor{darkgreen}{\underset{\hs}{\mathbf{max}}}\bigg||\overline{\eta}_{[D_i]}(\hs)|_{max}^2+2|V^H_{[D_i]\pi_j}(\hs)|\cdot|\overline{\eta}_{[D_i]}(\hs)|_{max}\bigg| -\mathcal{E}_{i,j}^* \nonumber
\end{align}
where $\omega_{[D_i]\pi_j}^h(\hs) = \underset{\substack{a \sim \pi \\ \hs' \sim P}}{\mathbb{E}} \big[Q_{[D_i]\pi_j}^h(\hs,a,\hs')^2\big] - V_{[D_i]\pi_j}^h(\hs)^2$ is defined as the variance of the state-action value function; $\bm \Omega_{[D_i]\pi_j}^{h} \doteq \begin{bmatrix}
    \omega_{[D_i]\pi_j}^h(\hs^1) &
    \omega_{[D_i]\pi_j}^h(\hs^2) &
    \hdots
\end{bmatrix}^\top$ is the vector of variance of state-action value; $MV_{[D_i]\pi_j} \doteq \underset{\substack{\hs_0 \sim \mu}}{\mathbb{E}}\bigg[\underset{\substack{{\hat \tau}\sim {\pi}}}{\mathbb{V}ar}\big[\mathcal{D}_{i\pi}^H(\hs_0)\big]\bigg]$ is the expectation of maximum state-wise cost variance over initial state distribution;
$\xi_j \doteq \frac{\pi(a|\hs)}{\pi_j(a|\hs)}$ is the action probability ratio;
$A_{i,j}^h \doteq A_{[D_i]\pi_j}^h$ is the cost advantage;\\
$\mathcal{E}_{i,j}^* \doteq \left\{\mathbf{min}\left\{\mathbf{max}\left\{0,\ {\mathcal{E}}^{l}_{[D_i]\pi, \pi_j}\right\}, {\mathcal{E}}^{u}_{[D_i]\pi, \pi_j}\right\}\right\}^2$ is the minimal squared expectation of $\mathcal{D}_{i\pi}(\hs_0)$;
the lower bound surrogate function of $\mathcal{E}_{[D_i]}(\pi)$ is defined as:
\begin{align}
    &\mathcal{E}^{l}_{[D_i]\pi, \pi_j}\doteq\mathcal{E}_{[D_i]}(\pi) +  \underset{\substack{{\hat s} \sim \bar d_{\pi_j} \\ a\sim \pi}}{\mathbb{E}}\Bigg[  A_{[D_i]\pi_j}^{H}({\hat s},a) - 2(H+1)\epsilon_{[D_i]}^{\pi} \sqrt{\frac{1}{2} \mathbb{E}_{{\hat s} \sim \bar d_{\pi_j}}[\mathcal{D}_{KL}( \pi \| \pi_j)[{\hat s}]]}\Bigg] 
\end{align}
Additionally,
\begin{align}
    \label{eq: k define}
    &~~~~|\overline{K}_{i}^h| \doteq \textcolor{darkblue}{|\overline{K}_{[D_i]}^h(\hs,a,\hs')|_{max}} = \left|\underset{\substack{\hs_0 = \hs' \\{\hat \tau} \sim \pi_j}}{\mathbb{E}}\bigg[\sum_{t=0}^{h-2} \bar A_{[D_i]\pi',\pi_j}^{h-1-t}(\hs_t)\bigg] - \underset{\substack{\hs_0 = \hs \\{\hat \tau} \sim \pi_j}}{\mathbb{E}}\bigg[\sum_{t=0}^{h-1} \bar A_{[D_i]\pi,\pi_j}^{h-t}(\hs_t)\bigg]\right| \\ \nonumber 
    &~~~~~~~~~~~~~~~~~~~~~~~~~~~~~~~~~~~~~+ \epsilon_{[D_i]} h(1-h)\mathcal{D}^{max}_{KL}(\pi\|\pi_j) \\ 
    &~~~~\textcolor{darkgreen}{|\overline{\eta}_{[D_i]}(\hs)|_{max}} = \left|\underset{\substack{\hs_0 = \hs \\{\hat \tau} \sim \pi_j}}{\mathbb{E}}\bigg[\sum_{t=0}^{H-1} \bar A^{H-t}_{[D_i]\pi,\pi_j}(\hs_t)\bigg]\right| + \epsilon_{[D_i]} H(1-H)\mathcal{D}^{max}_{KL}(\pi\|\pi_j), 
\end{align}
where $\epsilon_{[D_i]} = \underset{\hs,a,h}{\mathbf{max}}|A_{[D_i]\pi}^h(\hs,a)|$ and $\mathcal{D}^{max}_{KL}(\pi\|\pi_j) = \underset{\hs}{max}~\mathcal{D}_{KL}(\pi\|\pi_j)[\hs]$.

\subsection{ASCPO Optimization}
\label{sec: ascpo optimization main theory}
With the surrogate functions derived in \Cref{sec: surro func obj and cons}, ASCPO solves the following optimization by looking for the optimal policy within a set $\Pi_\theta \subset \Pi$ of $\theta$-parametrized policies:

\begin{align}
\label{eq: ascpo optimization final}
    & \pi_{j+1} =  \underset{\pi \in \Pi_{\theta}}{\textbf{argmax}}  ~ \mathcal{J}^l_{\pi, \pi_j},~\textbf{s.t.}~ \forall i,  ~{\mathcal{E}}^u_{[D_i]\pi, \pi_j} + k\left(\overline{MV}_{[D_i]\pi,\pi_j} + \overline{VM}_{[D_i]\pi,\pi_j} \right)  \leq w_i.
\end{align}

\begin{theorem}[High Probability State-wise Constraints Satisfaction]
\label{theo: high prob stasfication} Suppose $\pi, \pi'$ are related by \eqref{eq: ascpo optimization final}, then the maximum state-wise cost for $\pi'$ satisfies 
\begin{align}
    \forall i \in 1,\cdots ,m,  Pr\big({\mathcal{D}_{\pi^\prime}}_i(\hs_0)\leq w_i \big) \geq p .\nonumber 
\end{align}
\end{theorem}

\begin{proof}
    With \Cref{lem: upper bound of mean}, \Cref{lem: bound of MV} and \Cref{lem: bound of VM}, we have the following upper bound of absolute performance bound $\mathcal{B}_k(\pi', \gamma)$:
\begin{align}
\label{eq: bk upper bound}
    \mathcal{B}_k(\pi', \gamma) \leq \mathcal{E}^u_{\pi', \pi} + k\left({MV}_{\pi',\pi} + {VM}_{\pi',\pi} \right)
\end{align}

Note that according to \Cref{lem: bound of MV} and \Cref{lem: bound of VM}, we can only get \Cref{eq: bk upper bound} holds when $\gamma \in (0,1)$. 
To extend the result to non-discounted version, we observe \\
$\mathcal{F}(\gamma) \doteq \mathcal{E}^u_{\pi', \pi} + k\left({MV}_{\pi',\pi} + {VM}_{\pi',\pi} \right) - \mathcal{B}_k(\pi', \gamma)$ is a polynomial function and coefficients are all limited with the following conditions holds:

\begin{align}
\label{cond: F}
    &\mathcal{F}(\gamma) \geq 0, \text{when}~ \gamma \in (0,1)\\ \nonumber
    &\mathcal{F}(\gamma)\text{'s domain of definition is }  \mathcal{R} \\ \nonumber
    &\mathcal{F}(\gamma)\text{ is a polynomial function}
\end{align}

we have $\underset{\gamma\rightarrow 1^-}{\lim}\mathcal{F}(\gamma)$ exists and $\mathcal{F}(\gamma)$ is continuous at point $(1, \mathcal{F}(1)$). So $\mathcal{F}(1) = \underset{\gamma\rightarrow 1^-}{\lim}\mathcal{F}(\gamma)\geq 0$, which equals to:

\begin{align}
\nonumber
  \mathcal{B}_k(\pi', 1) \leq \mathcal{E}^u_{\pi', \pi} + k\left({\overline{MV}}_{\pi',\pi} + {\overline{VM}}_{\pi',\pi} \right)
\end{align}
where
\begin{align}
\nonumber
    &~~~~\overline{MV}_{\pi',\pi} = MV_{\pi} + \|\mu^\top\|_\infty \sum_{h=1}^H \bigg( \underset{\hs}{\mathbf{max}}\bigg|\underset{\substack{\\a\sim\pi\\\hs'\sim P}}{\mathbb{E}}\left[\left(\frac{\pi^{\prime}(a|\hs)}{\pi(a|\hs)}-1\right) A_{\pi}^h(\hs,a,\hs')^2\right] \\ \nonumber 
    &~~~~~~~~~~~~~~ + 2\underset{\substack{a \sim \pi \\ \hs' \sim P}}{\mathbb{E}}\left[\left(\frac{\pi^{\prime}(a|\hs)}{\pi(a|\hs)}\right)A_{\pi}^h(\hs,a,\hs')\right]|\overline{K}^h(\hs,a,\hs')|_{max} + |\overline{K}^h(\hs,a,\hs')|_{max}^2\bigg| + 2 \|\bm \Omega_\pi^{h}\|_\infty \bigg) \\ \nonumber
    &~~~~\overline{VM}_{\pi',\pi} = \underset{\hs_0 \sim \mu}{\mathbb{E}} [(V^H_\pi(\hs_0))^2] + \|\mu^T\|_\infty\underset{\hs}{\mathbf{max}}\bigg||\overline{\eta}(\hs)|_{max}^2+2|V^H_\pi(\hs)|\cdot|\overline{\eta}(\hs)|_{max}\bigg| \\ \nonumber
    &~~~~~~~~~~~~~~ -\left(\mathbf{min}\left\{\mathbf{max}\left\{0,\ \mathcal{E}^l_{\pi^{\prime}, \pi}\right\}, \mathcal{E}^u_{\pi^{\prime}, \pi}\right\}\right)^2 \\ \nonumber
    &~~~~|\overline{K}^h(\hs,a,\hs')|_{max} = \left|\underset{\substack{\hs_0 = \hs' \\{\hat \tau} \sim \pi}}{\mathbb{E}}\bigg[\sum_{t=0}^{h-2} A_{\pi',\pi}^{h-1-t}(\hs_t)\bigg] - \underset{\substack{\hs_0 = \hs \\{\hat \tau} \sim \pi}}{\mathbb{E}}\bigg[\sum_{t=0}^{h-1} A_{\pi',\pi}^{h-t}(\hs_t)\bigg]\right| + \epsilon h(1-h)\mathcal{D}^{max}_{KL}(\pi'\|\pi) \\ \nonumber
    &~~~~|\overline{\eta}(\hs)|_{max} = \left|\underset{\substack{\hs_0 = \hs \\{\hat \tau} \sim \pi}}{\mathbb{E}}\bigg[\sum_{t=0}^{H-1} \bar A^{H-t}_{\pi',\pi}(\hs_t)\bigg]\right| + \epsilon H(1-H)\mathcal{D}^{max}_{KL}(\pi'\|\pi)) 
\end{align}

We define $\mathcal{M}_k^j(\pi) = \mathcal{E}^u_{\pi, \pi_j} + k\left({\overline{MV}}_{\pi,\pi_j} + {\overline{VM}}_{\pi,\pi_j} \right)$, and by \Cref{eq: bk upper bound}, we have $\mathcal{B}_k(\pi_{j+1})\leq \mathcal{M}_k^j(\pi_{j+1})$. Thus, by constraining $\mathcal{M}_k^j$ under threshold $w$ at each iteration, we guarantee that the true $\mathcal{B}_k$ is under $w$. 


Mathematically, the following inequality holds true:
\begin{align}
    \label{eq: ineq of prob}
    Pr\big(\mathcal{R}_{\pi_{j+1}}(\hs_0)\leq w)& \geq 
    Pr\big(\mathcal{R}_{\pi_{j+1}}(\hs_0)\leq \mathcal{M}_k^j(\pi_{j+1}))  \\ \nonumber
   &\geq  Pr\big(\mathcal{R}_{\pi_{j+1}}(\hs_0)\leq \mathcal{B}_k(\pi_{j+1})) \geq p .
\end{align}
Thus by bringing cost increment function $\mathcal{D}_i$ into function $\mathcal{R}$, we prove that \Cref{theo: high prob stasfication} holds.
\end{proof}

Furthermore, according to (Theorem 1, \citep{achiam2017constrained}), we also have a performance guarantee for ASCPO:
\begin{theorem}[Monotonic Improvement of Performance]
\label{theo: performance improvement} Suppose $\pi, \pi'$ are related by \eqref{eq: ascpo optimization final}, then performance $\mathcal{J}(\pi)$ satisfies $\mathcal{J}(\pi') \geq \mathcal{J}(\pi)$.
\end{theorem}

\section{Practical Implementation}
\label{sec:practical}
The direct implementation of \eqref{eq: ascpo optimization final} presents challenges due to (i) small update steps caused by the inclusion of the KL divergence term in the objective function~\citep{schulman2015trust}, and (ii) the difficulty of precisely computing parameters, such as the infinity norm terms or the supremum terms. In this section, we demonstrate a more practical approach to implement \eqref{eq: ascpo optimization final} by (i) encouraging larger update steps through the use of a trust region constraint, and (ii) simplifying complex computations by equivalent transformations of the optimization problem.
Additinoally, we show how to (iii) encourage learning even when \eqref{eq: ascpo optimization final} becomes infeasible and (iv) handle the difficulty of fitting augmented value $V_{[D_i]}^\pi$.
The ASCPO pseudocode is summarized in \Cref{algo: ascpo_full}.

\paragraph{Trust Region Constraint} 
While the theoretical recommendations for the coefficients of the KL divergence terms in \eqref{eq: ascpo optimization final} often result in very small step sizes when followed strictly, a more practical approach is to impose a constraint on the KL divergence between the new and old policies. This constraint, commonly referred to as a trust region constraint~\citep{schulman2015trust}, allows for the taking of larger steps in a robust way:
\begin{align}
\label{eq: ascpo optimization weighted sum version}
    \pi_{j+1} &= \underset{\pi \in \Pi_{\theta}}{\textbf{argmax}} ~ \frac{1}{1-\gamma} \underset{\substack{\hat s \sim d_{\pi_j} \\ a\sim {\pi}}}{\mathbb{E}} \left[ A_{\pi_j}(s,a) \right]\\ \nonumber
    &\textbf{s.t.}~ \underset{{\hat s} \sim \bar d_{\pi_j}}{\mathbb{E}}[\mathcal{D}_{KL}(\pi \| \pi_j)[{\hat s}]]\leq \delta, \\ \nonumber
    & ~~~~~~~~~\forall i,  ~{{\mathcal{E}}}^u_{[D_i]\pi, \pi_j} + k\left(\overline{MV}_{[D_i]\pi,\pi_j} + \overline{VM}_{[D_i]\pi,\pi_j} \right)  \leq w_i.
\end{align}
where $\delta$ is the step size, 
the set $\{\pi \in \Pi_\theta ~: ~ \mathbb{E}_{{\hat s} \sim \bar d^{\pi_j}}[\mathcal{D}_{KL}(\pi \| \pi_k)[{\hat s}]] \leq \delta\}$ is called \textit{trust region}. 

\paragraph{Special Parameter Handling}
When implementing \ref{eq: ascpo optimization weighted sum version}, we first treat two items as hyperparameters. 
(i) \bm{$\textcolor{darkred}{\|\mu^\top\|_\infty}$}:~ 
Although the infinity norm of $\mu^\top$ is theoretically equal to 1, we found that treating it as a hyperparameter in $\mathbb{R}^+$ enhances performance in practical implementation. 
(ii) \bm{$\textcolor{darkblue}{|\overline{K}^h_{[D_i]}{(\hs,a,\hs')}|_{max}}$}: ~ We can either compute $|\overline{K}^h_{[D_i]}(\hs,a,\hs')|_{max}$ from the most recent policy with \eqref{eq: k define} or treat it as a hyperparameter since $|\overline{K}^h_{[D_i]}(\hs,a,\hs')|_{max}$ is bounded for any system with a bounded reward function. Due to the instability in the estimation error for this item and the highly erratic nature of taking the maximum value, the performance of the effect is highly unreliable. Consequently, we treated it as a hyperparameter in practice, which yielded excellent results.
(iii) \bm{$\textcolor{darkgreen}{|\overline{\eta}_{[D_i]}{(\hs)}|_{max}}$} and $\textcolor{darkgreen}{\underset{\hs}{\bm{max}}|~\bm{\cdot}~|}$: ~Furthermore, we find that taking the average of the state $\hs$ instead of the maximum yields superior and more stable convergence results. It is noteworthy that a similar technique has been employed in \citep{schulman2015trust} to manage maximum KL divergence.

\begin{algorithm}
\caption{Absolute State-wise Constrained Policy Optimization}
\label{algo: ascpo_full}
\begin{algorithmic}
\State \textbf{Input:} Initial policy $\pi_0\in\Pi_\theta$.
\For{$j=0,1,2,\dots$}
\State Sample trajectory $\tau\sim\pi_j=\pi_{\theta_j}$
\State Estimate gradient $g \gets \nabla_{\theta}\frac{1}{1-\gamma} \underset{\substack{\hs \sim d_{\pi_j} \\ a\sim {\pi}}}{\mathbb{E}} \left[ A_{\pi_j}(\hs,a) \right]\rvert_{\theta=\theta_j}$\Comment{\cref{sec:practical}}
\State Estimate gradient $b_i \gets \nabla_{\theta} \mathcal{X}_{\pi, \pi_j} \rvert_{\theta=\theta_j}, \forall i=1,2,\dots,m$\Comment{\cref{eq: x def,eq: mv def in x,eq: vm def in x}}
\State Estimate Hessian $H \gets \nabla^2_{\theta} \underset{{\hat s} \sim \bar d_{\pi_j}}{\mathbb{E}}[\mathcal{D}_{KL}(\pi \| \pi_j)[{\hat s}]]\rvert_{\theta=\theta_j}$
\State Compute $c_i, \forall i=1,2,\dots,m$ \Comment{\cref{eq: ci def}} 
\State Solve convex programming \Comment{\cite{achiam2017cpo}}\begin{align*}
    \theta^*_{j+1} = \underset{\theta}{\textbf{argmax}} &~~~ g^\top(\theta-\theta_j) \\
    \textbf{s.t.}~ &~~~\frac{1}{2}(\theta-\theta_j)^\top H (\theta-\theta_j) \leq \delta \\
    &~~~ c_i + b_i^\top(\theta-\theta_j) \leq 0,~i=1,2,\dots,m 
\end{align*} 
\State Get search direction $\Delta\theta^* \gets \theta^*_{j+1} - \theta_j$
\For{$k=0,1,2,\dots$} \Comment{Line search}
\State $\theta' \gets \theta_{j} + \xi^k\Delta\theta^*$ \Comment{$\xi\in(0,1)$ is the backtracking coefficient}
\If{
$\underset{{\hat s} \sim \bar d_{\pi_j}}{\mathbb{E}}[\mathcal{D}_{KL}(\pi_{\theta'} \| \pi_j)[{\hat s}]] \leq \delta$ 
\textbf{and} \Comment{Trust region, \cref{constr: cost decrease}}\\
$~~~~~~~~~~~~~~\mathcal{X}_{\pi_{\theta'}, \pi_j} - \mathcal{X}_{\pi_j, \pi_j} \leq \mathrm{max}(-c_i,0),~\forall i$ \textbf{and} \Comment{Costs, \cref{constr: cost decrease}}\\
$~~~~~~~~~~~~~~\left(\underset{\substack{\hs \sim d_{\pi_{\theta'}} \\ a\sim {\pi}}}{\mathbb{E}} \left[ A_{\pi_j}(\hs,a) \right] \geq \underset{\substack{\hs \sim d_{\pi_j} \\ a\sim {\pi}}}{\mathbb{E}} \left[ A_{\pi_j}(\hs,a) \right] \mathbf{or}~\mathrm{infeasible~\eqref{eq: ascpo optimization weighted sum version}}\right)$} \Comment{Rewards}
\State $\theta_{j+1} \gets \theta'$ \Comment{Update policy}
\State \textbf{break}
\EndIf
\EndFor
\EndFor
\end{algorithmic}
\end{algorithm}

\paragraph{Infeasible Constraints}
An update to $\theta$ is 
computed
every time \eqref{eq: ascpo optimization final} is solved.
However, due to approximation errors, sometimes \eqref{eq: ascpo optimization final} can become infeasible.
In that case, we propose an recovery update that 
only
decreases the constraint value within the trust region.
In addition, approximation errors can also cause the proposed policy update (either feasible or recovery) to violate the original constraints in \eqref{eq: ascpo optimization final}.
Hence, each policy update is followed by a backtracking line search to ensure constraint satisfaction.
If all these fails,
we relax the search condition by also accepting decreasing expected advantage with respect to the costs, when the cost constraints are already violated.
Define:
\begin{align}
    \label{eq: ci def}
    &c_i\doteq \mathcal{E}_{[D_i]}(\pi) + 2(H+1)\epsilon_{[D_i]}^{\pi} \sqrt{\frac{1}{2} \mathbb{E}_{{\hat s} \sim \bar d_{\pi_j}}[\mathcal{D}_{KL}( \pi \| \pi_j)[{\hat s}]]}  \\ \nonumber
    &~~~~~~~~~+ MV_{[D_i]\pi_j} + \underset{\hs_0 \sim \mu}{\mathbb{E}} [(V^H_{[D_i]\pi_j}(\hs_0))^2] - w_i\\
    \label{eq: x def}
    &\mathcal{X}_{\pi, \pi_j}\doteq \underset{\substack{\hs \sim \overline{d}_{\pi_j} \\ a\sim {\pi}}}{\mathbb{E}} \left[ A_{[D_i]\pi_j}(\hs,a) \right] + k\left(\widetilde{MV}_{[D_i]\pi,\pi_j} + \widetilde{VM}_{[D_i]\pi,\pi_j} \right)
\end{align}
where
\begin{align}    
    \label{eq: mv def in x}
    &\widetilde{MV}_{[D_i]\pi,\pi_j}\doteq\overline{MV}_{[D_i]\pi,\pi_j} - MV_{[D_i]\pi_j}\\
    \label{eq: vm def in x}
    &\widetilde{VM}_{[D_i]\pi.\pi_j}\doteq \overline{VM}_{[D_i]\pi,\pi_j}-\underset{\hs_0 \sim \mu}{\mathbb{E}} [(V^H_{[D_i]\pi_j}(\hs_0))^2].
\end{align}

The above criteria can be summarized into a set of new constraints as
\begin{align}
\left\{
\begin{aligned}
    &\underset{{\hat s} \sim \bar d_{\pi_j}}{\mathbb{E}}[\mathcal{D}_{KL}(\pi \| \pi_j)[{\hat s}]]\leq \delta \\
    &\mathcal{X}_{\pi, \pi_j} - \mathcal{X}_{\pi_j, \pi_j} \leq \mathrm{max}(-c_i,0),~\forall i 
\end{aligned}
\right.  
\label{constr: cost decrease}
\end{align}

\paragraph{Imbalanced Cost Value Targets}
\begin{wrapfigure}{r}{0.5\textwidth}
\centering
  \includegraphics[width=0.5\textwidth]{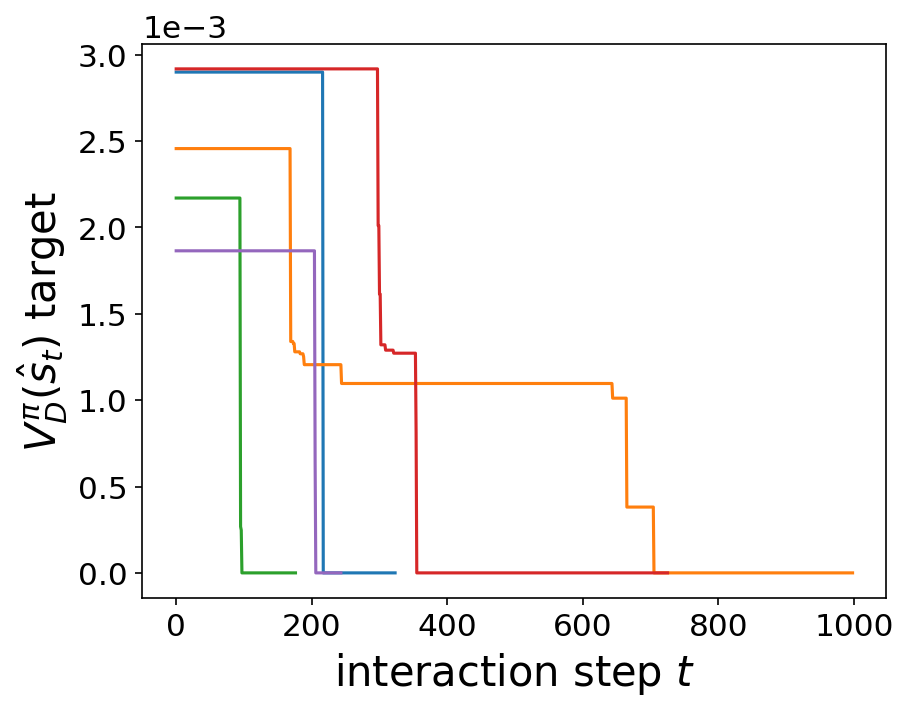}
  \caption{$V_{[D_i]\pi}(\hs)$ target of five sampled episodes.}
  \label{fig:vc_target}
\end{wrapfigure}
A critical step in solving \Cref{eq: ascpo optimization weighted sum version} involves fitting the cost increment value functions $V_{[D_i]\pi}(\hs)$. As demonstrated in \citep{zhao2024statewise}, $V_{[D_i]\pi}(\hs)$ is equal to the maximum cost increment in any future state over the maximal state-wise cost increment so far. In other words, $V_{[D_i]\pi}(\hs)$ forms a step function characterized by (i) pronounced zero-skewness, and (ii) monotonically decreasing trends. Here we visualize an example of $V_{[D_i]\pi}(\hs)$ in \Cref{fig:vc_target}.

To alleviate the unbalanced value target population, we adopt the sub-sampling technique introduced in \cite{zhao2024statewise}.
To encourage the fitting of the monotonically decreasing trend, we design an additional loss term $\mathcal{L}_{\tau}$ for penalizing non-monotonicity:
\begin{align}
\label{eq: monotonic descent}
\mathcal{L}_{\tau}(y, \hat y) = \frac{1}{n}\sum_{i=0}^{n}(y_i - \hat{y_{i}})^2 + w\cdot\sum_{i=0}^{n-1}(\max(0, y_{i+1} - y_{i}))^2,
\end{align}
where $\hat{y}_i$ represents the i-th true value of cost in episode $\tau$, $y_i$ denotes the i-th predicted value of cost in episode $\tau$, and $w$ signifies the weight of the monotonic-descent term. In essence, the rationale is to penalize any prediction that violates the non-increasing characteristics of the target sequence, thus improving the fitting quality. Further details and analysis are presented in \Cref{sec: evaluate ASCPO}.

\section{Experiments}
In our experiments, we aim to answer the following questions:

\textbf{Q1} 
How does ASCPO compare with other state-of-the-art methods for safe RL?

\textbf{Q2} 
What benefits are demonstrated by constraining the
upper probability bound of maximum state-wise cost? 
How did the illustration in \Cref{fig: cpo vs scpo vs ascpo} perform in the actual experiment?

\textbf{Q3}
How does the monotonic-descent trick of ASCPO in \cref{eq: monotonic descent} impact its
performance? Does it work for other baselines?

\textbf{Q4}
How does the resource usage of ASCPO compare to other algorithms?

\textbf{Q5}
Can ASCPO be extended to a PPO-based version?

\subsection{Experiment Setups}
\paragraph{GUARD} To demonstrate the efficacy of our absolute state-wise constrained policy optimization approach, we conduct experiments in the advanced safe reinforcement learning benchmark environment, GUARD ~\citep{zhao2024guard}. This environment has been augmented with additional robots and constraints integrated into the Safety Gym framework ~\citep{ray2019benchmarking}, allowing for more extensive and comprehensive testing scenarios. 

Our experiments are based on seven different robots: (i) \textbf{Point} (Shown in \Cref{fig: Point}):   A point mass robot ($\mathcal{A} \subseteq \mathbb{R}^{2}$) that can move on the ground. (ii) \textbf{Swimmer} (Shown in: \Cref{fig: Swimmer}) A three-link robot ($\mathcal{A} \subseteq \mathbb{R}^{2}$) that can move on the ground. (iii) \textbf{Arm3} (Shown in: \Cref{fig: Arm3}) A fixed three-joint robot arm ($\mathcal{A} \subseteq \mathbb{R}^{3}$) that can move its end effector around with high flexibility. (iv) \textbf{Drone} (Shown in: \Cref{fig: Drone}) A quadrotor robot ($\mathcal{A} \subseteq \mathbb{R}^{4}$) that can move in the air. (v) \textbf{Humanoid} (Shown in: \Cref{fig: Humanoid}) A bipedal robot($\mathcal{A} \subseteq \mathbb{R}^{6}$) that has a torso with a pair of legs and arms. Since the benchmark mainly focuses on the navigation ability of the robots in designed tasks, the arm joints of Humanoid are fixed. (vi) \textbf{Ant} (Shown in: \Cref{fig: Ant}) A quadrupedal robot ($\mathcal{A} \subseteq \mathbb{R}^{8}$) that can move on the ground. (vii) \textbf{Walker} (Shown in: \Cref{fig: Walker}) A bipedal robot ($\mathcal{A} \subseteq \mathbb{R}^{10}$) that can move on the ground. 

All of the experiments are based on the goal task where the robot must navigate to a goal. Additionally, since we are interested in episodic tasks (finite-horizon MDP), the environment will be reset once the goal is reached. 
Four
different types of constraints are considered: (i) \textbf{Hazard}: Dangerous areas as shown in~\Cref{fig: hazard}. Hazards are trespassable circles on the ground. The agent is penalized for entering them. (ii) \textbf{3D Hazard}: 3D Dangerous areas as shown in~\Cref{fig: 3Dhazard}. 3D Hazards are trespassable spheres in the air. The agent is penalized for entering them. (iii) \textbf{Pillar}: Non-traversable obstacles as shown in ~\Cref{fig: pillar}. The agent is penalized for hitting them. (iv) \textbf{Ghost}:
Moving circles as shown in ~\Cref{fig: ghost}. Ghosts can be either trespassable or non-trespassable. The robot is penalized for touching the non-trespassable ghosts and entering the trespassable ghosts.

Considering different robots, constraint types, and constraint difficulty levels, we design 19 test suites with 7 types of robots and 12 types of constraints, which are summarized in \Cref{tab: testing suites} in Appendix. We name these test suites as \texttt{\{Robot\}-\{Constraint Number\}-\{Constraint Type\}}.

\begin{figure}[t]
  \centering
  \begin{subfigure}[t]{0.135\linewidth}
    \includegraphics[scale=0.1]{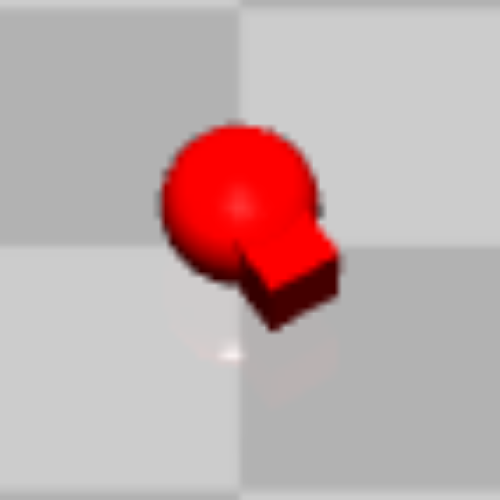}
    \centering
    \caption{Point}
    \label{fig: Point}
  \end{subfigure}
  \begin{subfigure}[t]{0.135\linewidth}
    \includegraphics[scale=0.1]{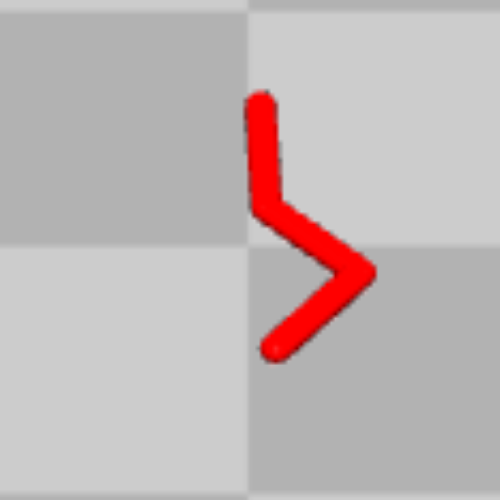}
    \centering
    \caption{Swimmer}
    \label{fig: Swimmer}
  \end{subfigure}
  \begin{subfigure}[t]{0.135\linewidth}
        \includegraphics[scale=0.1]{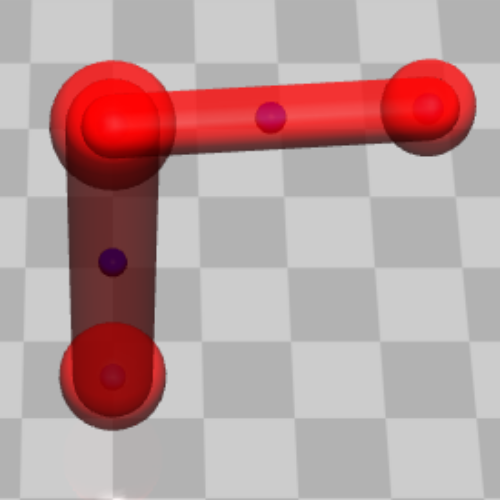}
    \centering
      \caption{Arm3}
      \label{fig: Arm3}
  \end{subfigure}
  \begin{subfigure}[t]{0.135\linewidth}
        \includegraphics[scale=0.1]{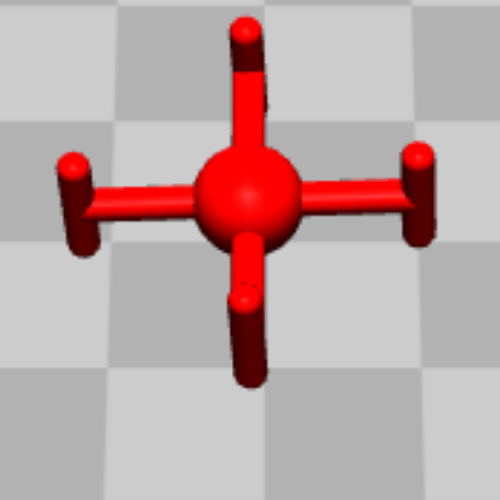}
    \centering
      \caption{Drone}
      \label{fig: Drone}
  \end{subfigure}
  \begin{subfigure}[t]{0.135\linewidth}
        \includegraphics[scale=0.1]{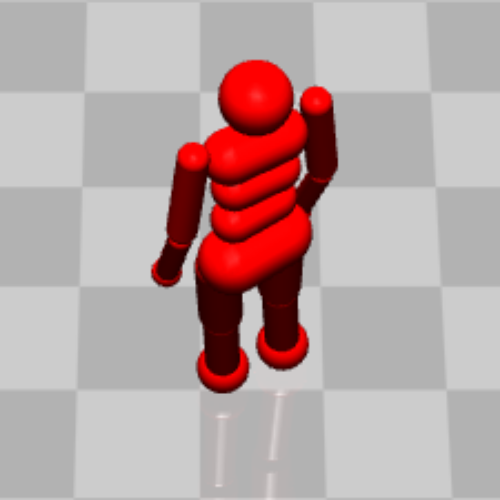}
    \centering
      \caption{Humanoid}
      \label{fig: Humanoid}
  \end{subfigure}
  \begin{subfigure}[t]{0.135\linewidth}
      \includegraphics[scale=0.1]{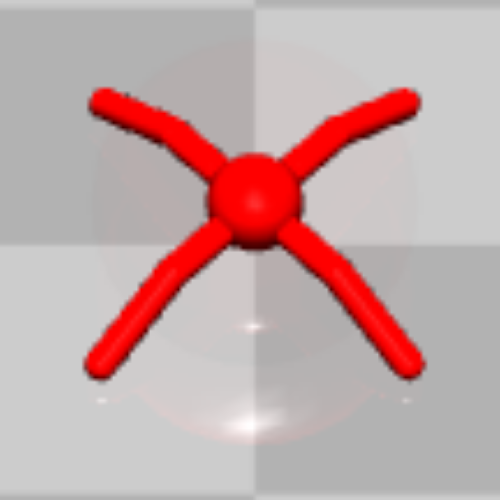}
    \centering
      \caption{Ant}
      \label{fig: Ant}
  \end{subfigure}
  \begin{subfigure}[t]{0.135\linewidth}
        \includegraphics[scale=0.1]{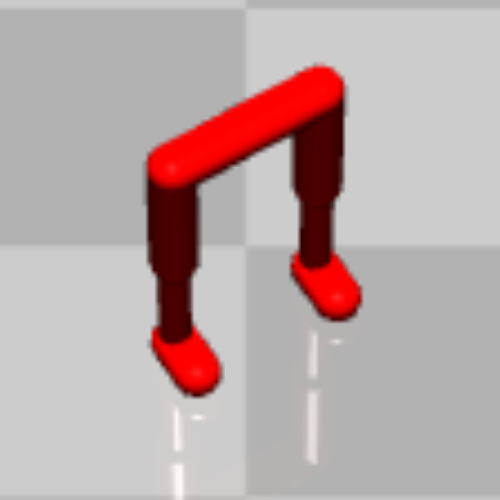}
    \centering
      \caption{Walker}
      \label{fig: Walker}
  \end{subfigure}
  \label{fig: robots}
  \caption{Robots of continuous control tasks benchmark GUARD.}
\end{figure}

\begin{figure}[t]
  \centering
  \begin{subfigure}[t]{0.24\linewidth}
    \includegraphics[scale=0.15]{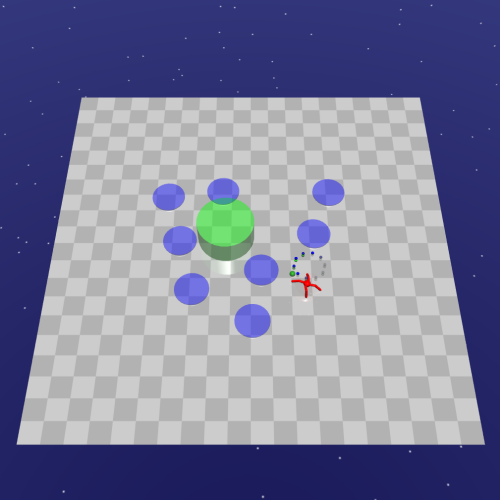}
    \centering
    \caption{Hazard}
    \label{fig: hazard}
  \end{subfigure}
  \begin{subfigure}[t]{0.24\linewidth}
    \includegraphics[scale=0.15]{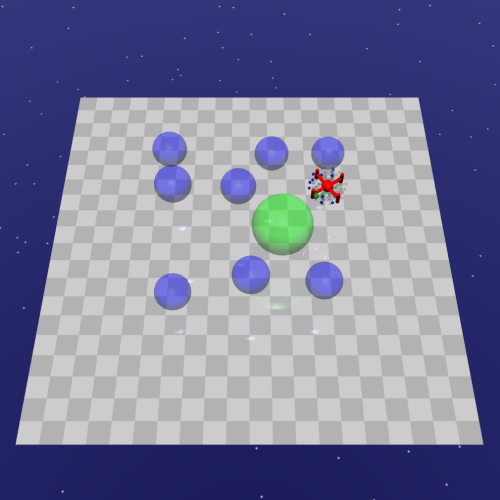}
    \centering
    \caption{3DHazard}
    \label{fig: 3Dhazard}
  \end{subfigure}
  \begin{subfigure}[t]{0.24\linewidth}
      \includegraphics[scale=0.15]{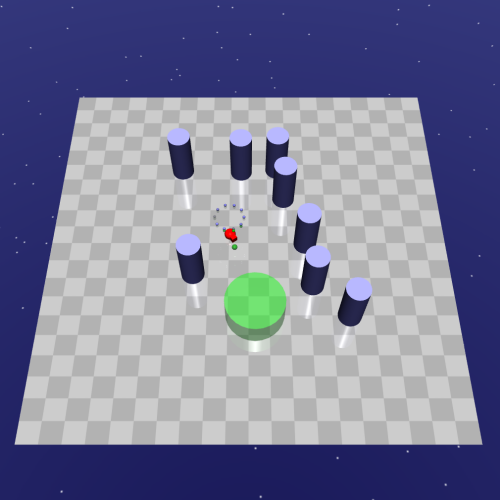}
    \centering
      \caption{Pillar}
      \label{fig: pillar}
  \end{subfigure}
  \begin{subfigure}[t]{0.24\linewidth}
      \includegraphics[scale=0.15]{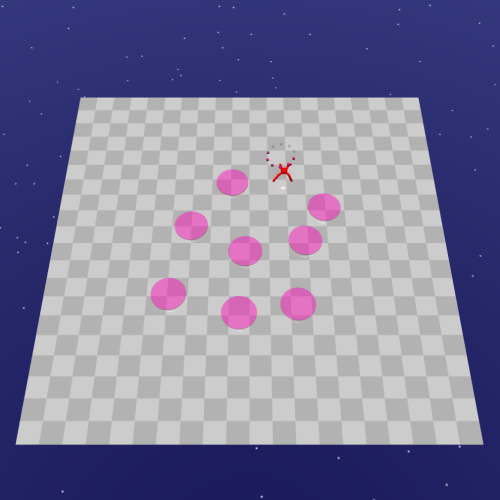}
    \centering
      \caption{Ghost}
      \label{fig: ghost}
  \end{subfigure}
  \label{fig: constraints}
  \caption{Tasks of continuous control tasks benchmark GUARD.}
\end{figure}

\paragraph{Comparison Group}
The comparison group encompasses various methods: (i) the unconstrained RL algorithm TRPO~\citep{schulman2015trust}; (ii) end-to-end constrained safe RL algorithms including SCPO~\citep{zhao2024statewise}, CPO~\citep{achiam2017constrained}, TRPO-Lagrangian~\citep{bohez2019value}, TRPO-FAC~\citep{ma2021feasible}, TRPO-IPO~\citep{liu2020ipo}, PCPO~\citep{yang2020projection}; and (iii) hierarchical safe RL algorithms such as TRPO-SL (TRPO-Safety Layer)\citep{dalal2018safe}, TRPO-USL (TRPO-Unrolling Safety Layer)\citep{zhang2022saferl}; and (iv) risk-sensitive algorithms TRPO-CVaR and CPO-CVaR~\citep{zhang2024cvarcpo}.
TRPO is chosen as the baseline method due to its state-of-the-art status and readily available safety-constrained derivatives for off-the-shelf testing. For hierarchical safe RL algorithms, a warm-up phase constituting one-third of the total epochs is employed, wherein unconstrained TRPO training is conducted. The data generated during this phase is then utilized to pre-train the safety critic for subsequent epochs.
Across all experiments, the policy $\pi$ and the values $(V^\pi, V_{D}^\pi)$ are encoded in feedforward neural networks featuring two hidden layers of size (64,64) with tanh activations. Additional details are provided in \Cref{sec: experiment details}.

\paragraph{Evaluation Metrics}
For comparative analysis, we assess algorithmic performance across three key metrics: (i) \textbf{reward performance} $J_r$, (ii) \textbf{average episode cost} $M_c$, and (iii) \textbf{cost rate (state-wise cost)} $\rho_c$. Detailed descriptions of these comparison metrics are provided in \Cref{sec:metrics}. To ensure consistency, we impose a cost limit of 0 for all safe RL algorithms, aligning with our objective to prevent any constraint violations. In conducting our comparison, we faithfully implement the baseline safe RL algorithms, adhering precisely to the policy update and action correction procedures outlined in their respective original papers. It is essential to note that for a fair comparison, 
we provide the baseline safe RL algorithms every advantage provided to ASCPO, including equivalent trust region policy updates.

\begin{figure}[t]
  \centering
  \begin{subfigure}[t]{1.0\linewidth}
      \centering
      \begin{subfigure}[t]{0.24\linewidth}
           \begin{subfigure}[t]{1.0\linewidth}
            \includegraphics[width=0.99\textwidth]{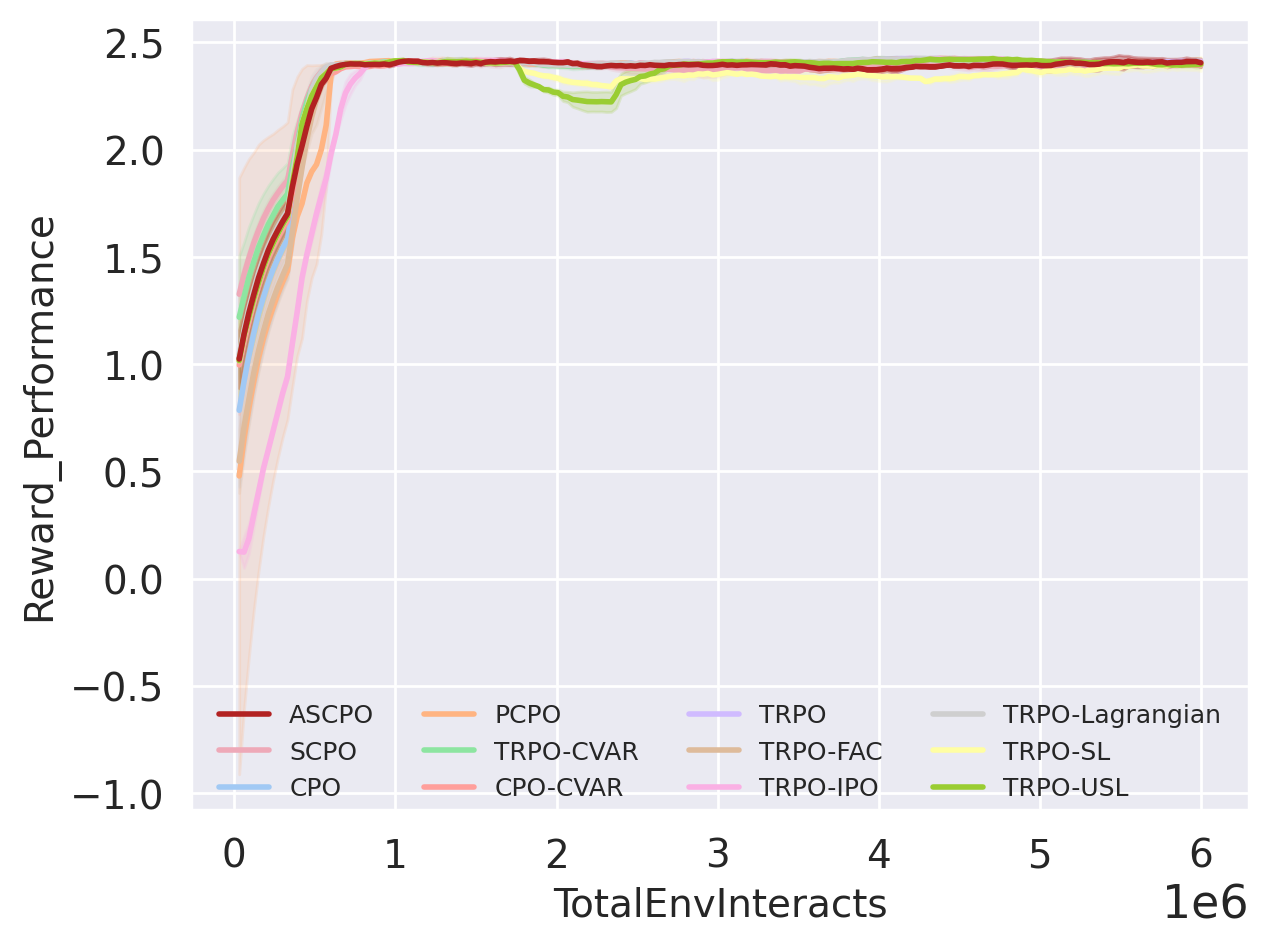}
          \end{subfigure}
          \hfill
          \begin{subfigure}[t]{1.0\linewidth}
              \includegraphics[width=0.99\textwidth]{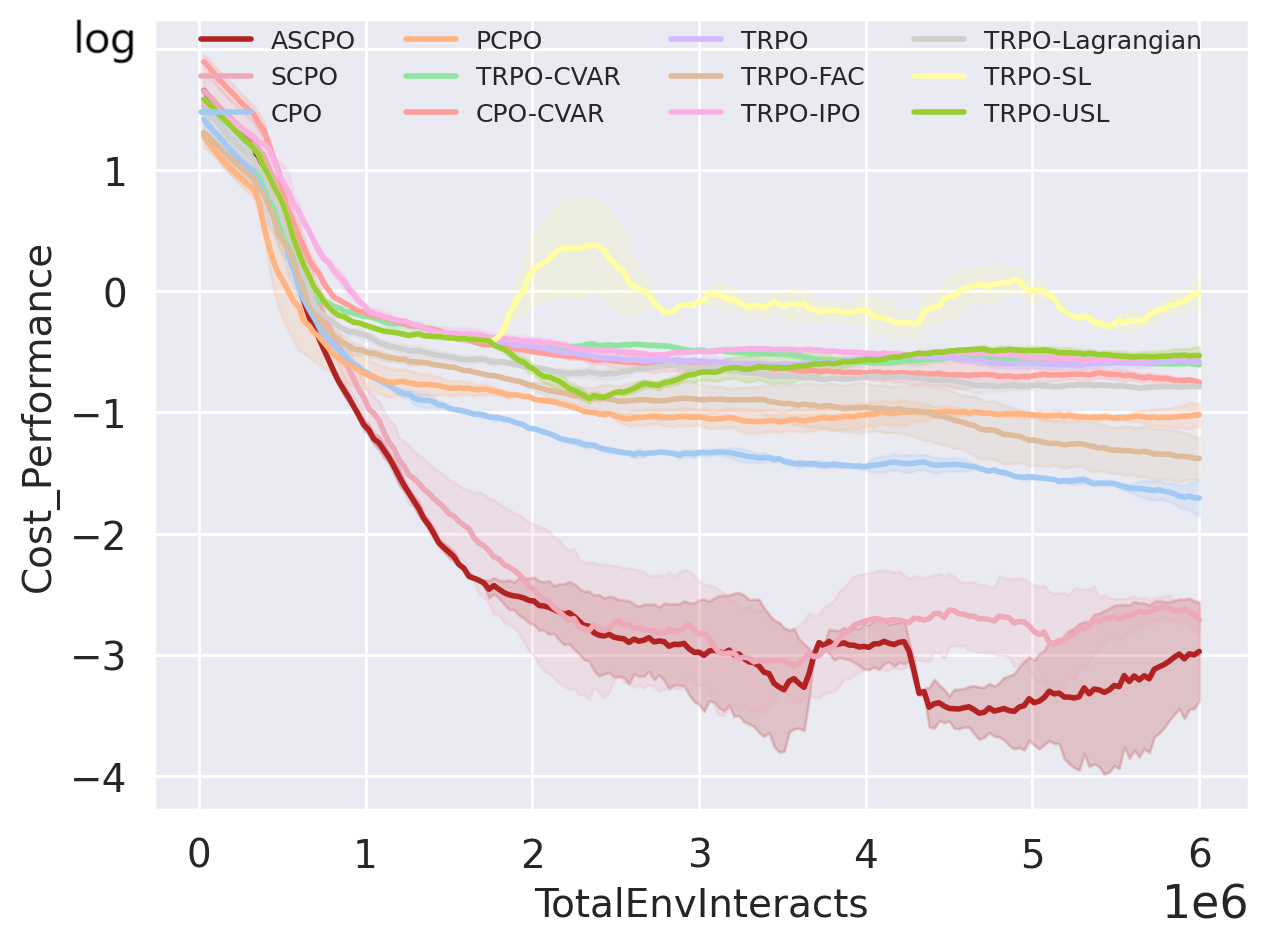}
          \end{subfigure}
          \hfill
          \begin{subfigure}[t]{1.0\linewidth}
              \includegraphics[width=0.99\textwidth]{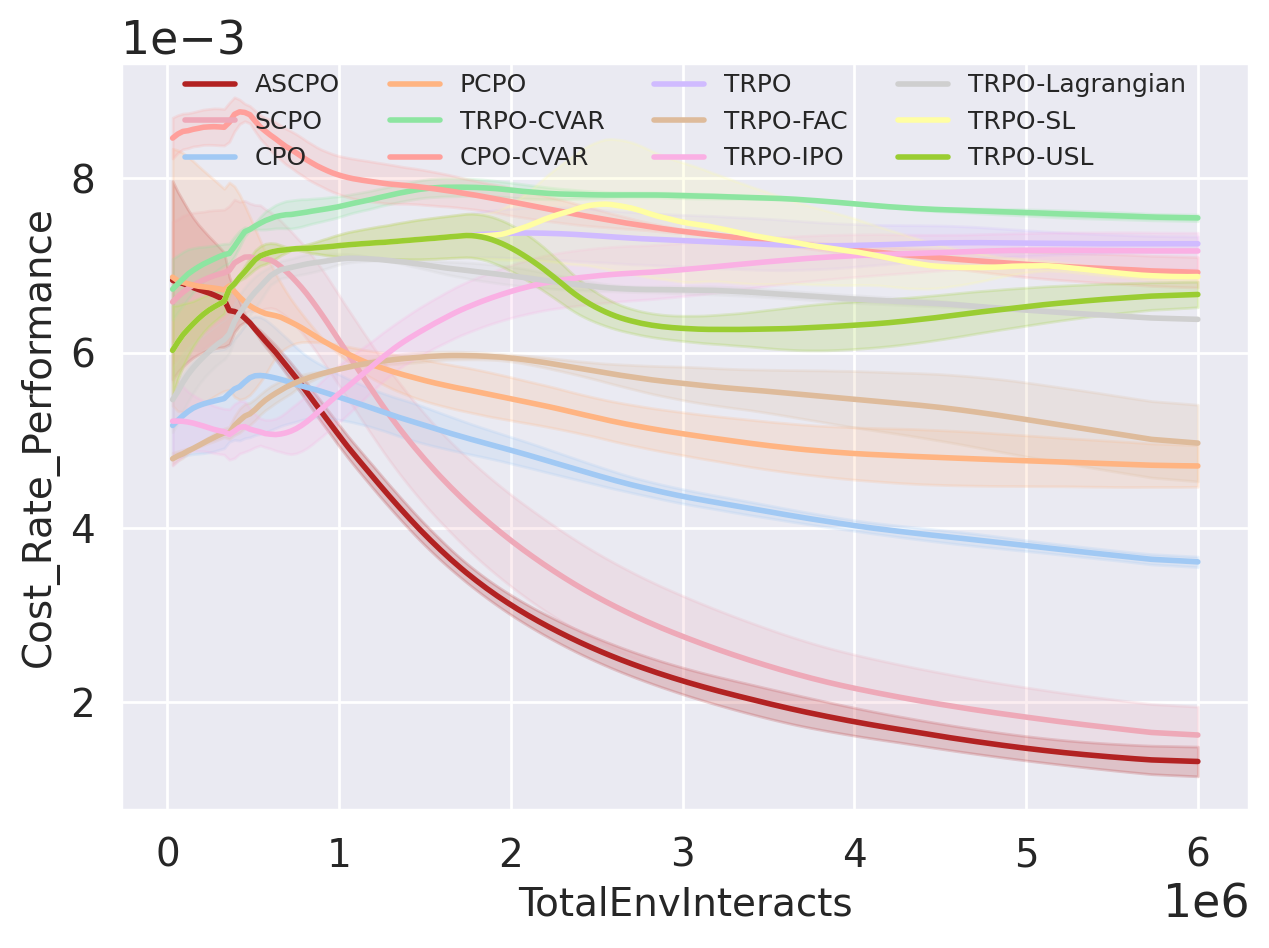}
          \end{subfigure}
      \caption{Point-8-Hazards}
      \label{Point-8-Hazards}
      \end{subfigure}
      \begin{subfigure}[t]{0.24\linewidth}
           \begin{subfigure}[t]{1.0\linewidth}
            \includegraphics[width=0.99\textwidth]{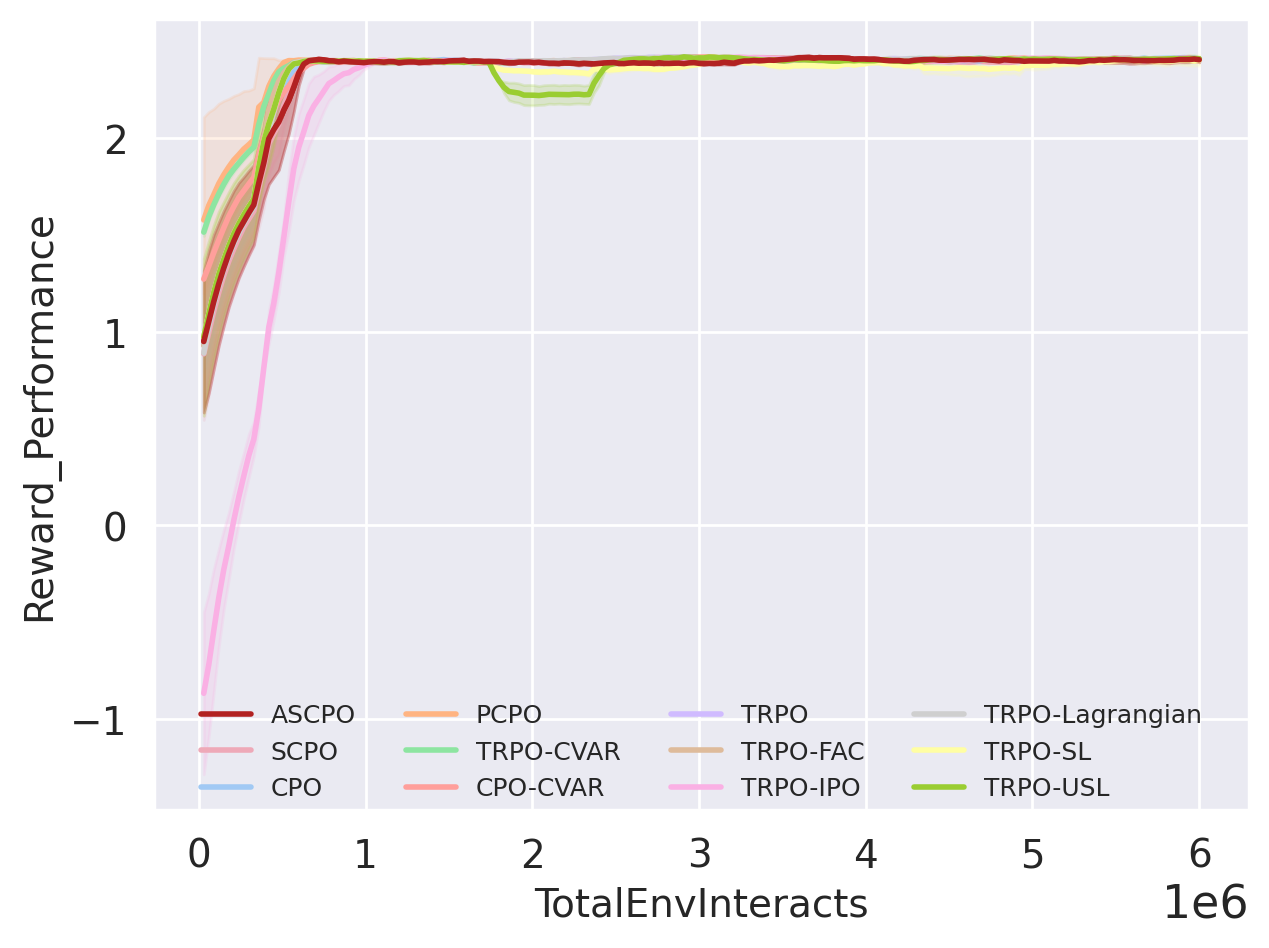}
          \end{subfigure}
          \hfill
          \begin{subfigure}[t]{1.0\linewidth}
              \includegraphics[width=0.99\textwidth]{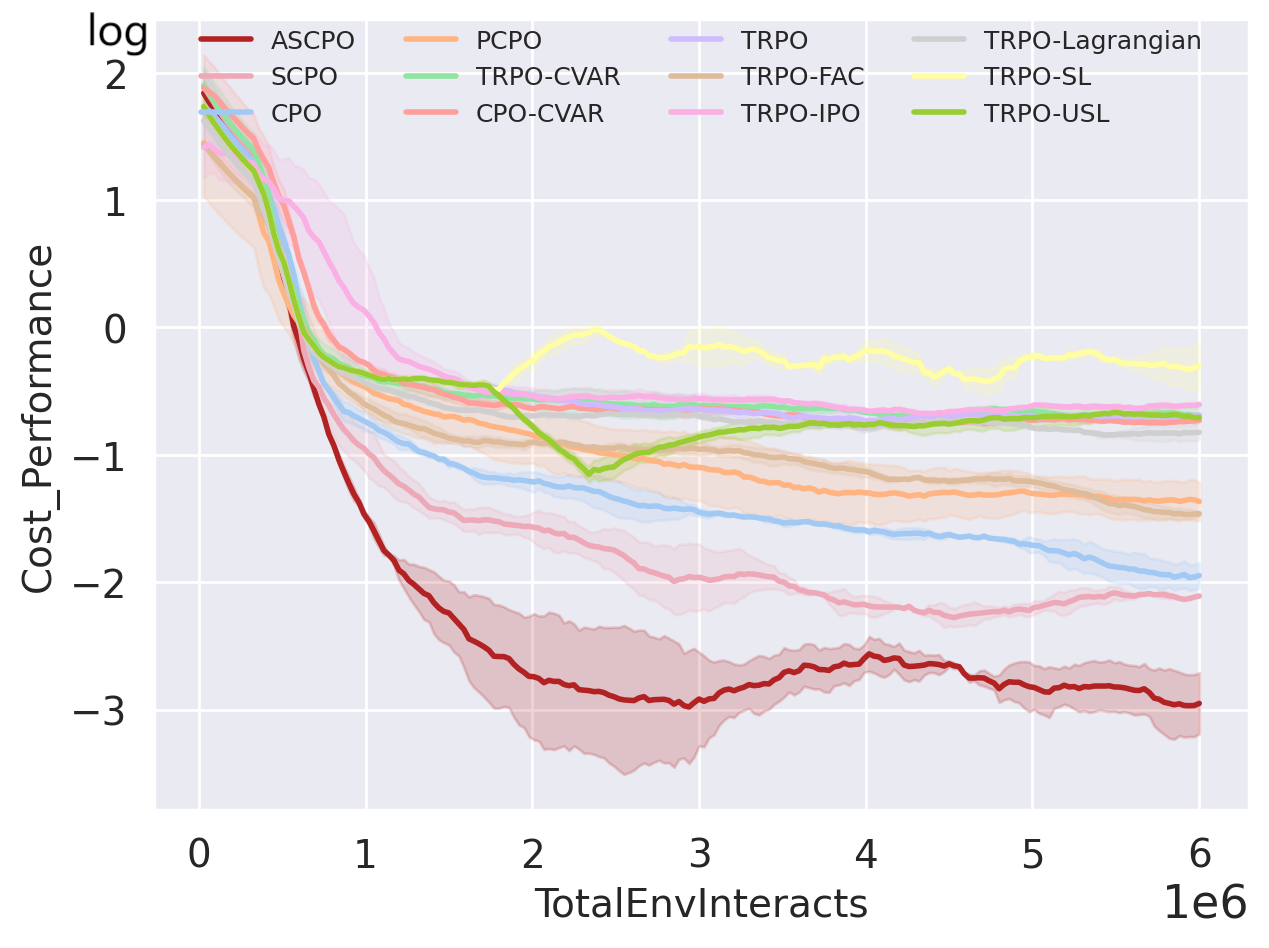}
          \end{subfigure}
          \hfill
          \begin{subfigure}[t]{1.0\linewidth}
              \includegraphics[width=0.99\textwidth]{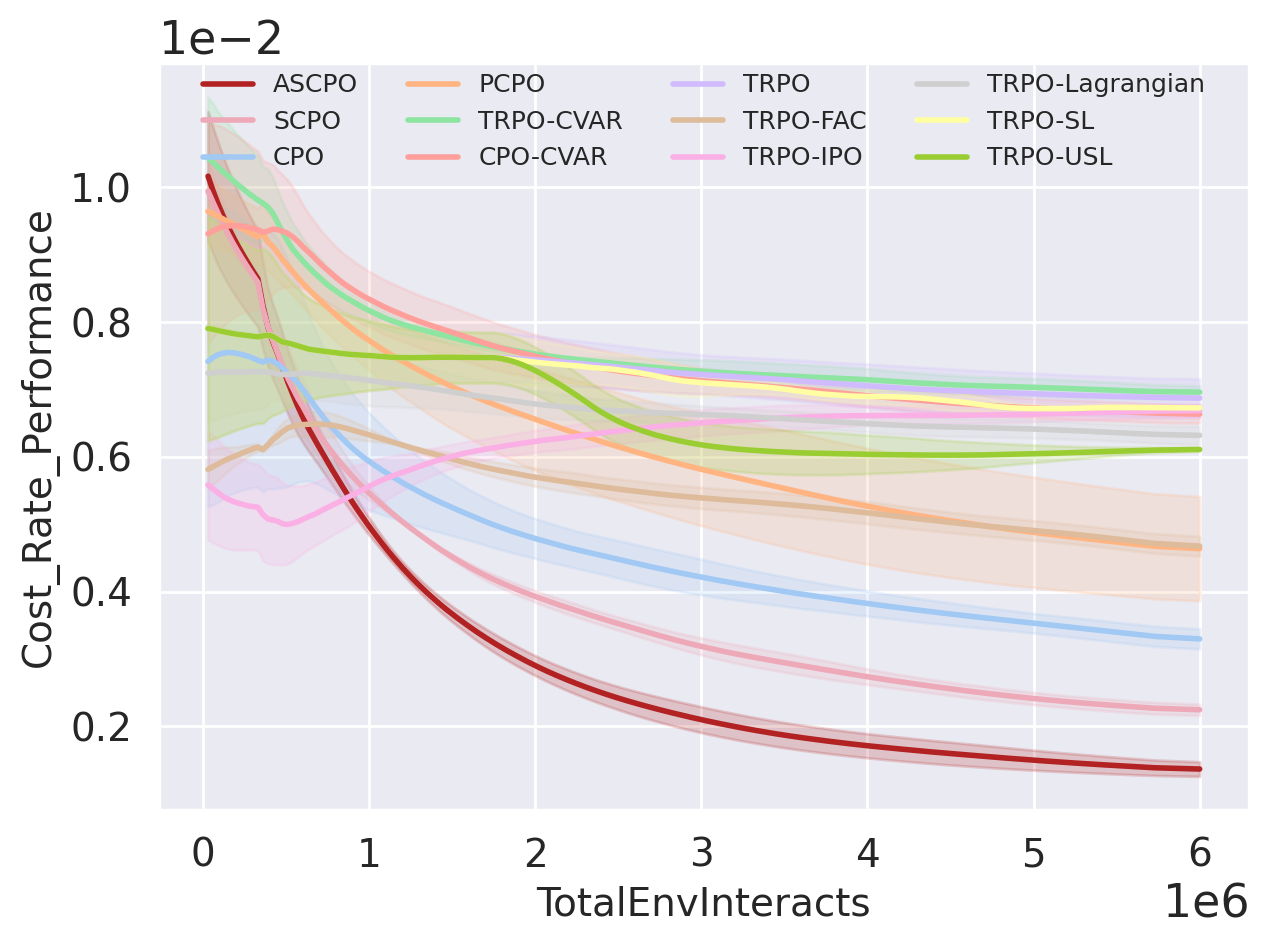}
          \end{subfigure}
      \caption{Point-8-Ghosts}
      \label{Point-8-Ghosts}
      \end{subfigure}
      \begin{subfigure}[t]{0.24\linewidth}
           \begin{subfigure}[t]{1.0\linewidth}
            \includegraphics[width=0.99\textwidth]{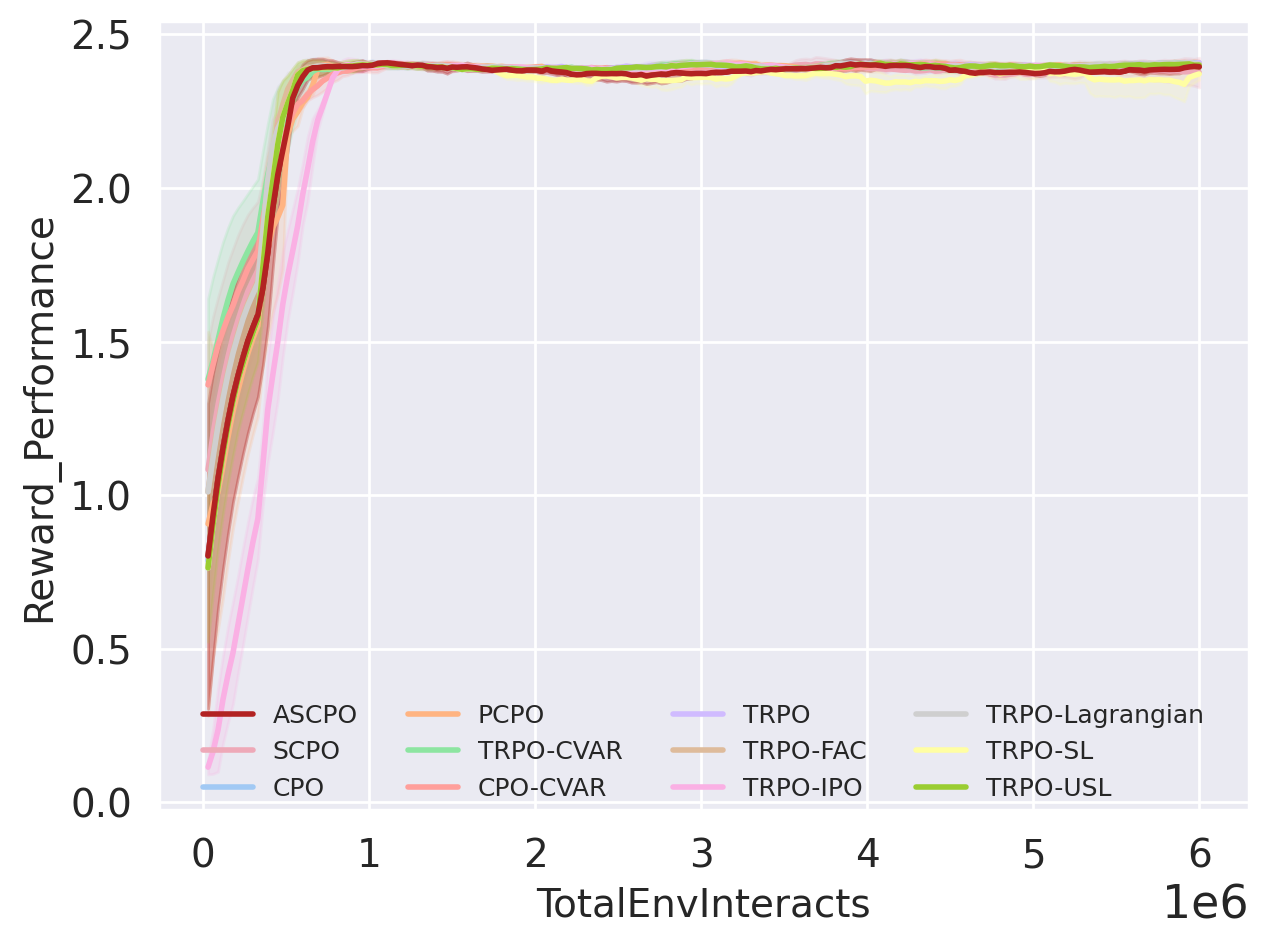}
          \end{subfigure}
          \hfill
          \begin{subfigure}[t]{1.0\linewidth}
              \includegraphics[width=0.99\textwidth]{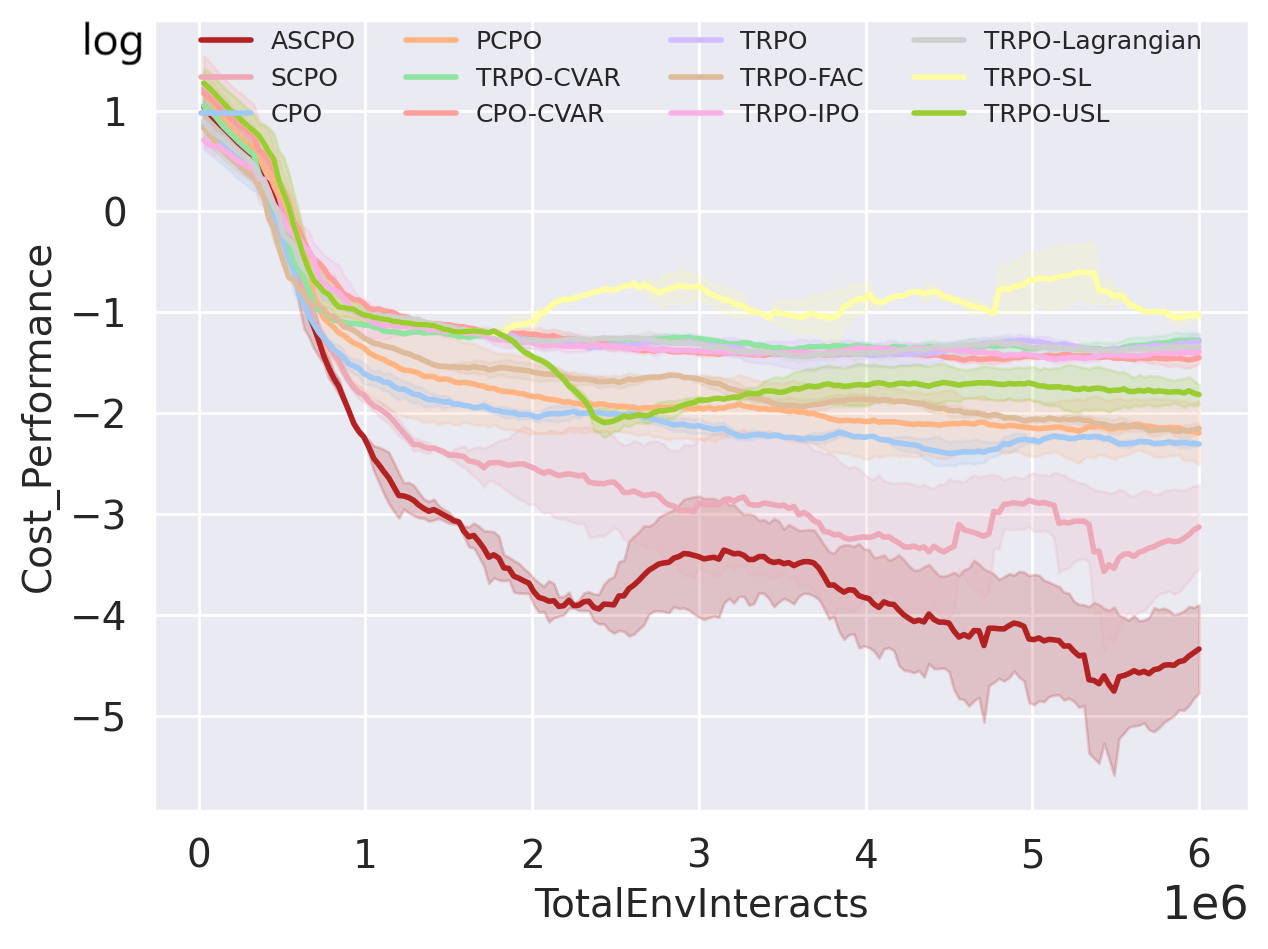}
          \end{subfigure}
          \hfill
          \begin{subfigure}[t]{1.0\linewidth}
              \includegraphics[width=0.99\textwidth]{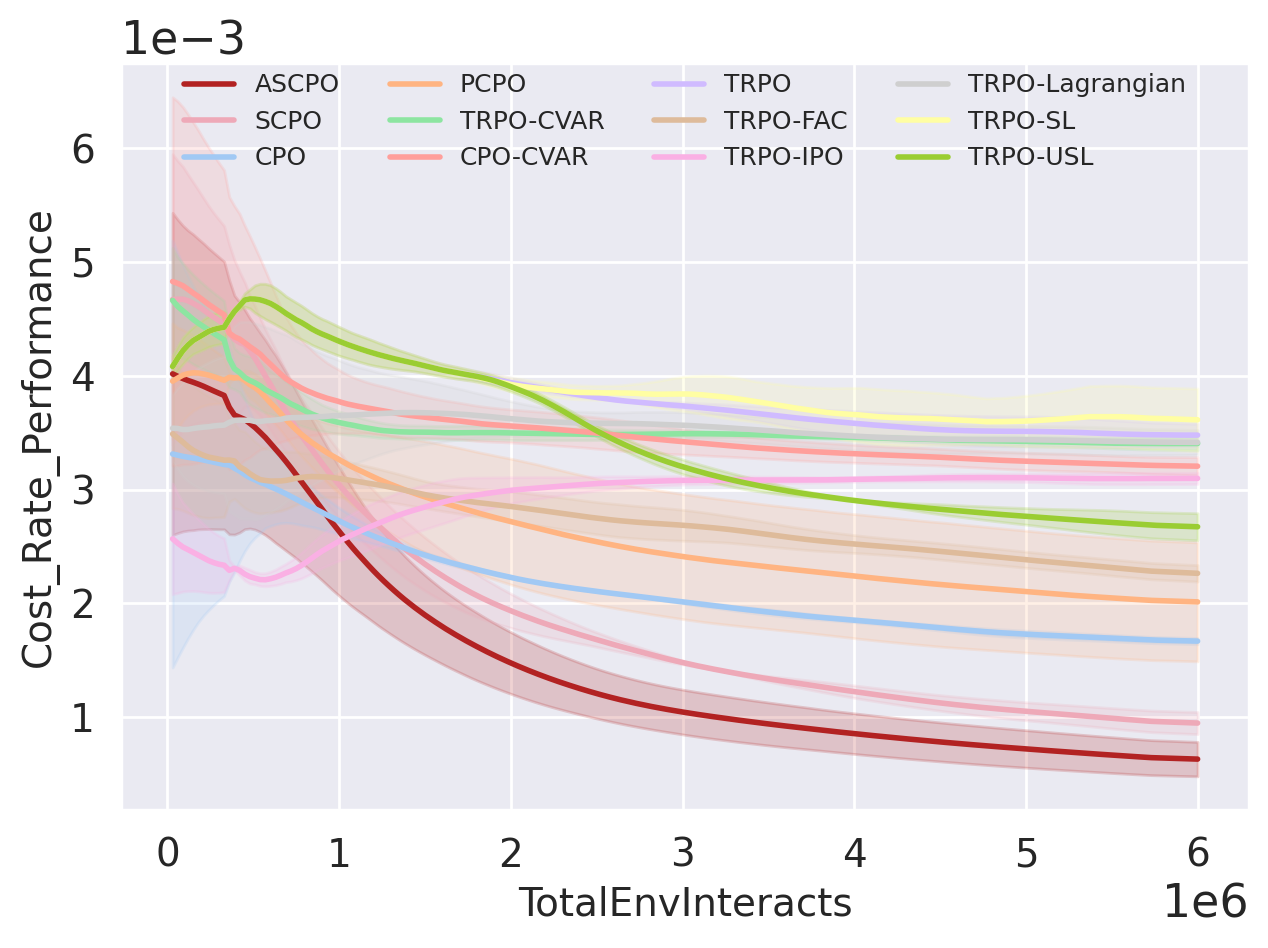}
          \end{subfigure}
      \caption{Point-4-Ghosts}
      \label{Point-4-Ghosts}
      \end{subfigure}
      \begin{subfigure}[t]{0.24\linewidth}
           \begin{subfigure}[t]{1.0\linewidth}
            \includegraphics[width=0.99\textwidth]{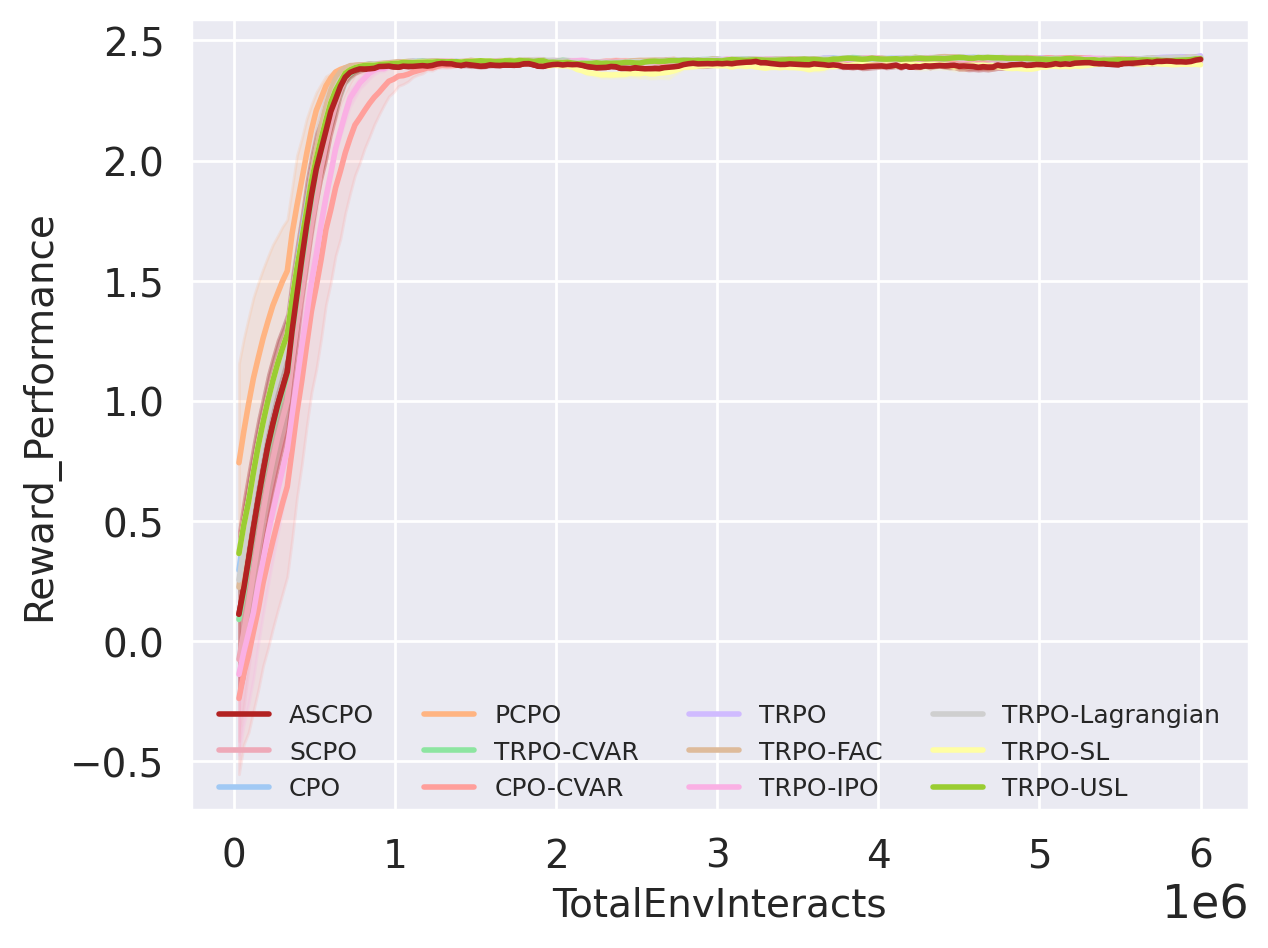}
          \end{subfigure}
          \hfill
          \begin{subfigure}[t]{1.0\linewidth}
              \includegraphics[width=0.99\textwidth]{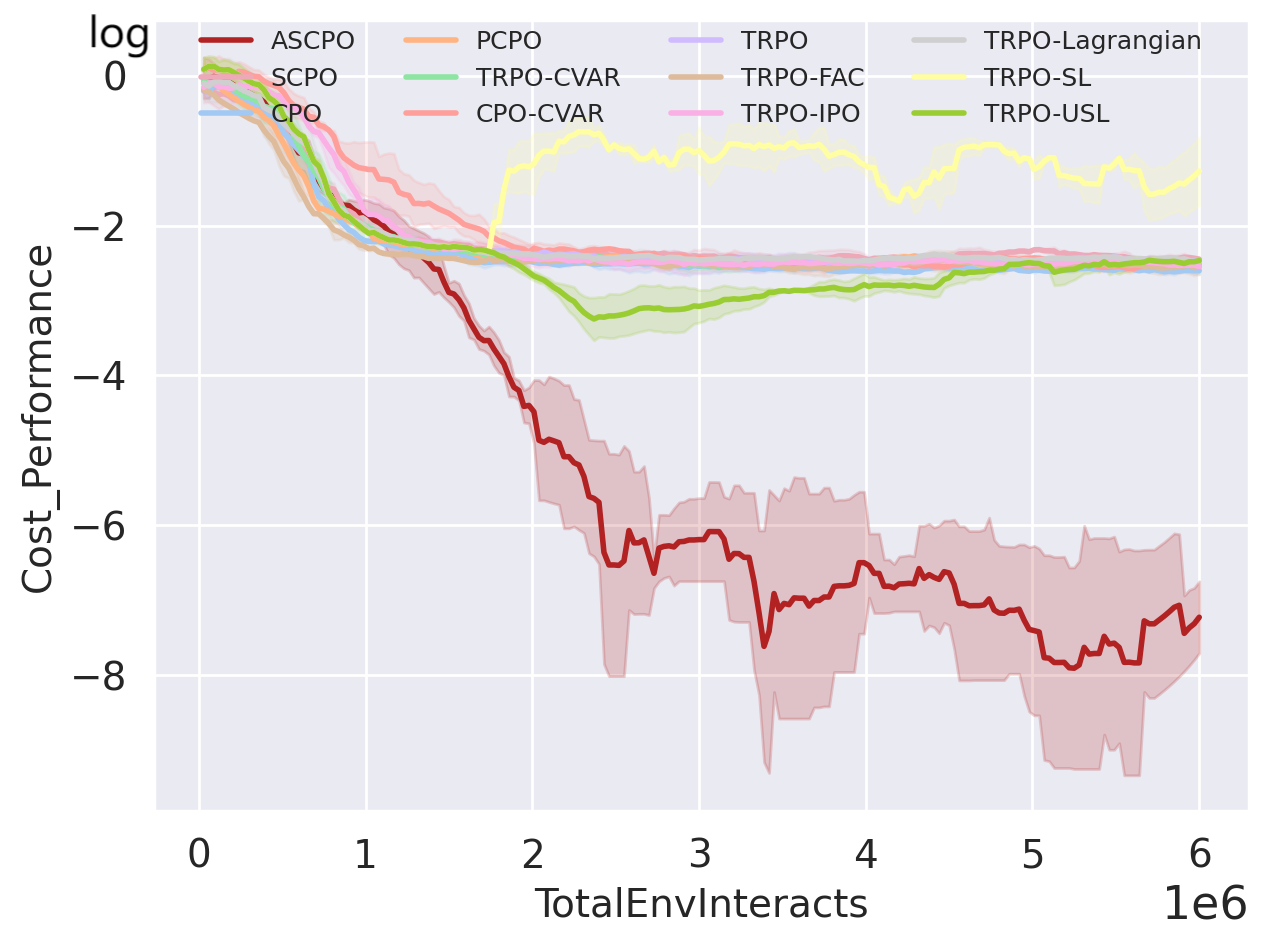}
          \end{subfigure}
          \hfill
          \begin{subfigure}[t]{1.0\linewidth}
              \includegraphics[width=0.99\textwidth]{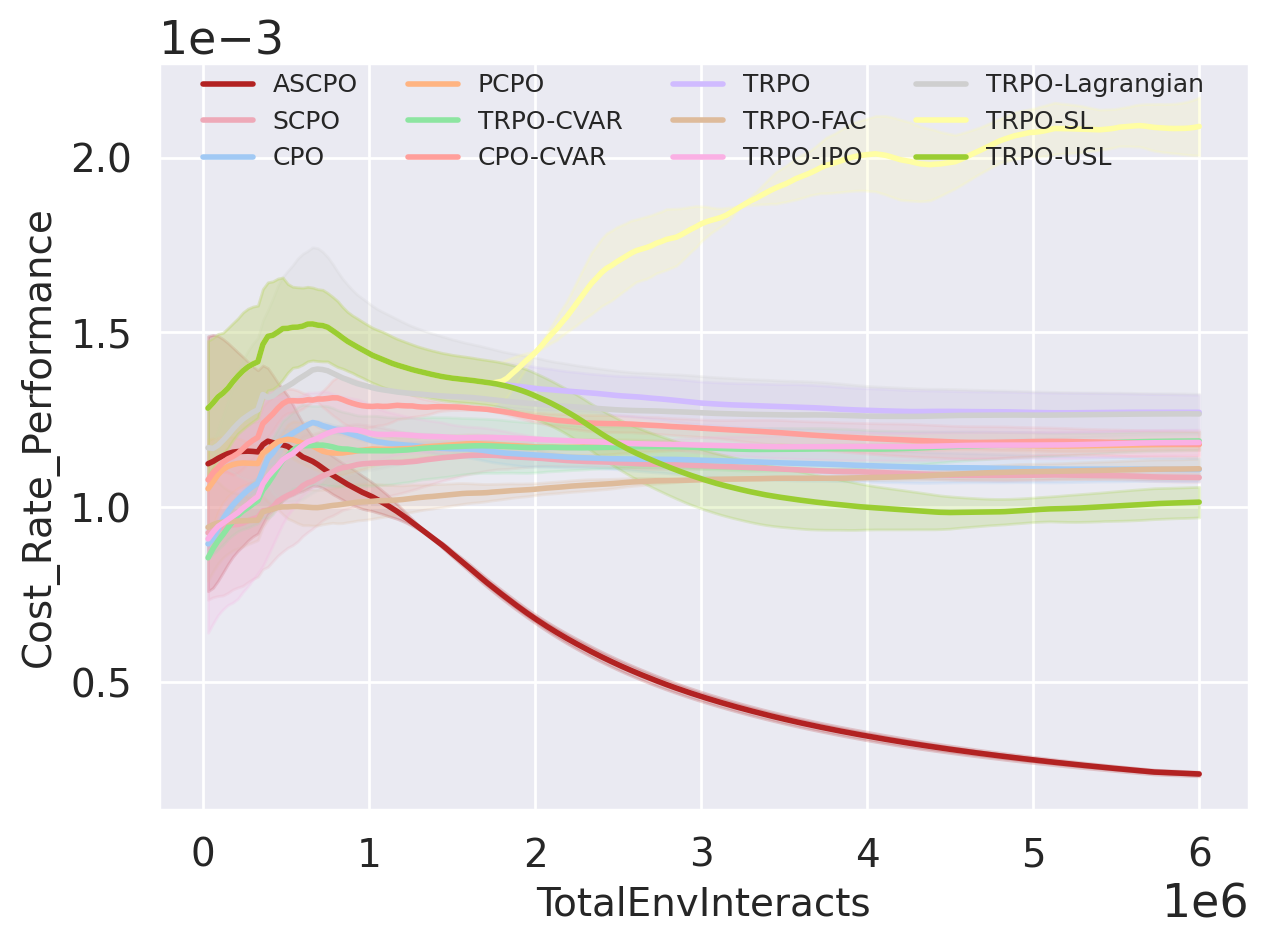}
          \end{subfigure}
      \caption{Swimmer-1-Hazards}
      \label{Swimmer-1-Hazards}
      \end{subfigure}
  \end{subfigure}

  \begin{subfigure}[t]{1.0\linewidth}
      \centering
      \begin{subfigure}[t]{0.24\linewidth}
           \begin{subfigure}[t]{1.0\linewidth}
            \includegraphics[width=0.99\textwidth]{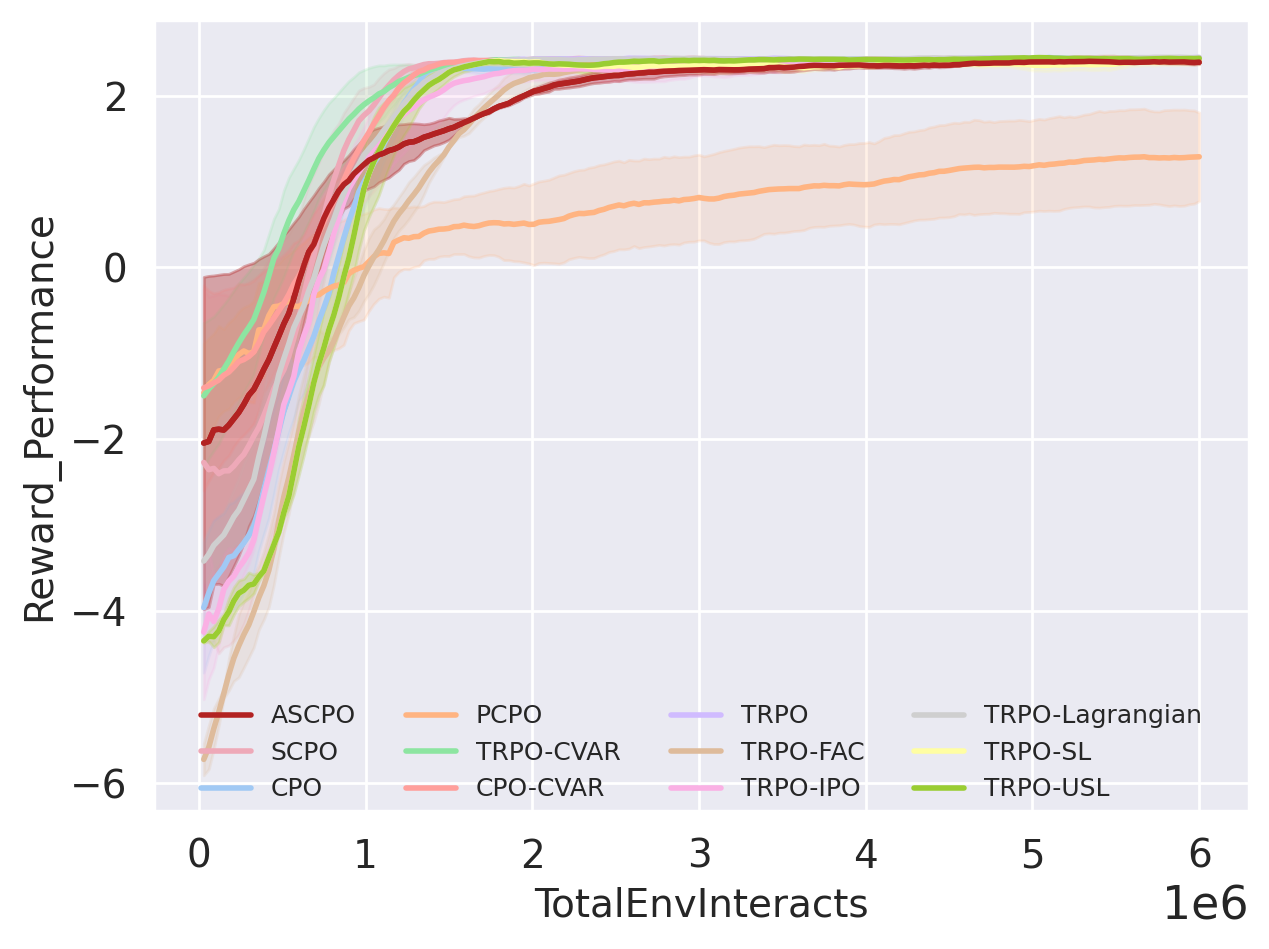}
          \end{subfigure}
          \hfill
          \begin{subfigure}[t]{1.0\linewidth}
              \includegraphics[width=0.99\textwidth]{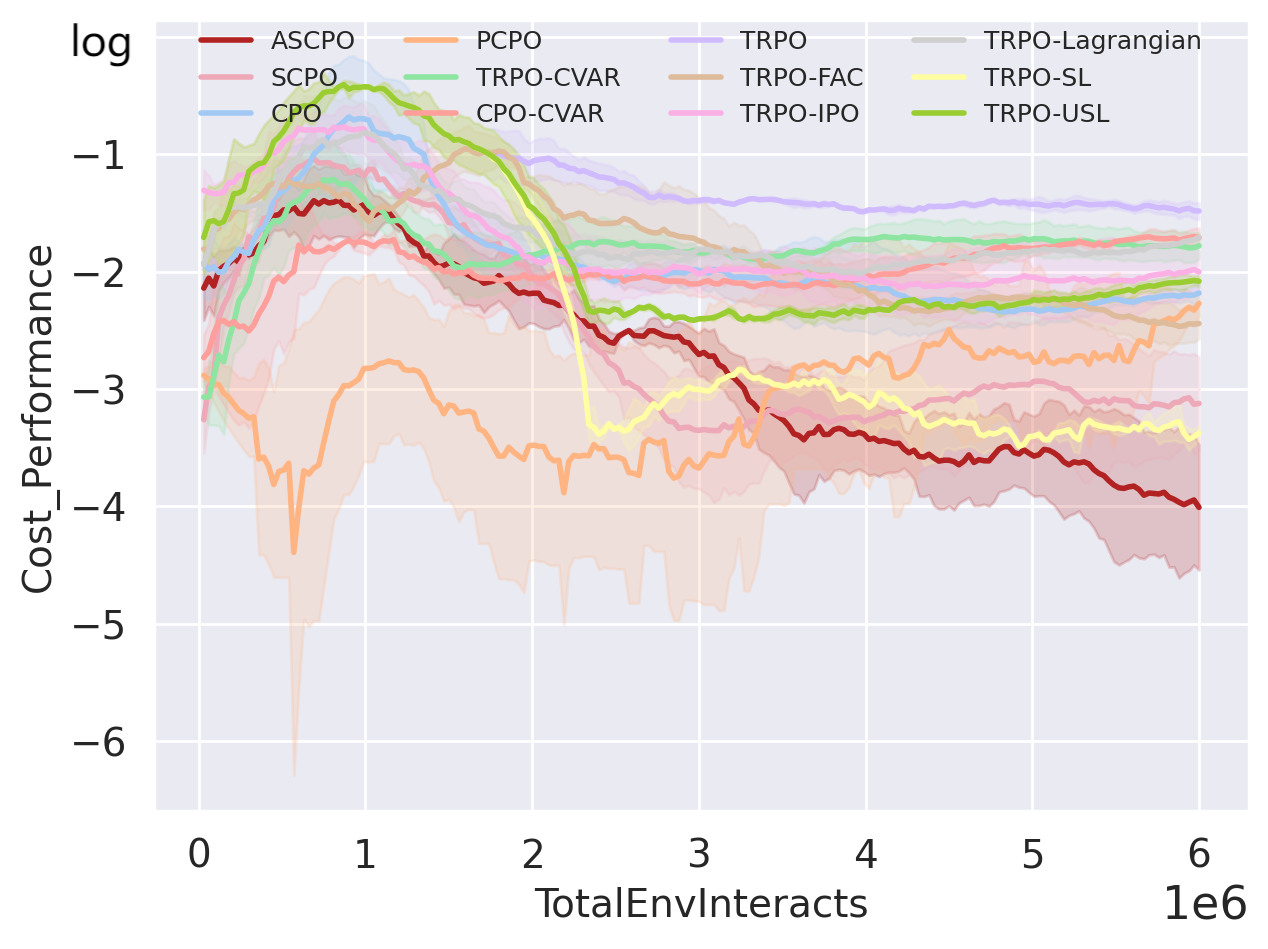}
          \end{subfigure}
          \hfill
          \begin{subfigure}[t]{1.0\linewidth}
              \includegraphics[width=0.99\textwidth]{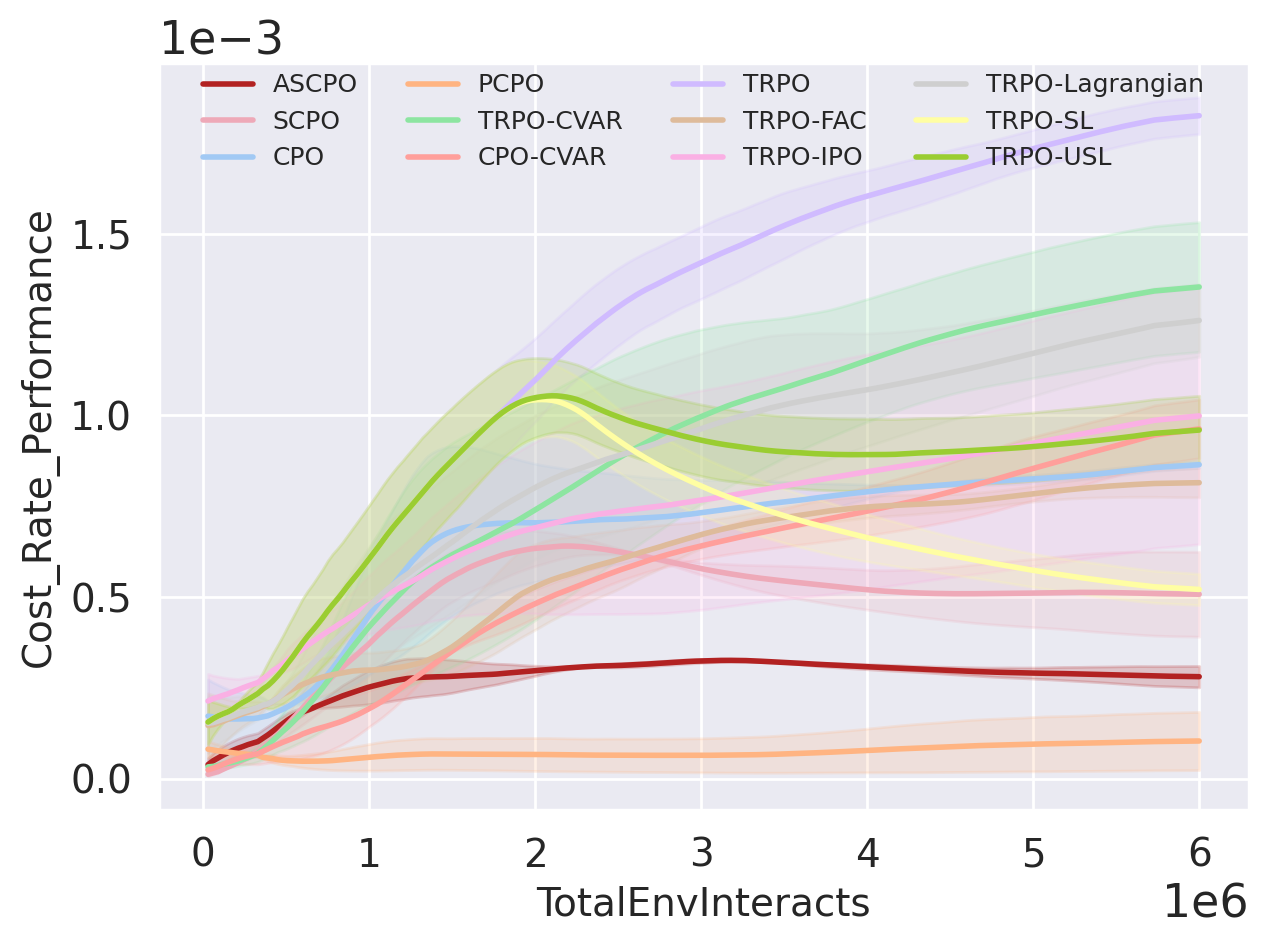}
          \end{subfigure}
      \caption{Drone-8-Hazards}
      \label{Drone-8-Hazards}
      \end{subfigure}
      \begin{subfigure}[t]{0.24\linewidth}
           \begin{subfigure}[t]{1.0\linewidth}
            \includegraphics[width=0.99\textwidth]{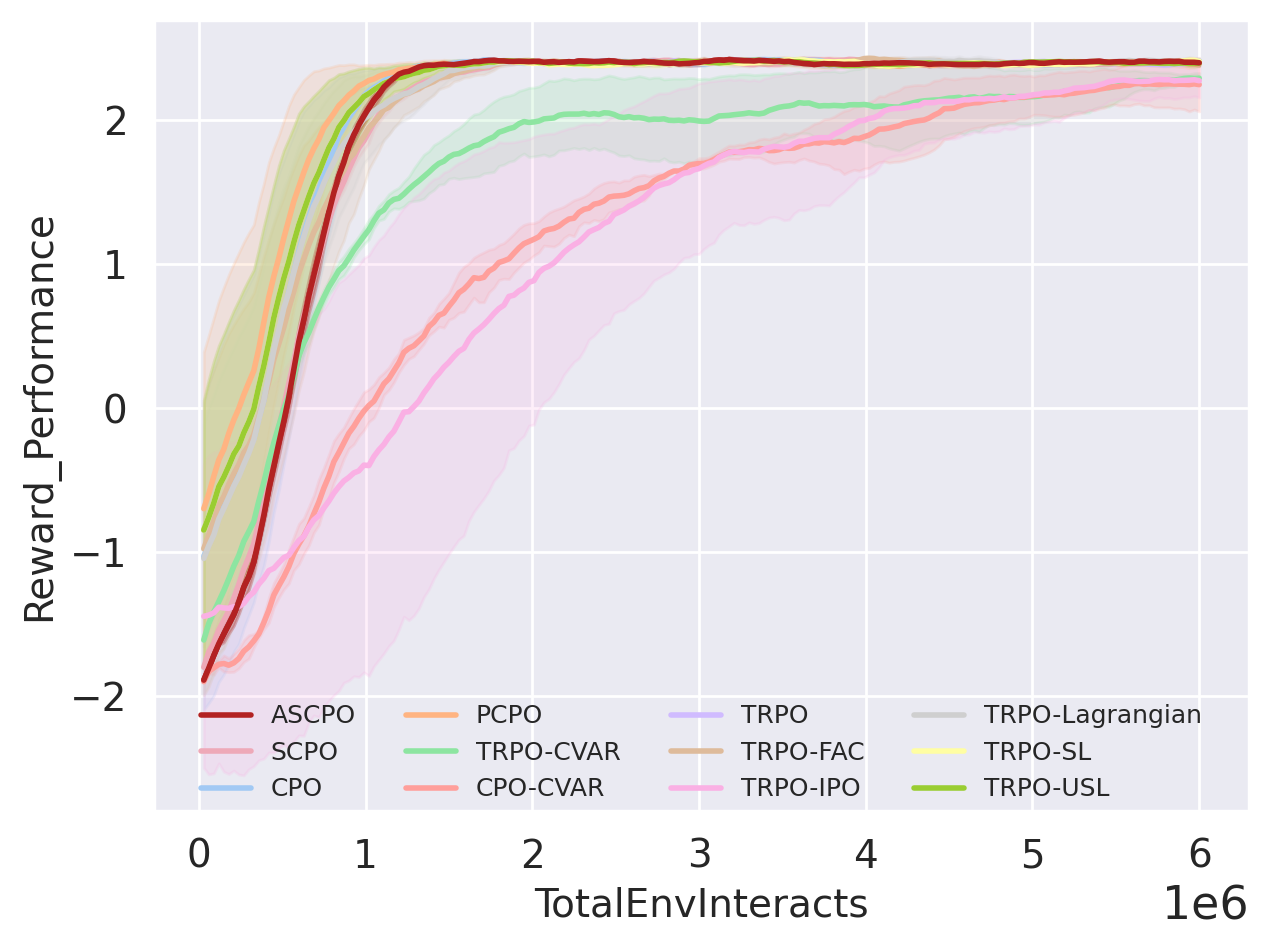}
          \end{subfigure}
          \hfill
          \begin{subfigure}[t]{1.0\linewidth}
              \includegraphics[width=0.99\textwidth]{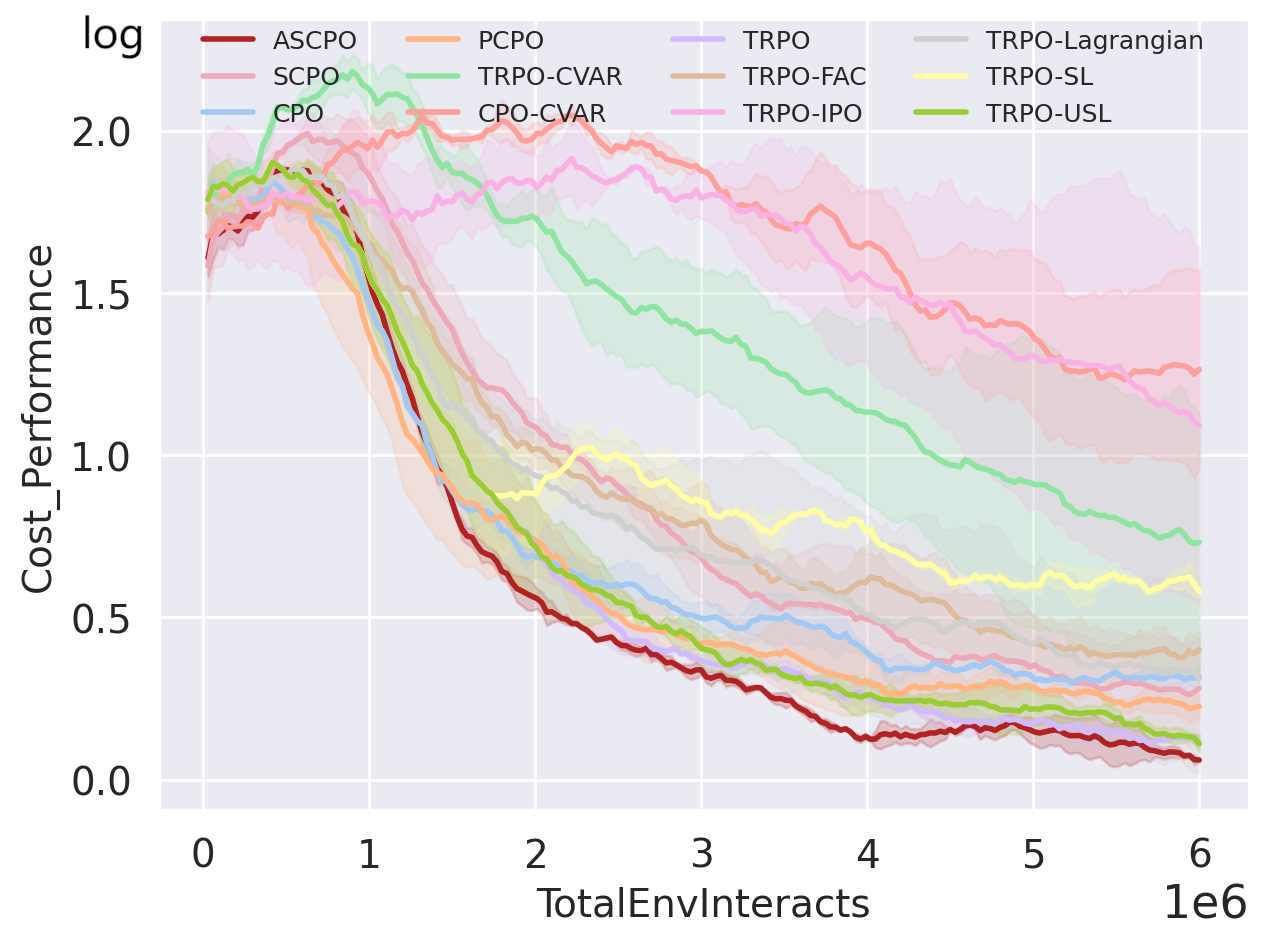}
          \end{subfigure}
          \hfill
          \begin{subfigure}[t]{1.0\linewidth}
              \includegraphics[width=0.99\textwidth]{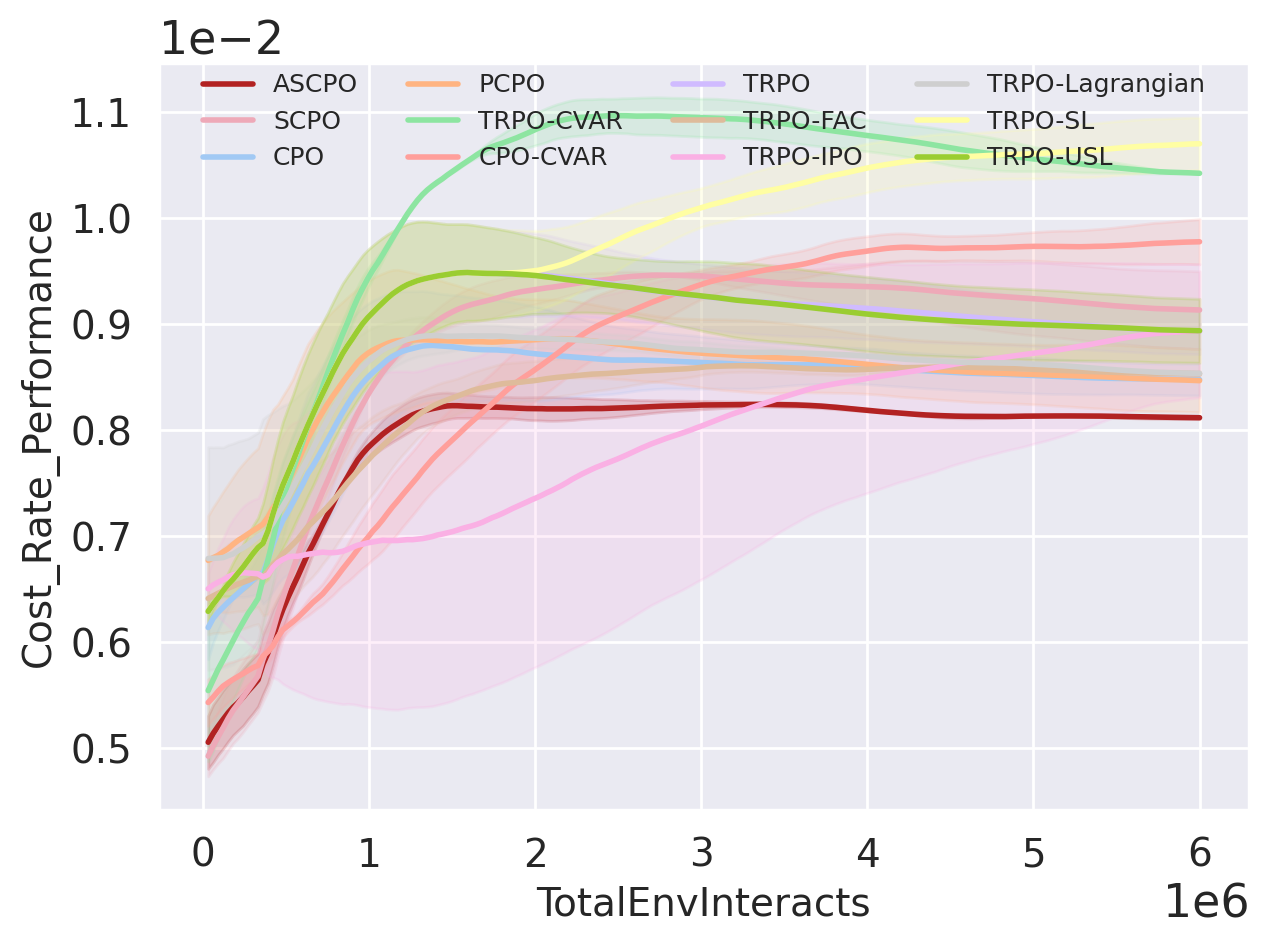}
          \end{subfigure}
      \caption{Humanoid-8-Hazards}
      \label{Humanoid-8-Hazards}
      \end{subfigure}
      \begin{subfigure}[t]{0.24\linewidth}
           \begin{subfigure}[t]{1.0\linewidth}
            \includegraphics[width=0.99\textwidth]{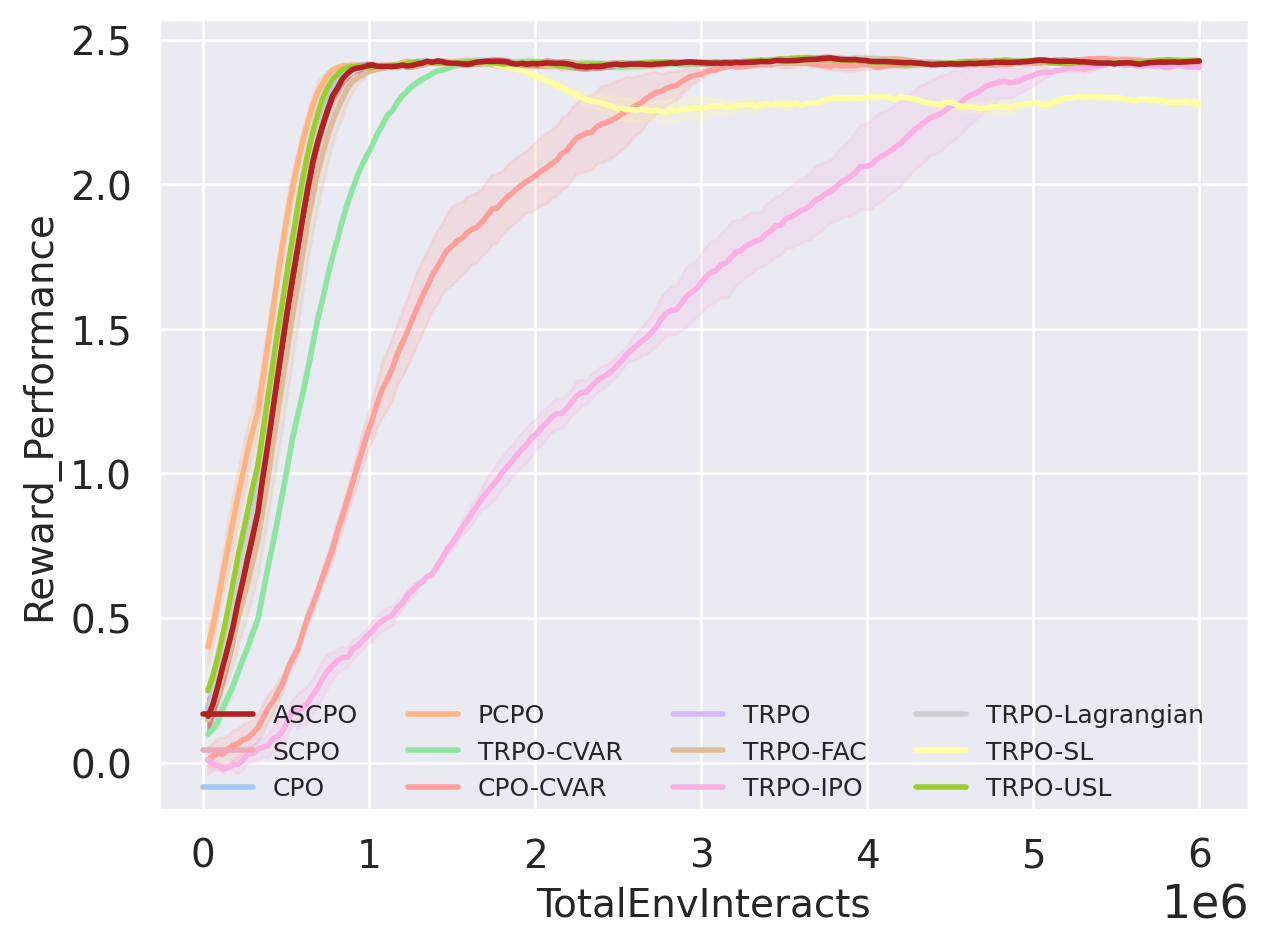}
          \end{subfigure}
          \hfill
          \begin{subfigure}[t]{1.0\linewidth}
              \includegraphics[width=0.99\textwidth]{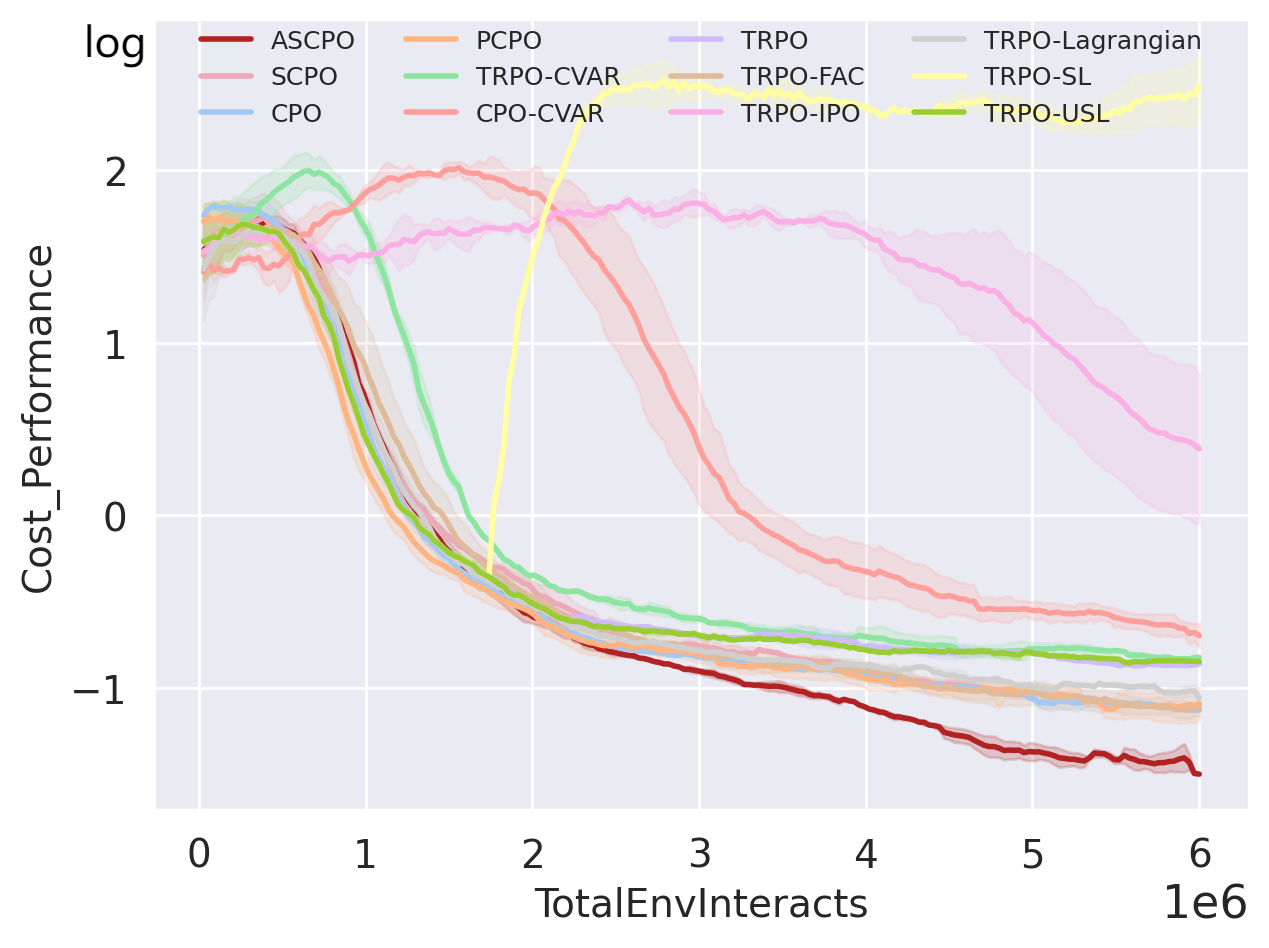}
          \end{subfigure}
          \hfill
          \begin{subfigure}[t]{1.0\linewidth}
              \includegraphics[width=0.99\textwidth]{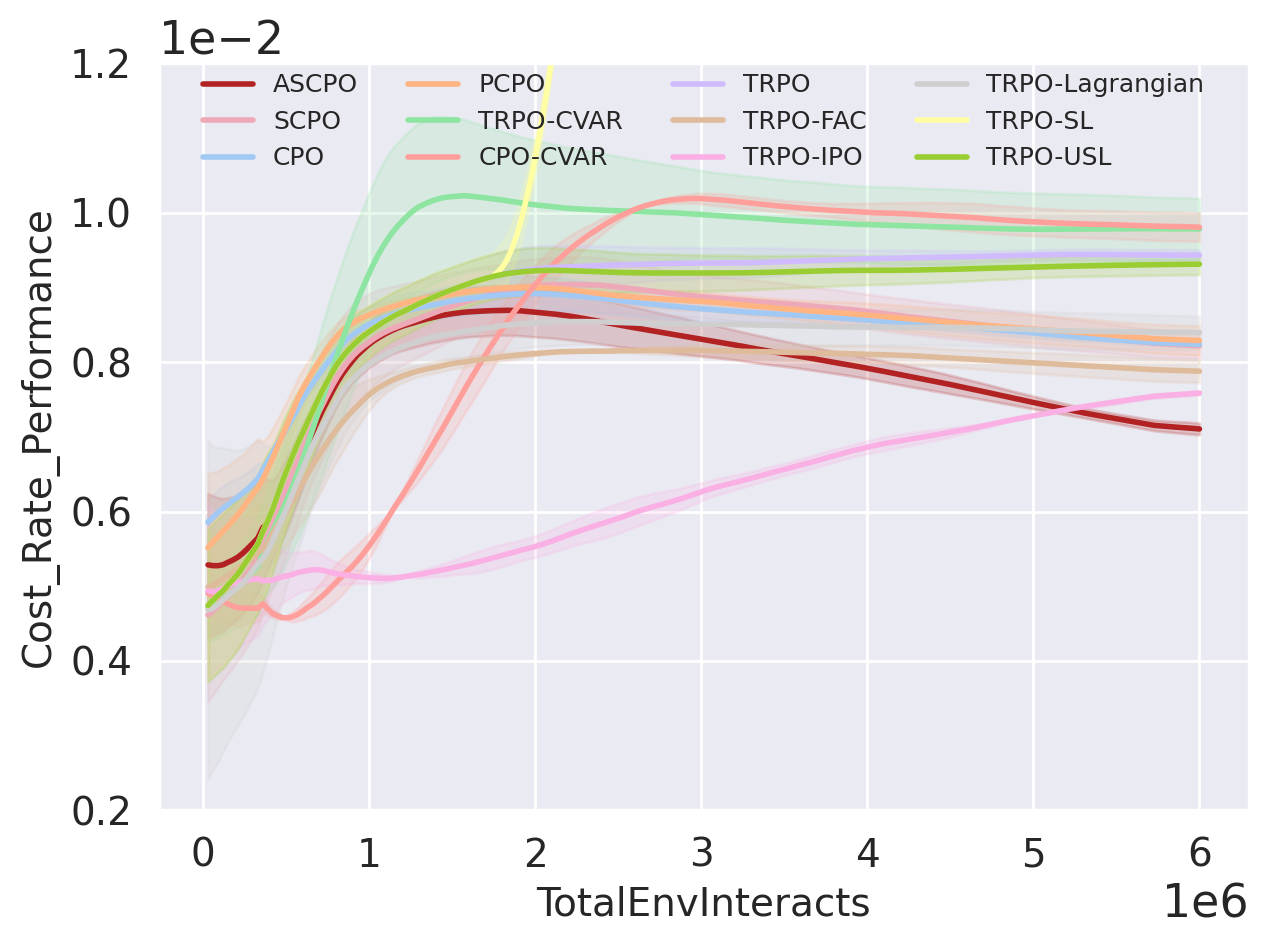}
          \end{subfigure}
      \caption{Ant-8-Hazards}
      \label{Ant-8-Hazards}
      \end{subfigure}
      \begin{subfigure}[t]{0.24\linewidth}
           \begin{subfigure}[t]{1.0\linewidth}
            \includegraphics[width=0.99\textwidth]{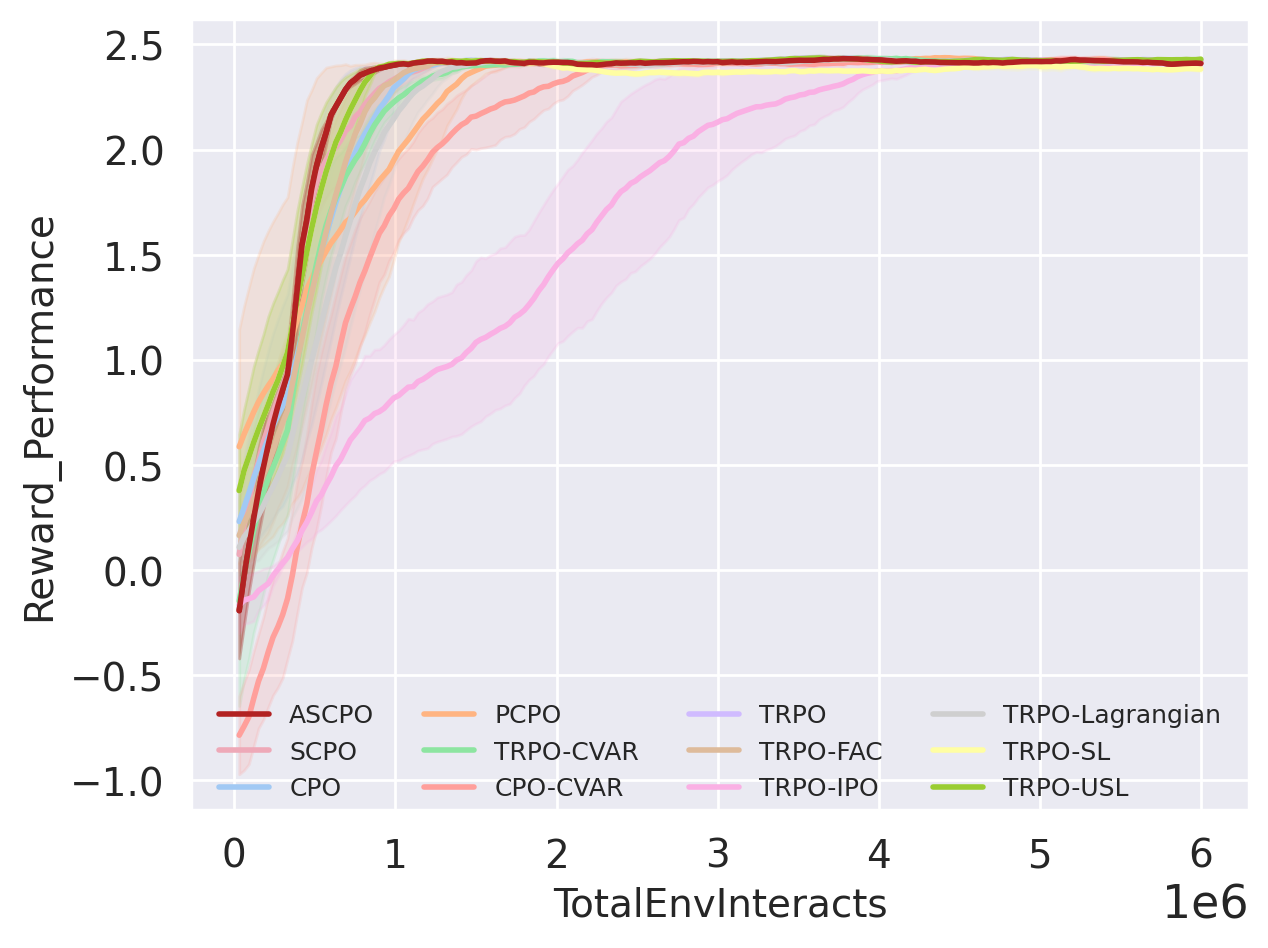}
          \end{subfigure}
          \hfill
          \begin{subfigure}[t]{1.0\linewidth}
              \includegraphics[width=0.99\textwidth]{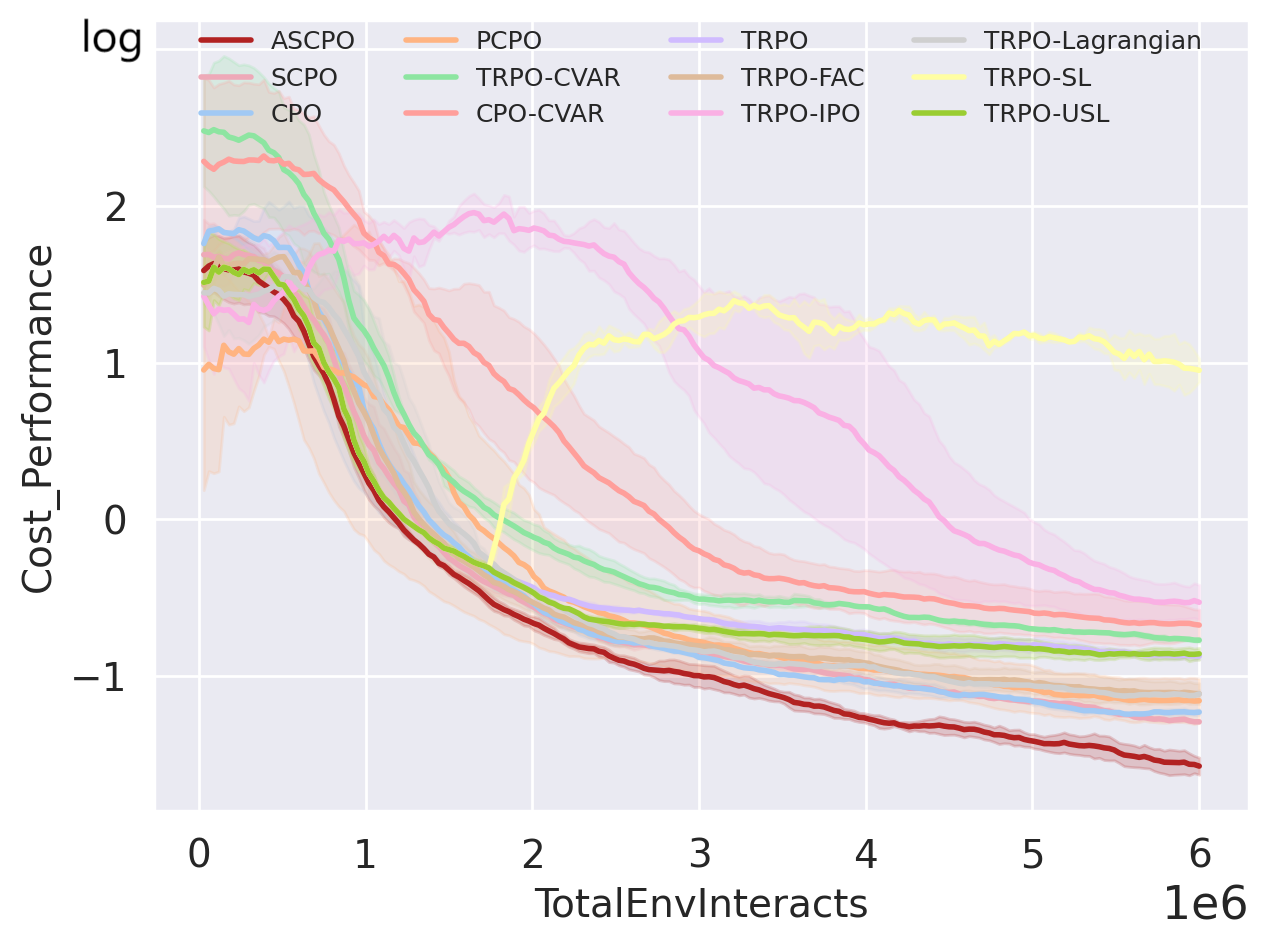}
          \end{subfigure}
          \hfill
          \begin{subfigure}[t]{1.0\linewidth}
              \includegraphics[width=0.99\textwidth]{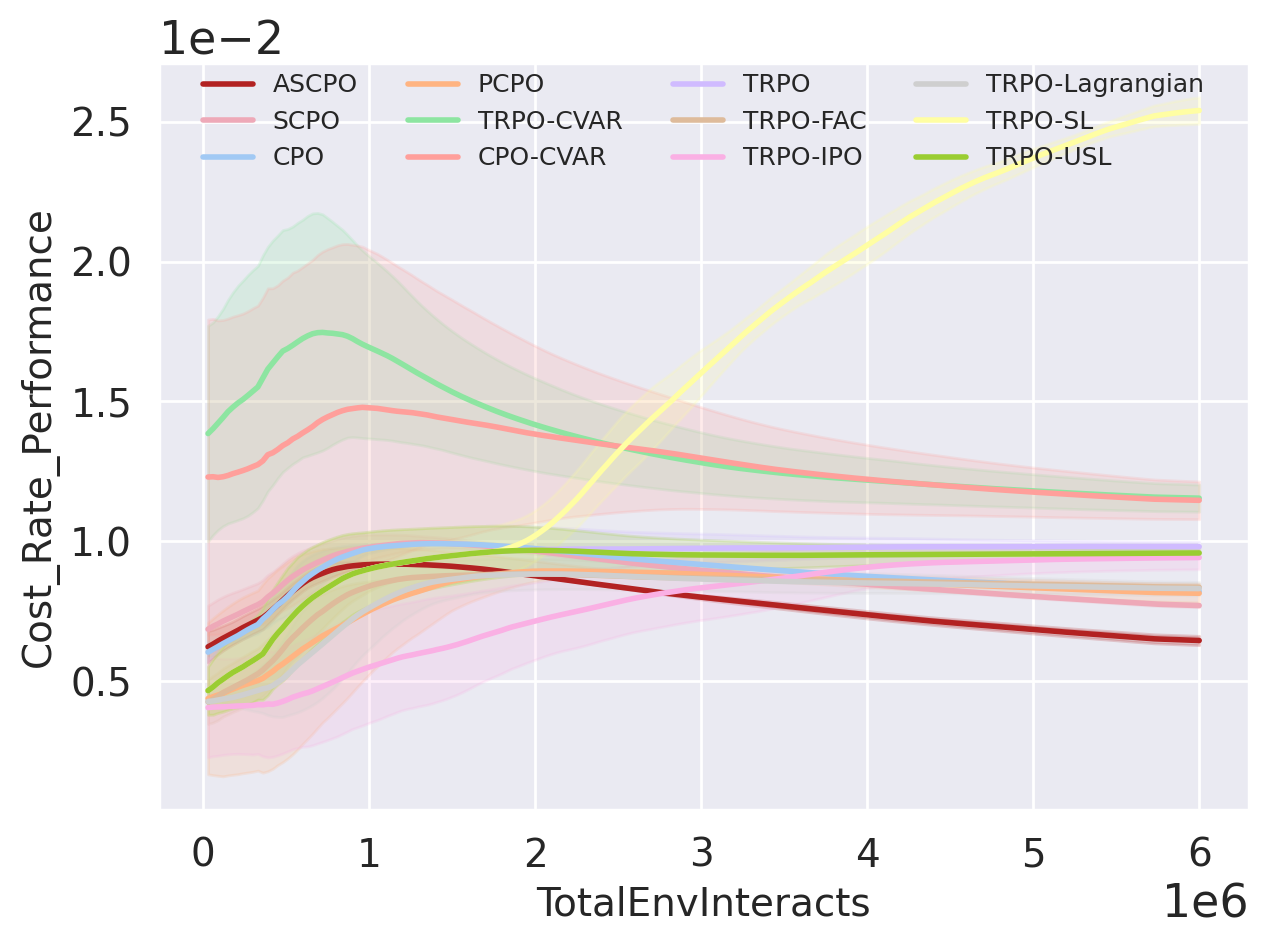}
          \end{subfigure}
      \caption{Walker-8-Hazards}
      \label{Walker-8-Hazards}
      \end{subfigure}
  \end{subfigure}
 
  \label{fig: main experiments}
  \caption{Comparison of results from (i) low dimensional systems with different types and numbers of constraints (Point, Swimmer: \cref{Point-8-Hazards,Point-8-Ghosts,Point-4-Ghosts,Swimmer-1-Hazards}) (ii) practical robots (Drone and Humanoid: \cref{Drone-8-Hazards,Humanoid-8-Hazards}) (iii) high dimensional systems (Ant, Walker:\cref{Ant-8-Hazards,Walker-8-Hazards}). The Y-axis of `Cost\_Performance' here uses logarithmic coordinates to show small level differences.}
\end{figure}

\subsection{Evaluating ASCPO and Comparison Analysis}
\label{sec: evaluate ASCPO}

\paragraph{Low Dimensional Systems}
We have selected four representative test suites (\cref{Point-8-Hazards,Point-8-Ghosts,Point-4-Ghosts,Swimmer-1-Hazards}) to exemplify ASCPO's performance on low-dimensional systems. The outcomes underscore ASCPO's remarkable ability to simultaneously achieve near-zero constraint violations and rapid reward convergence, a feat challenging for other algorithms within our comparison group. Specifically, compared to other baseline methods, ASCPO demonstrates: (i) near-zero average episode cost at the swiftest rate, (ii) substantially reduced cost rates, and (iii) stable and expedited reward convergence. 
Baseline end-to-end CMDP methods (CPO, PCPO, TRPO-Lagrangian, TRPO-FAC, TRPO-IPO) fall short of achieving near-zero cost performance even under a cost limit of zero. Similarly, while SCPO can approach near-zero cost, it exhibits deficiencies in both convergence speed and cost rate reduction, due to fundamental limitation in only regulating expectation in SCMDP methods. Risk-sensitive methods (TRPO-CVaR, CPO-CVaR) prove unsuitable for achieving optimal synthesis and are incapable of nearly zero cost performance.
Even with an explicit safety layer correcting unsafe actions at each time step, baseline hierarchical safe RL methods (TRPO-SL, TRPO-USL) fail to achieve near-zero cost performance due to inaccuracies stemming from the linear approximation of the cost function\citep{dalal2018safe} when confronted with highly nonlinear dynamics as encountered in our MuJoCo environments\citep{todorov2012mujoco}. A comprehensive summary of these results is provided in \Cref{sec:metrics}.

\paragraph{Practical Systems and High Dimension Systems}
To demonstrate the scalability of ASCPO and potential to solve complex robot learning problems, we conducted a series of experiments on practical and high-dimensional systems (\cref{Drone-8-Hazards,Humanoid-8-Hazards,Ant-8-Hazards,Walker-8-Hazards}). The results underscore ASCPO's ability to surpass all other algorithms by rapidly achieving both reward and cost convergence while maintaining a remarkably lower cost rate.

\begin{wrapfigure}{r}{0.5\textwidth}
    \vspace{-10pt}
    \centering
    \begin{subfigure}[b]{0.49\textwidth}
        \begin{subfigure}[t]{1.00\textwidth}
        \raisebox{-\height}{\includegraphics[width=\textwidth]{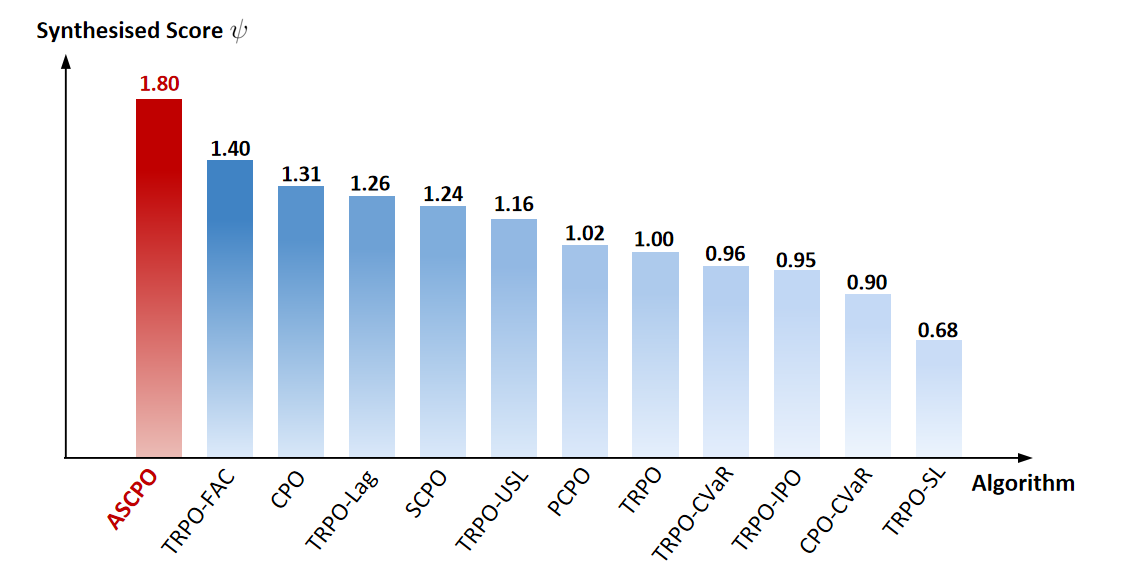}}
        \end{subfigure}
    \end{subfigure}
    \caption{The average of each algorithm's synthesised score $\psi$ on the 19 tasks.} 
    \label{fig: psi results}
    \vspace{-10pt}
\end{wrapfigure}

It is noteworthy that the PCPO demonstrates superior cost rate reduction performance in the Drone-8Hazards test suite. However, this comes at the expense of intolerably slow and unstable reward increases, rendering it unsuitable for simulation or real-world implementation. In the remaining three test suites, the Lagrangian method exhibits a favorable cost rate performance but concurrently displays poor performance of cost reduction and reward improvement during the early training stages. This discrepancy arises from the robots' initial struggle to learn task completion and reward optimization, resulting in longer episode lengths and consequently lower cost rates. This significant flaw in the Lagrangian approach underscores its limitations. More details and ablation experiments about Lagrangian method are provided in \Cref{sec: lagrandian ablation}.

\paragraph{Comprehensive Evaluation} To demonstrate the comprehensive ability of our algorithm, we take TRPO as baseline and design a new metric named synthesised score $\psi$ as follow:
\begin{align}
\label{eq: def of psi}
    \psi_{Current} = \frac{1}{3} \left( \frac{J_r^{Current}}{J_r^{TRPO}} + \frac{M_c^{TRPO}}{M_c^{Current}} + \frac{\rho_c^{TRPO}}{\rho_c^{Current}} \right)
\end{align}
This metric represents the average improvement magnitude of the current algorithm with respect to TRPO under the three metrics $J_r$, $M_c$ and $\rho_c$. It is used to show the comprehensive performance of the safe RL algorithm in a particular task. We averaged the $\psi$ for each of the 12 algorithms across the 19 tasks and present the results in \Cref{fig: psi results}. The metrics used for computing can be found in \Cref{tab: point_hazard,tab: point_pillar,tab: point_ghost,tab: swimmer_hazard,tab: drone_3Dhazard,tab: ant_walker_hazard} in appendix. The results show that our algorithm has a cliff-leading combined effect compared to other algorithms.

Above results demonstrate the superiority of ASCPO in comparison to various other safe RL state-of-the-art methods, which answer \textbf{Q1}.

\begin{figure}[t]
  \centering
  \begin{subfigure}[t]{0.24\linewidth}
    \includegraphics[width=0.99\textwidth]{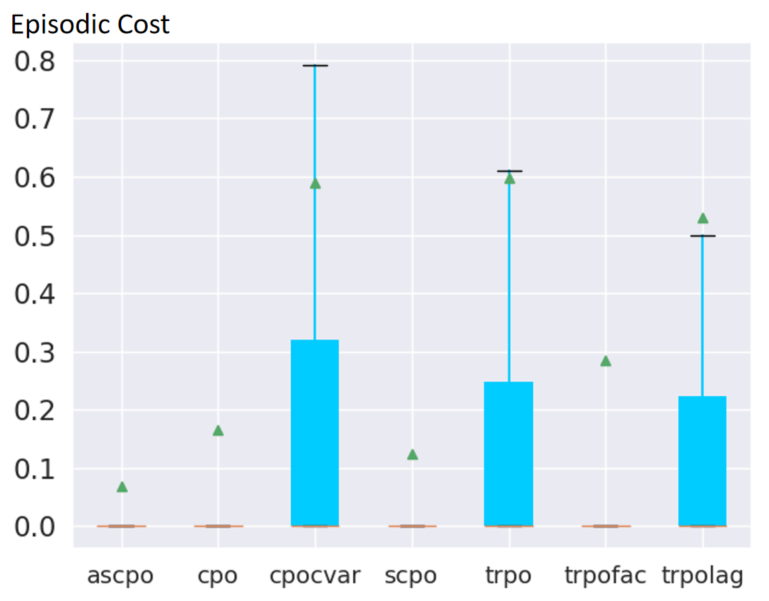}
    \centering
    \caption{Point-8-Hazards}
    \label{fig: Point-8-Hazards box}
  \end{subfigure}
  \begin{subfigure}[t]{0.24\linewidth}
      \includegraphics[width=0.99\textwidth]{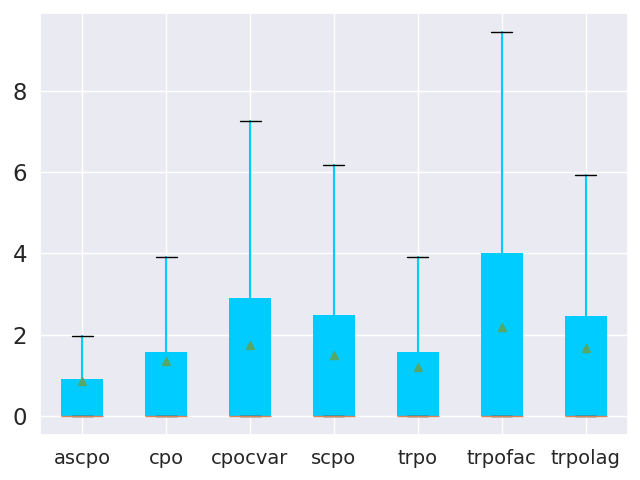}
    \centering
      \caption{Humanoid-8-Hazards}
      \label{fig: Humanoid-8-Hazards box}
  \end{subfigure}
  \begin{subfigure}[t]{0.24\linewidth}
      \includegraphics[width=0.99\textwidth]{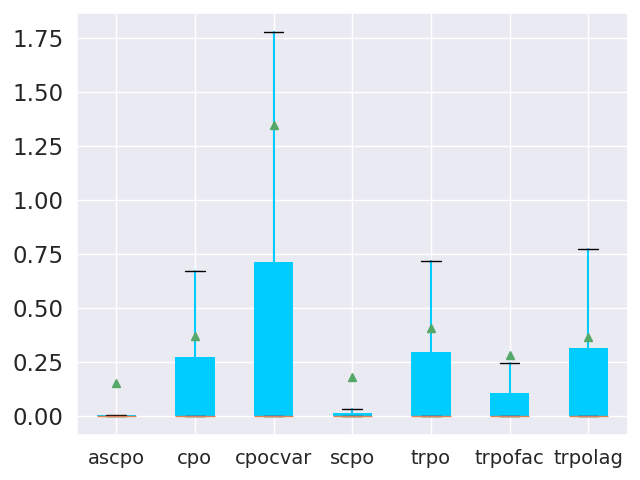}
    \centering
      \caption{Ant-8-Hazards}
      \label{fig: Ant-8-Hazards box}
  \end{subfigure}
  \begin{subfigure}[t]{0.24\linewidth}
      \includegraphics[width=0.99\textwidth]{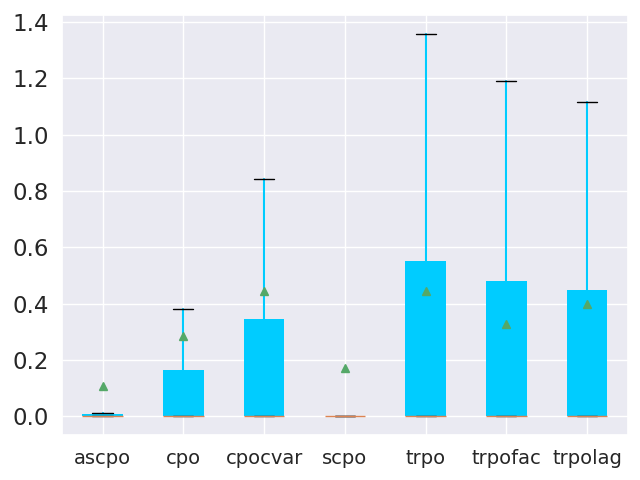}
    \centering
      \caption{Walker-8-Hazards}
      \label{fig: Walker-8-Hazards box}
  \end{subfigure}
  \caption{Box plots of episodic cost for testing representative algorithms on four representative test suites. Each plot features a box representing the interquartile range (IQR) of the data, accompanied by a line denoting the median. Whiskers extend from the box to illustrate the dataset's minimum and maximum values. In accordance with the scope of this study, no outliers were identified, which means that the entirety of the data was utilized for box plotting. Additionally, the mean values of episodic cost are indicated by green arrows.}
  \label{fig: box}
\end{figure}

\begin{figure}[t]
    \raisebox{-\height}{\includegraphics[width=\linewidth]{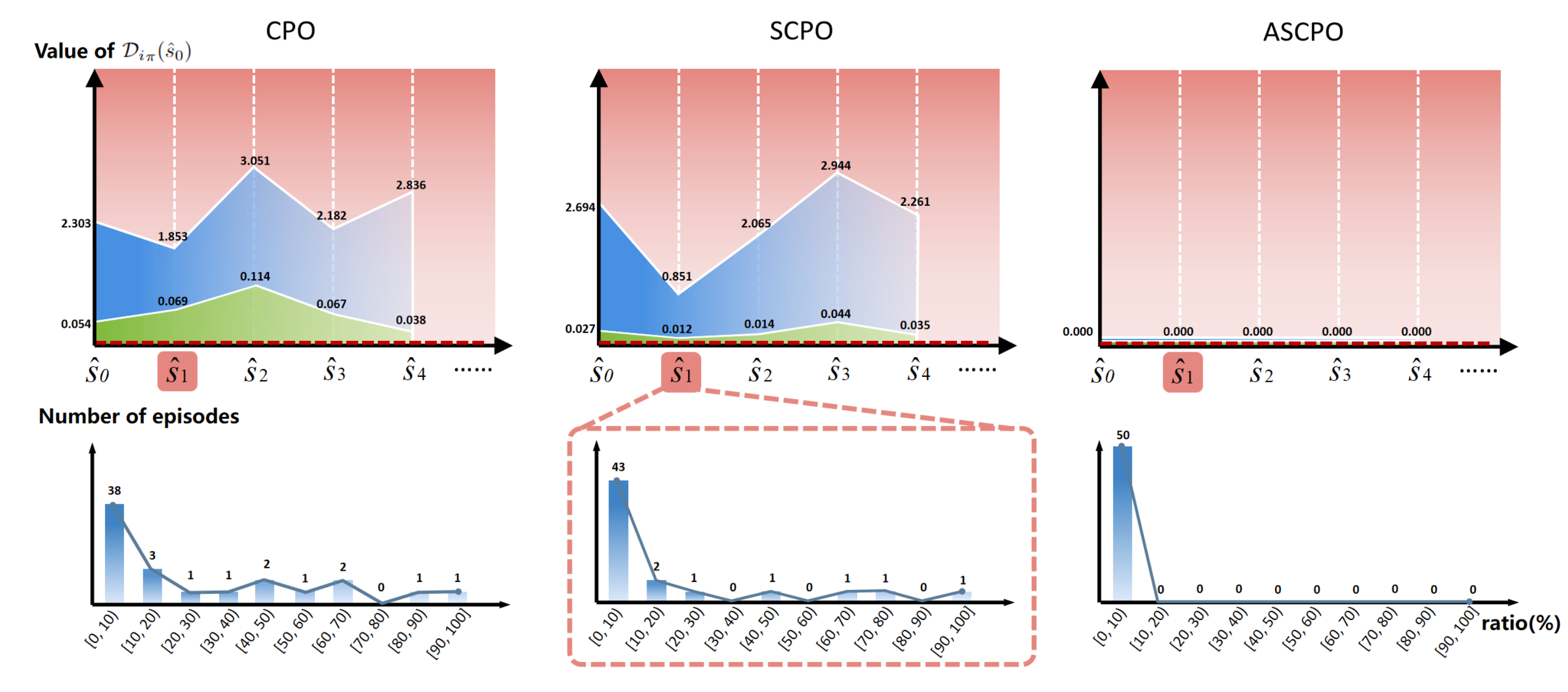}}
    \caption{Concretization of \Cref{fig: cpo vs scpo vs ascpo} with real data. We set threshold to 0 and take Swimmer-1-Hazard as a representative test suite. For each algorithm, we run 50 episodes under each of the five random seeds using the policy trained under 6e6 interactions to obtain the maximum (blue region) and expectation (green region) of episodic cost corresponding to each random seed (which represents a different initial state). Lower pictures show histograms of samples for a particular initial state.} 
    \label{fig:data proof of illustration}
\end{figure}

\paragraph{Absolute Maximum State-wise Cost}
In \Cref{fig: box}, we present a selection of four representative robots across varying dimensions to demonstrate ASCPO's proficiency in constraining episodic cost values below a specified threshold with exceptionally high probability. We include CPO, CPO-CVaR, SCPO, TRPO, TRPO-FAC and TRPO-Lagrangian as benchmark algorithms for comparison.
All algorithms undergo thorough training for 6 million steps, with data collection consisting of 30,000 steps under each of the three random seeds to ensure a robust data distribution. Our analysis reveals that ASCPO not only achieves the lowest mean episodic cost value but also effectively manages the maximum episodic cost, thereby demonstrating \textbf{ASCPO's success in controlling the entire distribution with high probability}. 
At the same time, in \Cref{fig:data proof of illustration}, we take Swimmer-1-Hazard as the representative test suite to concretize the illustration of \Cref{fig: cpo vs scpo vs ascpo} using real data. The results show that ASCPO can effectively constrain the maximum episodic cost within a safe threshold, demonstrating the superiority of our algorithm in controlling the entire distribution.
These results answer \textbf{Q2}.

\begin{wrapfigure}{r}{0.5\textwidth}
    \vspace{-20pt}
    \centering
    \begin{subfigure}[b]{0.24\textwidth}
        \begin{subfigure}[t]{1.00\textwidth}
        \raisebox{-\height}{\includegraphics[width=\textwidth]{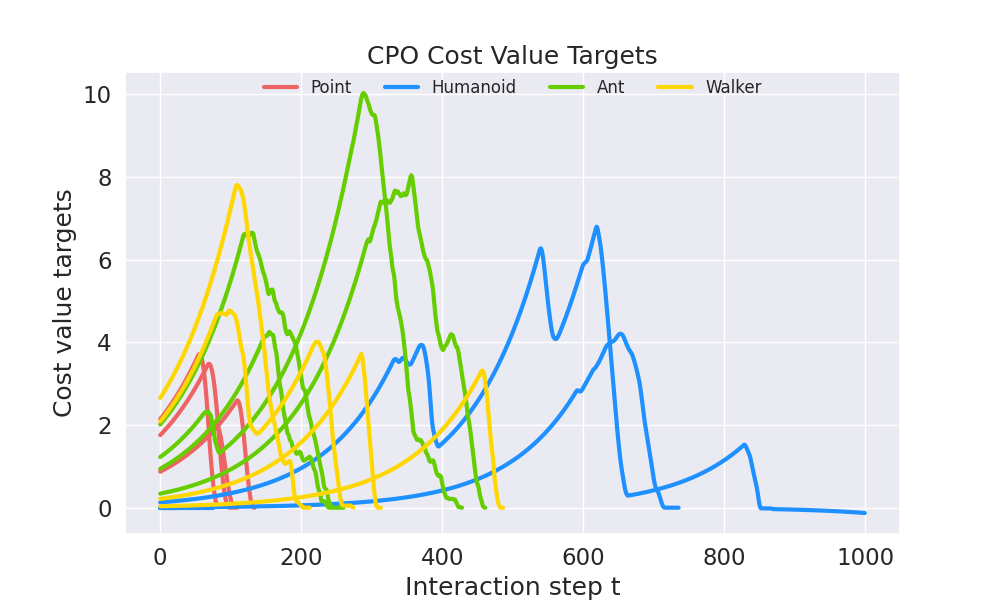}}
        \end{subfigure}
    \end{subfigure}
    \hfill
    \begin{subfigure}[b]{0.24\textwidth}
        \begin{subfigure}[t]{1.00\textwidth}
        \raisebox{-\height}{\includegraphics[width=\textwidth]{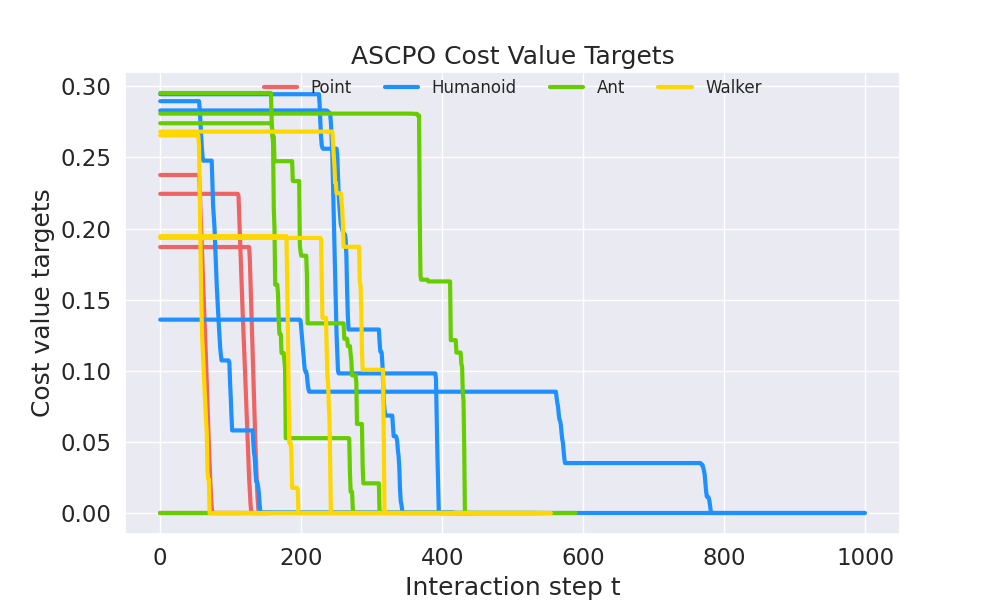}}
        \end{subfigure}
    \end{subfigure}
    \caption{Cost value function targets of four randomly sampled episodes of different tasks.} 
    \label{fig:cost value function target}
    \vspace{-10pt}
\end{wrapfigure}

\paragraph{Ablation on Monotonic-Descent Trick}
According to Zhao et al.~\citep{zhao2024statewise}, algorithms employing the MMDP framework feature cost target functions characterized by step functions, as illustrated in the right panel of \Cref{fig:cost value function target}. These functions delineate the maximum cost increment in any future state relative to the maximal state-wise cost encountered thus far. Consequently, strategies can be employed to enhance the neural network's fitting of these step functions. A pivotal characteristic of these functions is their monotonically decreasing nature, which informs the design of the loss \Cref{eq: monotonic descent}.
Moreover, as depicted in \Cref{fig:cost value function target}, the application of the monotonic-descent technique to non-MMDP scenarios (using CPO as an example) appears irrelevant and does not affect their performance. Subsequently, in \Cref{fig: ablation of monotonic descent}, we demonstrate the impact of employing the monotonic-descent strategy across four test suites. The results demonstrate that this approach accelerates the convergence of cost values towards near-zero values and effectively lowers the cost rate to a desirable level, which answer \textbf{Q3}.

\begin{figure}[t]
  \centering
  \begin{subfigure}[t]{1.0\linewidth}
      \centering
      \begin{subfigure}[t]{0.24\linewidth}
          \begin{subfigure}[t]{1.0\linewidth}
              \includegraphics[width=0.99\textwidth]{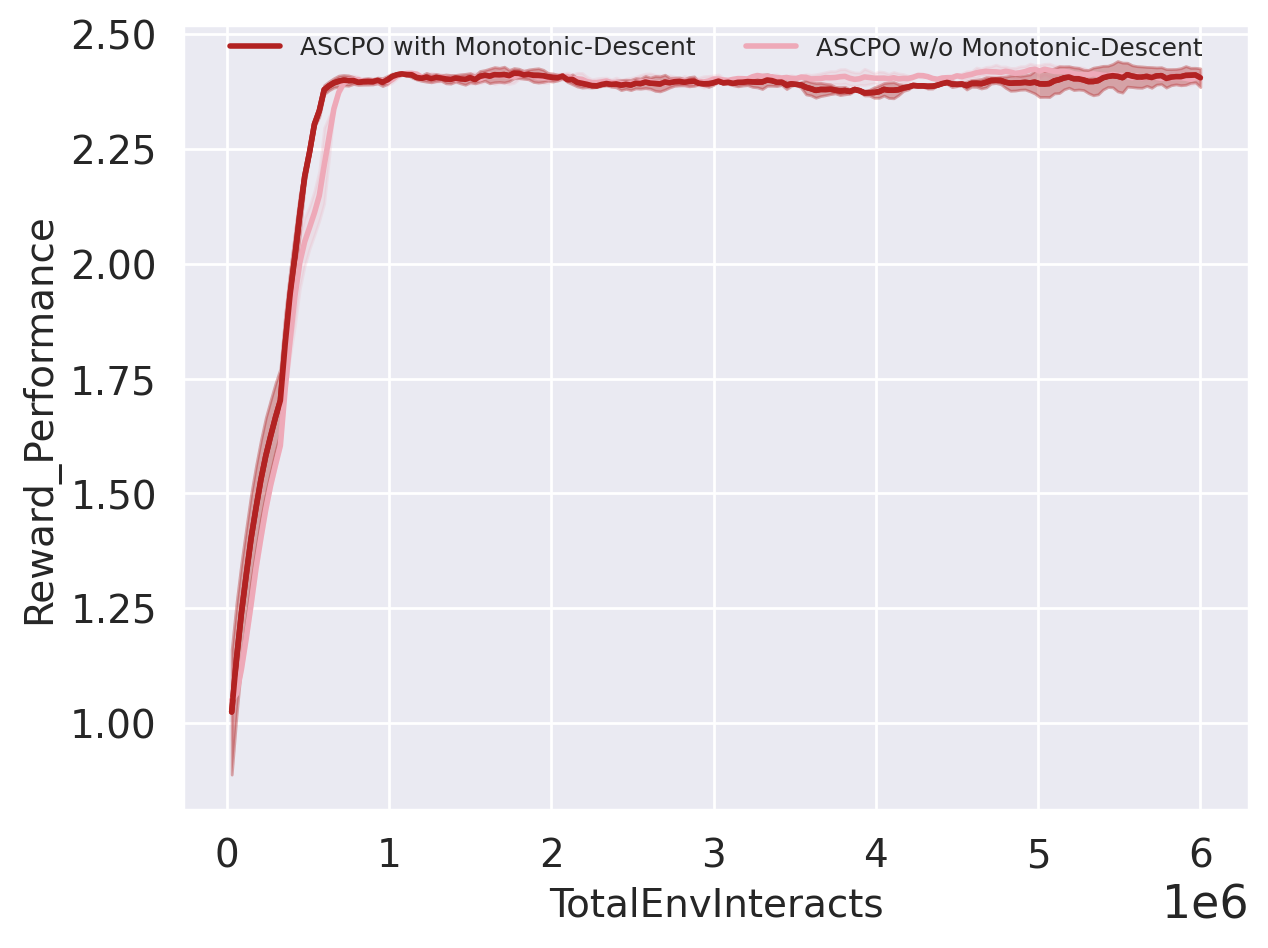}
          \end{subfigure}
          \hfill
          \begin{subfigure}[t]{1.0\linewidth}
              \includegraphics[width=0.99\textwidth]{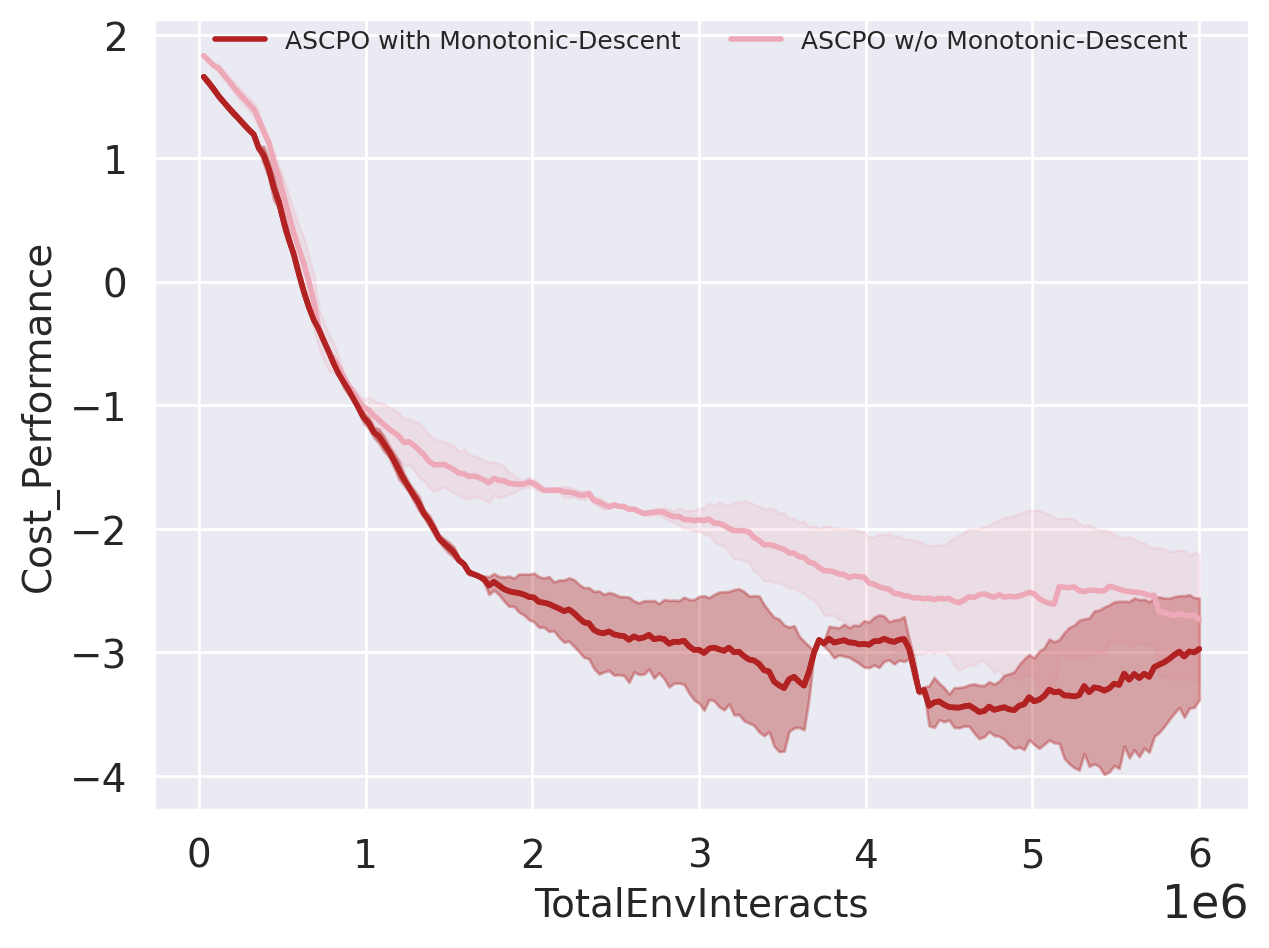}
          \end{subfigure}
          \hfill
          \begin{subfigure}[t]{1.0\linewidth}
              \includegraphics[width=0.99\textwidth]{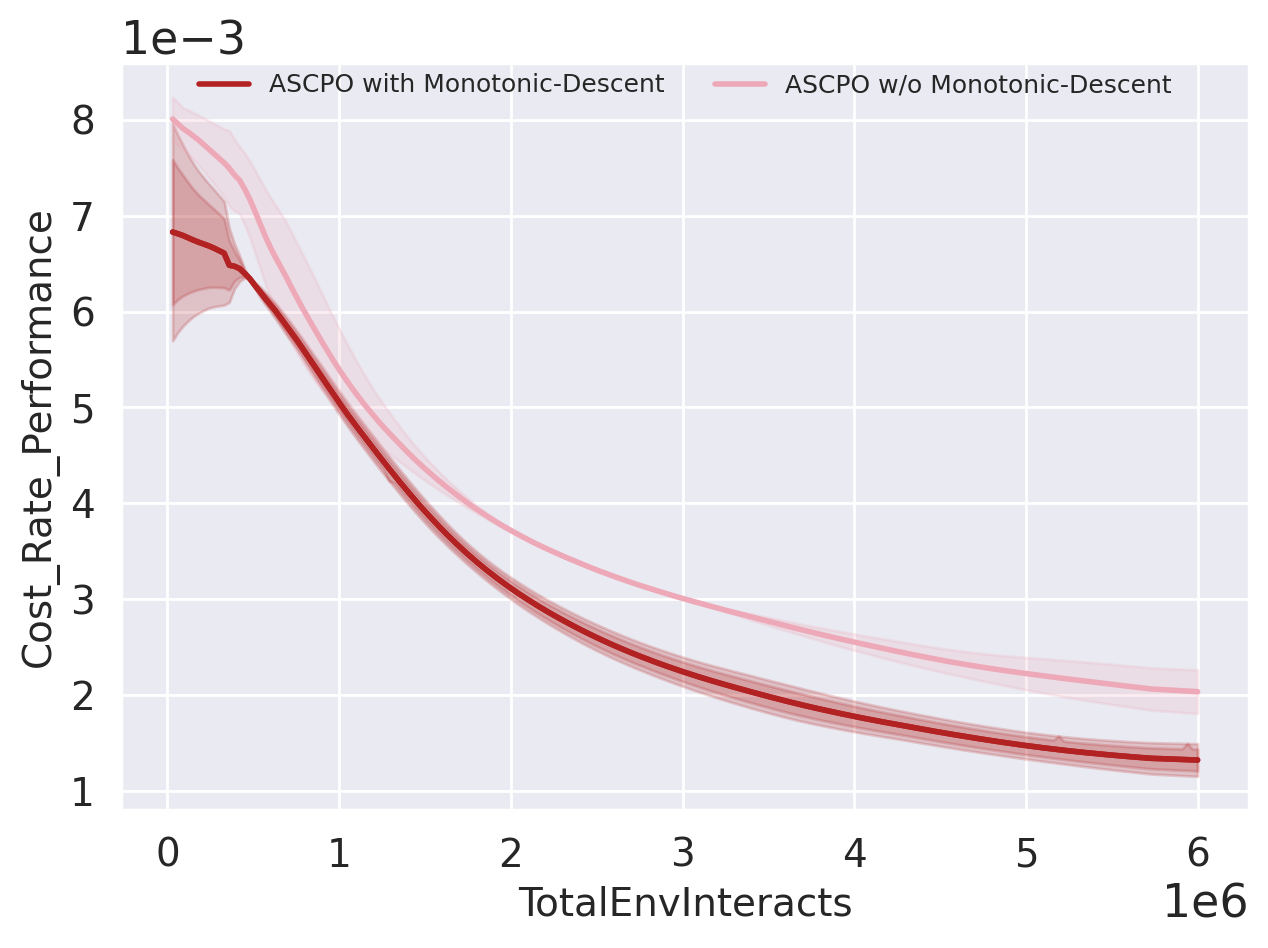}
          \end{subfigure}
      \caption{Point-8-Hazards}
      \end{subfigure}
      \begin{subfigure}[t]{0.24\linewidth}
          \begin{subfigure}[t]{1.0\linewidth}
              \includegraphics[width=0.99\textwidth]{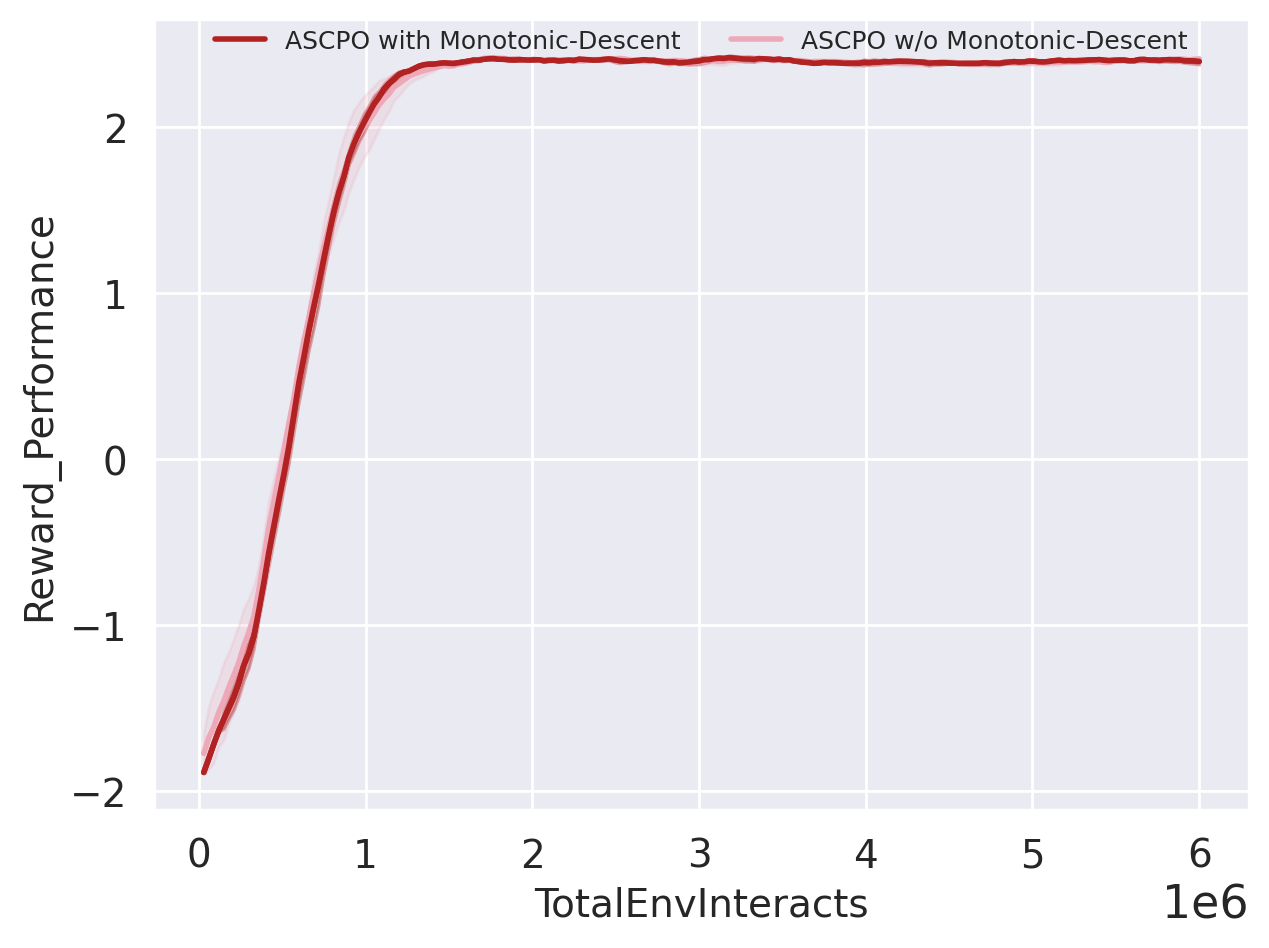}
          \end{subfigure}
          \hfill
          \begin{subfigure}[t]{1.0\linewidth}
              \includegraphics[width=0.99\textwidth]{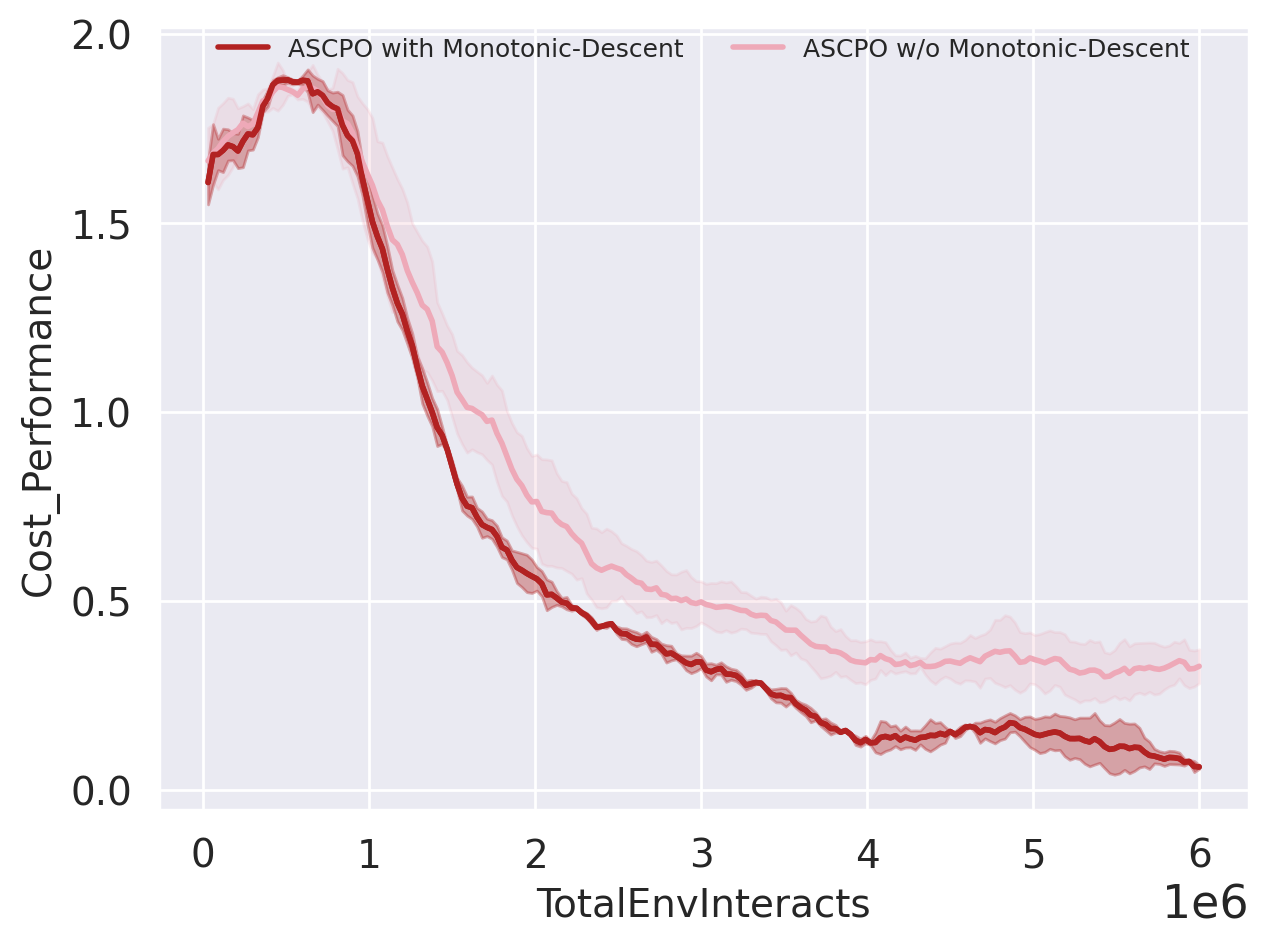}
          \end{subfigure}
          \hfill
          \begin{subfigure}[t]{1.0\linewidth}
              \includegraphics[width=0.99\textwidth]{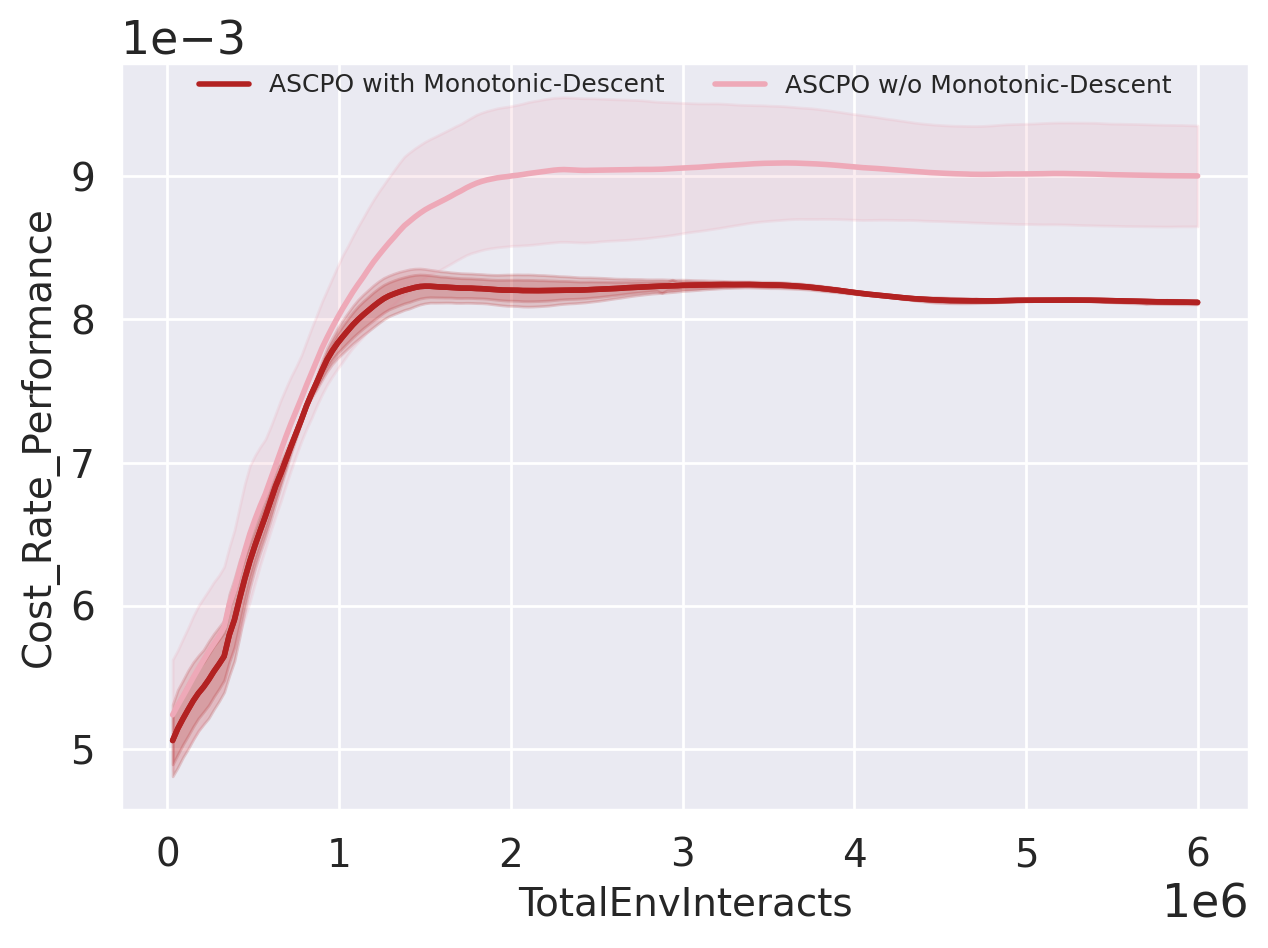}
          \end{subfigure}
      \caption{Humanoid-8-Hazards}
      \end{subfigure}
      \begin{subfigure}[t]{0.24\linewidth}
          \begin{subfigure}[t]{1.0\linewidth}
              \includegraphics[width=0.99\textwidth]{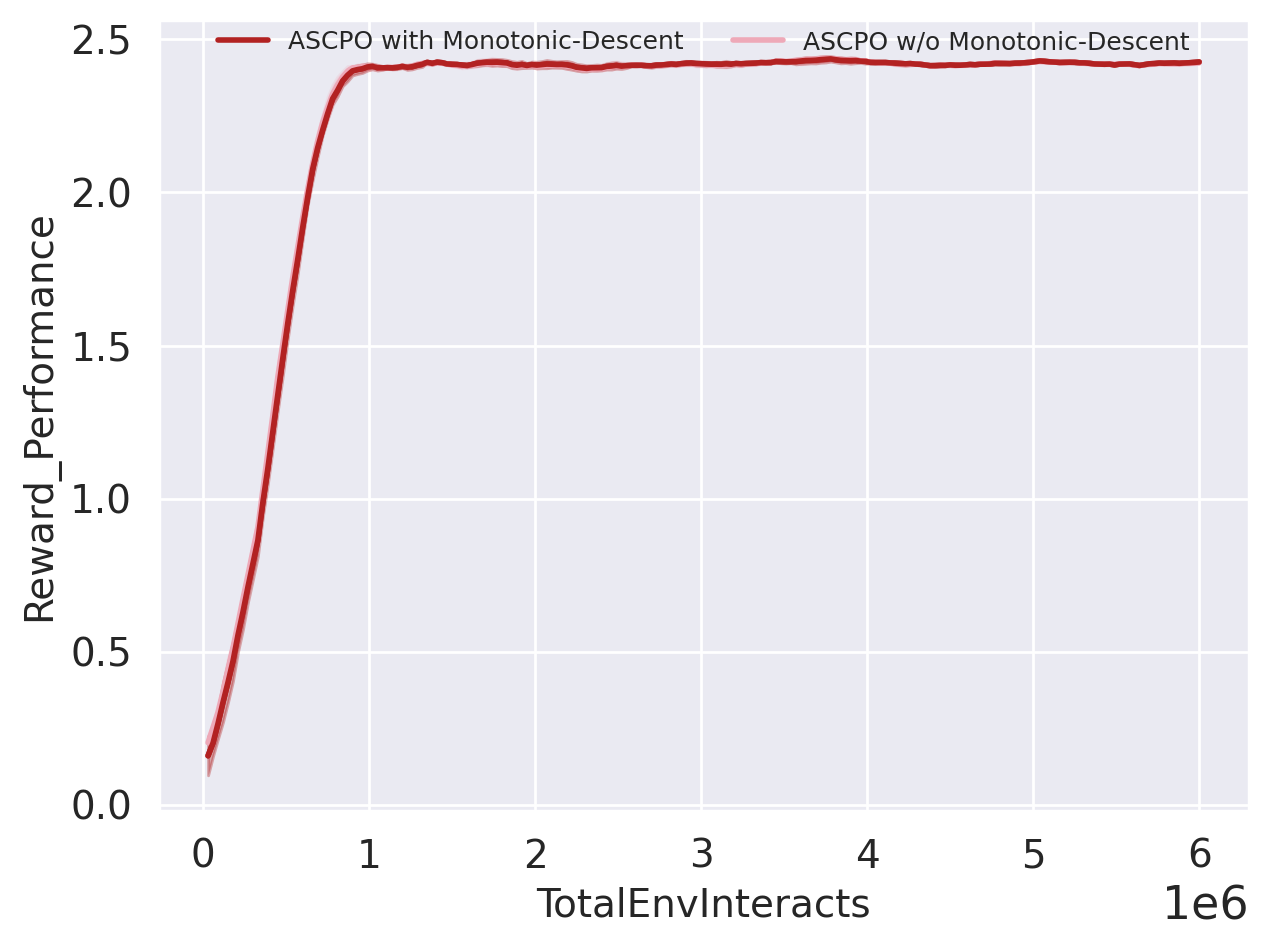}
          \end{subfigure}
          \hfill
          \begin{subfigure}[t]{1.0\linewidth}
              \includegraphics[width=0.99\textwidth]{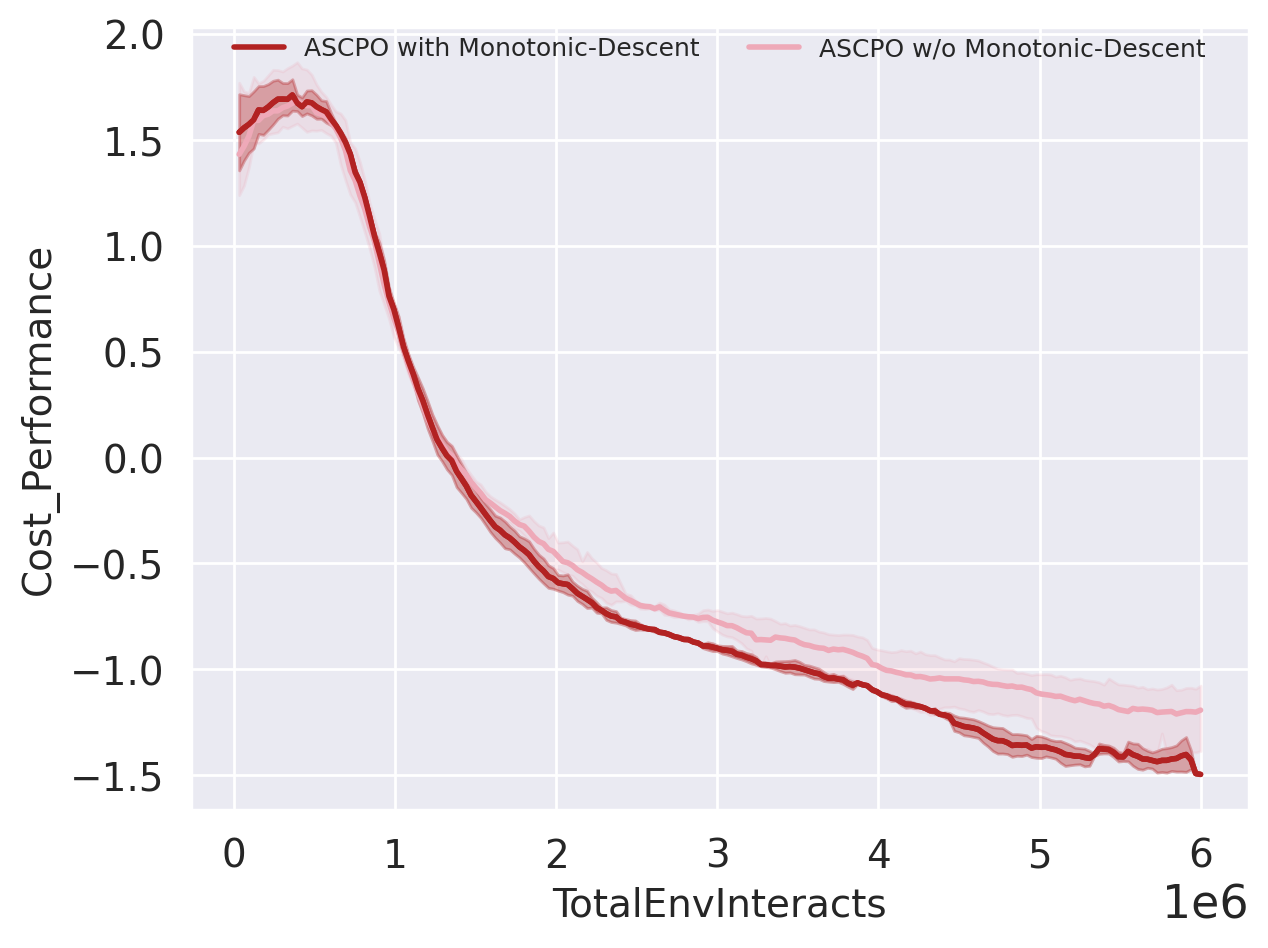}
          \end{subfigure}
          \hfill
          \begin{subfigure}[t]{1.0\linewidth}
              \includegraphics[width=0.99\textwidth]{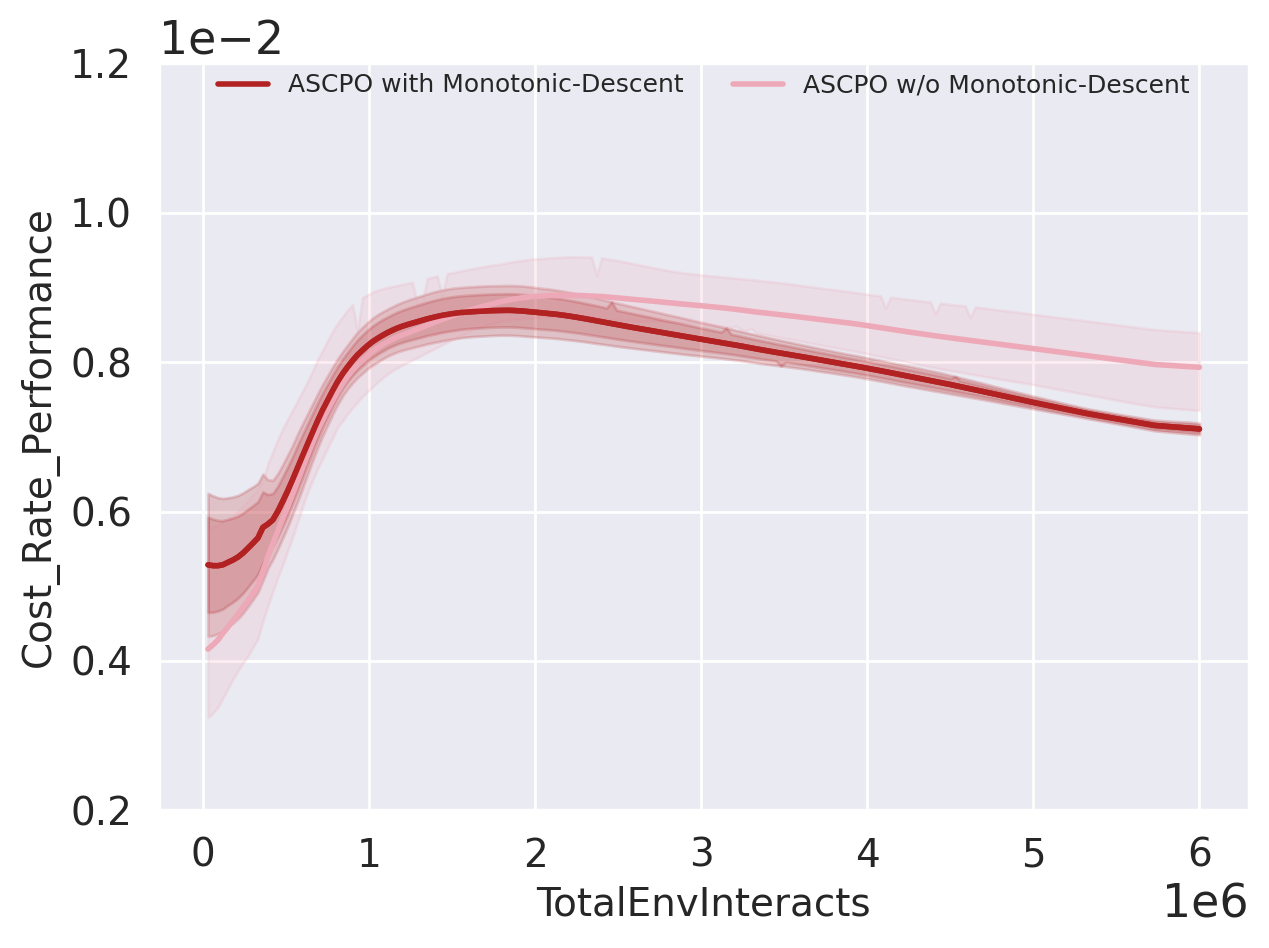}
          \end{subfigure}
      \caption{Ant-8-Hazards}
      \end{subfigure}
      \begin{subfigure}[t]{0.24\linewidth}
          \begin{subfigure}[t]{1.0\linewidth}
              \includegraphics[width=0.99\textwidth]{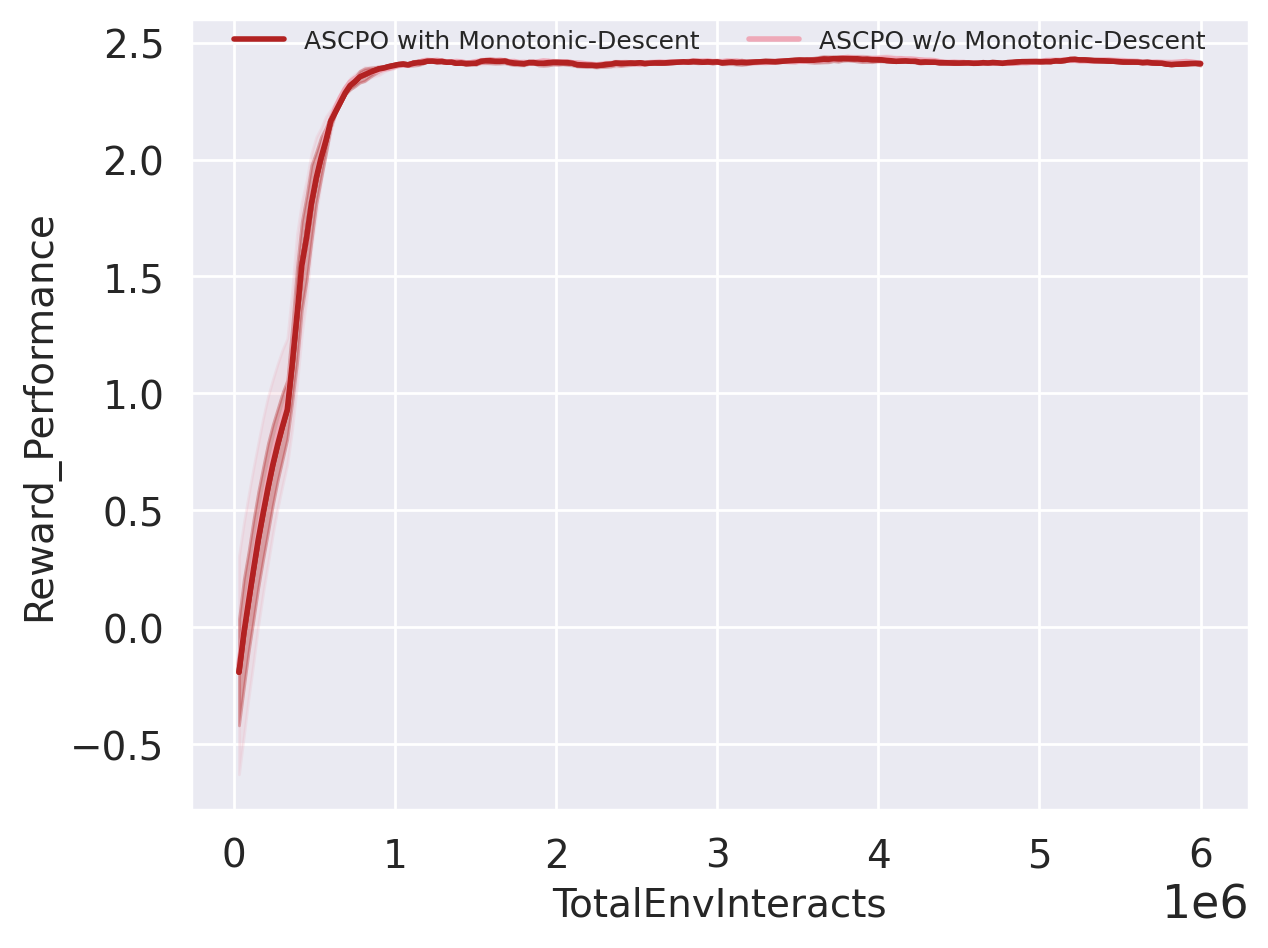}
          \end{subfigure}
          \hfill
          \begin{subfigure}[t]{1.0\linewidth}
              \includegraphics[width=0.99\textwidth]{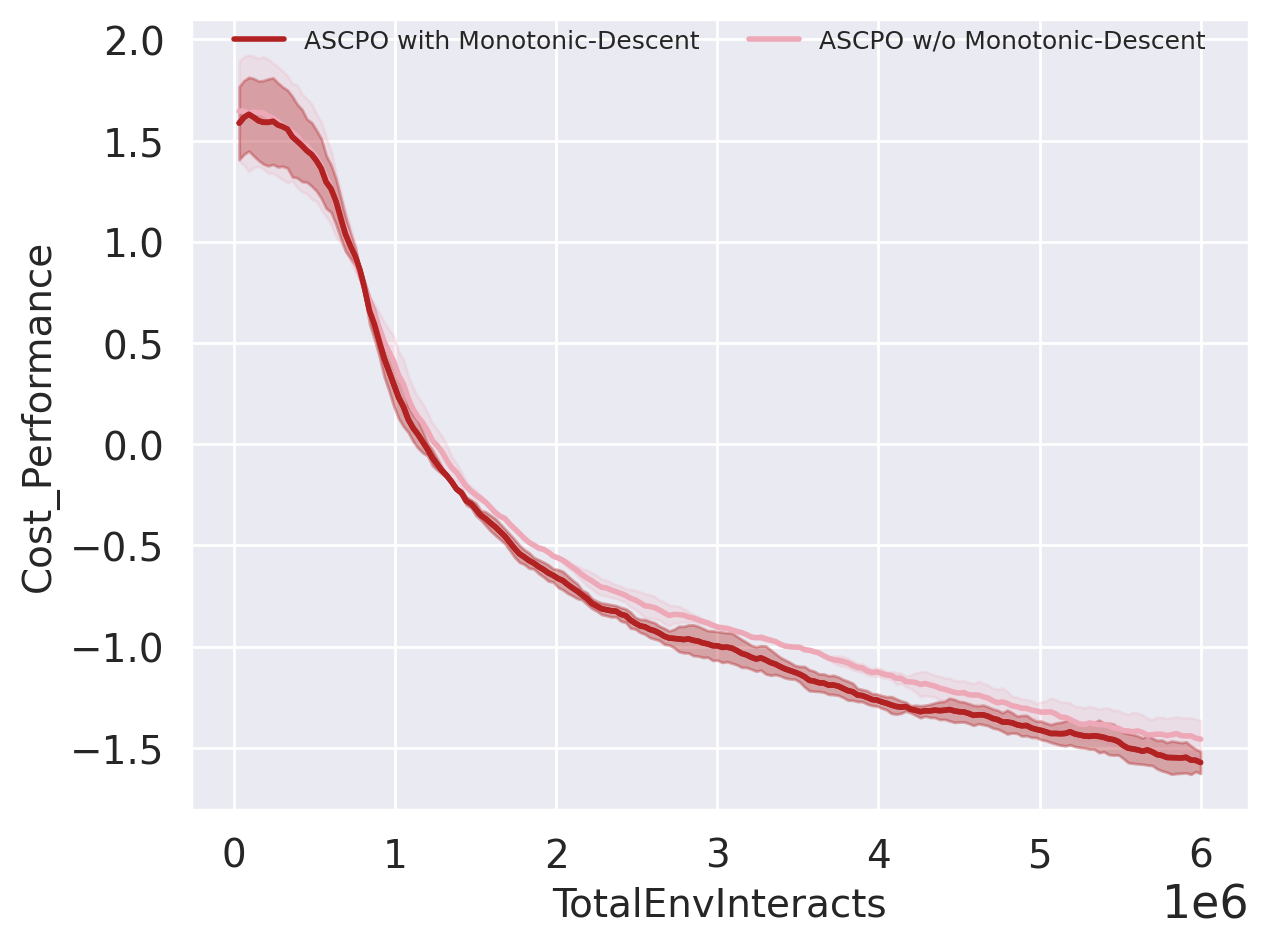}
          \end{subfigure}
          \hfill
          \begin{subfigure}[t]{1.0\linewidth}
              \includegraphics[width=0.99\textwidth]{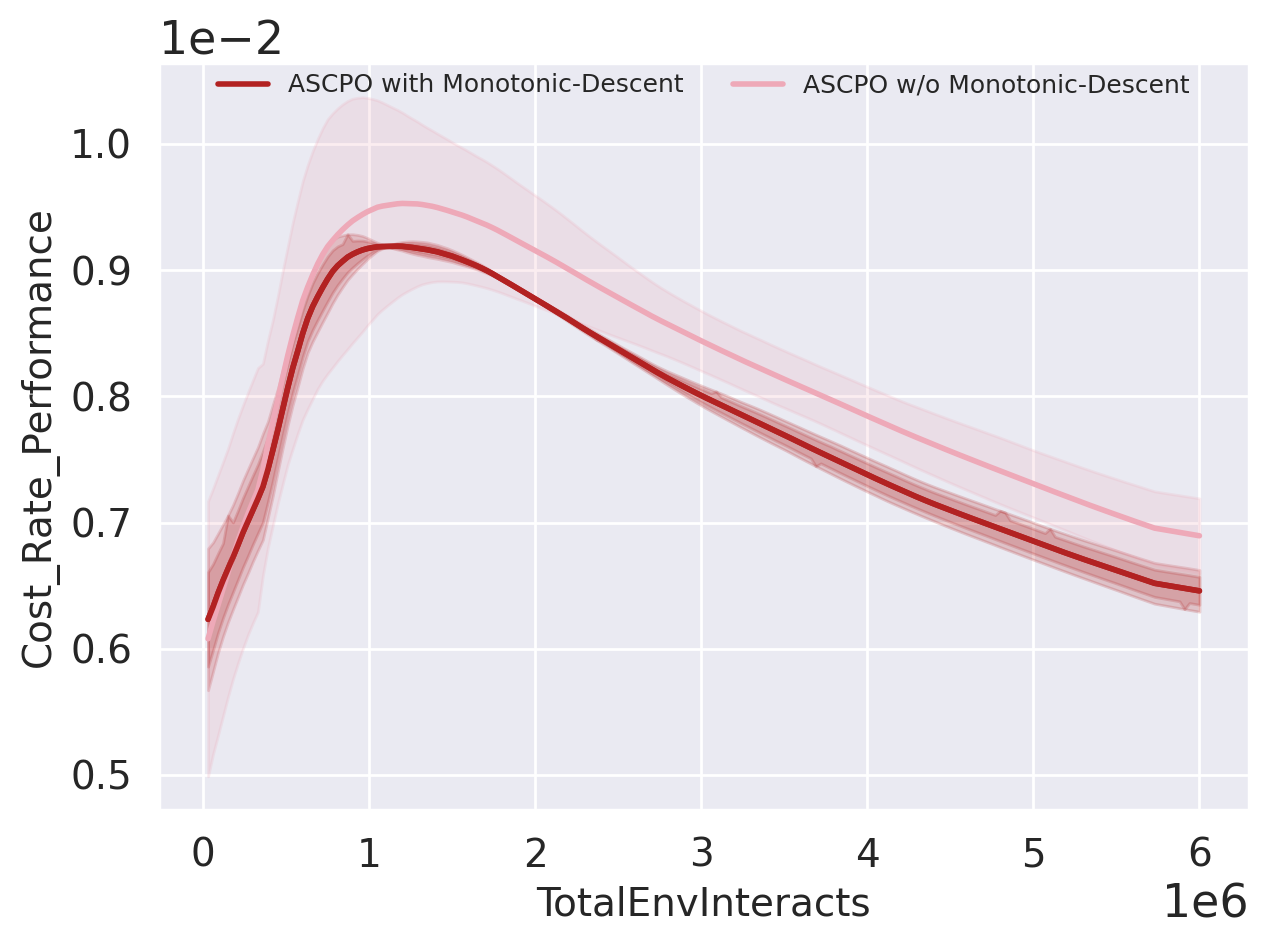}
          \end{subfigure}
      \caption{Walker-8-Hazards}
      \end{subfigure}
  \end{subfigure}
  \caption{ASCPO Monotonic-Descent ablation study with four representative test suites}
  \label{fig: ablation of monotonic descent}
\end{figure}

\begin{wrapfigure}{r}{0.5\textwidth}
    \vspace{-10pt}
    \centering
    \begin{subfigure}[b]{0.49\textwidth}
        \begin{subfigure}[t]{1.00\textwidth}
        \raisebox{-\height}{\includegraphics[width=\textwidth]{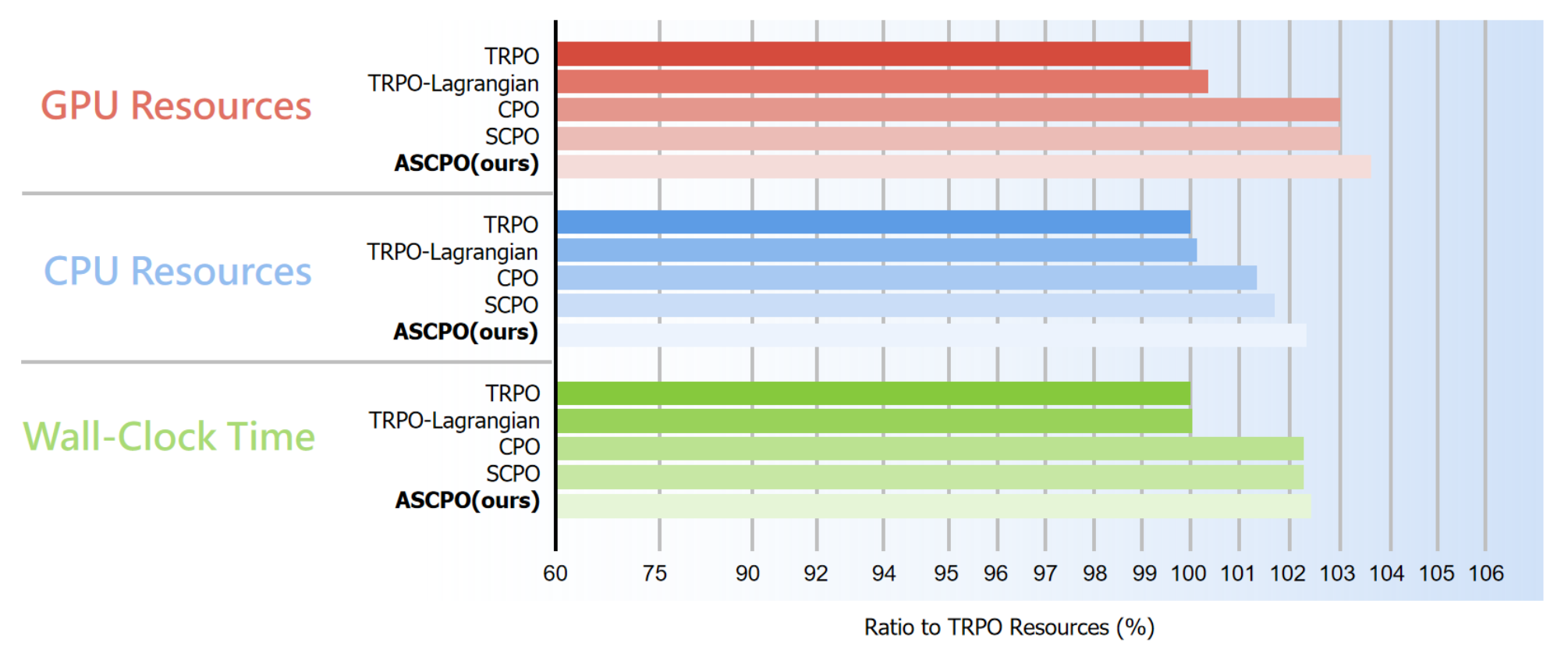}}
        \end{subfigure}
    \end{subfigure}
    \caption{Resource usage of ASCPO compared to other algorithms under the Goal-8-Hazard task, where TRPO is set as the baseline} 
    \label{fig: resources}
    \vspace{-10pt}
\end{wrapfigure}

\paragraph{Resources Usage} 
We conducted comprehensive tests comparing GPU and CPU memory usage, along with wall-clock time, across several algorithms, namely SCPO, CPO, TRPO, TRPO-Lagrangian, and ASCPO. These tests utilized identical system resources in the Goal-8-Hazard task. As illustrated in Figure \ref{fig: resources}, the results indicate that ASCPO marginally increases GPU and CPU resource consumption compared to CPO and SCPO, while exhibiting nearly identical wall-clock time performance. However, a closer examination of the horizontal coordinate magnitude reveals a significant performance enhancement achieved by ASCPO, requiring less than 1\% increase in resources for each aspect. This underscores the effectiveness of our algorithm in delivering exceptional results without imposing substantial demands on system resources and runtime. Furthermore, this observation addresses the pertinent question \textbf{Q4} regarding the efficacy of our approach.

\begin{figure}[t]
  \centering
  \begin{subfigure}[t]{1.0\linewidth}
      \centering
      \begin{subfigure}[t]{0.24\linewidth}
          \begin{subfigure}[t]{1.0\linewidth}
              \includegraphics[width=0.99\textwidth]{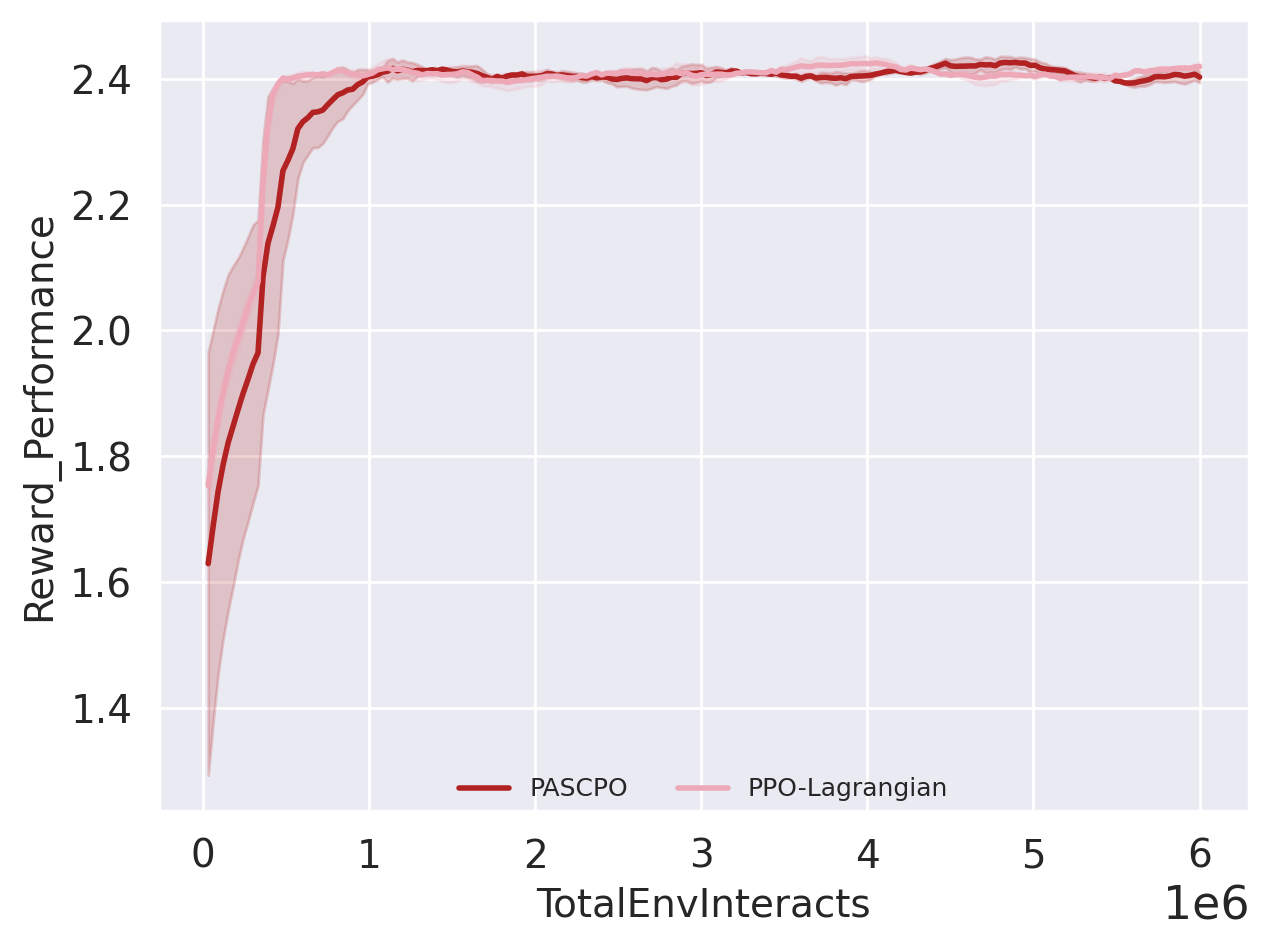}
          \end{subfigure}
          \hfill
          \begin{subfigure}[t]{1.0\linewidth}
              \includegraphics[width=0.99\textwidth]{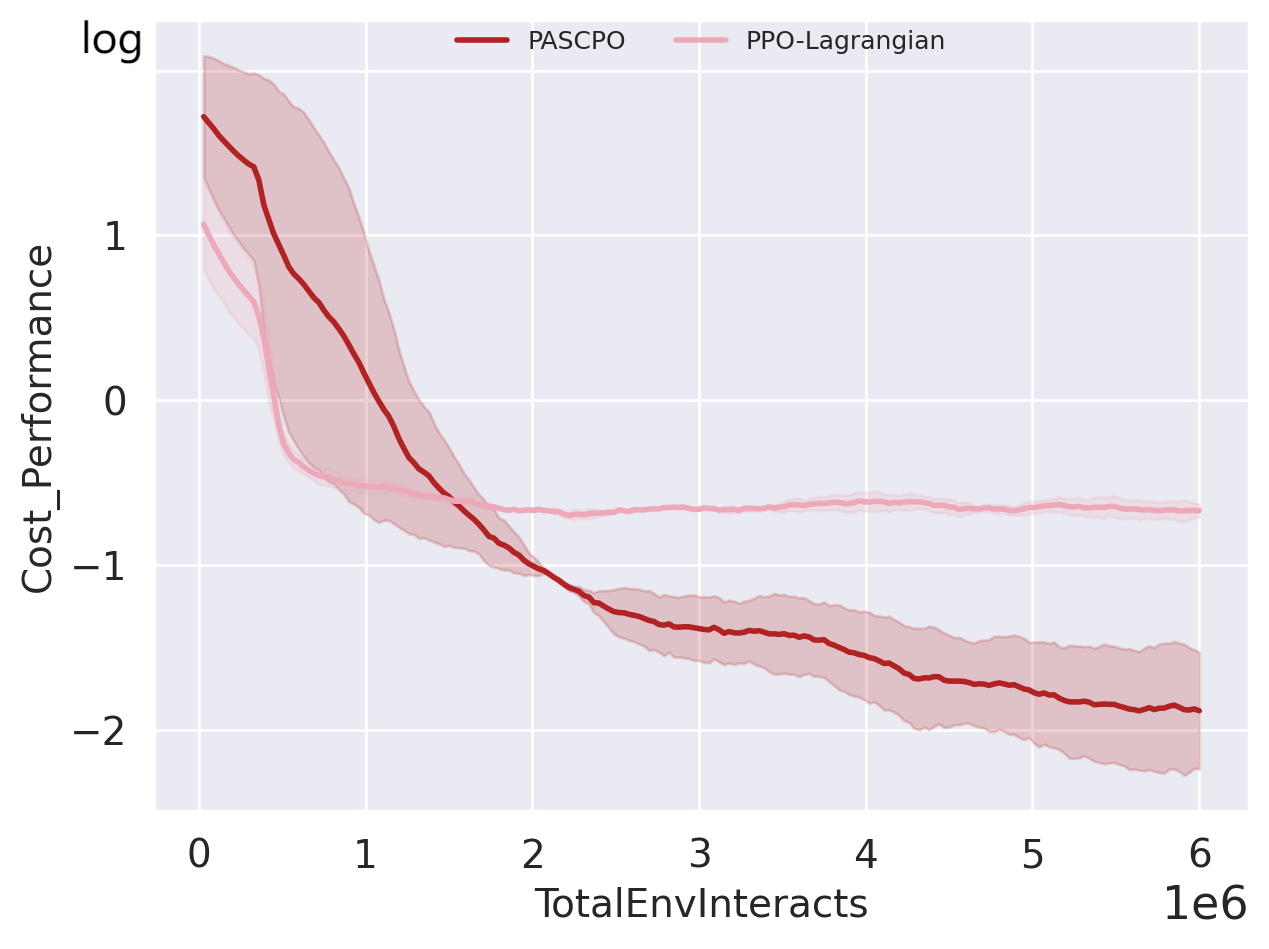}
          \end{subfigure}
          \hfill
          \begin{subfigure}[t]{1.0\linewidth}
              \includegraphics[width=0.99\textwidth]{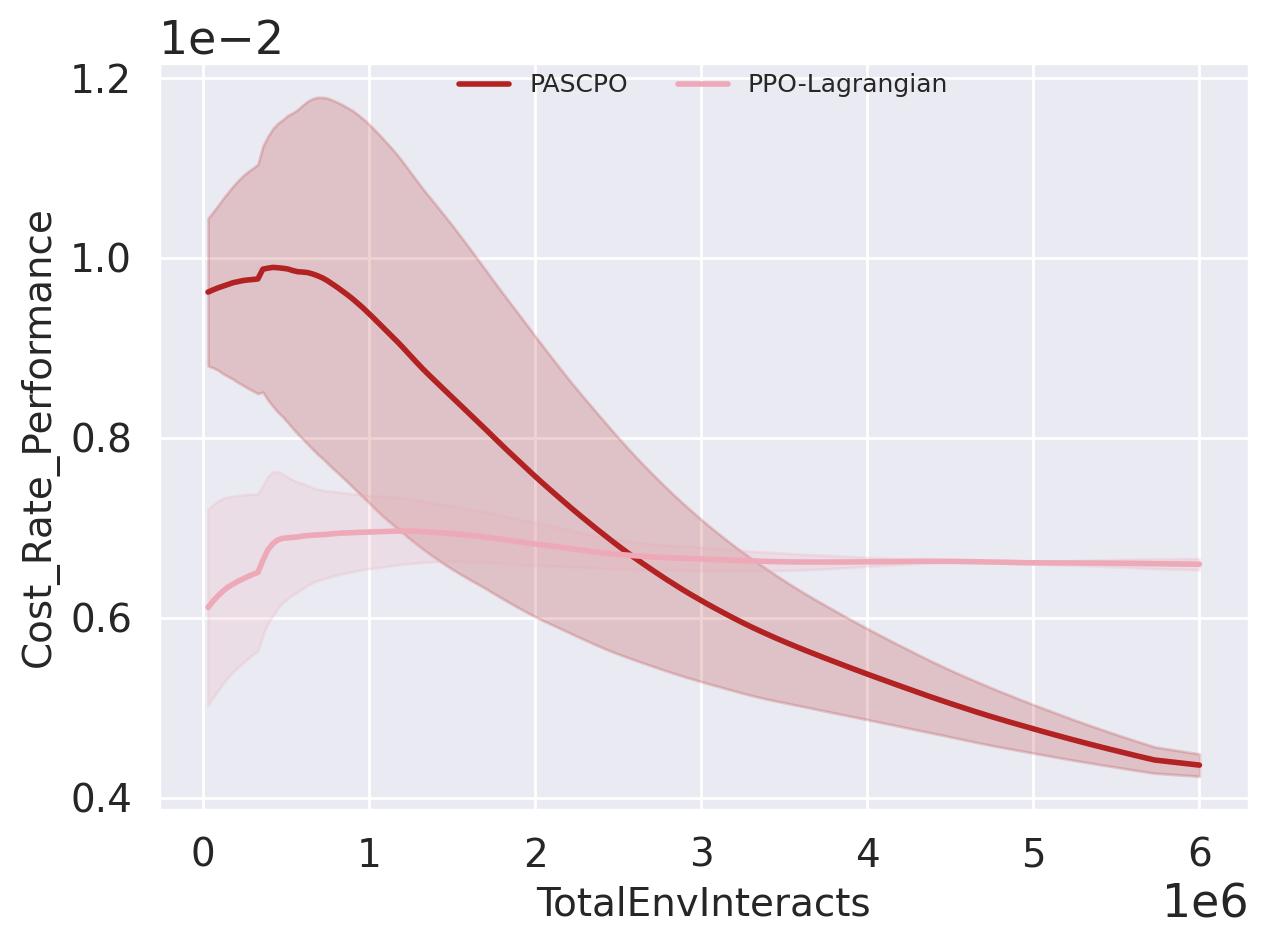}
          \end{subfigure}
          \hfill
          \begin{subfigure}[t]{1.0\linewidth}
              \includegraphics[width=0.99\textwidth]{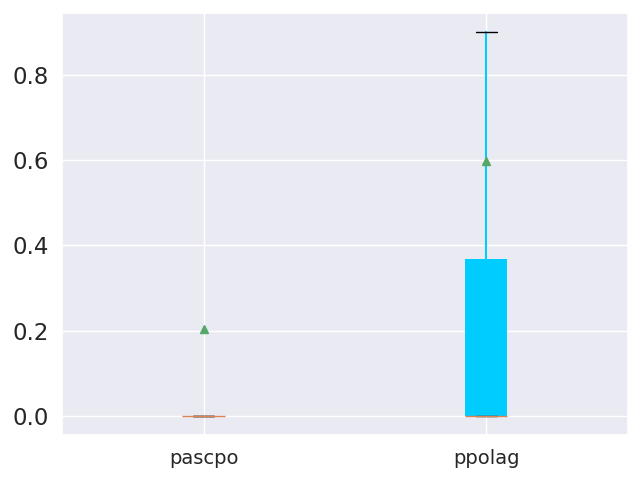}
          \end{subfigure}
      \caption{Point-8-Hazards}
      \end{subfigure}
      \begin{subfigure}[t]{0.24\linewidth}
          \begin{subfigure}[t]{1.0\linewidth}
              \includegraphics[width=0.99\textwidth]{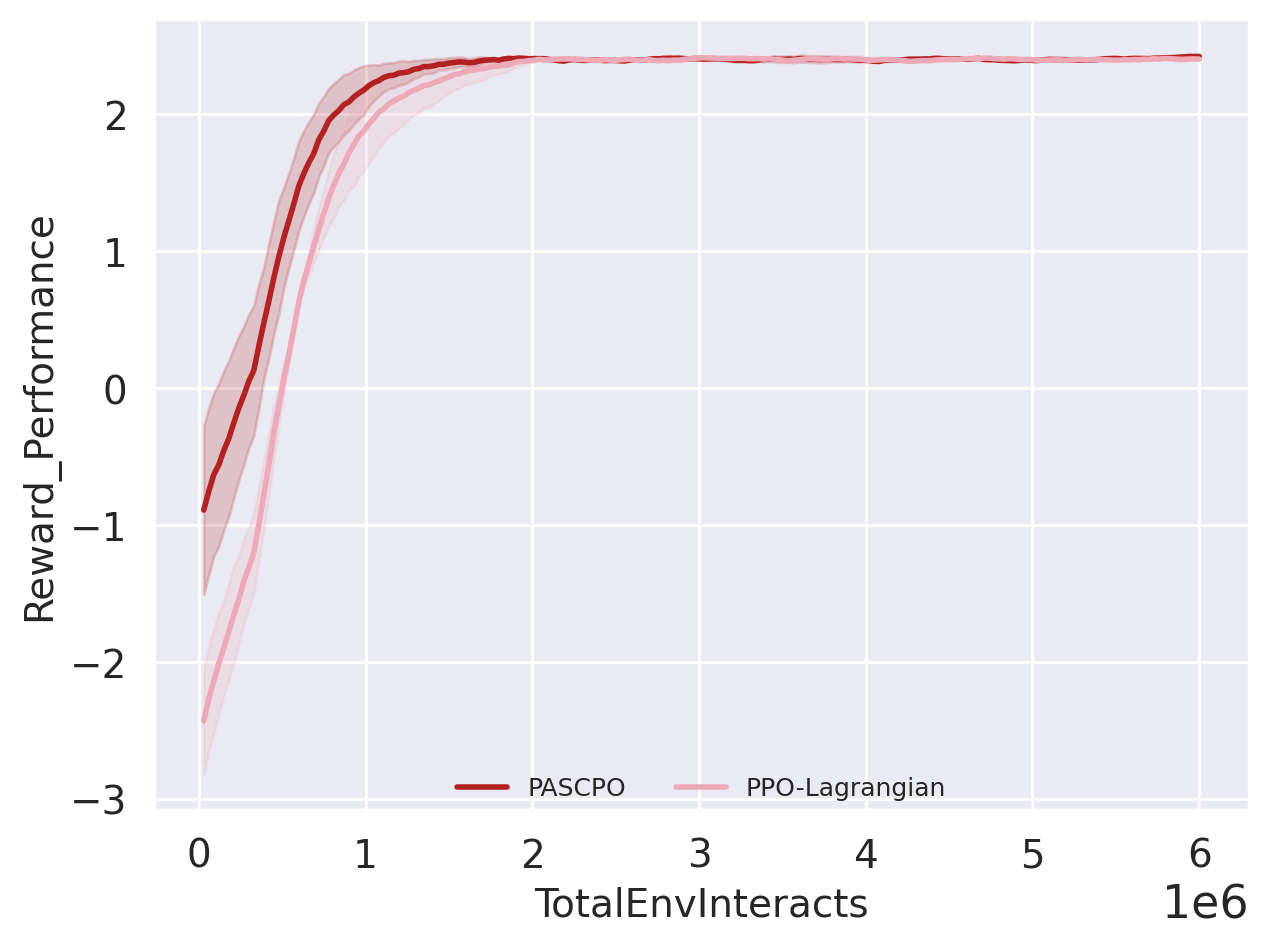}
          \end{subfigure}
          \hfill
          \begin{subfigure}[t]{1.0\linewidth}
              \includegraphics[width=0.99\textwidth]{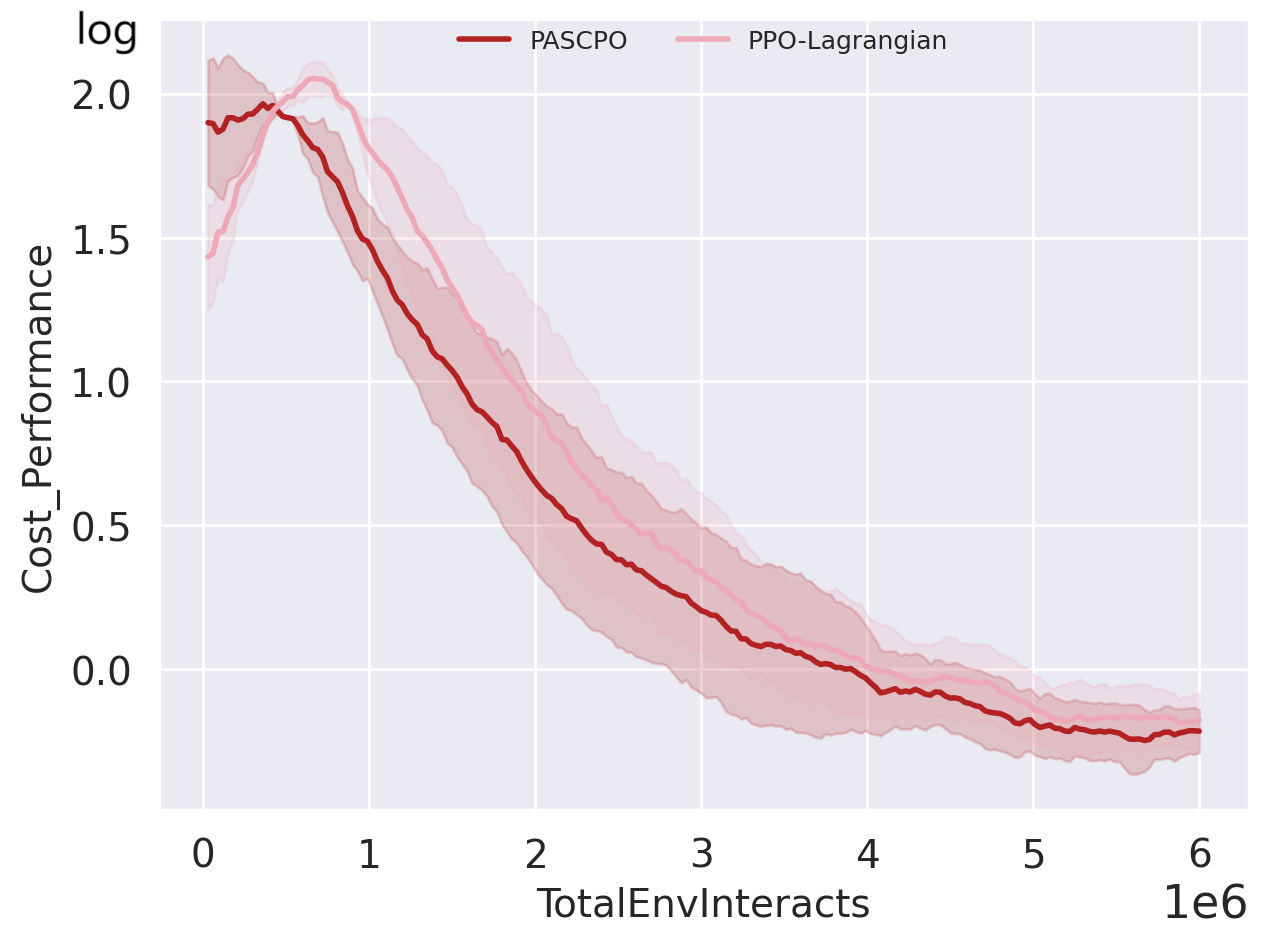}
          \end{subfigure}
          \hfill
          \begin{subfigure}[t]{1.0\linewidth}
              \includegraphics[width=0.99\textwidth]{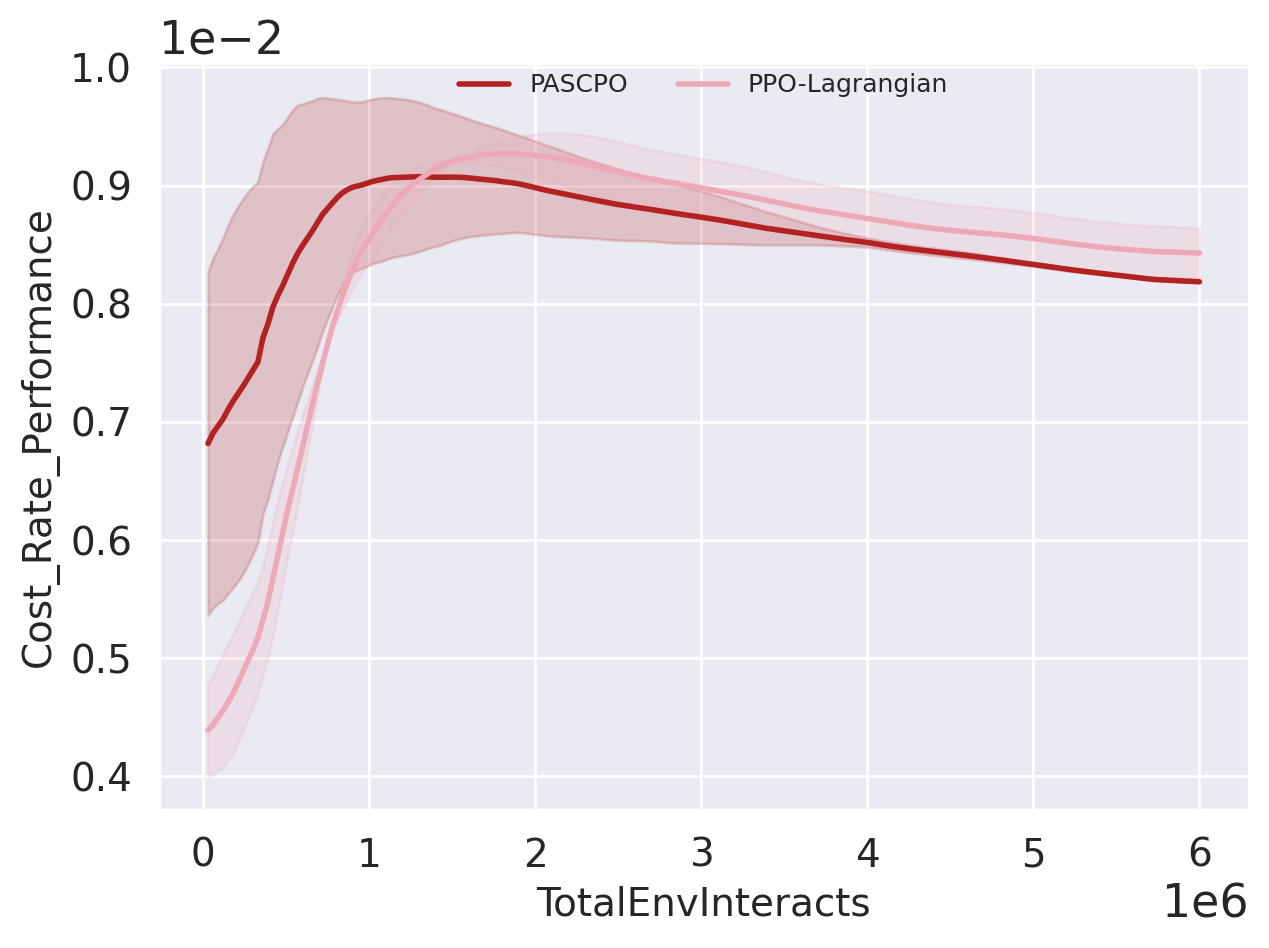}
          \end{subfigure}
          \hfill
          \begin{subfigure}[t]{1.0\linewidth}
              \includegraphics[width=0.99\textwidth]{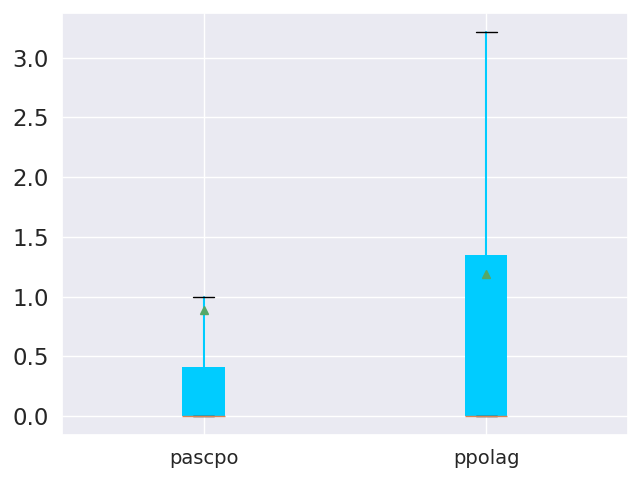}
          \end{subfigure}
      \caption{Humanoid-8-Hazards}
      \end{subfigure}
      \begin{subfigure}[t]{0.24\linewidth}
          \begin{subfigure}[t]{1.0\linewidth}
              \includegraphics[width=0.99\textwidth]{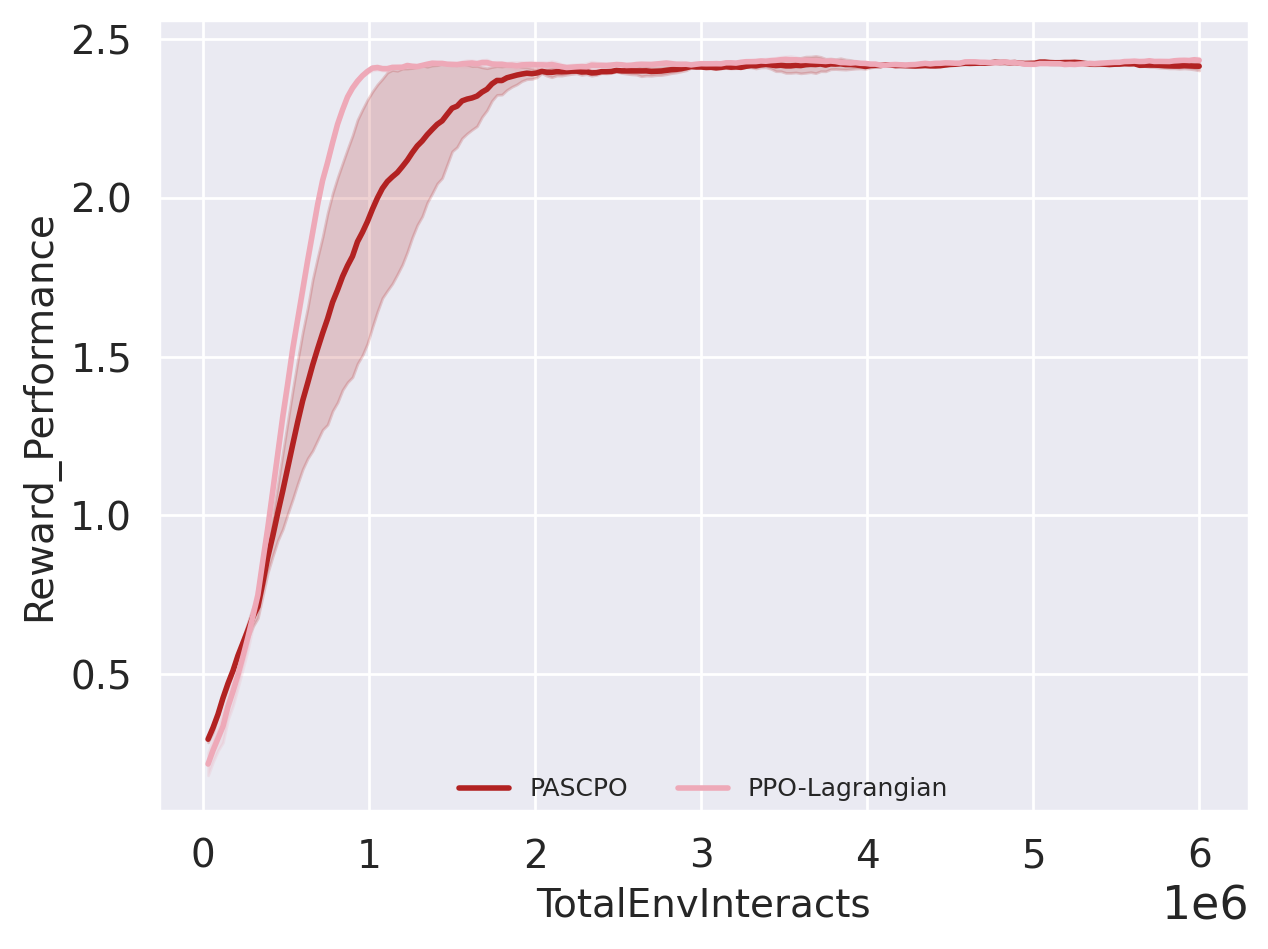}
          \end{subfigure}
          \hfill
          \begin{subfigure}[t]{1.0\linewidth}
              \includegraphics[width=0.99\textwidth]{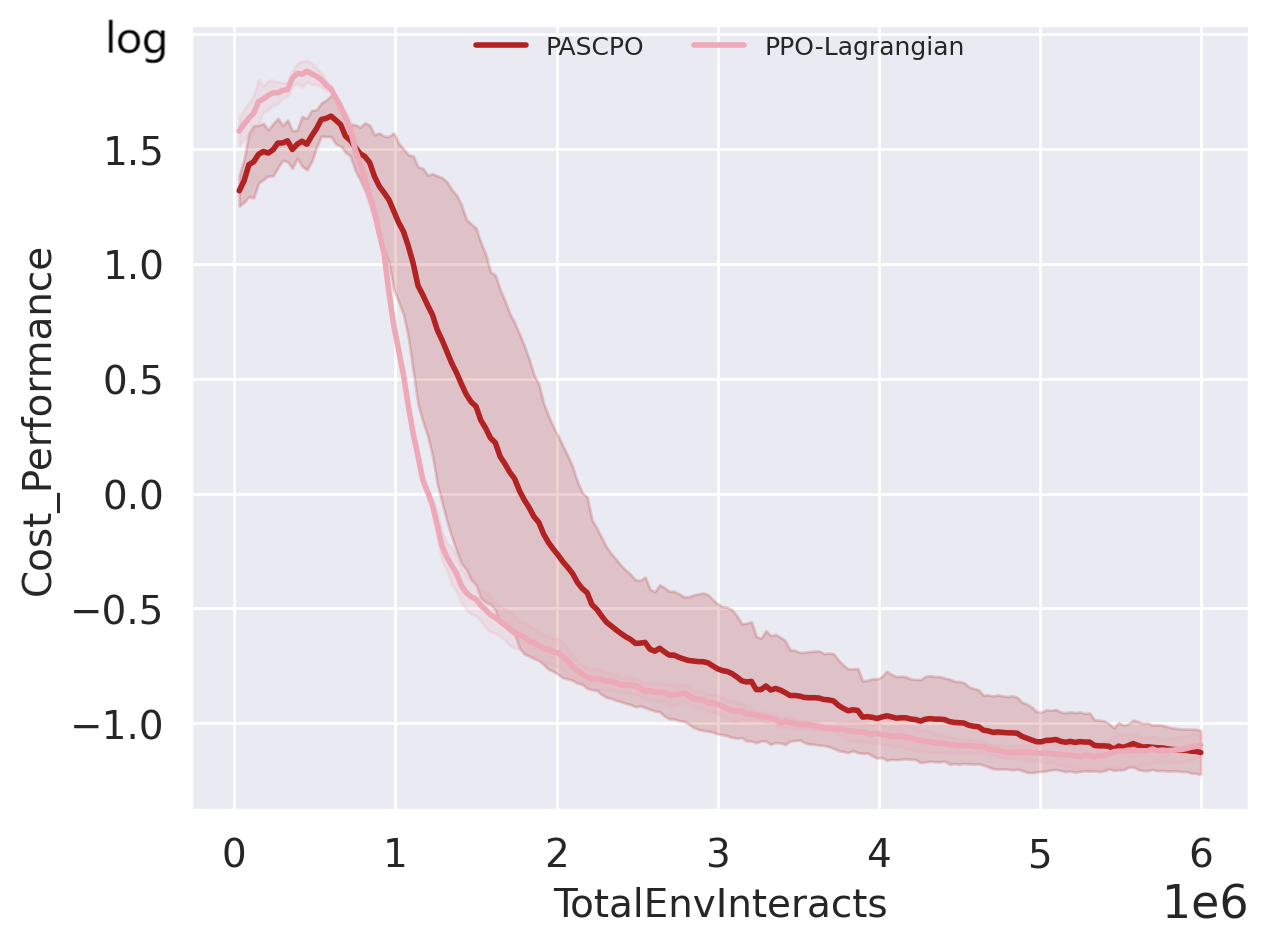}
          \end{subfigure}
          \hfill
          \begin{subfigure}[t]{1.0\linewidth}
              \includegraphics[width=0.99\textwidth]{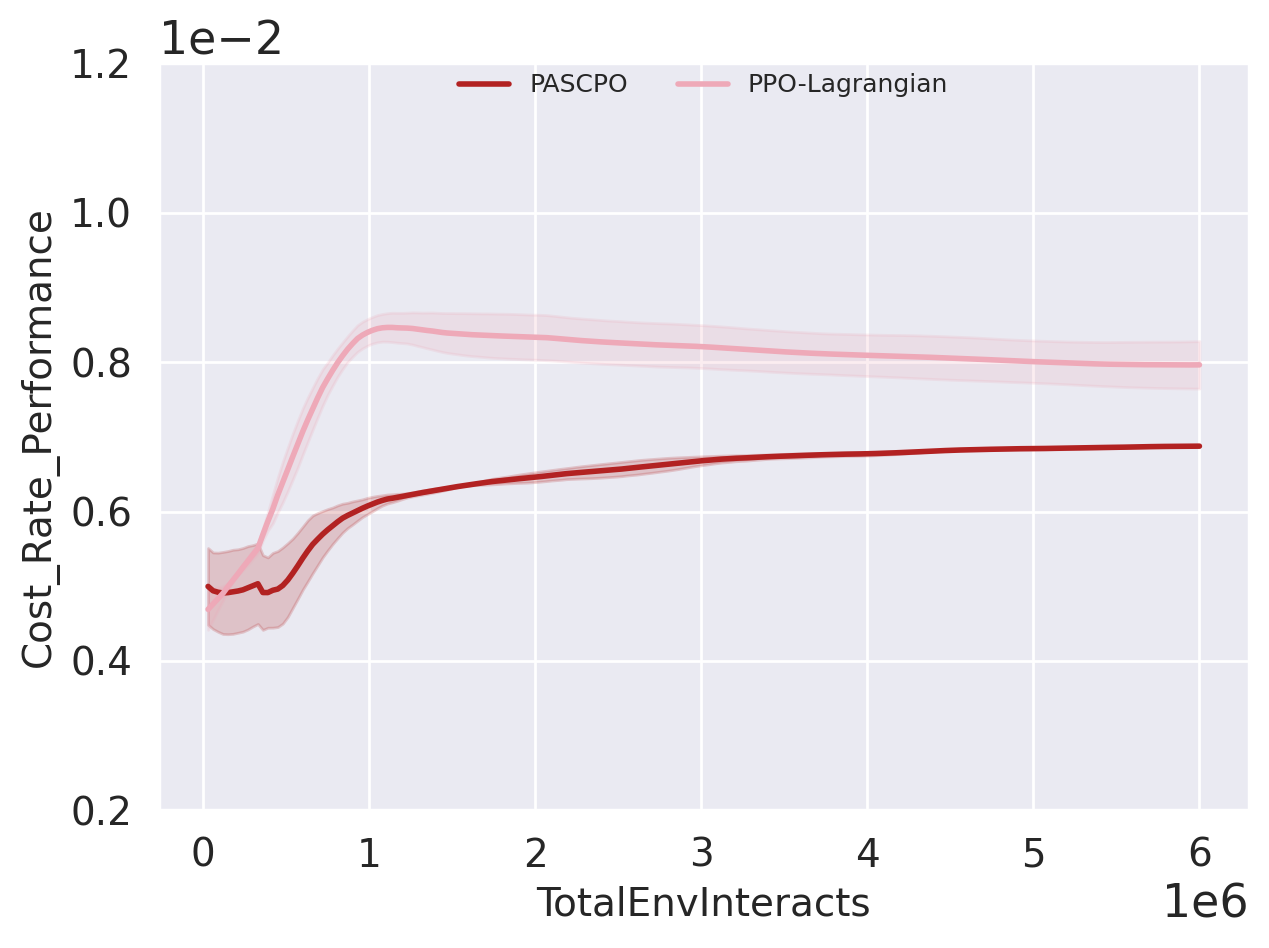}
          \end{subfigure}
          \hfill
          \begin{subfigure}[t]{1.0\linewidth}
              \includegraphics[width=0.99\textwidth]{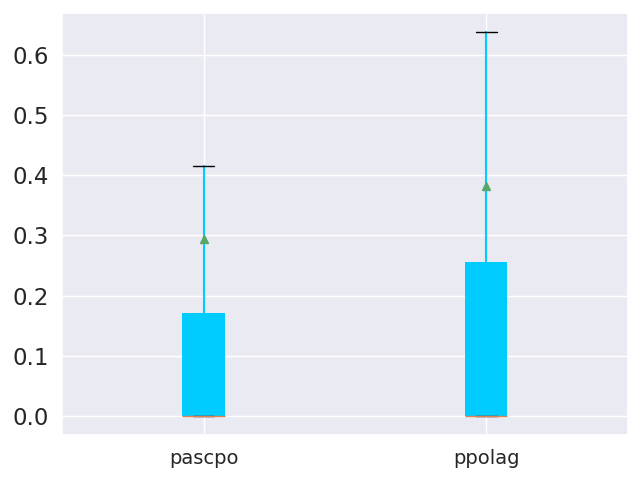}
          \end{subfigure}
      \caption{Ant-8-Hazards}
      \end{subfigure}
      \begin{subfigure}[t]{0.24\linewidth}
          \begin{subfigure}[t]{1.0\linewidth}
              \includegraphics[width=0.99\textwidth]{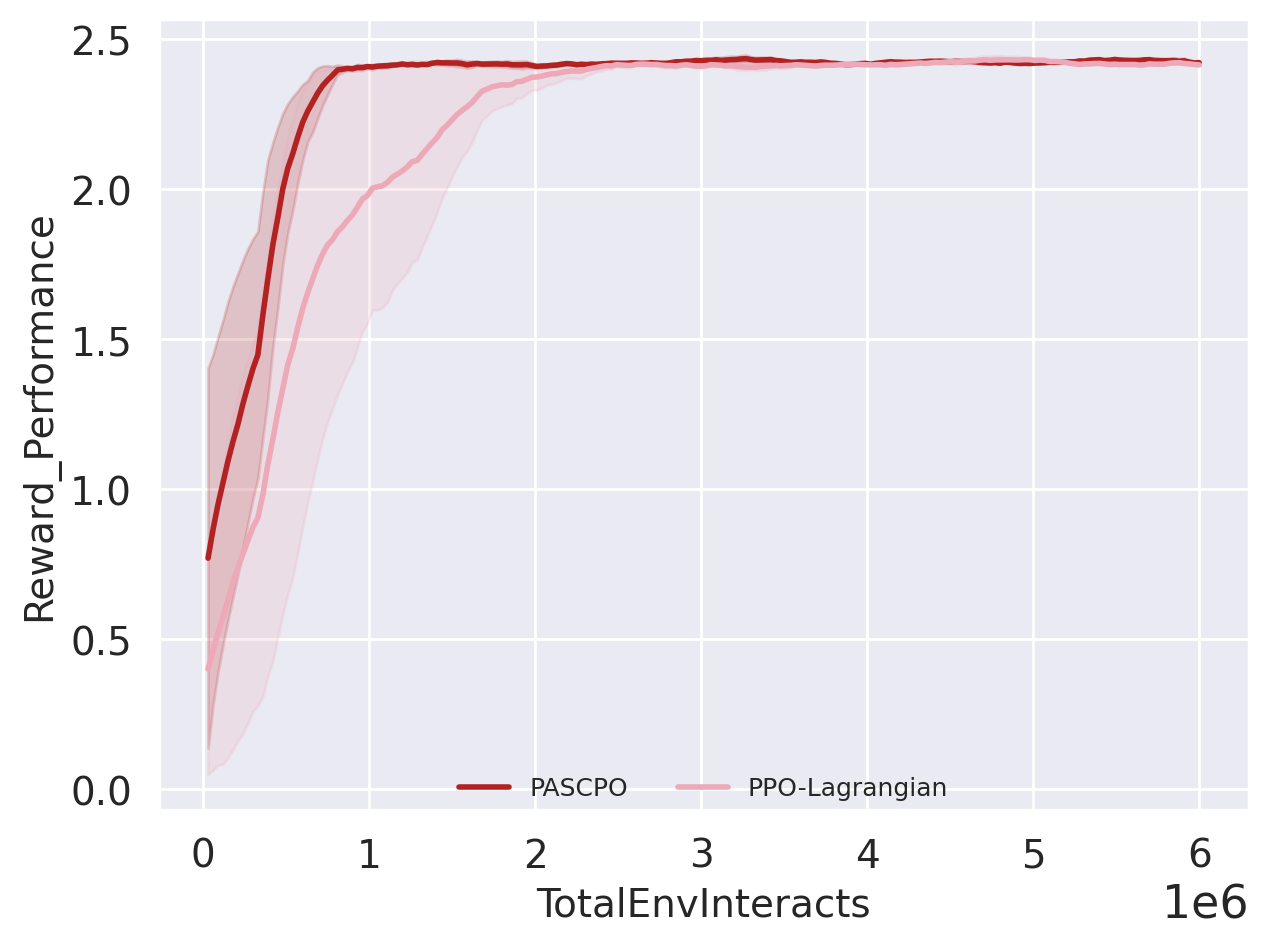}
          \end{subfigure}
          \hfill
          \begin{subfigure}[t]{1.0\linewidth}
              \includegraphics[width=0.99\textwidth]{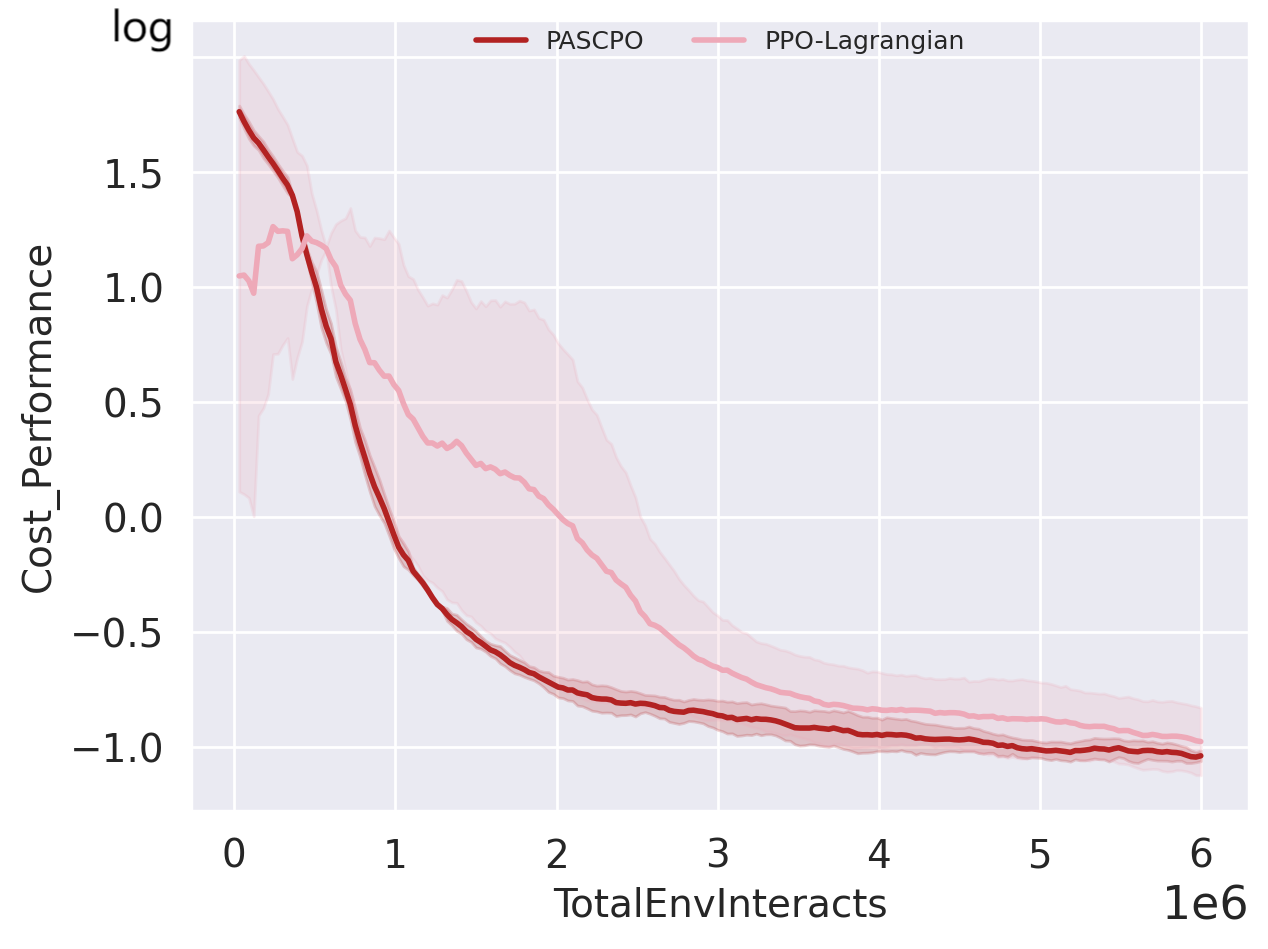}
          \end{subfigure}
          \hfill
          \begin{subfigure}[t]{1.0\linewidth}
              \includegraphics[width=0.99\textwidth]{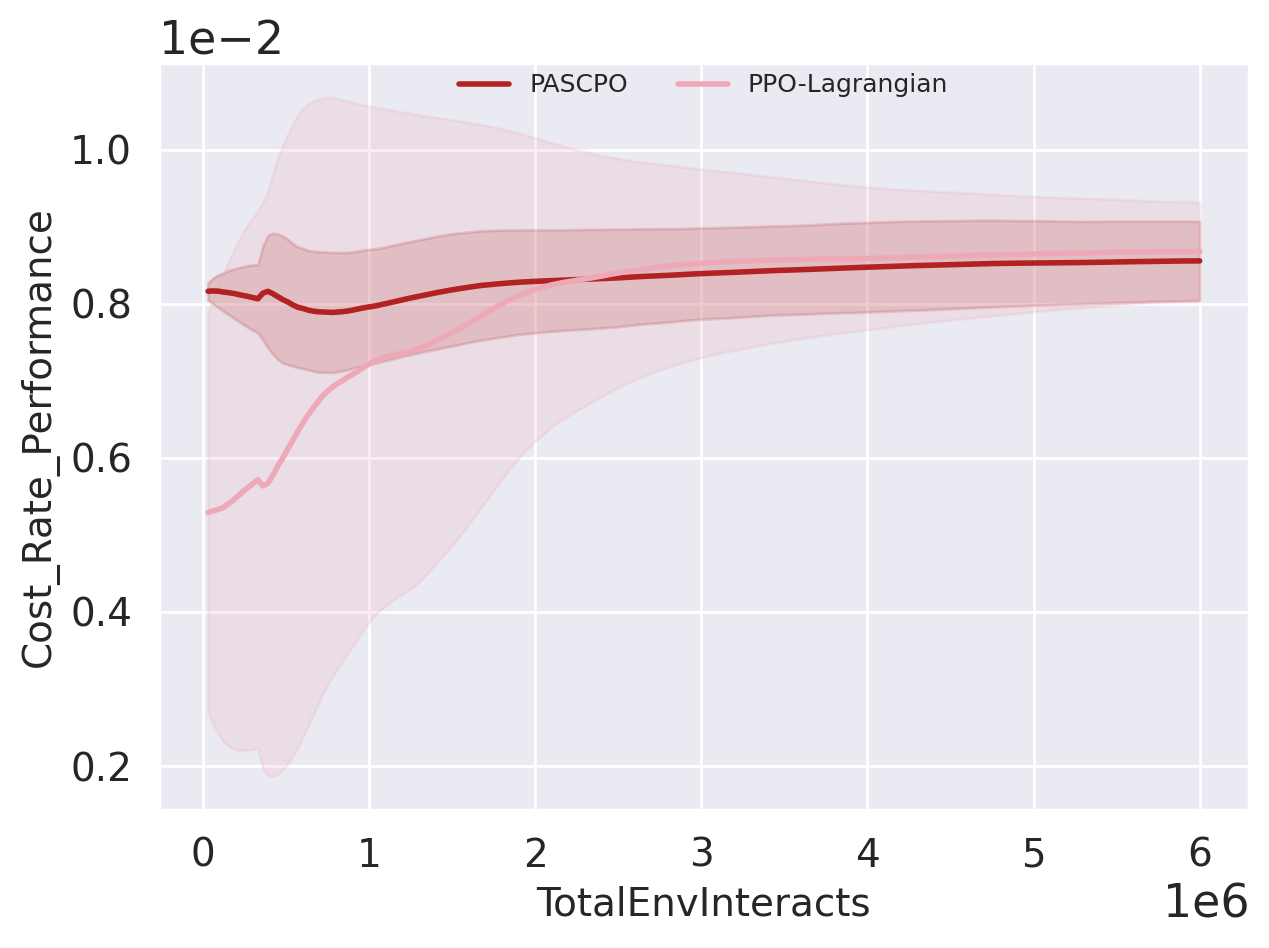}
          \end{subfigure}
          \hfill
          \begin{subfigure}[t]{1.0\linewidth}
              \includegraphics[width=0.99\textwidth]{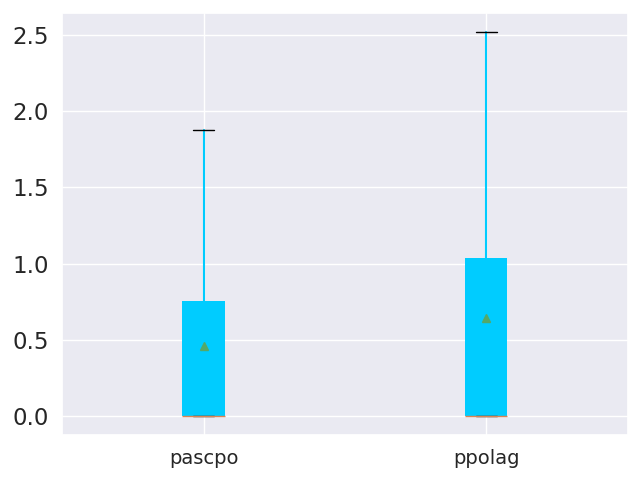}
          \end{subfigure}
      \caption{Walker-8-Hazards}
      \end{subfigure}
  \end{subfigure}
  \caption{PASCPO comparison curves and box plots of four representative test suites. The Y-axis of `Cost\_Performance' here uses logarithmic coordinates to show small level differences.}
  \label{fig: pascpo results}
\end{figure}

\begin{wrapfigure}{r}{0.5\textwidth}
    \vspace{-10pt}
    \centering
    \begin{subfigure}[b]{0.49\textwidth}
        \begin{subfigure}[t]{1.00\textwidth}
        \raisebox{-\height}{\includegraphics[width=\textwidth]{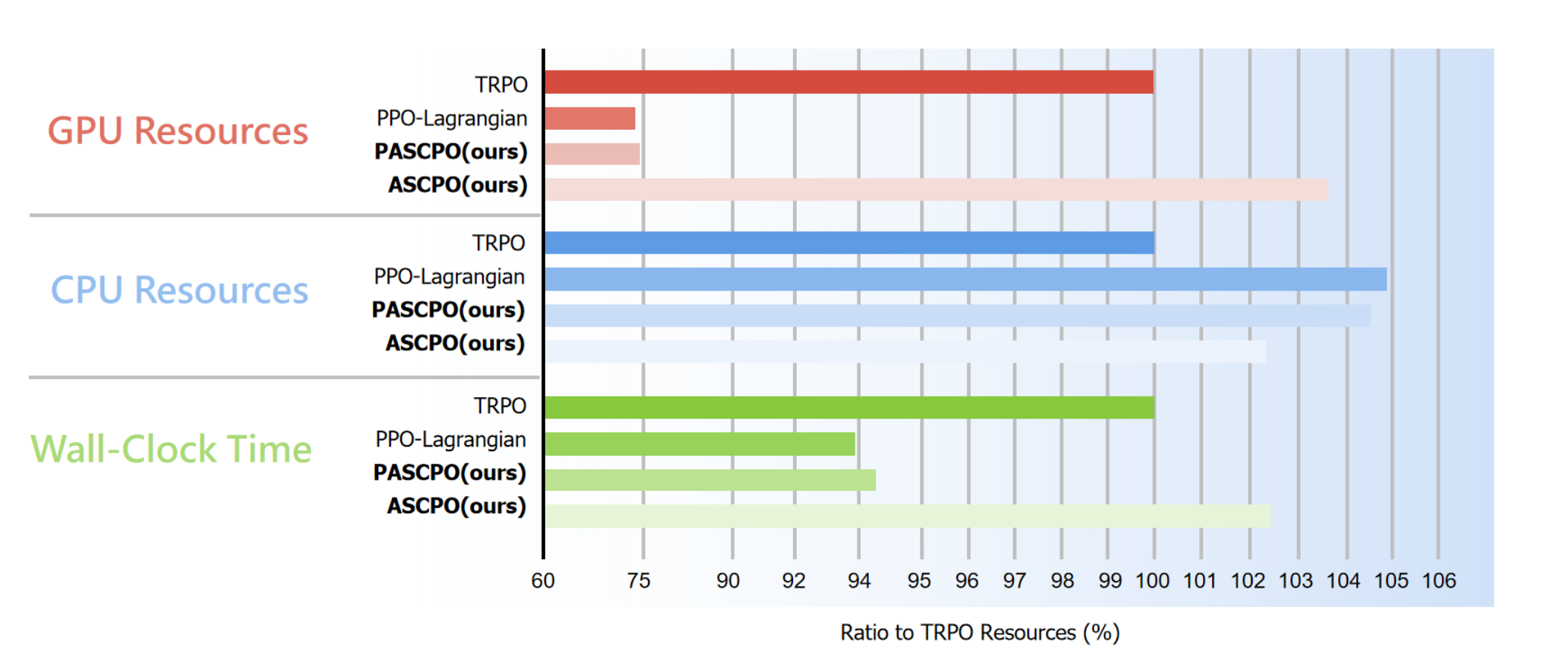}}
        \end{subfigure}
    \end{subfigure}
    \caption{Resources usage of PASCPO compared to other algorithms underi the Goal-8-Hazard task, where TRPO is set as the baseline.} 
    \label{fig: pascpo resources}
    \vspace{-10pt}
\end{wrapfigure}

\section{Proximal Absolute State-wise Constrained Policy Optimization}
With the proven success of the TRPO-based ASCPO in addressing various control tasks, a natural question arises: Can ASCPO be extended to a PPO-based version, similar to the PPO-Lagrangian implemented in \cite{ray2019benchmarking}? 
To explore this possibility, we introduce PASCPO, which integrates a \textbf{Clipped Surrogate Objective} \citep{schulman2017ppo} with the \textbf{Lagrangian Method}~\cite{ray2019benchmarking}. In detail, the original constrained optimization problem within the trust region, as outlined in \Cref{algo: ascpo_full}, is reformulated as a single-objective optimization problem. The local update is performed using a clipped surrogate objective, while the constraint is incorporated through the Lagrangian method. In \Cref{fig: pascpo results}, we compare our algorithm with fine-tuned PPO-Lagrangian in four representative tasks, demonstrating our effectiveness in achieving lower cost values and more efficient control across the entire distribution. 
Furthermore, as shown in \Cref{fig: pascpo resources}, PASCPO achieves significantly better performance while maintaining essentially the same hardware resource usage and wall-clock time. This demonstrates that our surrogate absolute bound, despite its complexity, does not introduce additional unacceptable computational costs. It is worth noting that PASCPO has significantly higher computing efficiency compared to ASCPO, albeit with some performance loss, making it suitable for specific scenarios.
These results address \textbf{Q5} and highlight the potential of extending ASCPO to a PPO-based method.

\section{Conclusion}
This paper proposed ASCPO, the first general-purpose policy search algorithm that ensures state-wise constraints satisfaction with high confidence.
We demonstrate ASCPO's effectiveness on challenging continuous control benchmark tasks, showing its significant performance improvement 
compared to existing methods and ability to handle state-wise constraints.


\acks{This work is partially supported by the National Science Foundation, Grant No. 2144489.}


\newpage

\appendix
\renewcommand{\contentsname}{Contents of Appendix}
\tableofcontents
\clearpage
\addtocontents{toc}{\protect\setcounter{tocdepth}{2}}

\clearpage
\section{Additional Proofs}
\label{proof: high prob stasfication}

\subsection{Proof of \Cref{prop: absolute bound definition}}
\label{sec: proof of prop absolute bound}
\begin{proof}
According to Selberg's inequality theory~\cite{saw1984chebyshev}, if random variable $\mathcal{R}_\pi(\hs_0)$ has finite non-zero variance $\mathcal{V}(\pi)$ and finite expected value $\mathcal{E}(\pi)$. Then for any real number $k\geq0$, following inequality holds:
\begin{align}
    Pr(\mathcal{R}_\pi(\hs_0) > \mathcal{E}(\pi)+k\mathcal{V}(\pi)) ) \le \frac{1}{k^2\mathcal{V}(\pi)+1}~~~,
\end{align}
which equals to:
\begin{align}
    Pr\big(\mathcal{R}_\pi(\hs_0)\leq \mathcal{B}_k(\pi)\big) \geq 1 - \frac{1}{k^2\mathcal{V}(\pi)+1}~~~.
\end{align}\\
Considering that $\pi\in\Pi$, then $\mathcal{V}(\pi)$ belongs to a corresponding variance space which has a non-zero minima of $\mathcal{V}(\pi)$, which is denoted as $\mathcal{V}_{min} \in \mathbb{R}^+$. Therefore, by treating $\psi = \mathcal{V}_{min}$, the following condition holds:
\begin{align}
        \text{There exists $\psi > 0$, s.t. }Pr\big(\mathcal{R}_\pi(\hs_0)\leq \mathcal{B}_k(\pi)\big) &\geq 1 - \frac{1}{k^2\mathcal{V}(\pi)+1} \\ \nonumber
        &\geq 1 - \frac{1}{k^2\mathcal{V}_{min}+1} \\ \nonumber 
        &= \underbrace{1 - \frac{1}{k^2\psi+1}}_{p_k^\psi} ~~~.
    \end{align}
\end{proof}




\subsection{Proof of \Cref{lem: bound of MV}}
\label{proof: MeanVariance Bound}

\begin{proof}
According to [Finite-horizon Version of Theorem 1, \citep{sobel1982variance}], the following proposition holds: 
\begin{prop}
\label{prop: V result}
Define $\bm X_\pi^H = \begin{bmatrix} \underset{\substack{{\hat \tau}\sim {\pi}}}{\mathbb{V}ar}[\mathcal{R}_\pi^H(\hs^1)] \\ \underset{\substack{{\hat \tau}\sim {\pi}}}{\mathbb{V}ar}[\mathcal{R}_\pi^H(\hs^2)] \\ \vdots \end{bmatrix}$,
and ${\hat P}_\pi = P_\pi^\top$, where $\hat{P}_\pi(i, j)$ denotes the probability of the transfer from i-th state to j-th state, the following equation holds 
\begin{align}
    \bm X_\pi^H =\gamma^2 \hat P_\pi \bm X_\pi^{H-1} + \bm \Omega_\pi^H
\end{align}
where $\bm X_\pi^0 = \begin{bmatrix}
    0 &
    0 &
    \hdots
\end{bmatrix}^\top$.
\end{prop}

With \Cref{prop: V result}, the complete $\bm X_\pi^H$ formula follows immediately
\begin{align}
\label{eq: complete X}
    \bm X_\pi^H &= \bm \Omega_\pi^H  + \gamma^2 \hat P_\pi \bm \Omega_\pi^{H-1} + (\gamma^2 \hat P_\pi)^2 \bm \Omega_\pi^{H-2} + \cdots (\gamma^2 \hat P_\pi)^{H-1}\Omega_\pi^{1} \\ \nonumber 
    &= \sum_{h=1}^H (\gamma^2 \hat P_\pi)^{H-h} \bm \Omega_\pi^h~~~.
\end{align}


With \eqref{eq: complete X}, The divergence of \textbf{MeanVariance} we want to bound can be written as:
\begin{align}
\label{eq:MV1 original}
    |MV_{\pi'} - MV_{\pi}| &=\big|\underset{\substack{\hs_0 \sim \mu}}{\mathbb{E}}[\underset{\substack{{\hat \tau}\sim {\pi}}}{\mathbb{V}ar}[\mathcal{R}_{\pi'}^H(\hs_0)] - \underset{\substack{\hs_0 \sim \mu}}{\mathbb{E}}[\underset{\substack{{\hat \tau}\sim {\pi}}}{\mathbb{V}ar}[\mathcal{R}_\pi^H(\hs_0)]\big| \\ \nonumber
    &= \|\mu^\top \big(\bm X_{\pi'}^H - \bm X_\pi^H) \|_\infty\\ \nonumber
    &\leq \|\mu^\top\|_\infty  \|\bm X_{\pi'}^H - \bm X_\pi^H\|_\infty \\ \nonumber  
    &\leq \|\mu^\top\|_\infty \bigg\| \sum_{h=1}^H (\gamma^2 \hat P_{\pi'})^{H-h} \bm \Omega_{\pi'}^h - \sum_{h=1}^{H} (\gamma^2 \hat P_\pi)^{H-h} \bm \Omega_\pi^h \bigg\|_\infty \\ \nonumber 
    &\leq \|\mu^\top\|_\infty \sum_{h=1}^H \bigg\|  (\gamma^2 \hat P_{\pi'})^{H-h} \bm \Omega_{\pi'}^h - (\gamma^2 \hat P_\pi)^{H-h} \bm \Omega_\pi^h \bigg\|_\infty
    ~~~~`.
\end{align}

Next, we will derive the upper bound of  $\bigg\| (\gamma^2 \hat P_{\pi'})^{H-h} \bm \Omega_{\pi'}^{h} - (\gamma^2 \hat P_\pi)^{H-h} \bm \Omega_\pi^{h} \bigg\|_\infty$

\begin{align}
\label{eq: MV_1}
    &\|(\gamma^2 \hat P_{\pi'})^{H-h} \bm \Omega_{\pi'}^{h} - (\gamma^2 \hat P_\pi)^{H-h} \bm \Omega_\pi^{h}\|_\infty \\ \nonumber 
    &= \|(\gamma^2 \hat P_{\pi'})^{H-h} \bm \Omega_{\pi'}^{h} - (\gamma^2 \hat P_{\pi'})^{H-h} \bm \Omega_\pi^{h} + (\gamma^2 \hat P_{\pi'})^{H-h} \bm \Omega_{\pi}^{h} - (\gamma^2 \hat P_\pi)^{H-h} \bm \Omega_\pi^{h}\|_\infty \\ \nonumber 
    &= \|(\gamma^2 \hat P_{\pi'})^{H-h} \big(\bm \Omega_{\pi'}^{h} - \bm \Omega_\pi^{h}\big) + \big((\gamma^2 \hat P_{\pi'})^{H-h} - (\gamma^2 \hat P_\pi)^{H-h}\big) \bm \Omega_\pi^{h}\|_\infty \\  \nonumber 
    &\leq \|(\gamma^2 \hat P_{\pi'})^{H-h} \big(\bm \Omega_{\pi'}^{h} - \bm \Omega_\pi^{h}\big)\|_\infty + \|\big((\gamma^2 \hat P_{\pi'})^{H-h} - (\gamma^2 \hat P_\pi)^{H-h}\big) \bm \Omega_\pi^{h}\|_\infty \\ \nonumber 
    &\leq \|(\gamma^2 \hat P_{\pi'})^{H-h}\|_\infty \|\bm \Omega_{\pi'}^{h} - \bm \Omega_\pi^{h}\|_\infty + \|(\gamma^2 \hat P_{\pi'})^{H-h} - (\gamma^2 \hat P_\pi)^{H-h}\|_\infty \|\bm \Omega_\pi^{h}\|_\infty
\end{align}

Since we already know 
\begin{align}
    \|(\gamma^2 \hat P_{\pi'})^{H-h}\|_\infty = ||\gamma^{2(H-h)} \hat P_{\pi'}^{H-h}||_\infty = \gamma^{2(H-h)}, \label{eq: mv 1 term} \\ 
    {\|\bm\Omega_\pi^{h)}\|_\infty} \text{  is a known constant vector given } \pi. \label{eq: mv 4 term}
\end{align}
We only need to bound $\|\bm \Omega_{\pi'}^{h} - \bm \Omega_\pi^{h}\|_\infty$ and $\|(\gamma^2 \hat P_{\pi'})^{H-h} - (\gamma^2 \hat P_\pi)^{H-h}\|_\infty$.

To address $\|(\gamma^2 \hat P_{\pi'})^{H-h} - (\gamma^2 \hat P_\pi)^{H-h}\|_\infty$, we have
\begin{align}
\label{eq:mv 3 term}
    &\|(\gamma^2 \hat P_{\pi'})^{H-h} - (\gamma^2 \hat P_\pi)^{H-h}\|_\infty \\ \nonumber 
    &\leq \|(\gamma^2 \hat P_{\pi'})^{H-h}\|_\infty + \|(\gamma^2 \hat P_\pi)^{H-h}\|_\infty \\ \nonumber 
    &= \gamma^{2(H-h)} (\|\hat P_{\pi'}^{H-h}\|_\infty + \|\hat P_\pi^{H-h}\|_\infty) \\ \nonumber 
    &= 2\gamma^{2(H-h)} ~~~.
\end{align}

To address $\|\bm\Omega_{\pi'}^{h} - \bm\Omega_\pi^{h}\|_\infty$, we notice that $\omega_\pi^h(\hs) = \underset{\substack{a \sim \pi \\ \hs' \sim P}}{\mathbb{V}ar}[Q_\pi^h(\hs,a,\hs')] = \underset{\substack{a \sim \pi \\ \hs' \sim P}}{\mathbb{V}ar}[A_\pi^h(\hs,a,\hs')]$, which means:
\begin{align}
    &\|\bm\Omega_{\pi'}^{h} -\bm\Omega_\pi^{h}\|_\infty = \underset{\hs}{\mathbf{max}}\bigg|\underset{\substack{a \sim \pi' \\ \hs' \sim P}}{\mathbb{V}ar}[A_{\pi'}^h(\hs,a,\hs')] - \underset{\substack{a \sim \pi \\ \hs' \sim P}}{\mathbb{V}ar}[A_\pi^h(\hs,a,\hs')]\bigg|\\ \nonumber
\end{align}
Where
\begin{align}
    A_\pi^h(\hs,a,\hs')&= R(\hs, a, \hs') + \gamma V_{\pi}^{h-1}(\hs') - V_{\pi}^h(\hs)\\ \nonumber
    A_{\pi'}^h(\hs,a,\hs')&= R(\hs, a, \hs') + \gamma V_{\pi'}^{h-1}(\hs') - V_{\pi'}^h(\hs)\\ \nonumber
\end{align}
Define $K^h(\hs,a,\hs')=A_{\pi'}^h(\hs,a,\hs') - A_{\pi}^h(\hs,a,\hs')$, we have:
\begin{align}
    K^h(\hs,a,\hs')&= \gamma( V_{\pi'}^{h-1}(\hs') - V_{\pi}^{h-1}(\hs')) - (V_{\pi'}^h(\hs) - V_{\pi}^h(\hs))\\ \nonumber
\end{align}
Similar to TRPO~\cite{schulman2015trust}:
\begin{align}
\label{eq: Expectation of A = -v + v}
    &\underset{\substack{\hs_0 = \hs \\{\hat \tau} \sim \pi'}}{\mathbb{E}}\bigg[\sum_{t=0}^{h-1} \gamma^tA_{\pi}^{h-t}(\hs_t,a_t,\hs_{t+1})\bigg]\\ \nonumber
    &=\underset{\substack{\hs_0 = \hs \\{\hat \tau} \sim \pi'}}{\mathbb{E}}\bigg[\sum_{t=0}^{h-1} \gamma^t(R_{\pi}(\hs_t,a_t,\hs_{t+1})+\gamma V_{\pi}^{h-t-1}(\hs_{t+1})-V_{\pi}^{h-t}(\hs_{t}))\bigg]\\ \nonumber
    &=\underset{\substack{\hs_0 = \hs \\{\hat \tau} \sim \pi'}}{\mathbb{E}}\bigg[-V_{\pi}^h(\hs_{0}) + \gamma^{h}V_{\pi}^0(\hs_h) +\sum_{t=0}^{h-1} \gamma^tR_{\pi}(\hs_t,a_t,\hs_{t+1})\bigg]\\ \nonumber
    &= \underset{\substack{\hs_0 = \hs \\{\hat \tau} \sim \pi'}}{\mathbb{E}}\bigg[-V_{\pi}^h(\hs_{0}) +\sum_{t=0}^{h-1} \gamma^tR_{\pi}(\hs_t,a_t,\hs_{t+1})\bigg]\\ \nonumber
    &=\underset{\substack{\hs_0 = \hs }}{\mathbb{E}}\bigg[-V_{\pi}^h(\hs_{0})\bigg] + \underset{\substack{\hs_0 = \hs \\{\hat \tau} \sim \pi'}}{\mathbb{E}}\bigg[\sum_{t=0}^{h-1} \gamma^tR_{\pi}(\hs_t,a_t,\hs_{t+1})\bigg]\\ \nonumber
    &=-V_{\pi}^h(\hs) + V_{\pi'}^h(\hs)\\ \nonumber
\end{align}
Then $K^h(\hs,a,\hs')$ can be written as:
\begin{align}
    K^h(\hs,a,\hs')&= \gamma\underset{\substack{\hs_0 = \hs' \\{\hat \tau} \sim \pi'}}{\mathbb{E}}\bigg[\sum_{t=0}^{h-2} \gamma^tA_{\pi}^{h-1-t}(\hs_t,a_t,\hs_{t+1})\bigg] - \underset{\substack{\hs_0 = \hs \\{\hat \tau} \sim \pi'}}{\mathbb{E}}\bigg[\sum_{t=0}^{h-1} \gamma^tA_{\pi}^{h-t}(\hs_t,a_t,\hs_{t+1})\bigg]\\ \nonumber
\end{align}
Define ${A_{\pi',\pi}^{h}(\hs)}$ to be the expected advantage of ${\pi'}$ over ${\pi}$ at state ${\hs}$:
\begin{align}
    A_{\pi',\pi}^{h}(\hs) =  \underset{a \sim \pi'}{\mathbb{E}}\bigg[A_{\pi}^{h}(\hs,a)\bigg]\\ \nonumber
\end{align}
Now $K^h(\hs,a,\hs')$ can be written as:
\begin{align}
    K^h(\hs,a,\hs')&= \gamma\underset{\substack{\hs_0 = \hs' \\{\hat \tau} \sim \pi'}}{\mathbb{E}}\bigg[\sum_{t=0}^{h-2} \gamma^t A_{\pi',\pi}^{h-1-t}(\hs_t)\bigg] - \underset{\substack{\hs_0 = \hs \\{\hat \tau} \sim \pi'}}{\mathbb{E}}\bigg[\sum_{t=0}^{h-1} \gamma^tA_{\pi',\pi}^{h-t}(\hs_t)\bigg]\\ \nonumber
\end{align}
Define ${L(\hs,a,\hs')}$ as:
\begin{align}
    L^h(\hs,a,\hs')&= \gamma\underset{\substack{\hs_0 = \hs' \\{\hat \tau} \sim \pi}}{\mathbb{E}}\bigg[\sum_{t=0}^{h-2} \gamma^t A_{\pi',\pi}^{h-1-t}(\hs_t)\bigg] - \underset{\substack{\hs_0 = \hs \\{\hat \tau} \sim \pi}}{\mathbb{E}}\bigg[\sum_{t=0}^{h-1} \gamma^tA_{\pi',\pi}^{h-t}(\hs_t)\bigg]\\ \nonumber
\end{align}
With $\epsilon = \underset{\hs,a,h}{\mathbf{max}}|A_\pi^h(\hs,a)|$ and the fact that $\big|\mathbb{E}_{\hs_t \sim \pi'}[A^h_{\pi',\pi}(\hs_t)] - \mathbb{E}_{\hs_t \sim \pi}[A^h_{\pi',\pi}(\hs_t)]\big| \leq 4\alpha (1 - (1-\alpha)^t )\epsilon$ ([Lemma3, \citep{schulman2015trust}]), where $\alpha = \underset{\hs}{\max} \mathcal{D}_{TV}(\pi'\|\pi)[\hs] = \mathcal{D}_{TV}^{max}(\pi'\|\pi)$, we have:
\begin{align}
    &|K^h(\hs,a,\hs') - L^h(\hs,a,\hs')| \\ \nonumber
    &= \bigg|\gamma\bigg(\underset{\substack{\hs_0 = \hs' \\{\hat \tau} \sim \pi'}}{\mathbb{E}}\bigg[\sum_{t=0}^{h-2} \gamma^t A_{\pi',\pi}^{h-1-t}(\hs_t)\bigg] - \underset{\substack{\hs_0 = \hs' \\{\hat \tau} \sim \pi}}{\mathbb{E}}\bigg[\sum_{t=0}^{h-2} \gamma^t A_{\pi',\pi}^{h-1-t}(\hs_t)\bigg]\bigg) \\ \nonumber
    &- \bigg(\underset{\substack{\hs_0 = \hs \\{\hat \tau} \sim \pi'}}{\mathbb{E}}\bigg[\sum_{t=0}^{h-1} \gamma^tA_{\pi',\pi}^{h-t}(\hs_t)\bigg] - \underset{\substack{\hs_0 = \hs \\{\hat \tau} \sim \pi}}{\mathbb{E}}\bigg[\sum_{t=0}^{h-1} \gamma^tA_{\pi',\pi}^{h-t}(\hs_t)\bigg]\bigg)\bigg|\\ \nonumber
    &\leq \gamma\bigg|\underset{\substack{\hs_0 = \hs' \\{\hat \tau} \sim \pi'}}{\mathbb{E}}\bigg[\sum_{t=0}^{h-2} \gamma^t A_{\pi',\pi}^{h-1-t}(\hs_t)\bigg] - \underset{\substack{\hs_0 = \hs' \\{\hat \tau} \sim \pi}}{\mathbb{E}}\bigg[\sum_{t=0}^{h-2} \gamma^t A_{\pi',\pi}^{h-1-t}(\hs_t)\bigg]\bigg| \\ \nonumber
    &+ \bigg|\underset{\substack{\hs_0 = \hs \\{\hat \tau} \sim \pi'}}{\mathbb{E}}\bigg[\sum_{t=0}^{h-1} \gamma^tA_{\pi',\pi}^{h-t}(\hs_t)\bigg] - \underset{\substack{\hs_0 = \hs \\{\hat \tau} \sim \pi}}{\mathbb{E}}\bigg[\sum_{t=0}^{h-1} \gamma^tA_{\pi',\pi}^{h-t}(\hs_t)\bigg]\bigg|\\ \nonumber
    &\leq 4\epsilon\alpha \bigg(\gamma\frac{1 - \gamma^{h-2}}{1-\gamma} - \gamma \frac{1 - \gamma^{h-2} (1-\alpha)^{h-2}}{1- \gamma(1-\alpha)} + \frac{1 - \gamma^{h-1}}{1-\gamma} - \frac{1 - \gamma^{h-1} (1-\alpha)^{h-1}}{1- \gamma(1-\alpha)}\bigg)\\\nonumber
    &\leq  \frac{8\epsilon(\gamma - \gamma^h)}{(1-\gamma)^2}(\mathcal{D}_{TV}^{max}(\pi'\|\pi))^2
\end{align}

Then according to \cite{pinskerInequ} $\mathcal{D}_{TV}(\pi'\|\pi)[\hs]\leq\sqrt{\frac{1}{2}\mathcal{D}_{KL}(\pi'\|\pi)[\hs]}$, we can then bound ${|K(\hs,a,\hs')|}$ with:
\begin{align}
    |K^h(\hs,a,\hs')| &\leq \left|L^h(\hs,a,\hs')\right| + \frac{8\epsilon(\gamma - \gamma^h)}{(1-\gamma)^2}(\mathcal{D}_{TV}^{max}(\pi'\|\pi))^2 \\ \nonumber
    &\leq \left|L^h(\hs,a,\hs')\right| + \frac{4\epsilon(\gamma - \gamma^h)}{(1-\gamma)^2}\mathcal{D}_{KL}^{max}(\pi'\|\pi) \doteq |K^h(\hs,a,\hs')|_{max}
\end{align}

where $\mathcal{D}_{KL}^{max}(\pi'\|\pi) = \max_s \mathcal{D}_{KL}(\pi'\|\pi)[s]$. With $A_{\pi'}^h(\hs,a,\hs') = A_{\pi}^h(\hs,a,\hs') + K^h(\hs,a,\hs')$, we have:

\begin{align}
    &\underset{\substack{a \sim \pi' \\ \hs' \sim P}}{\mathbb{V}ar}[A_{\pi'}^h(\hs,a,\hs')] - \underset{\substack{a \sim \pi \\ \hs' \sim P}}{\mathbb{V}ar}[A_\pi^h(\hs,a,\hs')]\\ \nonumber
    =&\underset{\substack{a \sim \pi' \\ \hs' \sim P}}{\mathbb{E}}[A_{\pi'}^h(\hs,a,\hs')^2] - \underset{\substack{a \sim \pi \\ \hs' \sim P}}{\mathbb{E}}[A_{\pi}^h(\hs,a,\hs')^2]\\ \nonumber
    =& \underset{\substack{a \sim \pi' \\ \hs' \sim P}}{\mathbb{E}}[(A_{\pi}^h(\hs,a,\hs') + K^h(\hs,a,\hs'))^2] - \underset{\substack{a \sim \pi \\ \hs' \sim P}}{\mathbb{E}}[A_{\pi}^h(\hs,a,\hs')^2]\\ \nonumber
    =& \underset{\substack{a \sim \pi' \\ \hs' \sim P}}{\mathbb{E}}[A_{\pi}^h(\hs,a,\hs')^2] - \underset{\substack{a \sim \pi \\ \hs' \sim P}}{\mathbb{E}}[A_{\pi}^h(\hs,a,\hs')^2] + 2\underset{\substack{a \sim \pi' \\ \hs' \sim P}}{\mathbb{E}}[A_{\pi}(\hs,a,\hs') K^h(\hs,a,\hs')] + \underset{\substack{a \sim \pi' \\ \hs' \sim P}}{\mathbb{E}}[K^h(\hs,a,\hs')^2]\\ \nonumber
    =& \underset{\substack{\\a\sim\pi\\\hs'\sim P}}{\mathbb{E}}\left[\left(\frac{\pi^{\prime}(a|\hs)}{\pi(a|\hs)}-1\right) A_{\pi}^h(\hs,a,\hs')^2\right] + 2\underset{\substack{a \sim \pi' \\ \hs' \sim P}}{\mathbb{E}}[A_{\pi}^h(\hs,a,\hs') K^h(\hs,a,\hs')] + \underset{\substack{a \sim \pi' \\ \hs' \sim P}}{\mathbb{E}}[K^h(\hs,a,\hs')^2]\\ \nonumber
    &\leq \underset{\substack{\\a\sim\pi\\\hs'\sim P}}{\mathbb{E}}\left[\left(\frac{\pi^{\prime}(a|\hs)}{\pi(a|\hs)}-1\right) A_{\pi}^h(\hs,a,\hs')^2 + 2\left(\frac{\pi^{\prime}(a|\hs)}{\pi(a|\hs)}\right)A_{\pi}^h(\hs,a,\hs')|K^h(\hs,a,\hs')|_{max} + |K^h(\hs,a,\hs')|_{max}^2 \right]\\ \nonumber
\end{align}
Then we can bound ${\|\bm\Omega_{\pi'} -\bm\Omega_\pi\|_\infty}$ with:
\begin{align}
\label{eq: Omega differences}
 &\|\bm\Omega_{\pi'} -\bm\Omega_\pi\|_\infty \\ \nonumber
 &\leq \underset{\hs}{\mathbf{max}}\bigg|\underset{\substack{\\a\sim\pi\\\hs'\sim P}}{\mathbb{E}}\bigg[\left(\frac{\pi^{\prime}(a|\hs)}{\pi(a|\hs)}-1\right) A_{\pi}^h(\hs,a,\hs')^2 \\ \nonumber 
 &~~~~~+ 2\left(\frac{\pi^{\prime}(a|\hs)}{\pi(a|\hs)}\right)A_{\pi}^h(\hs,a,\hs')|K^h(\hs,a,\hs')|_{max} + |K^h(\hs,a,\hs')|_{max}^2\bigg]\bigg|\\ \nonumber
\end{align}
By substituting \eqref{eq: mv 1 term}, \eqref{eq: mv 4 term}, \eqref{eq:mv 3 term} and \eqref{eq: Omega differences} into \Cref{eq: MV_1}, we have:
\begin{align}
    &\|(\gamma^2 \hat P_{\pi'})^{H-h} \bm \Omega_{\pi'}^{h} - (\gamma^2 P_\pi)^{H-h} \bm \Omega_\pi^{h}\|_1 \label{eq: MV1 final}\\ \nonumber 
    &\leq \gamma^{2(H-h)} \underset{\hs}{\mathbf{max}}\bigg|\underset{\substack{\\a\sim\pi\\\hs'\sim P}}{\mathbb{E}}\bigg[\left(\frac{\pi^{\prime}(a|\hs)}{\pi(a|\hs)}-1\right) A_{\pi}^h(\hs,a,\hs')^2 \\ \nonumber 
 &~~~~~+ 2\left(\frac{\pi^{\prime}(a|\hs)}{\pi(a|\hs)}\right)A_{\pi}^h(\hs,a,\hs')|K^h(\hs,a,\hs')|_{max} + |K^h(\hs,a,\hs')|_{max}^2\bigg]\bigg| + 2\gamma^{2(H-h)} \|\bm \Omega_\pi^{h}\|_\infty
\end{align}

By substituting \eqref{eq: MV1 final} into \eqref{eq:MV1 original}, \Cref{lem: bound of MV} is proved.

\end{proof}
\subsection{Proof of \Cref{lem: bound of VM}}
\label{proof: VarianceMean Bound}

\begin{proof}
\begin{align}
\label{eq: bound of VM}
    VM_{\pi'} = \underset{\hs_0 \sim \mu}{\mathbb{E}} [(V^H_{\pi'}(\hs_0))^2] - \mathcal{E}(\pi')^2
\end{align}
Since both terms on the right of \Cref{eq: bound of VM} are non-negative, we can bound $VM_{\pi'}$ with the upper bound of $\underset{\hs_0 \sim \mu}{\mathbb{E}} [(V_{\pi'}^H(\hs_0))^2]$ and the lower bound of $\mathcal{E}(\pi')^2$.\\
Define $\bm Y_\pi = \begin{bmatrix} (V_\pi^H(\hs^1))^2 \\ (V_\pi^H(\hs^2))^2 \\ \vdots \end{bmatrix}$, where $\underset{\hs_0 \sim \mu}{\mathbb{E}} [(V_\pi^H(\hs_0))^2] = \mu^\top \bm Y_\pi$.
Then we have
\begin{align}
\label{eq: bound EV}
    &\big|\underset{\hs_0 \sim \mu}{\mathbb{E}} [(V_{\pi'}^H(\hs_0))^2]-\underset{\hs_0 \sim \mu}{\mathbb{E}} [(V_\pi^H(\hs_0))^2]\big| \\ \nonumber
    &= \|\mu^\top \big(\bm Y_{\pi'} - \bm Y_\pi) \|_\infty\\ \nonumber
    &\leq \|\mu^\top\|_\infty  \|\bm Y_{\pi'} - \bm Y_\pi\|_\infty \\ \nonumber
\end{align}
To address $\|\bm Y_{\pi'} - \bm Y_\pi\|_\infty$, we have:
\begin{align}
    (V_{\pi'}^H(\hs))^2 - (V_\pi^H(\hs))^2 
    &= \bigg(V^H_{\pi'}(\hs) - V^H_\pi(\hs)\bigg)\bigg(V^H_{\pi'}(\hs)+V^H_\pi(\hs)\bigg) \\ \nonumber
\end{align}
According to \eqref{eq: Expectation of A = -v + v}:
\begin{align}
    V^H_{\pi'}(\hs) - V^H_{\pi}(\hs) &= \underset{\substack{\hs_0 = \hs \\{\hat \tau} \sim \pi'}}{\mathbb{E}}\bigg[\sum_{t=0}^{H-1} \gamma^t A^{H-t}_{\pi}(\hs_t,a_t,\hs_{t+1})\bigg] \\ \nonumber
    &= \underset{\substack{\hs_0 = \hs \\{\hat \tau} \sim \pi'}}{\mathbb{E}}\bigg[\sum_{t=0}^{H-1} \gamma^t\bar A^{H-t}_{\pi',\pi}(\hs_t)\bigg]\\ \nonumber
    &\dot = \ \ \eta(\hs)
\end{align}
Define $L(\hs) = \underset{\substack{\hs_0 = \hs \\{\hat \tau} \sim \pi}}{\mathbb{E}}\bigg[\sum_{t=0}^{H-1} \gamma^t\bar A^{H-t}_{\pi',\pi}(\hs_t)\bigg]$, then we have:
\begin{align}
    |\eta(\hs)-L(\hs)| &= 
    \bigg|\underset{\substack{\hs_0 = \hs \\{\hat \tau} \sim \pi'}}{\mathbb{E}}\bigg[\sum_{t=0}^{H-1} \gamma^t\bar A^{H-t}_{\pi',\pi}(\hs_t)\bigg] - \underset{\substack{\hs_0 = \hs \\{\hat \tau} \sim \pi}}{\mathbb{E}}\bigg[\sum_{t=0}^{H-1} \gamma^t\bar A^{H-t}_{\pi',\pi}(\hs_t)\bigg] \bigg| \\ \nonumber
    &\leq 4\epsilon\alpha\left(\frac{1-\gamma^{H-1}}{1-\gamma}-\frac{1-\gamma^{H-1}(1-\alpha)^{H-1}}{1-\gamma(1-\alpha)}\right) \\ \nonumber
    & \leq \frac{4\epsilon(\gamma - \gamma^H)}{(1-\gamma)^2}(\mathcal{D}_{TV}(\pi'\|\pi)[\hs])^2
\end{align}
And according to \cite{pinskerInequ}, we can bound $|\eta(\hs)|$ with:
\begin{align}
    |\eta(\hs)| &\leq \left|L(\hs)\right| + \frac{4\epsilon(\gamma - \gamma^H)}{(1-\gamma)^2}(\mathcal{D}_{TV}(\pi'\|\pi)[\hs])^2 \\ \nonumber
    &\leq \left|L(\hs)\right| + \frac{2\epsilon(\gamma - \gamma^H)}{(1-\gamma)^2}\mathcal{D}^{max}_{KL}(\pi'\|\pi) \doteq |\eta(\hs)|_{max} 
\end{align}
Further, we can obtain:
\begin{align}
    |V^H_{\pi'}(\hs)+V^H_\pi(\hs)| &\leq  |V^H_{\pi'}(\hs)|+|V^H_\pi(\hs)| \\ \nonumber
    &= |V^H_{\pi'}(\hs)| - |V^H_\pi(\hs)| + 2|V^H_\pi(\hs)| \\ \nonumber
    &\leq |V^H_{\pi'}(\hs)-V^H_\pi(\hs)| + 2|V^H_\pi(\hs)| \\ \nonumber
    &\leq |\eta(\hs)|_{max} + 2|V^H_\pi(\hs)| 
\end{align}
Thus the following inequality holds:
\begin{align}
\label{eq: bound V difference}
    &\|\bm Y_{\pi'} - \bm Y_\pi\|_\infty \\ \nonumber
    &\leq \underset{\hs}{\mathbf{max}}\bigg||V^H_{\pi'}(\hs)-V^H_\pi(\hs)|\cdot|V^H_{\pi'}(\hs)+V^H_\pi(\hs)|\bigg| \\ \nonumber
    &\leq \underset{\hs}{\mathbf{max}}\bigg||\eta(\hs)|_{max}\cdot\left(|\eta(\hs)|_{max} + 2|V^H_\pi(\hs)|\right)\bigg| \\ \nonumber
    &= \underset{\hs}{\mathbf{max}}\bigg||\eta(\hs)|_{max}^2+2|V^H_\pi(\hs)|\cdot|\eta(\hs)|_{max}\bigg|
\end{align}
Substitute \Cref{eq: bound V difference} into \Cref{eq: bound EV} the upper bound of $\underset{\hs_0 \sim \mu} {\mathbb{E}} [(V_{\pi'}^H(\hs_0))^2]$ is obtained:
\begin{align}
\label{eq: upper bound of the first term}
\underset{\hs_0 \sim \mu} {\mathbb{E}} [(V_{\pi'}^H(\hs_0))^2] &\leq \underset{\hs_0 \sim \mu}{\mathbb{E}} [(V^H_\pi(\hs_0))^2] + \|\mu^T\|_\infty\underset{\hs}{\mathbf{max}}\bigg||\eta(\hs)|_{max}^2+2|V^H_\pi(\hs)|\cdot|\eta(\hs)|_{max}\bigg| \\ \nonumber
\end{align}
The lower bound of $\mathcal{E}(\pi')^2$ can then be obtained according to \citep{zhao2024statewise}:
\begin{align}
\label{eq: lower bound of the second term}
    \mathcal{E}(\pi')^2 \geq \left(\mathbf{min}\left\{\mathbf{max}\left\{0,\ \mathcal{E}^l_{\pi^{\prime}, \pi}\right\}, \mathcal{E}^u_{\pi^{\prime}, \pi}\right\}\right)^2
\end{align}
where
\begin{align}
\nonumber
    \mathcal{E}^l_{\pi^{\prime}, \pi}&=\mathcal{E}(\pi) + \underset{\substack{\hs \sim \overline{d}_\pi \\ a\sim {\pi'}}}{\mathbb{E}} \bigg[ A^H_\pi(\hs,a) - 2(H+1)\epsilon^{\pi'} \sqrt{\frac 12 \mathcal{D}_{KL}({\pi'} \| \pi)[\hs]} \bigg] \\ \nonumber
    \mathcal{E}^u_{\pi^{\prime}, \pi}&=\mathcal{E}(\pi) + \underset{\substack{\hs \sim \overline{d}_\pi \\ a\sim {\pi'}}}{\mathbb{E}} \bigg[ A^H_\pi(\hs,a) + 2(H+1)\epsilon^{\pi'} \sqrt{\frac 12 \mathcal{D}_{KL}({\pi'} \| \pi)[\hs]} \bigg] \\
    \nonumber
    \overline{d}_\pi &= \sum\limits_{t=0}^H P(\hs_t=\hs|\pi)
\end{align}
By substituting \Cref{eq: upper bound of the first term} and \Cref{eq: lower bound of the second term} into \Cref{eq: bound of VM} \Cref{lem: bound of VM} is proved.
\end{proof}

\clearpage

\section{Experiment Details}
\label{sec: experiment details}
\subsection{Environment Settings}
\paragraph{Goal Task}
In the Goal task environments, the reward function is:
\begin{equation}\notag
\begin{split}
    & r(x_t) = d^{g}_{t-1} - d^{g}_{t} + \mathbf{1}[d^g_t < R^g]~,\\
\end{split}
\end{equation}
where $d^g_t$ is the distance from the robot to its closest goal and $R^g$ is the size (radius) of the goal. When a goal is achieved, the goal location is randomly reset to someplace new while keeping the rest of the layout the same. The test suites of our experiments are summarized in \Cref{tab: testing suites}.

\begin{table}[h]
\vskip 0.15in
\caption{The test suites environments of our experiments}
\begin{center}
\begin{tabular}{c|ccc|ccc|c}
\toprule
 & \multicolumn{6}{c|}{Ground robot} &\multicolumn{1}{c}{Aerial robot}\\
\cline{2-8}\\[-1.02em]
\textbf{Task Setting} & \multicolumn{3}{c|}{Low dimension} &\multicolumn{3}{c|}{High dimension} &\\
\cline{2-8}\\[-1.02em]
&Point &Swimmer & \multicolumn {1}{c|}{Arm3} 
& Humanoid &Ant & \multicolumn {1}{c|}{Walker} & Drone\\
\cline{1-8}\\[-1.02em]
1-Hazard    & \checkmark & \checkmark &      &      &       &     &  \cellcolor[HTML]{C0C0C0}  \\
4-Hazard    & \checkmark & \checkmark &      &      &       &     &  \cellcolor[HTML]{C0C0C0}  \\
8-Hazard    & \checkmark & \checkmark &\checkmark& \checkmark & \checkmark&\checkmark &  \cellcolor[HTML]{C0C0C0}  \\
1-Pillar    & \checkmark &            &      &      &       &     &  \cellcolor[HTML]{C0C0C0}  \\
4-Pillar    & \checkmark &            &      &      &       &     &  \cellcolor[HTML]{C0C0C0}  \\
8-Pillar    & \checkmark &            &      &      &       &     &  \cellcolor[HTML]{C0C0C0}  \\
1-Ghost & \checkmark  & & &  &  &  & \cellcolor[HTML]{C0C0C0}\\
4-Ghost & \checkmark   & & &  &  &  & \cellcolor[HTML]{C0C0C0}\\
8-Ghost &  \checkmark  & & &  &  &  & \cellcolor[HTML]{C0C0C0}\\
3DHazard-1 & \cellcolor[HTML]{C0C0C0} & \cellcolor[HTML]{C0C0C0}& \cellcolor[HTML]{C0C0C0}& \cellcolor[HTML]{C0C0C0} & \cellcolor[HTML]{C0C0C0} & \cellcolor[HTML]{C0C0C0} & \checkmark\\
3DHazard-4 &  \cellcolor[HTML]{C0C0C0} & \cellcolor[HTML]{C0C0C0}& \cellcolor[HTML]{C0C0C0}& \cellcolor[HTML]{C0C0C0} & \cellcolor[HTML]{C0C0C0} & \cellcolor[HTML]{C0C0C0} & \checkmark\\
3DHazard-8 &  \cellcolor[HTML]{C0C0C0} & \cellcolor[HTML]{C0C0C0}& \cellcolor[HTML]{C0C0C0}& \cellcolor[HTML]{C0C0C0} & \cellcolor[HTML]{C0C0C0} & \cellcolor[HTML]{C0C0C0} & \checkmark\\
\bottomrule
\end{tabular}
\label{tab: testing suites}
\end{center}
\end{table}

\paragraph{Hazard Constraint}
In the Hazard constraint environments, the cost function is:
\begin{equation}\notag
\begin{split}
    & c(x_t) = \max(0, R^h - d^h_t)~,\\
\end{split}
\end{equation}
where $d^h_t$ is the distance to the closest hazard and $R^h$ is the size (radius) of the hazard.

\paragraph{Pillar Constraint}
In the Pillar constraint environments, the cost $c_t = 1$ if the robot contacts with the pillar otherwise $c_t = 0$. 

\paragraph{Ghost Constraint} 
In the Ghost constraint environments, the cost function is:
\begin{equation}\notag
\begin{split}
    & c(x_t) = \max(0, R^h - d^h_t)~,\\
\end{split}
\end{equation}
where $d^h_t$ is the distance to the closest ghost and $R^h$ is the size (radius) of the ghost. And dynamics of ghosts are as follow:
\begin{equation}
\dot x_{object}= 
\begin{cases}
    \begin{aligned}
    v_0 * &d_{origin} ,& \text{if } &\norm{d_{origin}} > r_0\\
    v_1 * &d_{robot} ,& \text{if } &\norm{d_{origin}} \leq r_0 \text{ and } \norm{d_{robot}} > r_1\\
    &0 ,& \text{if }& \norm{d_{origin}} \leq r_0 \text{ and } \norm{d_{robot}} \leq r_1\\
    \end{aligned}
\end{cases},
\end{equation}
where $d_{\text{origin}} = x_{\text{origin}} - x_{\text{object}}$ represents the distance from the position of the dynamic object $x_{\text{object}}$, $d_{\text{robot}} = x_{\text{robot}} - x_{\text{object}}$ represents the distance from the dynamic object $x_{\text{object}}$ to the position of the robot $x_{\text{robot}}$,  $r_0$ defines a circular area centered at the origin point within which the objects are limited to move. $r_1$ represents the threshold distance that the dynamic objects strive to maintain from the robot and $v_0$, $v_1$ are configurable non-negative velocity constants for the dynamic objects.

\paragraph{State Space}
The state space is composed of two parts. The internal state spaces describe the state of the robots, which can be obtained from standard robot sensors (accelerometer, gyroscope, magnetometer, velocimeter, joint position sensor, joint velocity sensor and touch sensor). The details of the internal state spaces of the robots in our test suites are summarized in \Cref{tab:internal_state_space}.
The external state spaces are describe the state of the environment observed by the robots, which can be obtained from 2D lidar or 3D lidar (where each lidar sensor perceives objects of a single kind). The state spaces of all the test suites are summarized in \Cref{tab:external_state_space}. Note that Vase and Gremlin are two other constraints in Safety Gym \citep{ray2019benchmarking} and all the returns of vase lidar and gremlin lidar are zero vectors (i.e., $[0, 0, \cdots, 0] \in \mathbb{R}^{16}$) in our experiments since none of our test suites environments has vases.

\begin{table}[h]
\vskip 0.15in
\caption{The internal state space components of different test suites environments.}
\begin{center}
\begin{tabular}{c|ccccccc}
\toprule
\textbf{Internal State Space} & Point  & Swimmer & Walker & Ant & Drone & Arm3 & Humanoid\\
\hline
Accelerometer ($\mathbb{R}^3$) & \checkmark & \checkmark & \checkmark & \checkmark & \checkmark  & \checkmark $\times 5$ & \checkmark\\
Gyroscope ($\mathbb{R}^3$) & \checkmark & \checkmark & \checkmark & \checkmark & \checkmark  & \checkmark $\times 5$ & \checkmark\\
Magnetometer ($\mathbb{R}^3$) & \checkmark & \checkmark & \checkmark & \checkmark & \checkmark  & \checkmark $\times 5$ & \checkmark\\
Velocimeter ($\mathbb{R}^{3}$) & \checkmark & \checkmark & \checkmark & \checkmark & \checkmark & \checkmark $\times 5$ & \checkmark\\
Joint position sensor ($\mathbb{R}^{n}$) & ${n=0}$ & ${n=2}$ & ${n=10}$ & ${n=8}$ & ${n=0}$ & ${n=3}$ & ${n=17}$\\
Joint velocity sensor ($\mathbb{R}^{n}$)  & ${n=0}$ & ${n=2}$ & ${n=10}$ & ${n=8}$ & ${n=0}$ & ${n=3}$ & ${n=17}$\\
Touch sensor ($\mathbb{R}^{n}$) & ${n=0}$ & ${n=4}$ & ${n=2}$ & ${n=8}$ & ${n=0}$ & ${n=1}$ & ${n=2}$ \\
\bottomrule
\end{tabular}
\label{tab:internal_state_space}
\end{center}
\end{table}

\begin{table}[h]
\vskip 0.15in
\caption{The external state space components of different test suites environments.}
\begin{center}
\begin{tabular}{c|ccc}
\toprule
\textbf{External State Space} &  Goal-Hazard & 3D-Goal-Hazard & Goal-Pillar\\
\hline\\[-0.8em]
Goal Compass ($\mathbb{R}^{3}$) & \checkmark & \checkmark & \checkmark\\
Goal Lidar ($\mathbb{R}^{16}$) & \checkmark & $\times$ & \checkmark\\
3D Goal Lidar ($\mathbb{R}^{60}$) & $\times$ & \checkmark & $\times$\\
Hazard Lidar ($\mathbb{R}^{16}$) & \checkmark & $\times$ & $\times$\\
3D Hazard Lidar ($\mathbb{R}^{60}$) & $\times$ & \checkmark & $\times$\\
Pillar Lidar ($\mathbb{R}^{16}$) & $\times$ & $\times$ & \checkmark\\
Vase Lidar ($\mathbb{R}^{16}$) & \checkmark & $\times$ & \checkmark\\
Gremlin Lidar ($\mathbb{R}^{16}$) & \checkmark & $\times$ & \checkmark\\
\bottomrule
\end{tabular}
\label{tab:external_state_space}
\end{center}
\end{table}

\paragraph{Control Space}
For all the experiments, the control space of all robots are continuous, and linearly scaled to [-1, +1].

\subsection{Policy Settings}
The hyper-parameters used in our experiments are listed in \Cref{tab:policy_setting} as default.

Our experiments use separate multi-layer perception with ${tanh}$ activations for the policy network, value network and cost network. Each network consists of two hidden layers of size (64,64). All of the networks are trained using $Adam$ optimizer with learning rate of 0.01.

We apply an on-policy framework in our experiments. During each epoch the agent interact $B$ times with the environment and then perform a policy update based on the experience collected from the current epoch. The maximum length of the trajectory is set to 1000 and the total epoch number $N$ is set to 200 as default.

The policy update step is based on the scheme of TRPO, which performs up to 100 steps of backtracking with a coefficient of 0.8 for line searching.

For all experiments, we use a discount factor of $\gamma = 0.99$, an advantage discount factor $\lambda =0.95$, and a KL-divergence step size of $\delta_{KL} = 0.02$.

For experiments which consider cost constraints we adopt a target cost $\delta_{c} = 0.0$ to pursue a zero-violation policy.

Other unique hyper-parameters for each algorithms are hand-tuned to attain reasonable performance. 

Each model is trained on a server with a 48-core Intel(R) Xeon(R) Silver 6426Y CPU @ 2.5.GHz, Nvidia RTX A6000 GPU with 48GB memory, and Ubuntu 22.04.

For all tasks, we train each model for 6e6 steps which takes around seven hours.

\begin{sidewaystable}
\vskip 0.15in
\caption{Important hyper-parameters of different algorithms in our experiments}
\begin{center}
\resizebox{\textwidth}{!}{%
\begin{tabular}{lr|cccccccccccc}
\toprule
\textbf{Policy Parameter} & & TRPO & TRPO-Lagrangian & TRPO-SL [18' Dalal]& TRPO-USL & TRPO-IPO & TRPO-FAC & CPO & PCPO &  SCPO & TRO-CVaR & CPO-CVaR & ASCPO \\
\hline\\[-1.0em]
Epochs & $N$ & 200 & 200 & 200 & 200 & 200 & 200 & 200 & 200 & 200 & 200 & 200 & 200\\
Steps per epoch & $B$& 30000 & 30000 & 30000 & 30000 & 30000 & 30000 & 30000 & 30000 & 30000 & 30000 & 30000 & 30000\\
Maximum length of trajectory & $L$ & 1000 & 1000 & 1000 & 1000 & 1000 & 1000 & 1000 & 1000 & 1000 & 1000 & 1000 & 1000\\
Policy network hidden layers & & (64, 64) & (64, 64) & (64, 64) & (64, 64) & (64, 64) & (64, 64) & (64, 64) & (64, 64) & (64, 64) & (64, 64) & (64, 64) & (64, 64)\\
Discount factor  &  $\gamma$ & 0.99 & 0.99 & 0.99 & 0.99 & 0.99 & 0.99 & 0.99 & 0.99 & 0.99 & 0.99 & 0.99 & 0.99\\
Advantage discount factor  & $\lambda$ & 0.97 & 0.97 & 0.97 & 0.97 & 0.97 & 0.97 & 0.97 & 0.97 & 0.97 & 0.97 & 0.97 & 0.97\\
TRPO backtracking steps & &100 &100 &100 &100 &100 &100 &100 & - &100 &100 &100 &100\\
TRPO backtracking coefficient & &0.8 &0.8 &0.8 &0.8 &0.8 &0.8 &0.8 & - &0.8 &0.8 &0.8 &0.8\\
Target KL & $\delta_{KL}$& 0.02 & 0.02 & 0.02 & 0.02 & 0.02 & 0.02 & 0.02 & 0.02 & 0.02 & 0.02 & 0.02 & 0.02\\

Value network hidden layers & & (64, 64) & (64, 64) & (64, 64) & (64, 64) & (64, 64) & (64, 64) & (64, 64) & (64, 64) & (64, 64) & (64, 64) & (64, 64) & (64, 64)\\
Value network iteration & & 80 & 80 & 80 & 80 & 80 & 80 & 80 & 80 & 80 & 80 & 80 & 80\\
Value network optimizer & & Adam & Adam & Adam & Adam & Adam & Adam & Adam & Adam & Adam & Adam & Adam & Adam\\
Value learning rate & & 0.001 & 0.001 & 0.001 & 0.001 & 0.001 & 0.001 & 0.001 & 0.001 & 0.001 & 0.001 & 0.001 & 0.001\\

Cost network hidden layers & & - & (64, 64) & (64, 64) & (64, 64) & - & (64, 64) & (64, 64) & (64, 64) & (64, 64) & (64, 64) & (64, 64) & (64, 64)\\
Cost network iteration & & - & 80 & 80 & 80 & - & 80 & 80 & 80 & 80 & 80 & 80 & 80\\
Cost network optimizer & & - & Adam & Adam & Adam & - & Adam & Adam & Adam & Adam & Adam & Adam & Adam\\
Cost learning rate & & - & 0.001 & 0.001 & 0.001 & - & 0.001 & 0.001 & 0.001 & 0.001 & 0.001 & 0.001 & 0.001\\
Target Cost & $\delta_{c}$& - & 0.0 &  0.0 &  0.0 &  0.0 & 0.0 & 0.0 & 0.0 & 0.0 & 0.0 & 0.0 & 0.0\\

Lagrangian optimizer & & - & - & - & - & - & Adam & - & - & - & - & - & -\\
Lagrangian learning rate & & - & 0.005 & - & - & - & 0.0001 & - & - & - & - & - & -\\
USL correction iteration & & - & - & - & 20 & - & - & - & - & - & - & - & -\\
USL correction rate & & - & - & - & 0.05 & - & - & - & - & - & - & - & -\\
Warmup ratio & & - & - & 1/3 & 1/3 & - & - & - & - & - & - & - & -\\
IPO parameter  & ${t}$ & - & - & - & - & 0.01 & - & - & - & - & - & - & -\\
Cost reduction & & - & - & - & - & - & - & 0.0 & - & 0.0 & - & 0.0 & 0.0\\
Probability factor & k & - & - & - & - & - & - & - & - & - & - & - & 7.0\\
\bottomrule
\end{tabular}
}
\label{tab:policy_setting}
\end{center}
\end{sidewaystable}

\subsection{Metrics Comparison}
\label{sec:metrics}
In \Cref{tab: point_hazard,tab: point_pillar,tab: point_ghost,tab: swimmer_hazard,tab: drone_3Dhazard,tab: ant_walker_hazard}, we report all the $19$ results of our test suites by three metrics:
\begin{itemize}
    \item The average episode return $J_r$.
    \item The average episodic sum of costs $M_c$.
    \item The average state-wise cost over the entirety of training $\rho_c$.
\end{itemize}
All of the three metrics were obtained from the final epoch after convergence. Each metric was averaged over two random seed.

The learning curves of all experiments are shown in \Cref{fig:exp-point-hazard,fig:exp-point-pillar,fig:exp-point-ghost,fig:exp-swimmer-hazard,fig:exp-drone-hazard,fig:exp-ant-walker-hazard}. 

\subsection{Ablation study on large penalty for infractions}
\label{sec: lagrandian ablation}

\begin{wrapfigure}{r}{0.6\textwidth}
    \centering
    \begin{subfigure}[b]{0.19\textwidth}
        \begin{subfigure}[t]{1.00\textwidth}
        \raisebox{-\height}{\includegraphics[width=\textwidth]{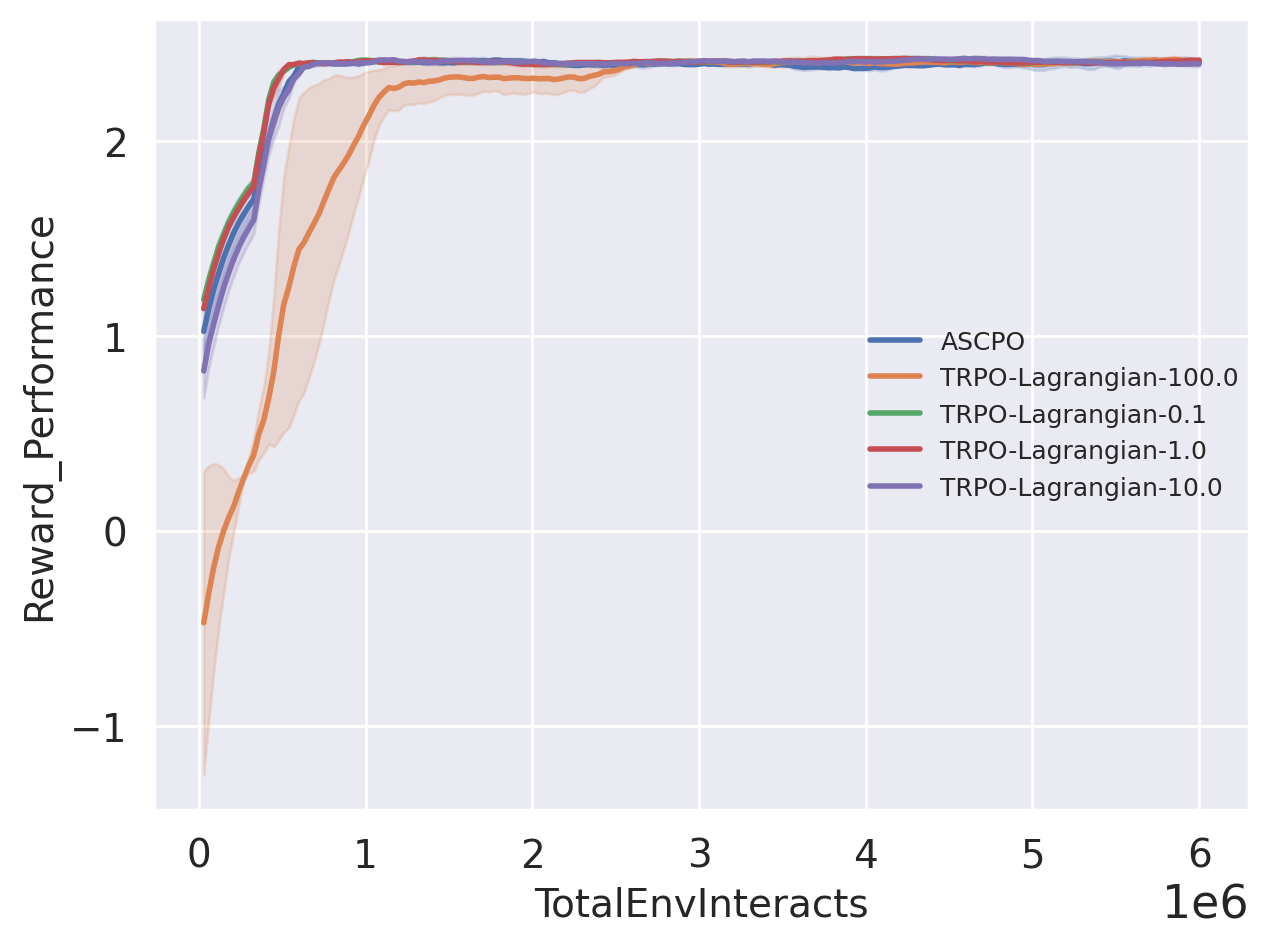}}
        \end{subfigure}
    \end{subfigure}
    \hfill
    \begin{subfigure}[b]{0.19\textwidth}
        \begin{subfigure}[t]{1.00\textwidth}
        \raisebox{-\height}{\includegraphics[width=\textwidth]{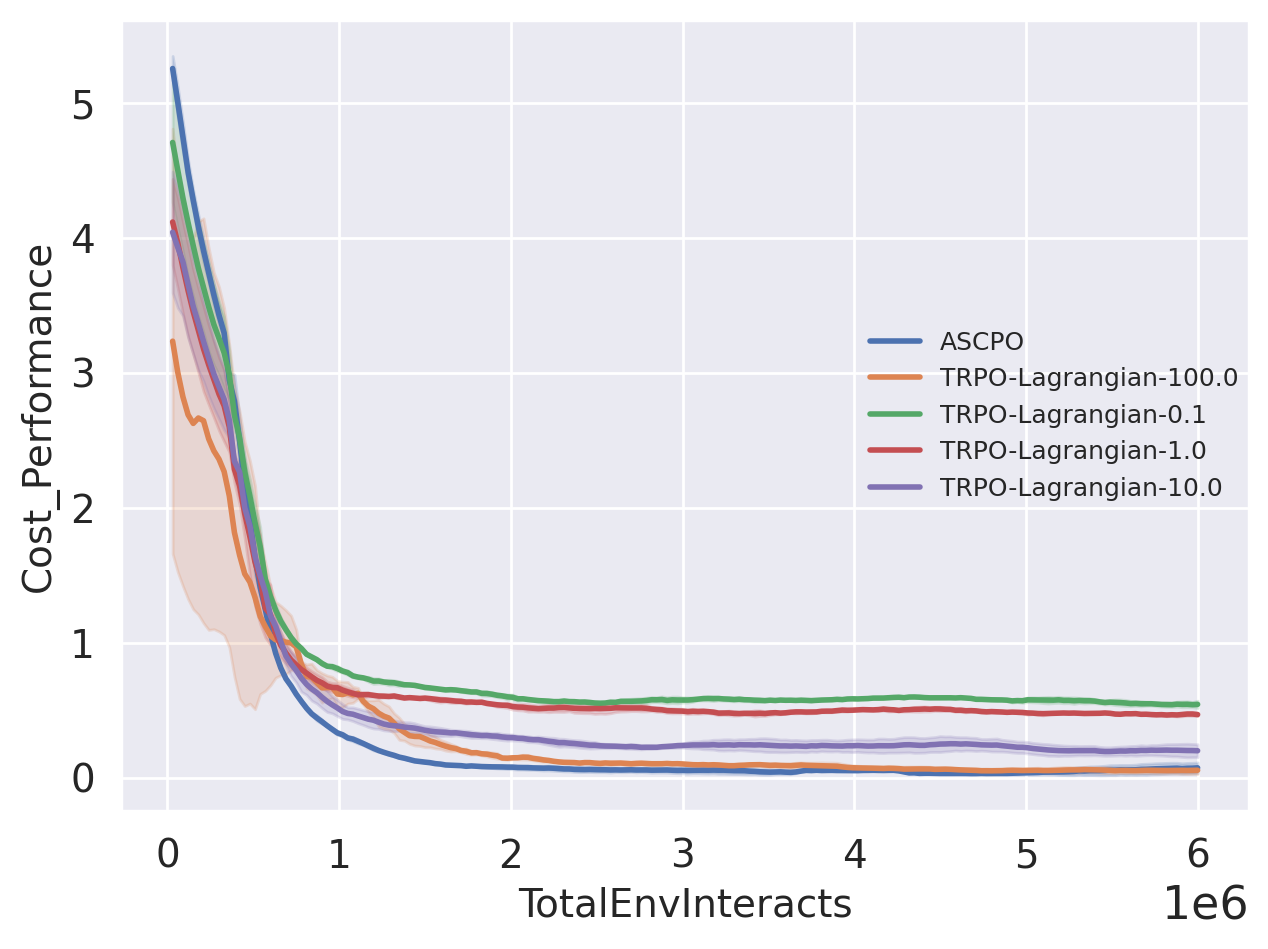}}
        \end{subfigure}
    \end{subfigure} 
    \hfill
    \begin{subfigure}[b]{0.19\textwidth}
        \begin{subfigure}[t]{1.00\textwidth}
        \raisebox{-\height}{\includegraphics[width=\textwidth]{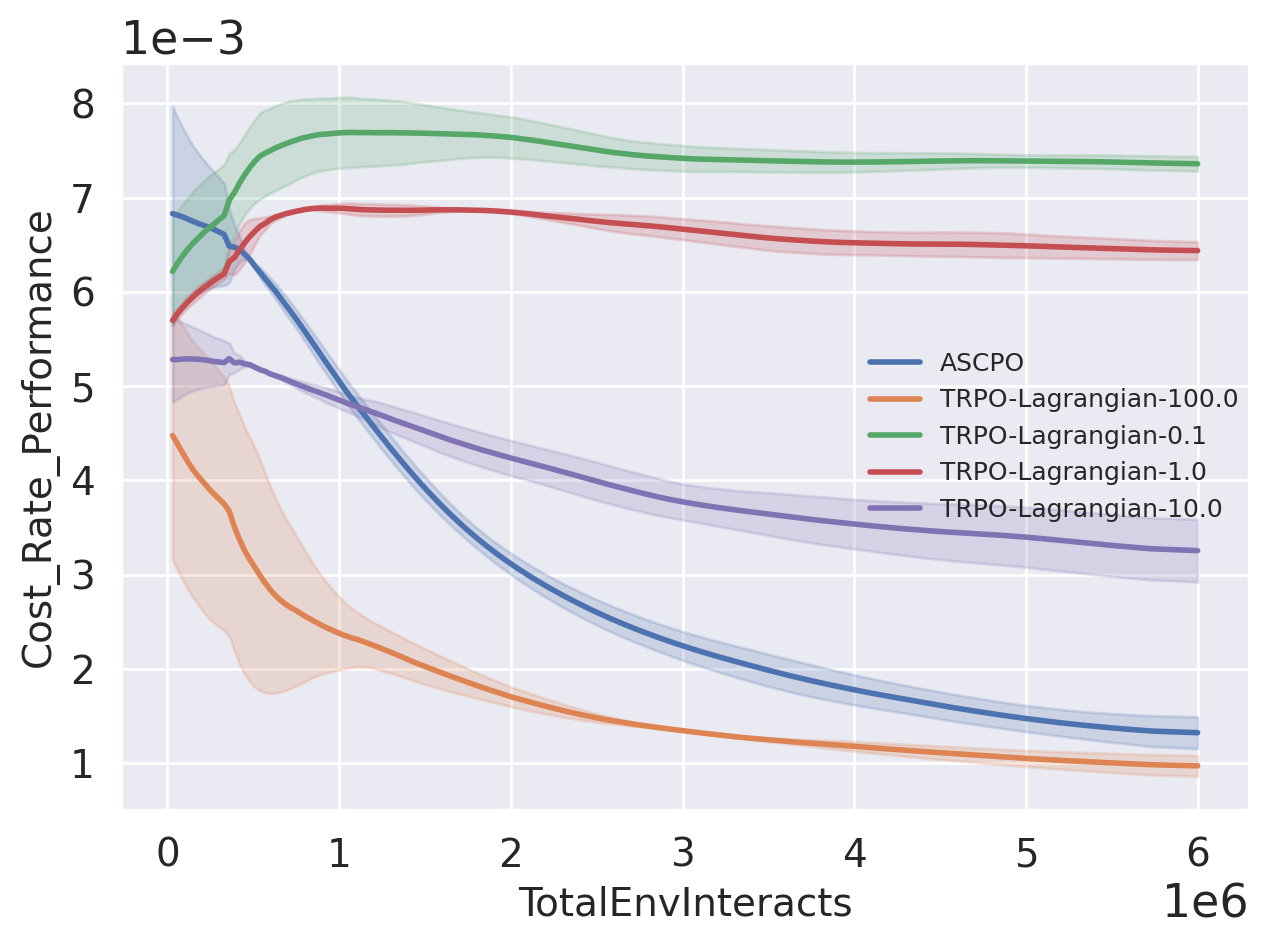}}
        \end{subfigure}
    \end{subfigure}
    \caption{TRPO-Lagrangian method ablation study with Point-8-Hazard} 
    \label{fig: lag ablation}
    \vspace{-10pt}
\end{wrapfigure}

In our experiments employing the Lagrangian method, we utilized an adaptive penalty coefficient. Consequently, we augmented this coefficient by a factor denoted as $\lambda$ to explore the trade-off between optimizing rewards and satisfying constraints. We designated these experiments as TRPO-Lagrangian-{$\lambda$} and juxtaposed them with ASCPO in \Cref{fig: lag ablation}. Observably, as $\lambda$ increases, both the cost rate and cost value exhibit a notable decrease in the Lagrangian method. However, this reduction in convergence speed of rewards is concurrently observed. Conversely, ASCPO demonstrates the swiftest convergence alongside the most favorable convergence values, achieving notable progress in both reward convergence and cost reduction. Moreover, ASCPO performs comparably, if not superiorly, in terms of cost rate. These results underscore the inadequacy of simplistic coefficient adjustments within the Lagrangian method when compared to the efficacy of our algorithm.

\begin{table}[p]
\caption{Metrics of three \textbf{Point-Hazard} environments obtained from the final epoch.}
\begin{subtable}[tb]{4cm}
\caption{Point-1-Hazard}
\begin{center}
\vskip -0.1in
\resizebox{!}{1.3cm}{
\begin{tabular}{c|ccc}
\toprule
\textbf{Algorithm} & $\bar{J}_r$ & $\bar{M}_c$ & $\bar{\rho}_c$\\
\hline\\[-0.8em]
TRPO & 2.3840 & 0.0697 & 0.0010 \\ 
TRPO-Lagrangian & 2.3777 & 0.0596 & 0.0009 \\ 
TRPO-SL & 2.3614 & 0.0480 & 0.0009 \\ 
TRPO-USL & 2.3859 & 0.0894 & 0.0007 \\ 
TRPO-IPO & 2.3580 & 0.0719 & 0.0009 \\ 
TRPO-FAC & 2.3646 & 0.0516 & 0.0007 \\ 
CPO & \textbf{2.3937} & 0.0315 & 0.0005 \\ 
PCPO & 2.3831 & 0.0565 & 0.0006 \\ 
SCPO & 1.2580 & 0.0021 & \textbf{0.0001} \\ 
TRPO-CVaR & 2.3887 & 0.0764 & 0.0009 \\ 
CPO-CVaR & 2.3459 & 0.0539 & 0.0008 \\ 
ASCPO & 2.3415 & \textbf{0.0010} & 0.0002 \\ 

\bottomrule
\end{tabular}}
\label{tab: point_hazard_1}
\end{center}
\end{subtable}
\hfill
\begin{subtable}[tb]{4cm}
\caption{Point-4-Hazard}
\vspace{-0.1in}
\begin{center}
\resizebox{!}{1.3cm}{
\begin{tabular}{c|ccc}
\toprule
\textbf{Algorithm} & $\bar{J}_r$ & $\bar{M}_c$ & $\bar{\rho}_c$\\
\hline\\[-0.8em]
TRPO & 2.3397 & 0.2697 & 0.0038 \\ 
TRPO-Lagrangian & 2.3797 & 0.2197 & 0.0033 \\ 
TRPO-SL & 2.4054 & 0.4020 & 0.0035 \\ 
TRPO-USL & 2.4083 & 0.2428 & 0.0029 \\ 
TRPO-IPO & 2.3807 & 0.2349 & 0.0034 \\ 
TRPO-FAC & 2.3806 & 0.1225 & 0.0025 \\ 
CPO & 2.4243 & 0.0924 & 0.0019 \\ 
PCPO & 2.4123 & 0.1107 & 0.0023 \\ 
SCPO & 2.3686 & 0.0313 & 0.0011 \\ 
TRPO-CVaR & 2.4179 & 0.3165 & 0.0036 \\ 
CPO-CVaR & 2.3999 & 0.2846 & 0.0036 \\ 
ASCPO & \textbf{2.4272} & \textbf{0.0290} & \textbf{0.0007} \\ 
\bottomrule
\end{tabular}}
\label{tab: point_hazard_4}
\end{center}
\end{subtable}
\hfill
\begin{subtable}[tb]{4cm}
\caption{Point-8-Hazard}
\vspace{-0.1in}
\begin{center}
\resizebox{!}{1.3cm}{
\begin{tabular}{c|ccc}
\toprule
\textbf{Algorithm} & $\bar{J}_r$ & $\bar{M}_c$ & $\bar{\rho}_c$\\
\hline\\[-0.8em]
TRPO & 2.4603 & 0.6222 & 0.0073 \\ 
TRPO-Lagrangian & 2.4067 & 0.4531 & 0.0064 \\ 
TRPO-SL & 2.4465 & 0.9176 & 0.0069 \\ 
TRPO-USL & 2.3686 & 0.5688 & 0.0067 \\ 
TRPO-IPO & 2.4138 & 0.5673 & 0.0072 \\ 
TRPO-FAC & 2.3708 & 0.2330 & 0.0049 \\ 
CPO & 2.4412 & 0.1961 & 0.0036 \\ 
PCPO & 2.4205 & 0.3953 & 0.0047 \\ 
SCPO & 2.3666 & 0.0919 & 0.0016 \\ 
TRPO-CVaR & 2.3783 & 0.5255 & 0.0075 \\ 
CPO-CVaR & 2.3685 & 0.4479 & 0.0069 \\ 
ASCPO & \textbf{2.4785} & \textbf{0.0413} & \textbf{0.0013} \\ 
\bottomrule
\end{tabular}}
\label{tab: point_hazard_8}
\end{center}
\end{subtable}
\hfill
\label{tab: point_hazard}
\end{table}

\begin{table}[p]
\caption{Metrics of three \textbf{Point-Pillar} experiments obtained from the final epoch.}
\begin{subtable}[tb]{4cm}
\caption{Point-1-Pillar}
\vspace{-0.1in}
\begin{center}
\resizebox{!}{1.3cm}{
\begin{tabular}{c|ccc}
\toprule
\textbf{Algorithm} & $\bar{J}_r$ & $\bar{M}_c$ & $\bar{\rho}_c$\\
\hline\\[-0.8em]
TRPO & 2.3742 & 0.1181 & 0.0015 \\ 
TRPO-Lagrangian & 2.3823 & 0.0694 & 0.0012 \\ 
TRPO-SL & 2.3721 & 0.0319 & 0.0008 \\ 
TRPO-USL & 2.3905 & 0.1753 & 0.0011 \\ 
TRPO-IPO & 2.3702 & 0.1594 & 0.0013 \\ 
TRPO-FAC & 2.3745 & 0.0595 & 0.0008 \\ 
CPO & 2.3779 & 0.0657 & 0.0011 \\ 
PCPO & 2.3634 & 0.0809 & 0.0012 \\ 
SCPO & 2.4085 & 0.0523 & \textbf{0.0006} \\ 
TRPO-CVaR & 2.3648 & 0.0724 & 0.0010 \\ 
CPO-CVaR & 2.3650 & 0.1990 & 0.0012 \\ 
ASCPO & \textbf{2.4202} & \textbf{0.0272} & 0.0007 \\ 
\bottomrule
\end{tabular}}
\label{tab: point_pillar_1}
\end{center}
\end{subtable}
\hfill
\begin{subtable}[tb]{4cm}
\caption{Point-4-Pillar}
\vspace{-0.1in}
\begin{center}
\resizebox{!}{1.3cm}{
\begin{tabular}{c|ccc}
\toprule
\textbf{Algorithm} & $\bar{J}_r$ & $\bar{M}_c$ & $\bar{\rho}_c$\\
\hline\\[-0.8em]
TRPO & 2.4169 & 0.4133 & 0.0057 \\ 
TRPO-Lagrangian & 2.3515 & 0.3376 & 0.0061 \\ 
TRPO-SL & 2.3071 & 0.2210 & \textbf{0.0027} \\ 
TRPO-USL & 2.4058 & 0.5679 & 0.0044 \\ 
TRPO-IPO & 2.4075 & 0.2745 & 0.0051 \\ 
TRPO-FAC & 2.3883 & \textbf{0.1982} & 0.0032 \\ 
CPO & 2.4012 & 0.2964 & 0.0060 \\ 
PCPO & 2.3992 & 0.3868 & 0.0065 \\ 
SCPO & 2.3918 & 0.2512 & 0.0042 \\ 
TRPO-CVaR & 2.3960 & 0.3610 & 0.0063 \\ 
CPO-CVaR & 2.3970 & 0.3109 & 0.0056 \\ 
ASCPO & \textbf{2.4405} & 0.2341 & 0.0041 \\ 
\bottomrule
\end{tabular}}
\label{tab: point_pillar_4}
\end{center}
\end{subtable}
\hfill
\begin{subtable}[tb]{4cm}
\caption{Point-8-Pillar}
\vspace{-0.1in}
\begin{center}
\resizebox{!}{1.3cm}{
\begin{tabular}{c|ccc}
\toprule
\textbf{Algorithm} & $\bar{J}_r$ & $\bar{M}_c$ & $\bar{\rho}_c$\\
\hline\\[-0.8em]
TRPO & 2.3295 & 2.3581 & 0.0216 \\ 
TRPO-Lagrangian & 2.3494 & 0.5958 & 0.0114 \\ 
TRPO-SL & 2.3756 & 0.5944 & \textbf{0.0066} \\ 
TRPO-USL & 2.3309 & 0.7557 & 0.0133 \\ 
TRPO-IPO & 2.3824 & 1.1487 & 0.0140 \\ 
TRPO-FAC & 2.4024 & \textbf{0.3398} & 0.0082 \\ 
CPO & 2.4128 & 0.8449 & 0.0157 \\ 
PCPO & 2.4003 & 4.6241 & 0.0193 \\ 
SCPO & 2.3897 & 3.2559 & 0.0129 \\ 
TRPO-CVaR & 2.4081 & 1.1789 & 0.0166 \\ 
CPO-CVaR & 2.3700 & 1.1933 & 0.0160 \\ 
ASCPO & \textbf{2.4202} & 0.4341 & 0.0106 \\ 
\bottomrule
\end{tabular}}
\label{tab: point_pillar_8}
\end{center}
\end{subtable}
\hfill
\label{tab: point_pillar}
\end{table}

\begin{table}[p]
\caption{Metrics of three \textbf{Point-Ghost} experiments obtained from the final epoch.}
\begin{subtable}[tb]{4cm}
\caption{Point-1-Ghost}
\vspace{-0.1in}
\begin{center}
\resizebox{!}{1.3cm}{
\begin{tabular}{c|ccc}
\toprule
\textbf{Algorithm} & $\bar{J}_r$ & $\bar{M}_c$ & $\bar{\rho}_c$\\
\hline\\[-0.8em]
TRPO & 2.4178 & 0.0854 & 0.0008 \\ 
TRPO-Lagrangian & 2.3868 & 0.0396 & 0.0008 \\ 
TRPO-SL & 2.3880 & 0.0256 & 0.0006 \\ 
TRPO-USL & 2.3867 & 0.0487 & 0.0006 \\ 
TRPO-IPO & 2.4201 & 0.0737 & 0.0007 \\ 
TRPO-FAC & 2.3861 & 0.0438 & 0.0006 \\ 
CPO & 2.3718 & 0.0220 & 0.0005 \\ 
PCPO & \textbf{2.4337} & 0.0579 & 0.0006 \\ 
SCPO & 1.4640 & 0.0000 & 0.0001 \\ 
TRPO-CVaR & 2.3739 & 0.0716 & 0.0009 \\ 
CPO-CVaR & 2.4110 & 0.0653 & 0.0007 \\ 
ASCPO & 2.4014 & \textbf{0.0000} & \textbf{0.0001} \\ 
\bottomrule
\end{tabular}}
\label{tab: point_ghost_1}
\end{center}
\end{subtable}
\hfill
\begin{subtable}[tb]{4cm}
\caption{Point-4-Ghost}
\vspace{-0.1in}
\begin{center}
\resizebox{!}{1.3cm}{
\begin{tabular}{c|ccc}
\toprule
\textbf{Algorithm} & $\bar{J}_r$ & $\bar{M}_c$ & $\bar{\rho}_c$\\
\hline\\[-0.8em]
TRPO & 2.4035 & 0.2809 & 0.0035 \\ 
TRPO-Lagrangian & 2.3938 & 0.2694 & 0.0034 \\ 
TRPO-SL & 2.3657 & 0.4658 & 0.0036 \\ 
TRPO-USL & 2.4127 & 0.1475 & 0.0027 \\ 
TRPO-IPO & 2.3940 & 0.3212 & 0.0031 \\ 
TRPO-FAC & 2.4002 & 0.1206 & 0.0022 \\ 
CPO & 2.4004 & 0.0962 & 0.0017 \\ 
PCPO & 2.4123 & 0.1101 & 0.0020 \\ 
SCPO & 2.2944 & 0.0377 & 0.0009 \\ 
TRPO-CVaR & 2.4487 & 0.3079 & 0.0034 \\ 
CPO-CVaR & 2.3877 & 0.2340 & 0.0032 \\ 
ASCPO & \textbf{2.4503} & \textbf{0.0262} & \textbf{0.0006} \\ 
\bottomrule
\end{tabular}}
\label{tab: point_ghost_4}
\end{center}
\end{subtable}
\hfill
\begin{subtable}[tb]{4cm}
\caption{Point-8-Ghost}
\vspace{-0.1in}
\begin{center}
\resizebox{!}{1.3cm}{
\begin{tabular}{c|ccc}
\toprule
\textbf{Algorithm} & $\bar{J}_r$ & $\bar{M}_c$ & $\bar{\rho}_c$\\
\hline\\[-0.8em]
TRPO & 2.4178 & 0.5109 & 0.0069 \\ 
TRPO-Lagrangian & 2.3985 & 0.4600 & 0.0063 \\ 
TRPO-SL & 2.3613 & 0.8013 & 0.0067 \\ 
TRPO-USL & 2.4162 & 0.4935 & 0.0061 \\ 
TRPO-IPO & 2.4118 & 0.5494 & 0.0067 \\ 
TRPO-FAC & 2.3956 & 0.2185 & 0.0046 \\ 
CPO & 2.4170 & 0.1653 & 0.0033 \\ 
PCPO & 2.4043 & 0.2444 & 0.0046 \\ 
SCPO & 2.3784 & 0.0904 & 0.0022 \\ 
TRPO-CVaR & \textbf{2.4333} & 0.5494 & 0.0070 \\ 
CPO-CVaR & 2.4280 & 0.5318 & 0.0066 \\ 
ASCPO & 2.4239 & \textbf{0.0321} & \textbf{0.0013} \\ 
\bottomrule
\end{tabular}}
\label{tab: point_ghost_8}
\end{center}
\end{subtable}
\hfill
\label{tab: point_ghost}
\end{table}

\begin{table}[p]
\caption{Metrics of three \textbf{Swimmer-Hazard} experiments obtained from the final epoch.}
\begin{subtable}[tb]{4cm}
\caption{Swimmer-1-Hazard}
\vspace{-0.1in}
\begin{center}
\resizebox{!}{1.3cm}{
\begin{tabular}{c|ccc}
\toprule
\textbf{Algorithm} & $\bar{J}_r$ & $\bar{M}_c$ & $\bar{\rho}_c$\\
\hline\\[-0.8em]
TRPO & \textbf{2.4834} & 0.0908 & 0.0013 \\ 
TRPO-Lagrangian & 2.3957 & 0.0813 & 0.0013 \\ 
TRPO-SL & 2.3730 & 0.0699 & 0.0021 \\ 
TRPO-USL & 2.4425 & 0.0605 & 0.0010 \\ 
TRPO-IPO & 2.4667 & 0.0777 & 0.0012 \\ 
TRPO-FAC & 2.4553 & 0.0708 & 0.0011 \\ 
CPO & 2.4484 & 0.0907 & 0.0011 \\ 
PCPO & 2.4348 & 0.0957 & 0.0012 \\ 
SCPO & 2.3924 & 0.0885 & 0.0011 \\ 
TRPO-CVaR & 2.4332 & 0.0759 & 0.0012 \\ 
CPO-CVaR & 2.4276 & 0.0847 & 0.0012 \\ 
ASCPO & 2.4531 & \textbf{0.0000} & \textbf{0.0002} \\ 
\bottomrule
\end{tabular}}
\label{tab: swimmer_hazard_1}
\end{center}
\end{subtable}
\hfill
\begin{subtable}[tb]{4cm}
\caption{Swimmer-4-Hazard}
\vspace{-0.1in}
\begin{center}
\resizebox{!}{1.3cm}{
\begin{tabular}{c|ccc}
\toprule
\textbf{Algorithm} & $\bar{J}_r$ & $\bar{M}_c$ & $\bar{\rho}_c$\\
\hline\\[-0.8em]
TRPO & 2.4441 & 0.3554 & 0.0051 \\ 
TRPO-Lagrangian & \textbf{2.4568} & 0.3284 & 0.0050 \\ 
TRPO-SL & 2.2038 & 0.9845 & 0.0068 \\ 
TRPO-USL & 2.4148 & 0.3786 & 0.0046 \\ 
TRPO-IPO & 2.4327 & 0.3464 & 0.0049 \\ 
TRPO-FAC & 2.4413 & 0.3097 & 0.0047 \\ 
CPO & 2.4172 & 0.3648 & 0.0045 \\ 
PCPO & 2.3920 & 0.2950 & 0.0047 \\ 
SCPO & 2.4096 & 0.3636 & 0.0043 \\ 
TRPO-CVaR & 2.4183 & 0.2984 & 0.0051 \\ 
CPO-CVaR & 2.4208 & 0.3456 & 0.0050 \\ 
ASCPO & 2.4233 & \textbf{0.2933} & \textbf{0.0040} \\ 
\bottomrule
\end{tabular}}
\label{tab: swimmer_hazard_4}
\end{center}
\end{subtable}
\hfill
\begin{subtable}[tb]{4cm}
\caption{Swimmer-8-Hazard}
\vspace{-0.1in}
\begin{center}
\resizebox{!}{1.3cm}{
\begin{tabular}{c|ccc}
\toprule
\textbf{Algorithm} & $\bar{J}_r$ & $\bar{M}_c$ & $\bar{\rho}_c$\\
\hline\\[-0.8em]
TRPO & 2.4218 & 0.7754 & 0.0103 \\ 
TRPO-Lagrangian & 2.4559 & 0.6977 & 0.0099 \\ 
TRPO-SL & 2.4321 & 2.3075 & 0.0111 \\ 
TRPO-USL & 2.4297 & \textbf{0.5784} & 0.0093 \\ 
TRPO-IPO & 2.4396 & 0.6749 & 0.0098 \\ 
TRPO-FAC & 2.4042 & 0.6703 & 0.0096 \\ 
CPO & 2.4433 & 0.8106 & 0.0088 \\ 
PCPO & 2.4388 & 0.6881 & 0.0096 \\ 
SCPO & 2.4581 & 1.0087 & 0.0089 \\ 
TRPO-CVaR & 2.4535 & 0.7250 & 0.0102 \\ 
CPO-CVaR & 2.4620 & 0.7371 & 0.0103 \\ 
ASCPO & \textbf{2.4889} & 0.7629 & \textbf{0.00751} \\ 
\bottomrule
\end{tabular}}
\label{tab: swimmer_hazard_8}
\end{center}
\end{subtable}
\hfill
\label{tab: swimmer_hazard}
\end{table}

\begin{table}[p]
\caption{Metrics of three \textbf{Drone-3DHazard} experiments obtained from the final epoch.}
\begin{subtable}[tb]{4cm}
\caption{Drone-1-3DHazard}
\vspace{-0.1in}
\begin{center}
\resizebox{!}{1.3cm}{
\begin{tabular}{c|ccc}
\toprule
\textbf{Algorithm} & $\bar{J}_r$ & $\bar{M}_c$ & $\bar{\rho}_c$\\
\hline\\[-0.8em]
TRPO & 2.4316 & 0.0504 & 0.0002 \\ 
TRPO-Lagrangian & 2.4052 & 0.0125 & 0.0001 \\ 
TRPO-SL & 2.3553 & 0.0227 & 0.0000 \\ 
TRPO-USL & 2.4022 & 0.0157 & 0.0001 \\ 
TRPO-IPO & 2.3784 & 0.0148 & 0.0002 \\ 
TRPO-FAC & 2.4197 & 0.0126 & 0.0001 \\ 
CPO & \textbf{2.4345} & 0.0044 & 0.0002 \\ 
PCPO & 2.3066 & 0.0556 & 0.0001 \\ 
SCPO & 2.3306 & 0.0104 & 0.0001 \\ 
TRPO-CVaR & 2.4244 & 0.0183 & 0.0001 \\ 
CPO-CVaR & 2.3877 & 0.0172 & 0.0001 \\ 
ASCPO & 2.4181 & \textbf{0.0039} & \textbf{0.0000} \\ 
\bottomrule
\end{tabular}}
\label{tab: drone_3Dhazard_1}
\end{center}
\end{subtable}
\hfill
\begin{subtable}[tb]{4cm}
\caption{Drone-4-3DHazard}
\vspace{-0.1in}
\begin{center}
\resizebox{!}{1.3cm}{
\begin{tabular}{c|ccc}
\toprule
\textbf{Algorithm} & $\bar{J}_r$ & $\bar{M}_c$ & $\bar{\rho}_c$\\
\hline\\[-0.8em]
TRPO & 2.4336 & 0.0749 & 0.0005 \\ 
TRPO-Lagrangian & 2.4330 & 0.0644 & 0.0006 \\ 
TRPO-SL & 2.3883 & 0.0059 & 0.0002 \\ 
TRPO-USL & 2.4322 & 0.0545 & 0.0004 \\ 
TRPO-IPO & 2.3787 & 0.0497 & 0.0006 \\ 
TRPO-FAC & 2.4153 & 0.0663 & 0.0004 \\ 
CPO & \textbf{2.4602} & 0.0640 & 0.0006 \\ 
PCPO & 2.3981 & 0.0672 & 0.0004 \\ 
SCPO & 2.3430 & 0.0233 & 0.0002 \\ 
TRPO-CVaR & 2.4513 & 0.0700 & 0.0006 \\ 
CPO-CVaR & 2.4458 & 0.0906 & 0.0005 \\ 
ASCPO & 2.3558 & \textbf{0.0059} & \textbf{0.0001} \\ 
\bottomrule
\end{tabular}}
\label{tab: drone_3Dhazard_4}
\end{center}
\end{subtable}
\hfill
\begin{subtable}[tb]{4cm}
\caption{Drone-8-3DHazard}
\vspace{-0.1in}
\begin{center}
\resizebox{!}{1.3cm}{
\begin{tabular}{c|ccc}
\toprule
\textbf{Algorithm} & $\bar{J}_r$ & $\bar{M}_c$ & $\bar{\rho}_c$\\
\hline\\[-0.8em]
TRPO & 2.4727 & 0.2310 & 0.0018 \\ 
TRPO-Lagrangian & 2.4760 & 0.1952 & 0.0013 \\ 
TRPO-SL & 2.3681 & 0.0225 & 0.0005 \\ 
TRPO-USL & 2.4470 & 0.1275 & 0.0010 \\ 
TRPO-IPO & 2.4245 & 0.1355 & 0.0010 \\ 
TRPO-FAC & 2.4172 & 0.0956 & 0.0008 \\ 
CPO & 2.4092 & 0.1018 & 0.0009 \\ 
PCPO & 1.2867 & 0.2227 & \textbf{0.0001} \\ 
SCPO & \textbf{2.4766} & 0.0729 & 0.0005 \\ 
TRPO-CVaR & 2.4068 & 0.1244 & 0.0014 \\ 
CPO-CVaR & 2.4219 & 0.1809 & 0.0010 \\ 
ASCPO & 2.3960 & \textbf{0.0126} & 0.0003 \\ 
\bottomrule
\end{tabular}}
\label{tab: drone_3Dhazard_8}
\end{center}
\end{subtable}
\hfill
\label{tab: drone_3Dhazard}
\end{table}

\begin{table}[p]
\caption{Metrics of \textbf{Arm3-Hazard}, \textbf{Humanoid-Hazard}, \textbf{Ant-Hazard}, \textbf{Walker-Hazard} experiments obtained from the final epoch.}
\centering
\begin{subtable}[tb]{5cm}
\caption{Arm3-8-Hazard}
\vspace{-0.1in}
\begin{center}
\resizebox{!}{1.3cm}{
\begin{tabular}{c|ccc}
\toprule
\textbf{Algorithm} & $\bar{J}_r$ & $\bar{M}_c$ & $\bar{\rho}_c$\\
\hline\\[-0.8em]
TRPO & 2.2280 & 1.0687 & 0.0136 \\ 
TRPO-Lagrangian & 2.2478 & 0.3002 & \textbf{0.0064} \\ 
TRPO-SL & 1.8348 & 6.8884 & 0.0109 \\ 
TRPO-USL & \textbf{2.2777} & 1.0435 & 0.0125 \\ 
TRPO-IPO & 2.2583 & 1.8212 & 0.0168 \\ 
TRPO-FAC & 2.2672 & 0.5234 & 0.0072 \\ 
CPO & 2.2449 & 0.5486 & 0.0084 \\ 
PCPO & 2.2567 & 0.9550 & 0.0099 \\ 
SCPO & 2.2536 & 0.6708 & 0.0100 \\ 
TRPO-CVaR & 2.2496 & 2.0251 & 0.0233 \\ 
CPO-CVaR & 2.2189 & 2.3813 & 0.0248 \\ 
ASCPO & 2.3452 & \textbf{0.2658} & 0.0075 \\ 
\bottomrule
\end{tabular}}
\label{tab: arm3_hazard_8}
\end{center}
\end{subtable}
\begin{subtable}[tb]{5cm}
\caption{Humanoid-8-Hazard}
\vspace{-0.1in}
\begin{center}
\resizebox{!}{1.3cm}{
\begin{tabular}{c|ccc}
\toprule
\textbf{Algorithm} & $\bar{J}_r$ & $\bar{M}_c$ & $\bar{\rho}_c$\\
\hline\\[-0.8em]
TRPO & 2.2280 & 1.0687 & 0.0136 \\ 
TRPO-Lagrangian & 2.2478 & 0.3002 & 0.0064 \\ 
TRPO-SL & 1.8348 & 6.8884 & 0.0109 \\ 
TRPO-USL & 2.2452 & 1.0435 & 0.0125 \\ 
TRPO-IPO & 2.2583 & 1.8212 & 0.0168 \\ 
TRPO-FAC & 2.2672 & 0.5234 & 0.0072 \\ 
CPO & 2.2449 & 0.5486 & 0.0084 \\ 
PCPO & 2.2567 & 0.9550 & 0.0099 \\ 
SCPO & 2.2536 & 0.3354 & 0.0100 \\ 
TRPO-CVaR & 2.2496 & 2.0251 & 0.0233 \\ 
CPO-CVaR & 2.2189 & 2.3813 & 0.0248 \\ 
ASCPO & \textbf{2.2777} & \textbf{0.2658} & \textbf{0.0075} \\ 
\bottomrule
\end{tabular}}
\label{tab: humanoid_hazard_8}
\end{center}
\end{subtable}
\begin{subtable}[tb]{5cm}
\caption{Ant-8-Hazard}
\vspace{-0.1in}
\begin{center}
\resizebox{!}{1.3cm}{
\begin{tabular}{c|ccc}
\toprule
\textbf{Algorithm} & $\bar{J}_r$ & $\bar{M}_c$ & $\bar{\rho}_c$\\
\hline\\[-0.8em]
TRPO & 2.4243 & 0.4039 & 0.0094 \\ 
TRPO-Lagrangian & 2.4266 & 0.3640 & 0.0084 \\ 
TRPO-SL & 2.2924 & 11.5143 & 0.0500 \\ 
TRPO-USL & \textbf{2.4599} & 0.4651 & 0.0093 \\ 
TRPO-IPO & 2.3908 & 2.0326 & 0.0076 \\ 
TRPO-FAC & 2.4479 & 0.3556 & 0.0079 \\ 
CPO & 2.4364 & 0.3279 & 0.0082 \\ 
PCPO & 2.4259 & 0.3630 & 0.0083 \\ 
SCPO & 2.4244 & 0.2980 & 0.0082 \\ 
TRPO-CVaR & 2.4284 & 0.4308 & 0.0098 \\ 
CPO-CVaR & 2.4404 & 0.5410 & 0.0098 \\ 
ASCPO & 2.4207 & \textbf{0.2010} & \textbf{0.0070} \\ 
\bottomrule
\end{tabular}}
\label{tab: ant_hazard_8}
\end{center}
\end{subtable}
\begin{subtable}[tb]{5cm}
\caption{Walker-8-Hazard}
\vspace{-0.1in}
\begin{center}
\resizebox{!}{1.3cm}{
\begin{tabular}{c|ccc}
\toprule
\textbf{Algorithm} & $\bar{J}_r$ & $\bar{M}_c$ & $\bar{\rho}_c$\\
\hline\\[-0.8em]
TRPO & 2.3856 & 0.4224 & 0.0098 \\ 
TRPO-Lagrangian & 2.4049 & 0.3227 & 0.0083 \\ 
TRPO-SL & 2.4113 & 1.8346 & 0.0256 \\ 
TRPO-USL & 2.4390 & 0.4296 & 0.0096 \\ 
TRPO-IPO & 2.4064 & 0.5490 & 0.0094 \\ 
TRPO-FAC & 2.4363 & 0.3403 & 0.0083 \\ 
CPO & \textbf{2.4429} & 0.3221 & 0.0082 \\ 
PCPO & 2.4131 & 0.2849 & 0.0081 \\ 
SCPO & 2.4173 & 0.2761 & 0.0077 \\ 
TRPO-CVaR & 2.4293 & 0.4696 & 0.0115 \\ 
CPO-CVaR & 2.3853 & 0.4909 & 0.0114 \\ 
ASCPO & 2.4090 & \textbf{0.1936} & \textbf{0.0064} \\ 
\bottomrule
\end{tabular}}
\label{tab: walker_hazard_8}
\end{center}
\end{subtable}
\label{tab: ant_walker_hazard}
\end{table}

\begin{figure}[p]
    \centering
    \begin{subfigure}[t]{0.32\textwidth}
    \begin{subfigure}[t]{1.00\textwidth}
        \raisebox{-\height}{\includegraphics[height=0.7\textwidth]{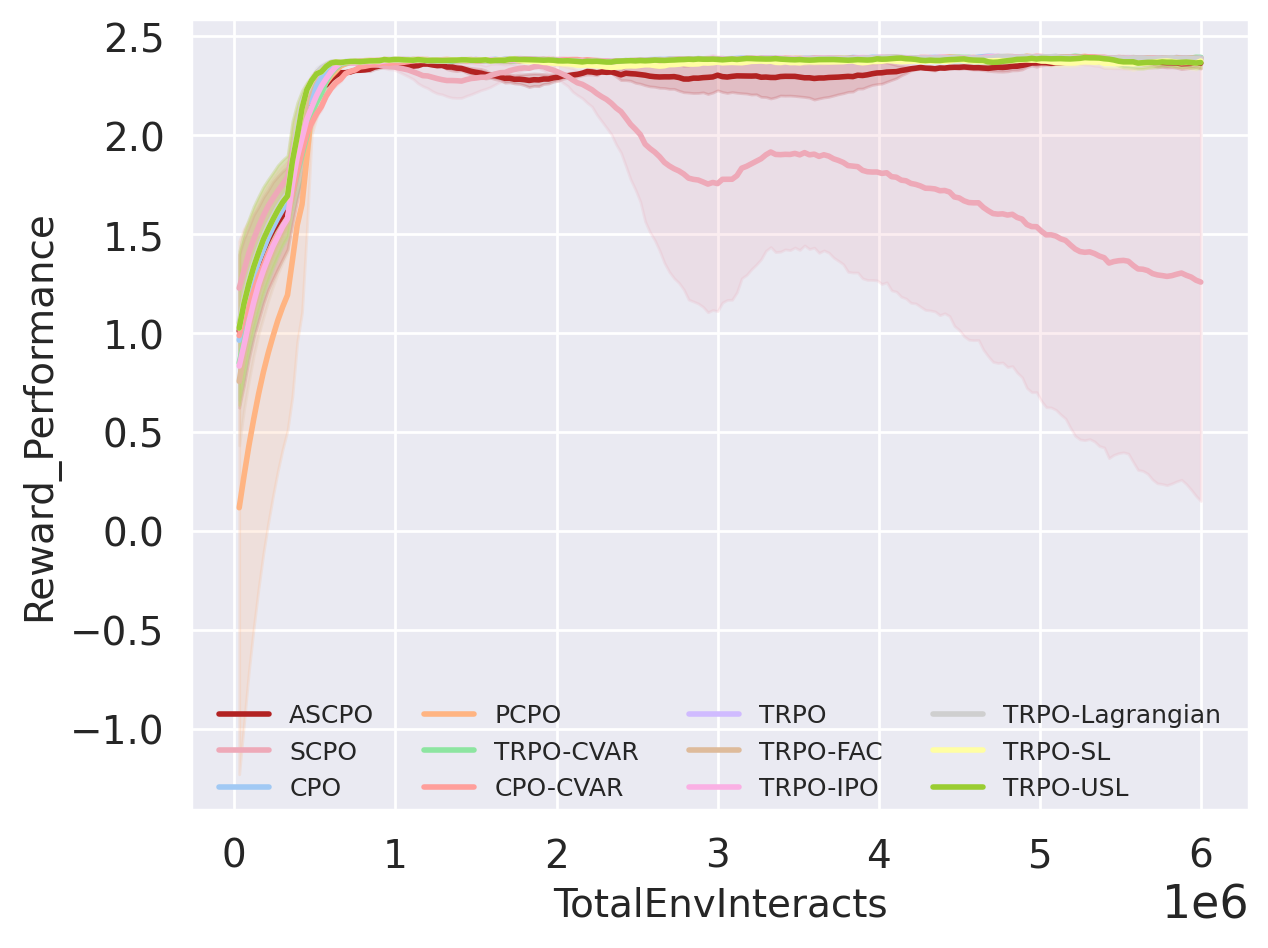}}
        \label{fig:point-hazard-1-Performance}
    \end{subfigure}
    \hfill
    \begin{subfigure}[t]{1.00\textwidth}
        \raisebox{-\height}{\includegraphics[height=0.7\textwidth]{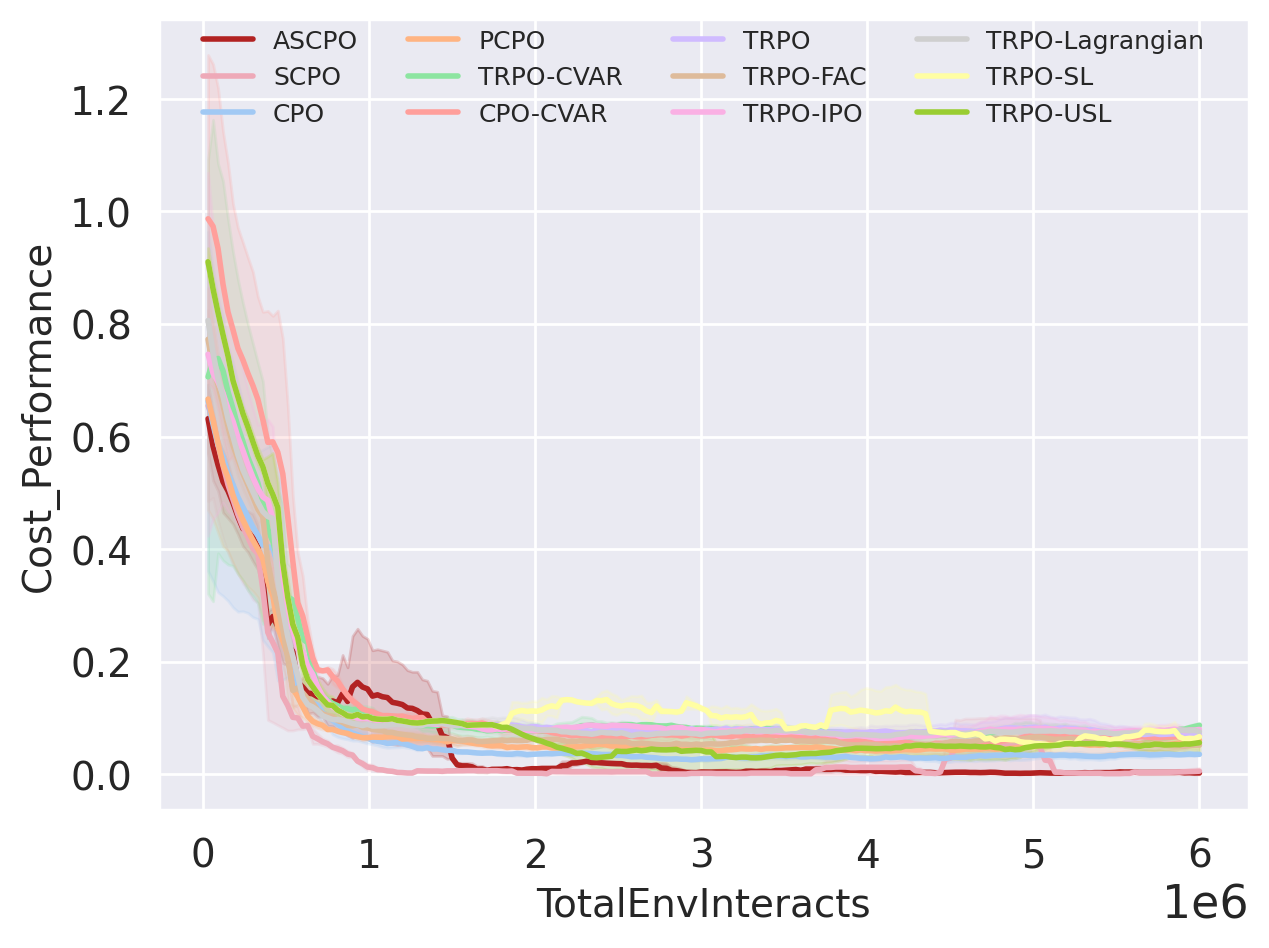}}
    \label{fig:point-hazard-1-AverageEpCost}
    \end{subfigure}
    \hfill
    \begin{subfigure}[t]{1.00\textwidth}
        \raisebox{-\height}{\includegraphics[height=0.7\textwidth]{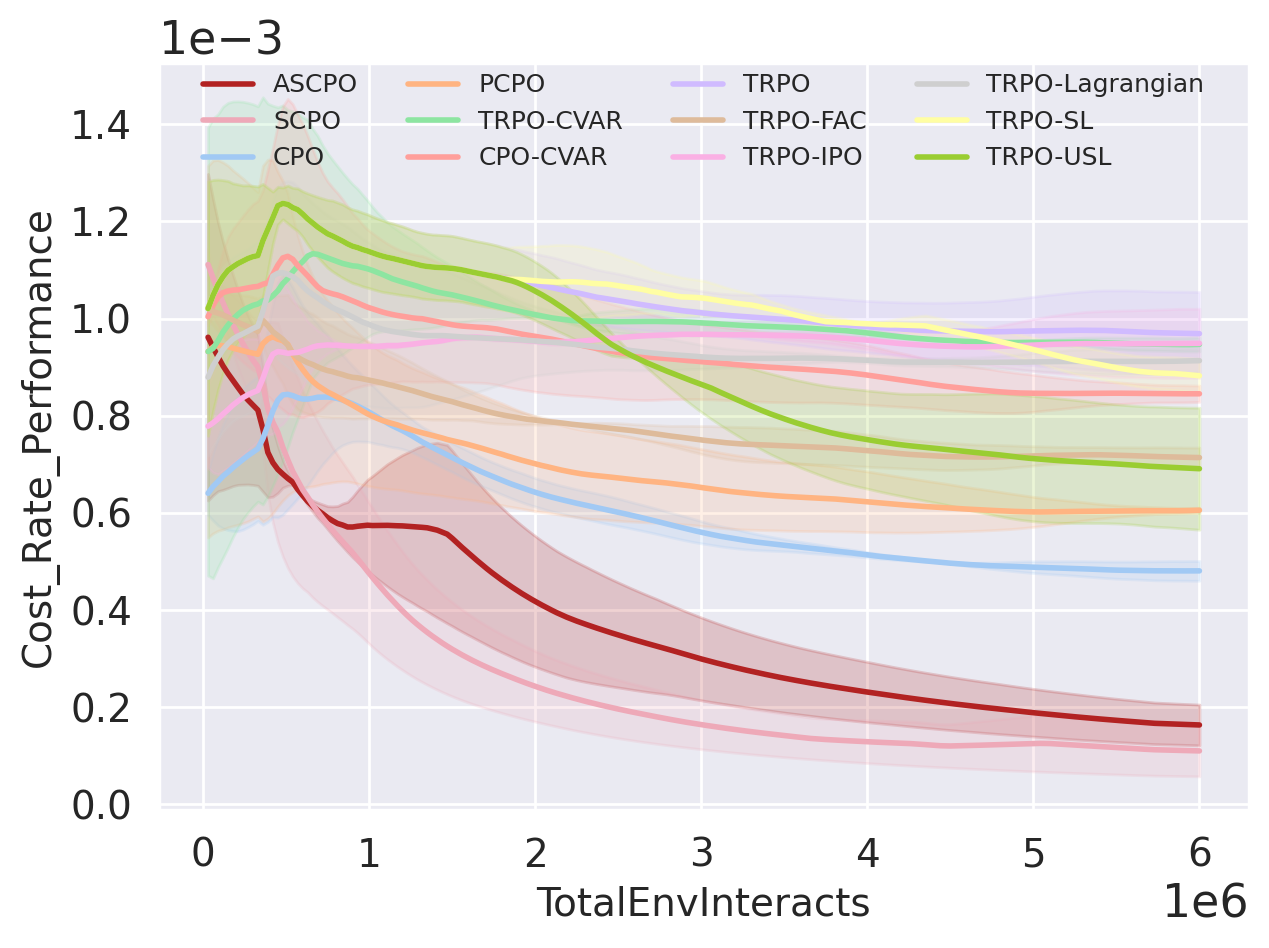}}
    \label{fig:point-hazard-1-CostRate}
    \end{subfigure}
    \caption{Point-1-Hazard}
    \label{fig:point-hazard-1}
    \end{subfigure}
   \begin{subfigure}[t]{0.32\textwidth}
    \begin{subfigure}[t]{1.00\textwidth}
        \raisebox{-\height}{\includegraphics[height=0.7\textwidth]{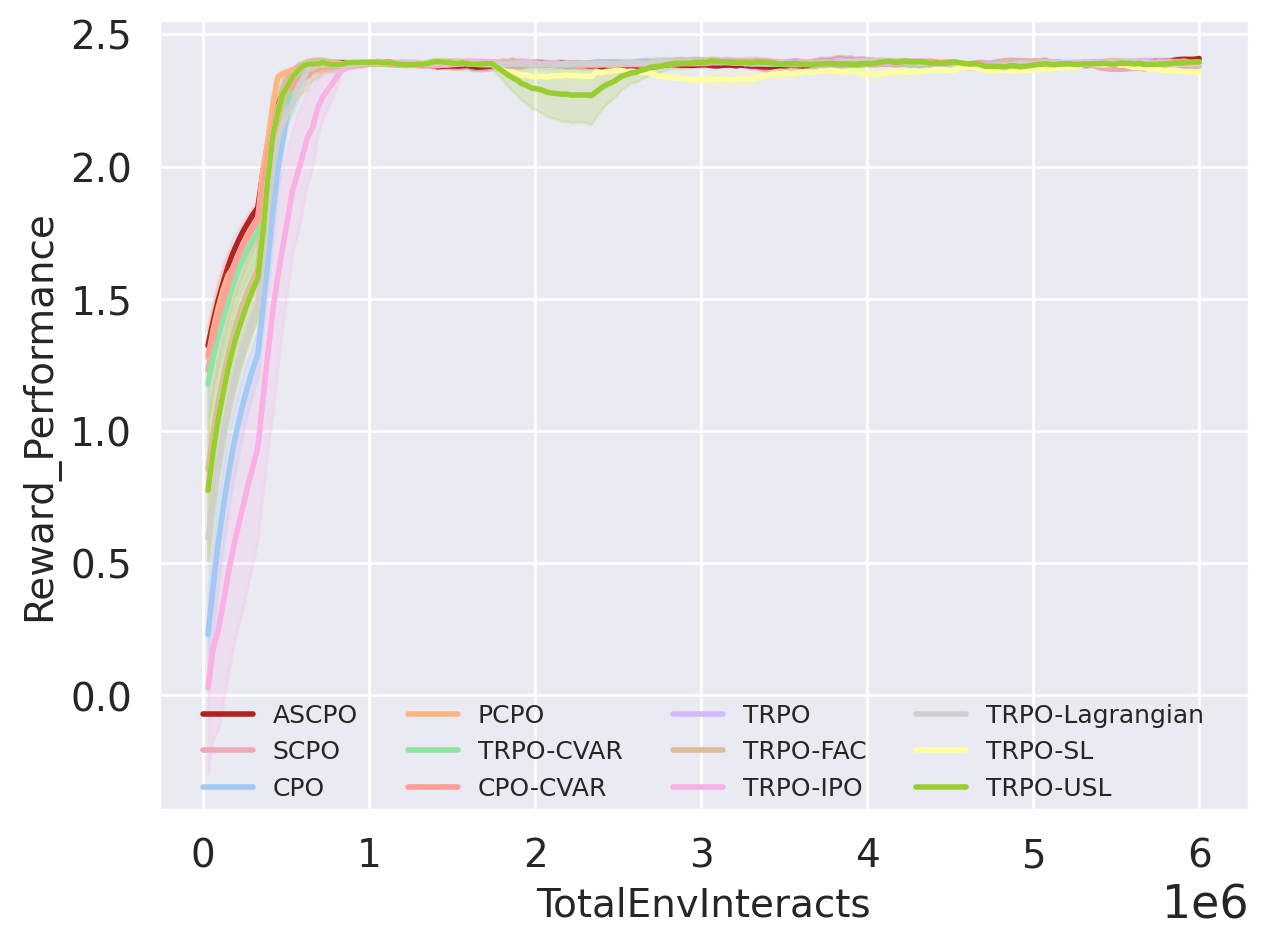}}
        \label{fig:point-hazard-4-Performance}
    \end{subfigure}
    \hfill
    \begin{subfigure}[t]{1.00\textwidth}
        \raisebox{-\height}{\includegraphics[height=0.7\textwidth]{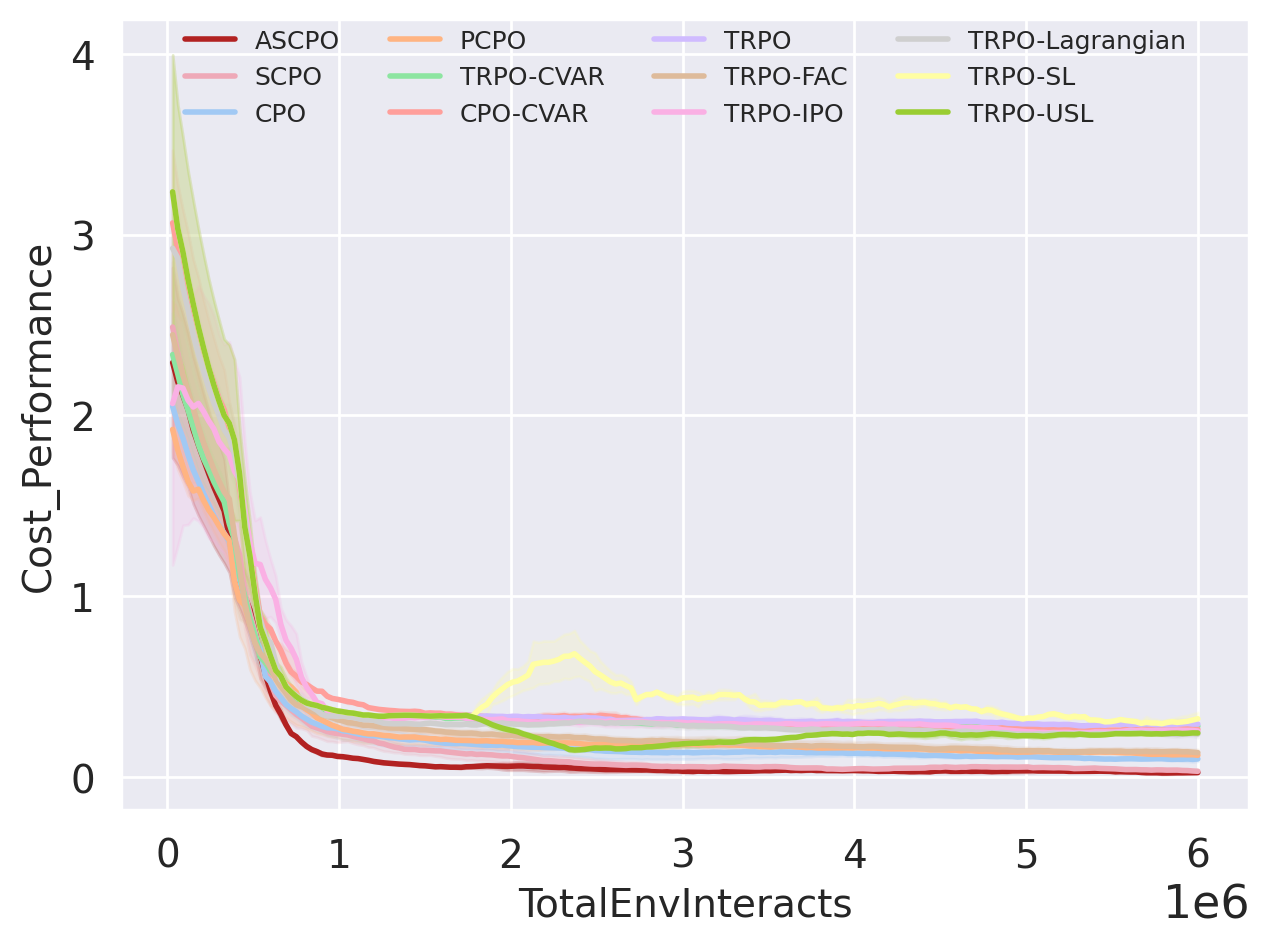}}
    \label{fig:point-hazard-4-AverageEpCost}
    \end{subfigure}
    \hfill
    \begin{subfigure}[t]{1.00\textwidth}
        \raisebox{-\height}{\includegraphics[height=0.7\textwidth]{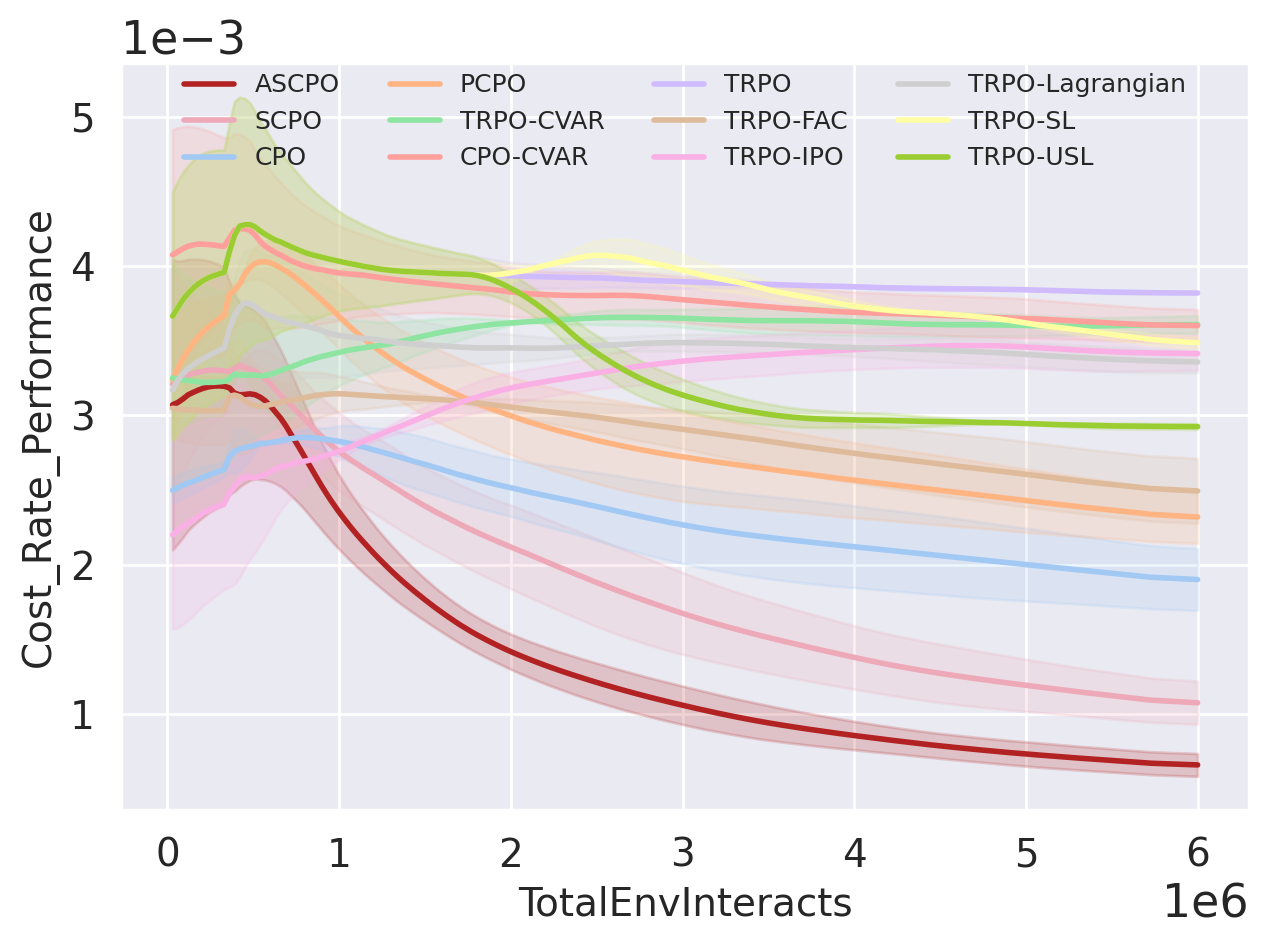}}
    \label{fig:point-hazard-4-CostRate}
    \end{subfigure}
    \caption{Point-4-Hazard}
    \label{fig:point-hazard-4}
    \end{subfigure}
    \begin{subfigure}[t]{0.32\textwidth}
    \begin{subfigure}[t]{1.00\textwidth}
        \raisebox{-\height}{\includegraphics[height=0.7\textwidth]{fig/guard/Goal_Point_8Hazards/Goal_Point_8Hazards_Reward_Performance.png}}
        \label{fig:point-hazard-8-Performance}
    \end{subfigure}
    \hfill
    \begin{subfigure}[t]{1.00\textwidth}
        \raisebox{-\height}{\includegraphics[height=0.7\textwidth]{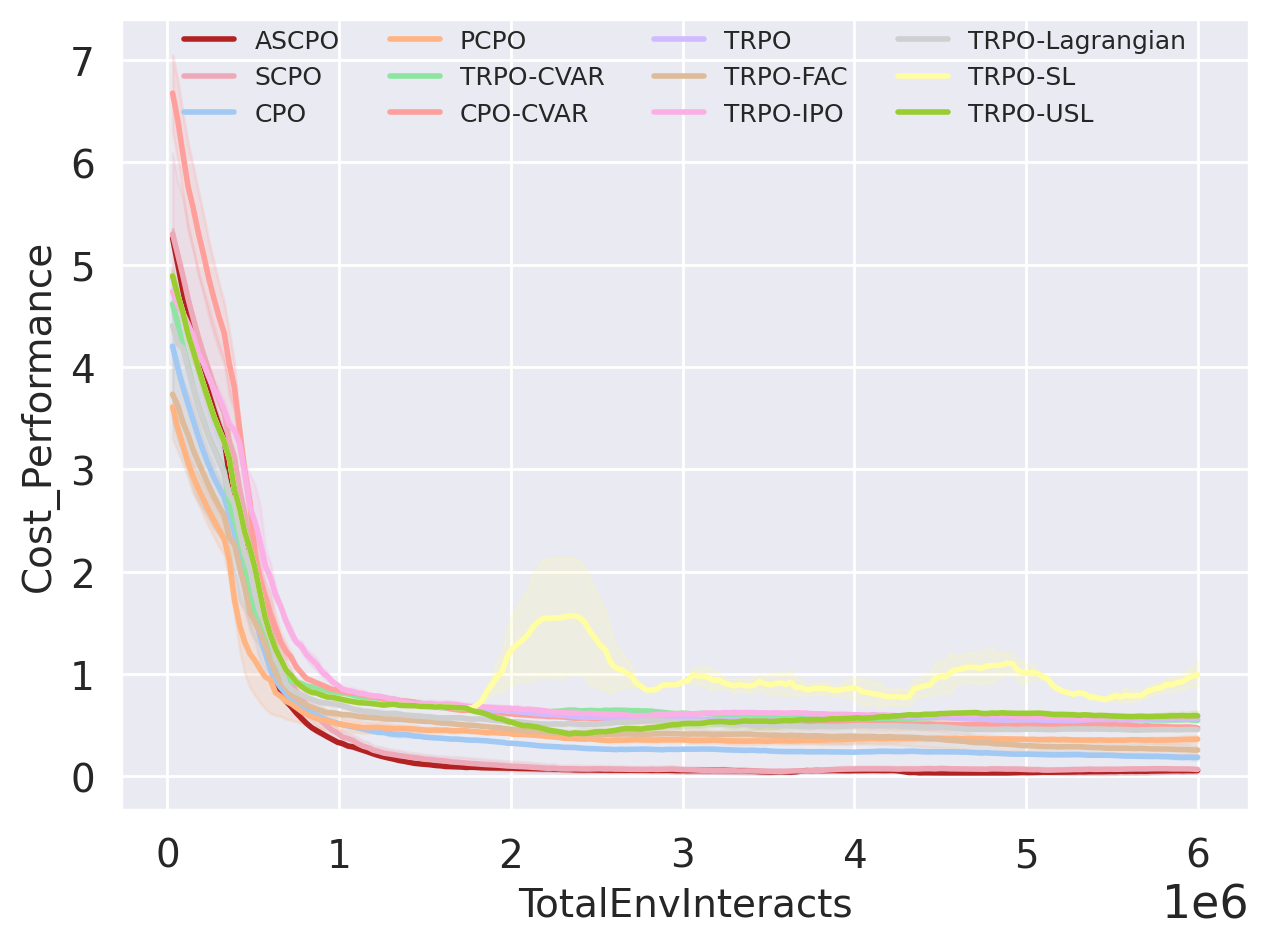}}
    \label{fig:point-hazard-8-AverageEpCost}
    \end{subfigure}
    \hfill
    \begin{subfigure}[t]{1.00\textwidth}
        \raisebox{-\height}{\includegraphics[height=0.7\textwidth]{fig/guard/Goal_Point_8Hazards/Goal_Point_8Hazards_Cost_Rate_Performance.png}}
    \label{fig:point-hazard-8-CostRate}
    \end{subfigure}
    \caption{Point-8-Hazard}
    \label{fig:point-hazard-8}
    \end{subfigure}
    \caption{Point-Hazard} 
    \label{fig:exp-point-hazard}
\end{figure}

\begin{figure}[p]
    \centering
    \begin{subfigure}[t]{0.32\textwidth}
    \begin{subfigure}[t]{1.00\textwidth}
        \raisebox{-\height}{\includegraphics[height=0.7\textwidth]{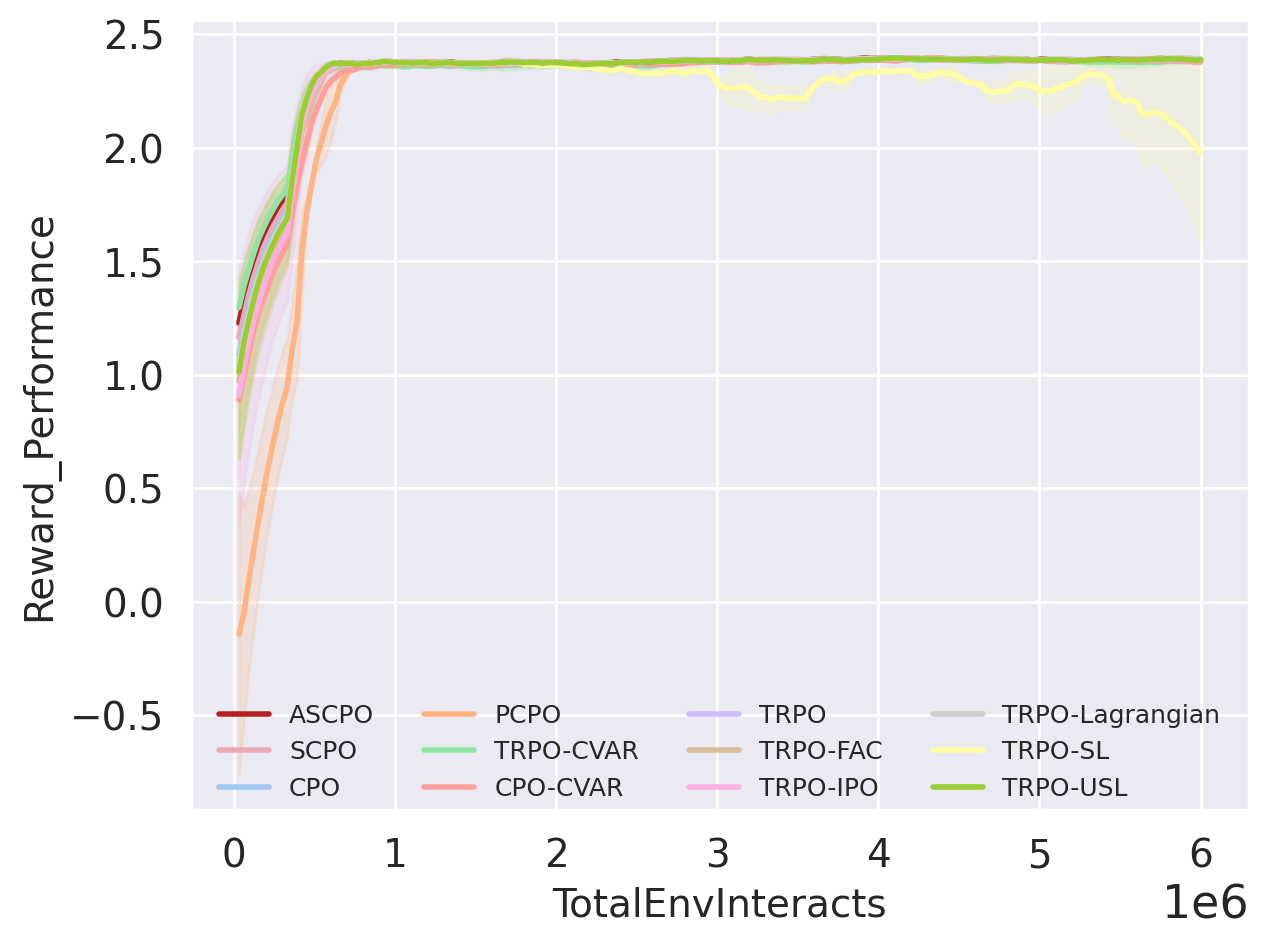}}
        \label{fig:point-pillar-1-Performance}
    \end{subfigure}
    \hfill
    \begin{subfigure}[t]{1.00\textwidth}
        \raisebox{-\height}{\includegraphics[height=0.7\textwidth]{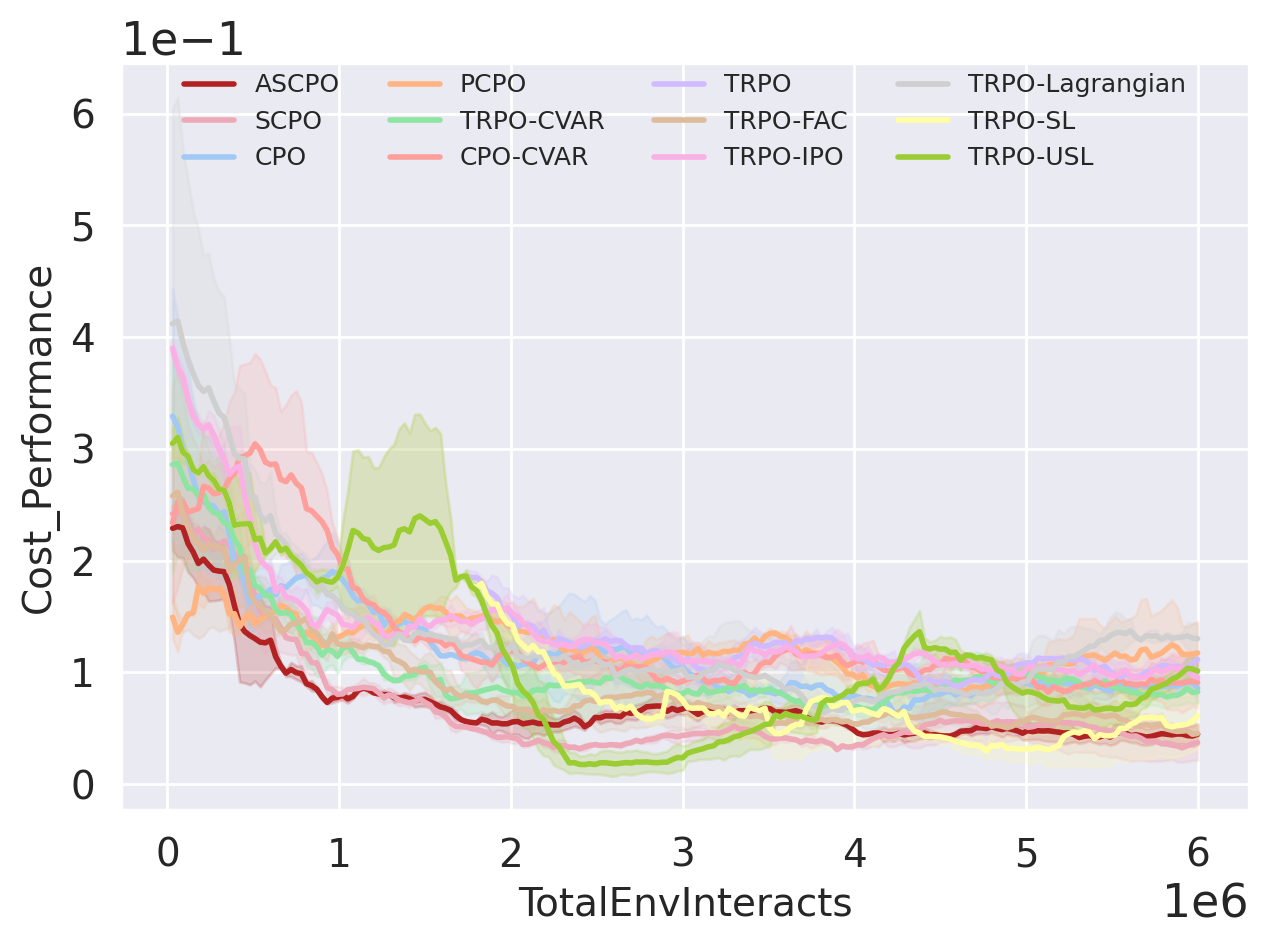}}
    \label{fig:point-pillar-1-AverageEpCost}
    \end{subfigure}
    \hfill
    \begin{subfigure}[t]{1.00\textwidth}
        \raisebox{-\height}{\includegraphics[height=0.7\textwidth]{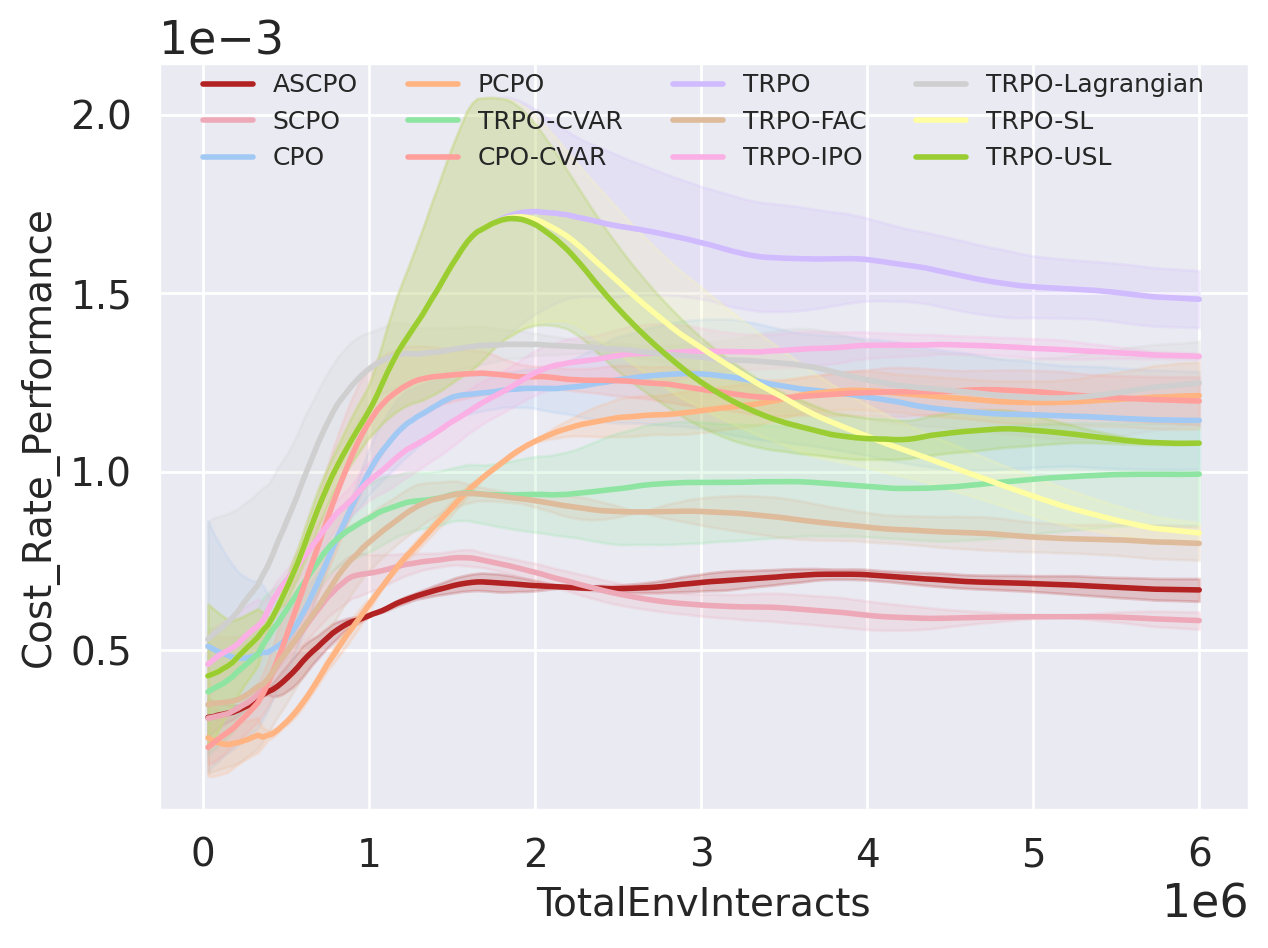}}
    \label{fig:point-pillar-1-CostRate}
    \end{subfigure}
    \caption{Point-1-Pillar}
    \label{fig:point-pillar-1}
    \end{subfigure}
   \begin{subfigure}[t]{0.32\textwidth}
    \begin{subfigure}[t]{1.00\textwidth}
        \raisebox{-\height}{\includegraphics[height=0.7\textwidth]{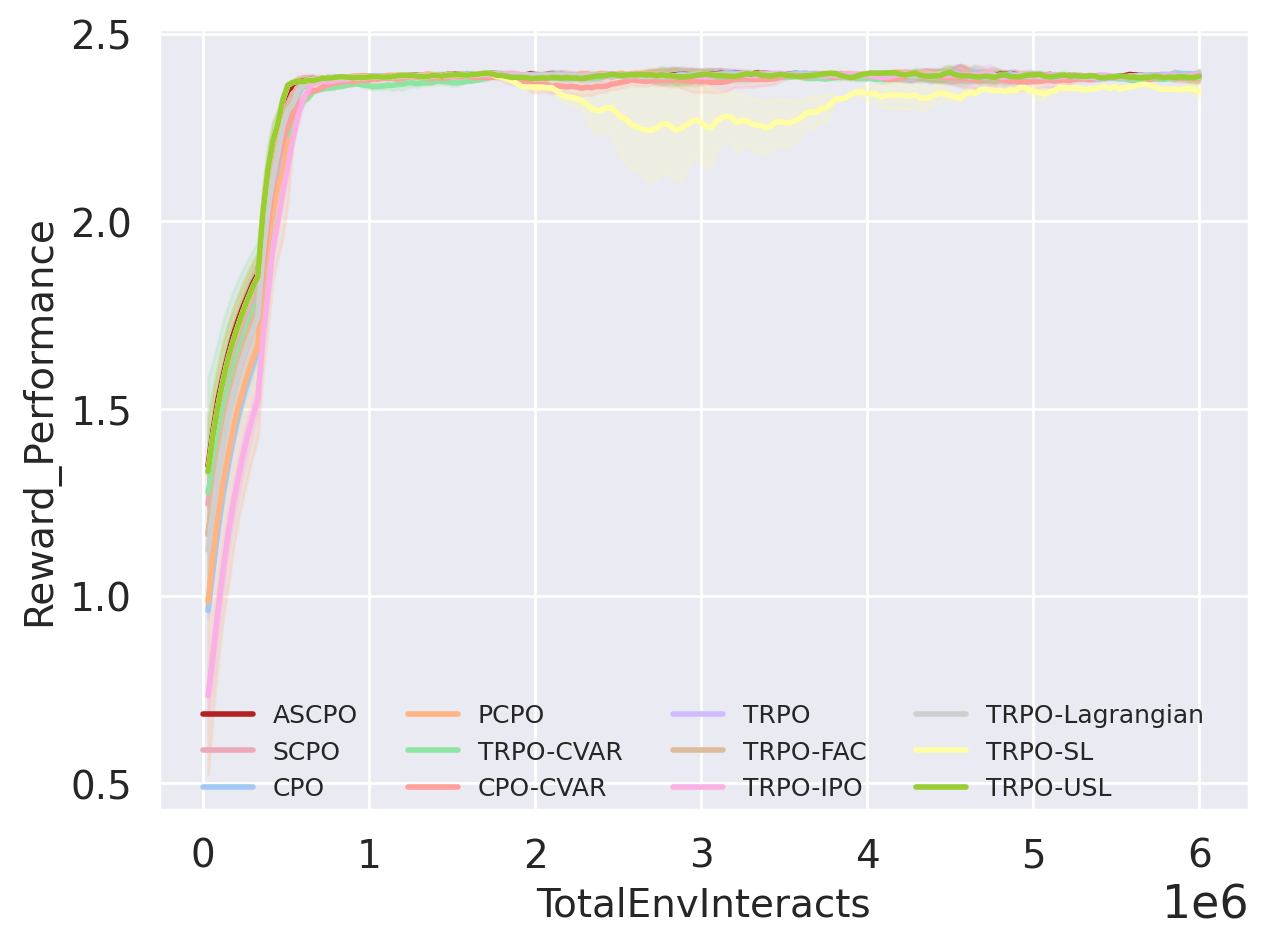}}
        \label{fig:point-pillar-4-Performance}
    \end{subfigure}
    \hfill
    \begin{subfigure}[t]{1.00\textwidth}
        \raisebox{-\height}{\includegraphics[height=0.7\textwidth]{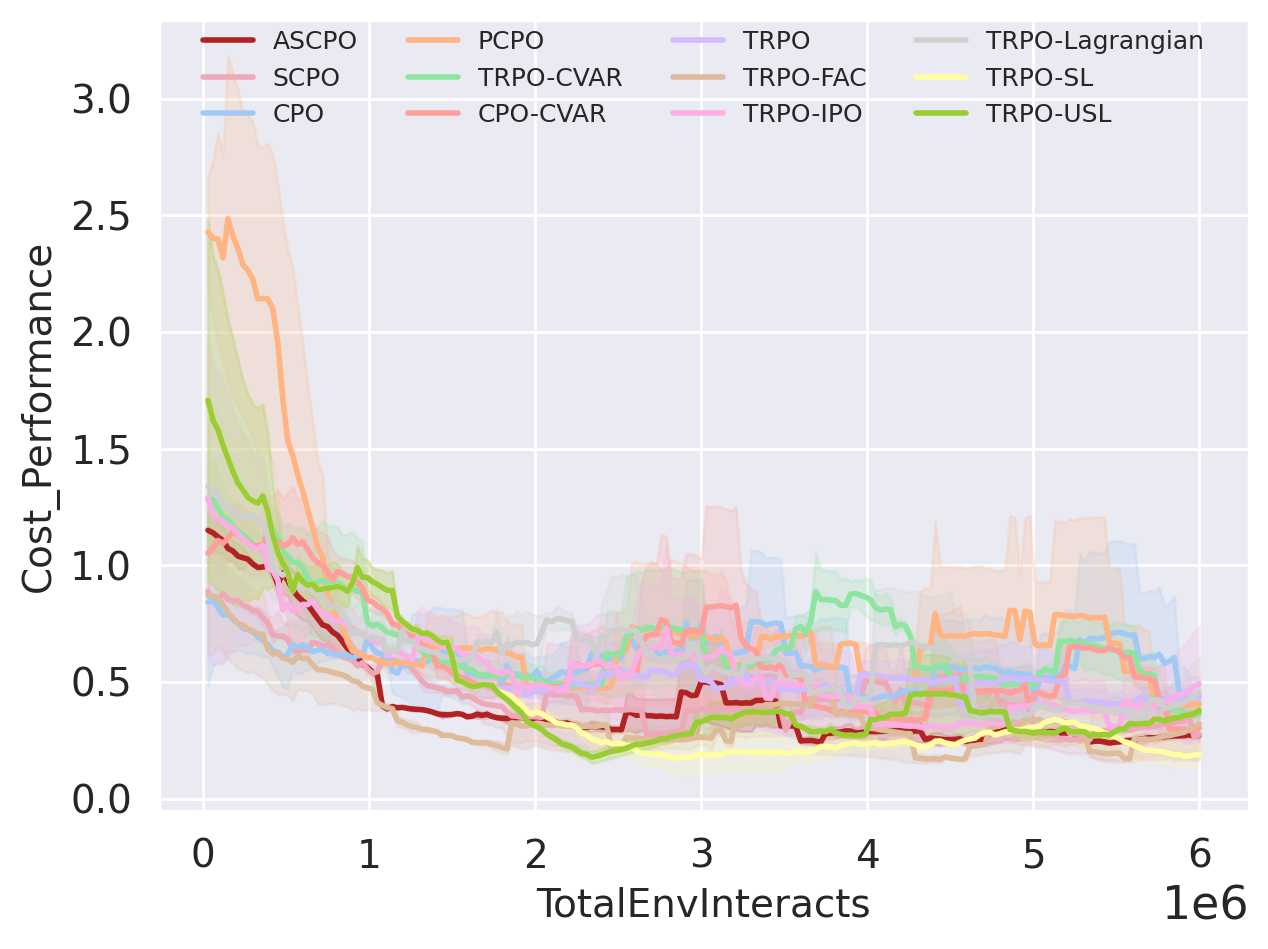}}
    \label{fig:point-pillar-4-AverageEpCost}
    \end{subfigure}
    \hfill
    \begin{subfigure}[t]{1.00\textwidth}
        \raisebox{-\height}{\includegraphics[height=0.7\textwidth]{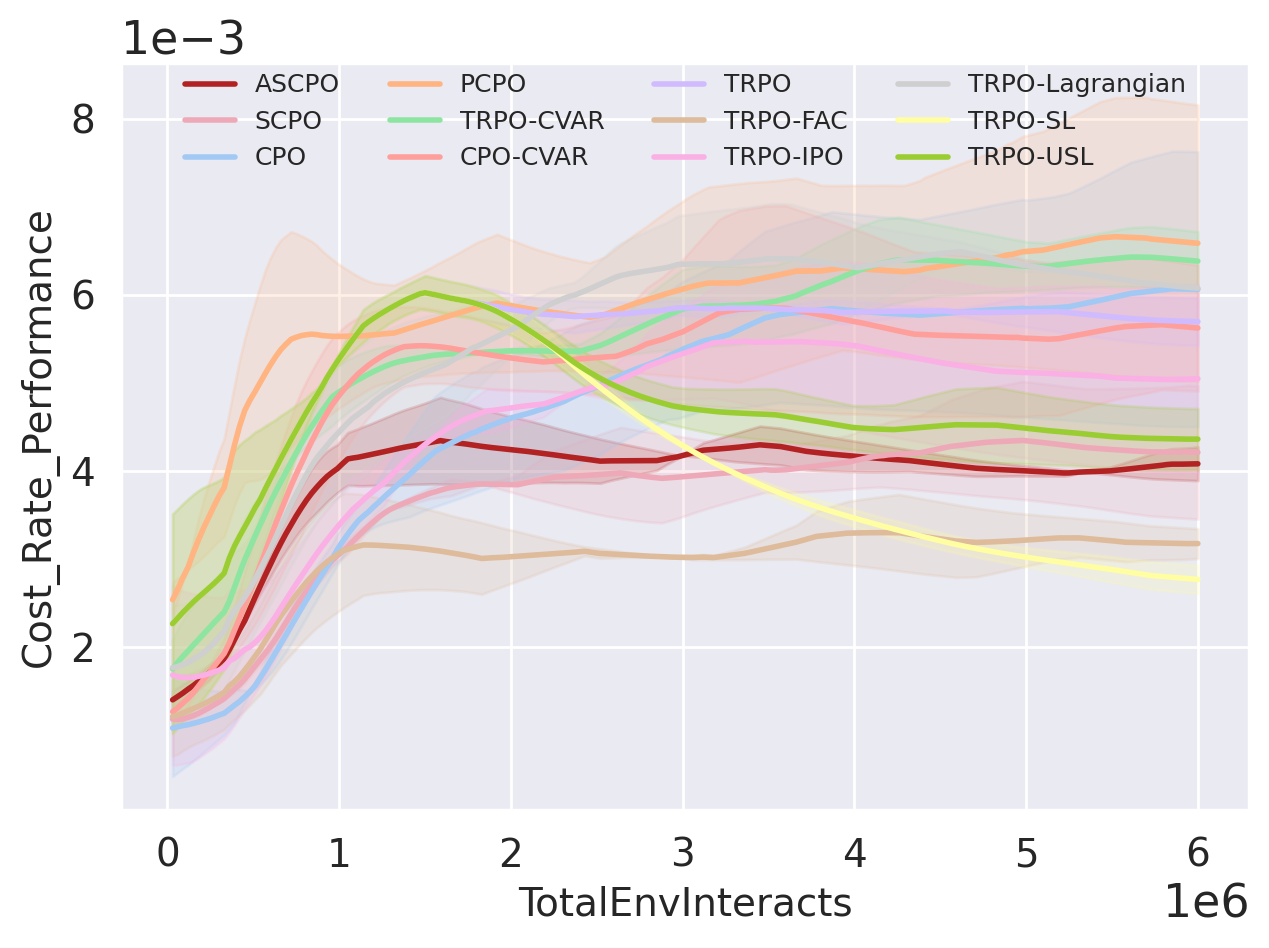}}
    \label{fig:point-pillar-4-CostRate}
    \end{subfigure}
    \caption{Point-4-Pillar}
    \label{fig:point-pillar-4}
    \end{subfigure}
    \begin{subfigure}[t]{0.32\textwidth}
    \begin{subfigure}[t]{1.00\textwidth}
        \raisebox{-\height}{\includegraphics[height=0.7\textwidth]{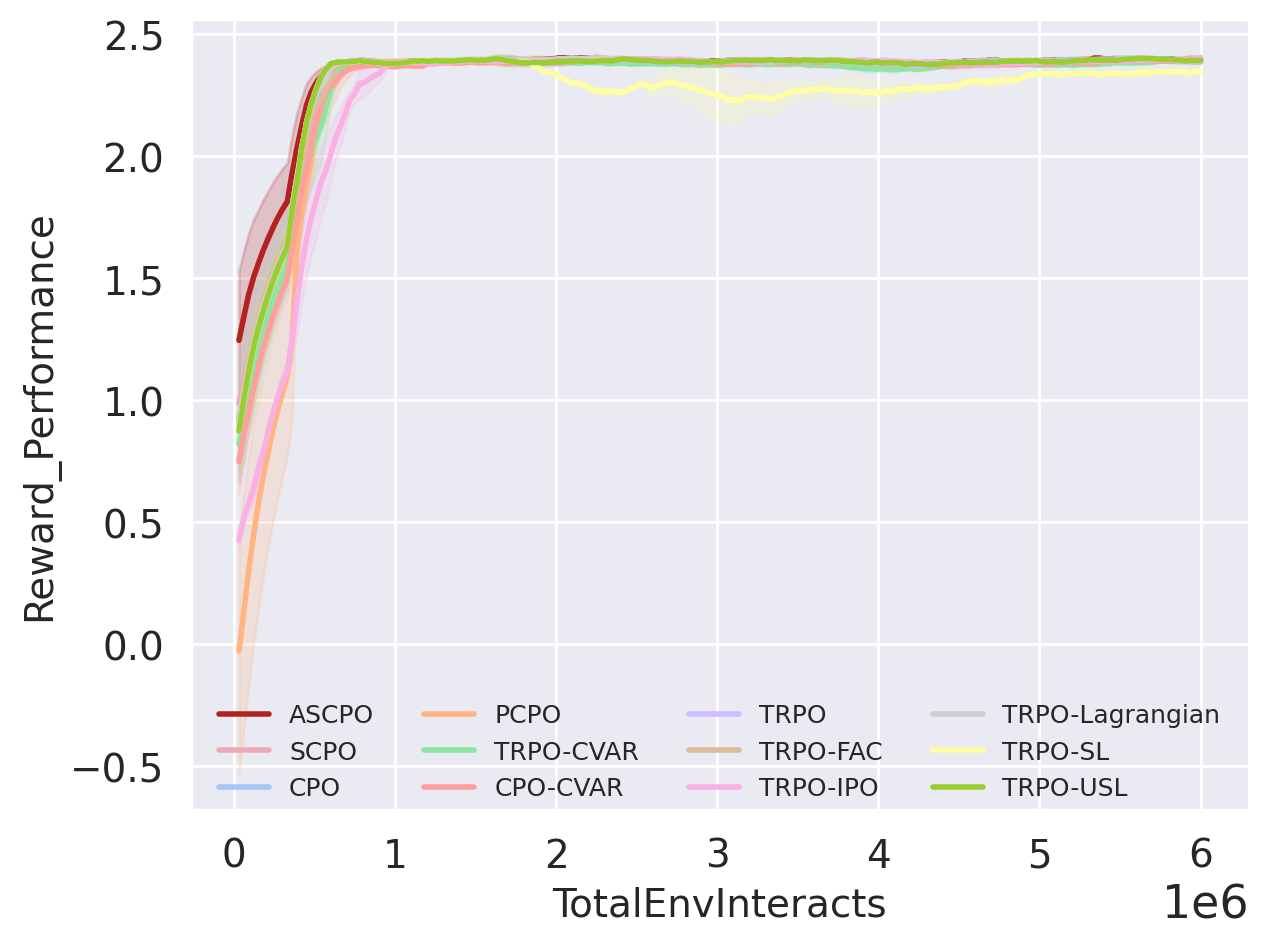}}
        \label{fig:point-pillar-8-Performance}
    \end{subfigure}
    \hfill
    \begin{subfigure}[t]{1.00\textwidth}
        \raisebox{-\height}{\includegraphics[height=0.7\textwidth]{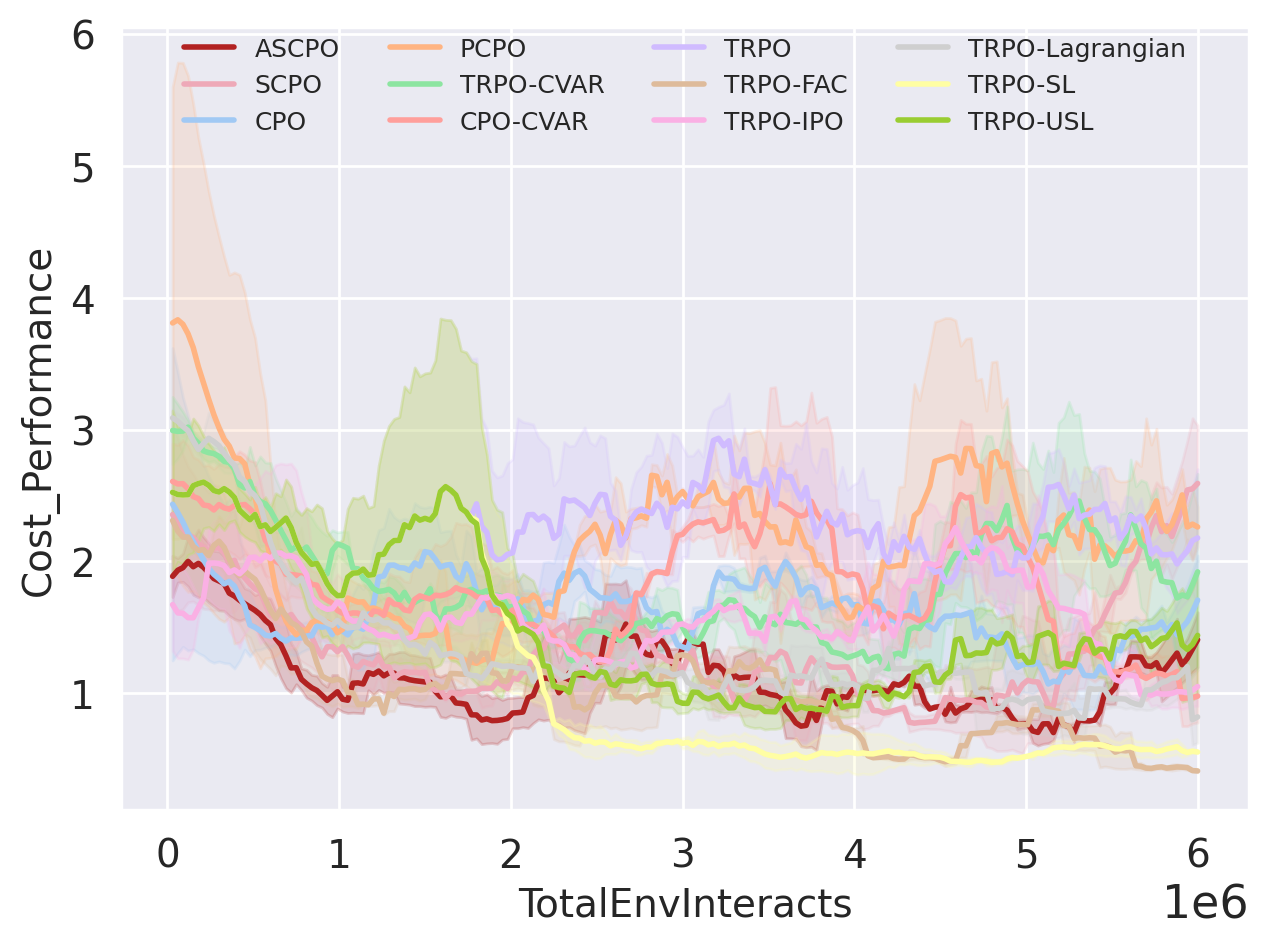}}
    \label{fig:point-pillar-8-AverageEpCost}
    \end{subfigure}
    \hfill
    \begin{subfigure}[t]{1.00\textwidth}
        \raisebox{-\height}{\includegraphics[height=0.7\textwidth]{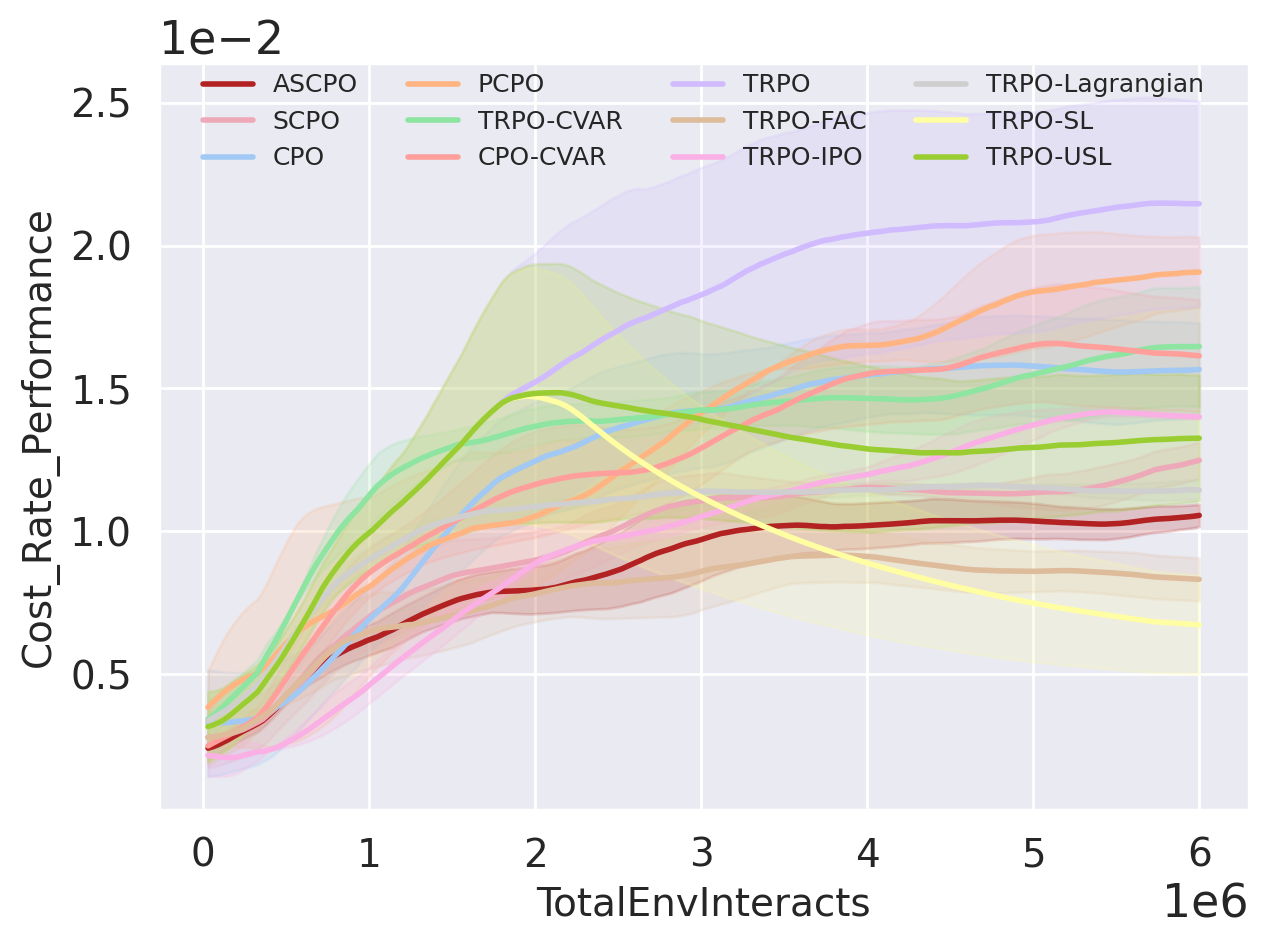}}
    \label{fig:point-pillar-8-CostRate}
    \end{subfigure}
    \caption{Point-8-Pillar}
    \label{fig:point-pillar-8}
    \end{subfigure}
    \caption{Point-Pillar}
    \label{fig:exp-point-pillar}
\end{figure}

\begin{figure}[p]
    \centering
    \begin{subfigure}[t]{0.32\textwidth}
    \begin{subfigure}[t]{1.00\textwidth}
        \raisebox{-\height}{\includegraphics[height=0.7\textwidth]{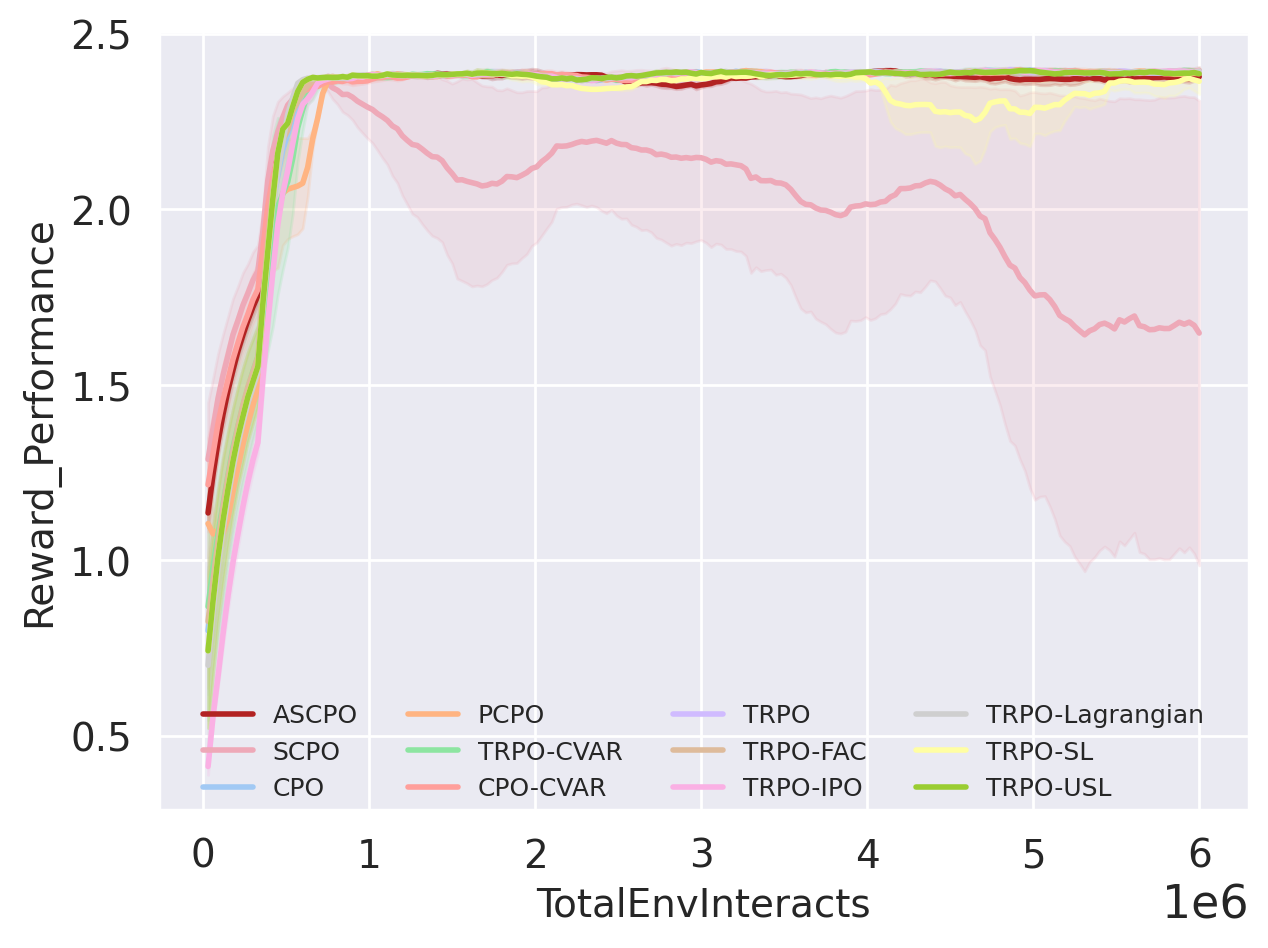}}
        \label{fig:point-ghost-1-Performance}
    \end{subfigure}
    \hfill
    \begin{subfigure}[t]{1.00\textwidth}
        \raisebox{-\height}{\includegraphics[height=0.7\textwidth]{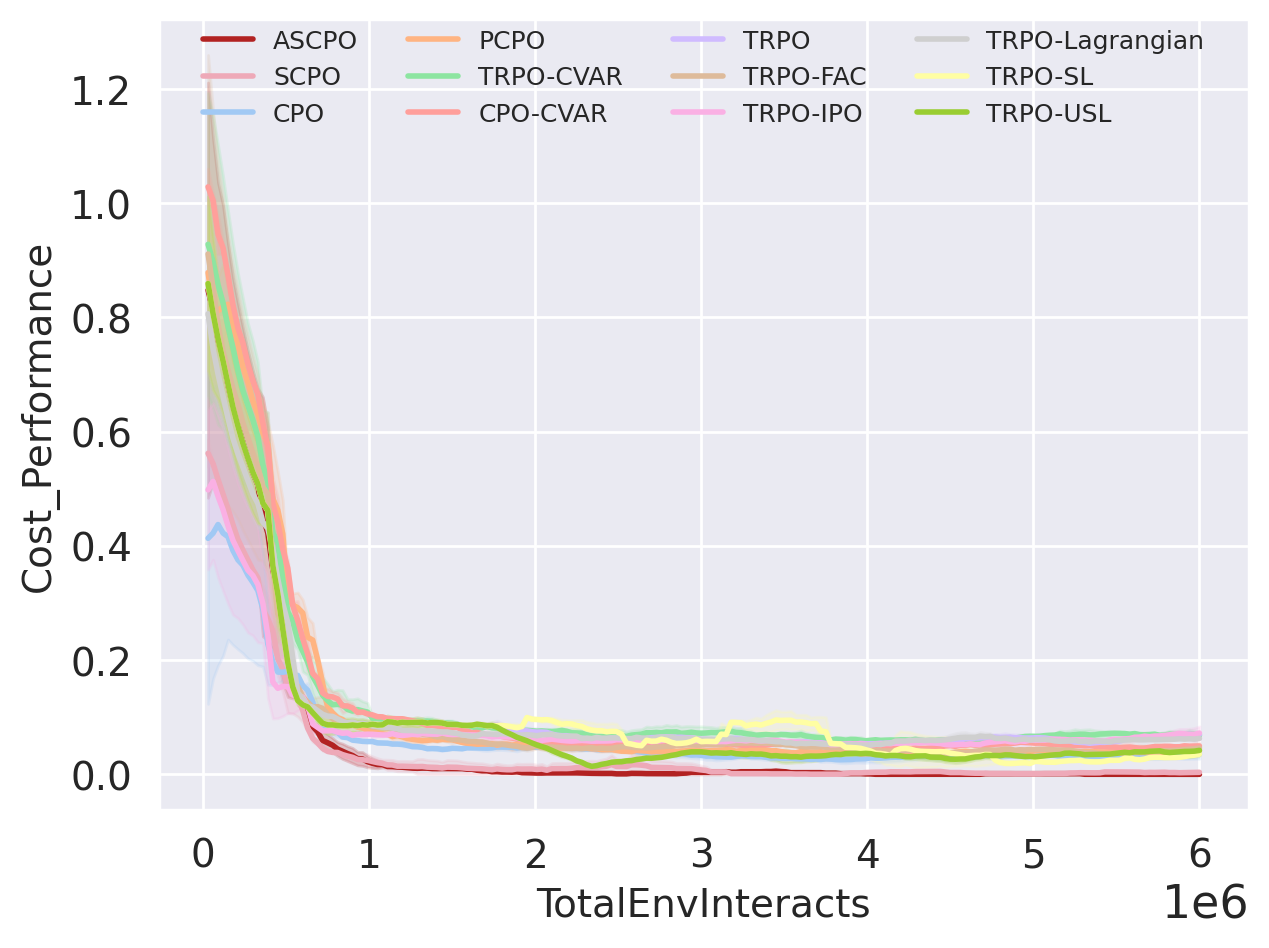}}
    \label{fig:point-ghost-1-AverageEpCost}
    \end{subfigure}
    \hfill
    \begin{subfigure}[t]{1.00\textwidth}
        \raisebox{-\height}{\includegraphics[height=0.7\textwidth]{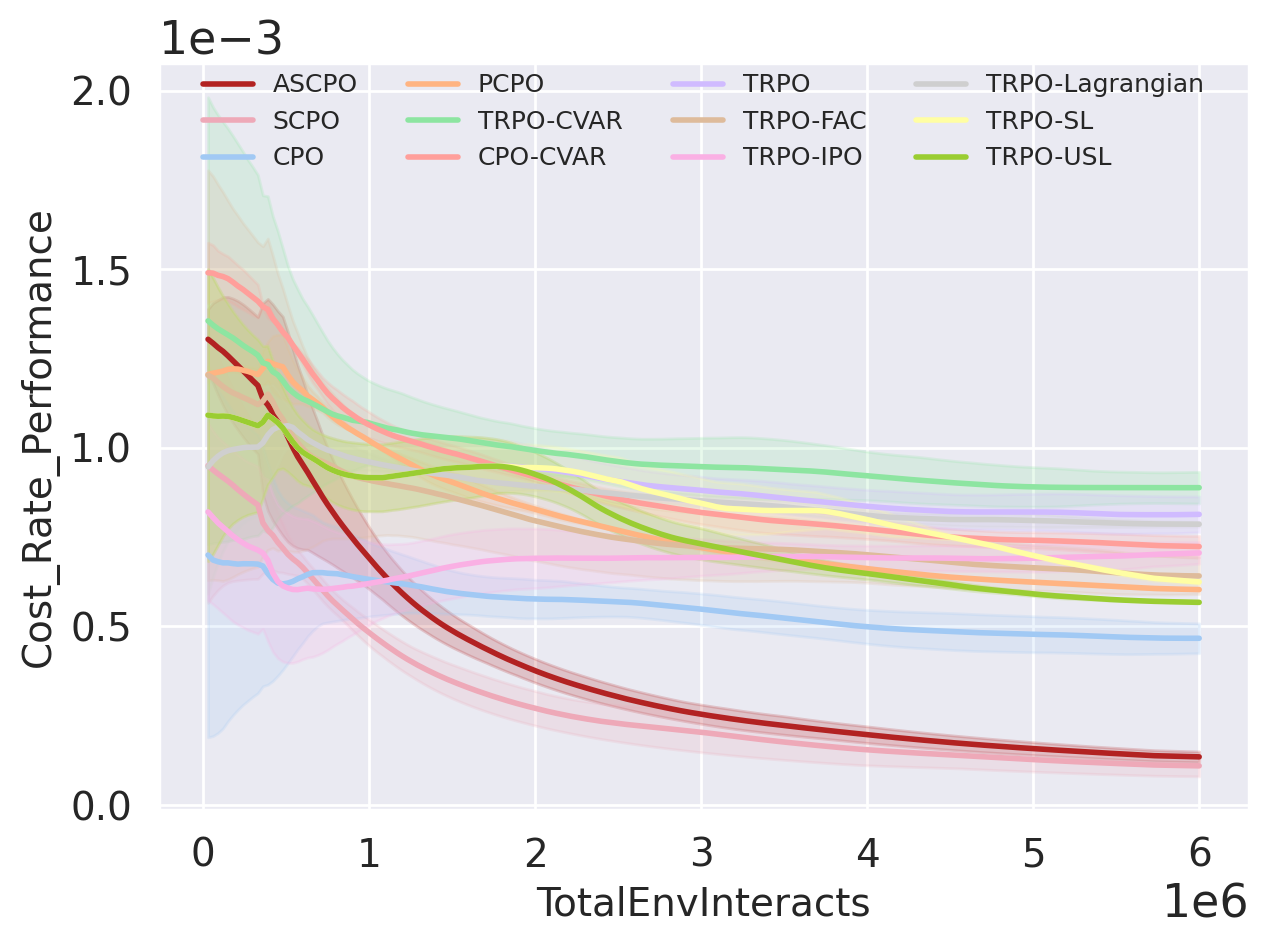}}
    \label{fig:point-ghost-1-CostRate}
    \end{subfigure}
    \caption{Point-1-Ghost}
    \label{fig:point-ghost-1}
    \end{subfigure}
   \begin{subfigure}[t]{0.32\textwidth}
    \begin{subfigure}[t]{1.00\textwidth}
        \raisebox{-\height}{\includegraphics[height=0.7\textwidth]{fig/guard/Goal_Point_4Ghosts/Goal_Point_4Ghosts_Reward_Performance.png}}
        \label{fig:point-ghost-4-Performance}
    \end{subfigure}
    \hfill
    \begin{subfigure}[t]{1.00\textwidth}
        \raisebox{-\height}{\includegraphics[height=0.7\textwidth]{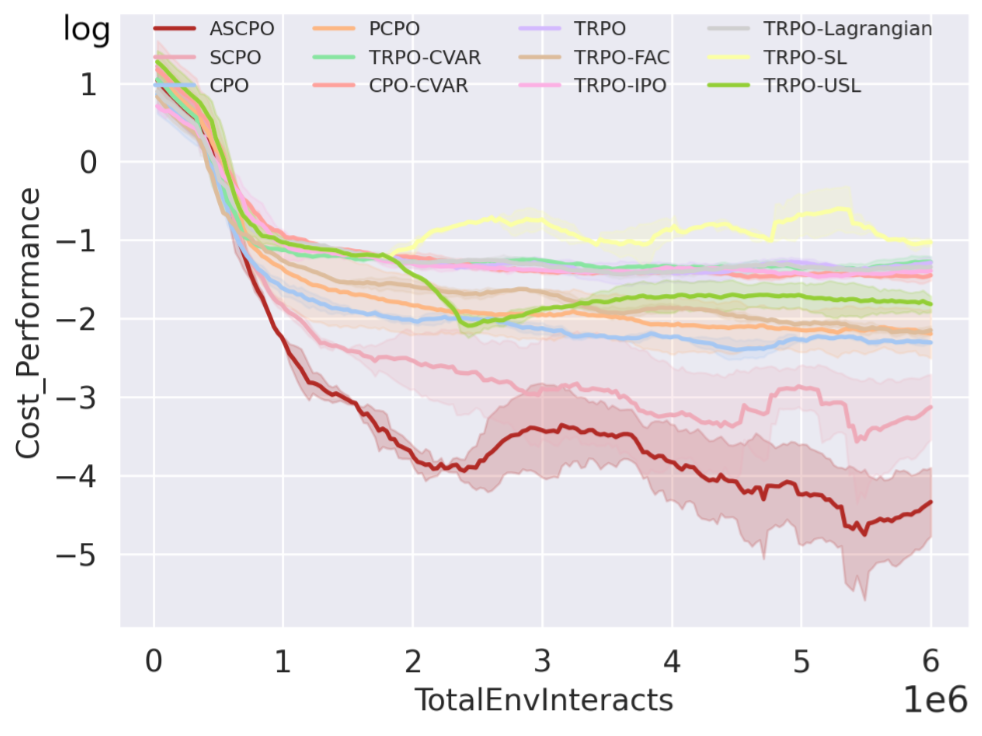}}
    \label{fig:point-ghost-4-AverageEpCost}
    \end{subfigure}
    \hfill
    \begin{subfigure}[t]{1.00\textwidth}
        \raisebox{-\height}{\includegraphics[height=0.7\textwidth]{fig/guard/Goal_Point_4Ghosts/Goal_Point_4Ghosts_Cost_Rate_Performance.png}}
    \label{fig:point-ghost-4-CostRate}
    \end{subfigure}
    \caption{Point-4-Ghost}
    \label{fig:point-ghost-4}
    \end{subfigure}
    \begin{subfigure}[t]{0.32\textwidth}
    \begin{subfigure}[t]{1.00\textwidth}
        \raisebox{-\height}{\includegraphics[height=0.7\textwidth]{fig/guard/Goal_Point_8Ghosts/Goal_Point_8Ghosts_Reward_Performance.png}}
        \label{fig:point-ghost-8-Performance}
    \end{subfigure}
    \hfill
    \begin{subfigure}[t]{1.00\textwidth}
        \raisebox{-\height}{\includegraphics[height=0.7\textwidth]{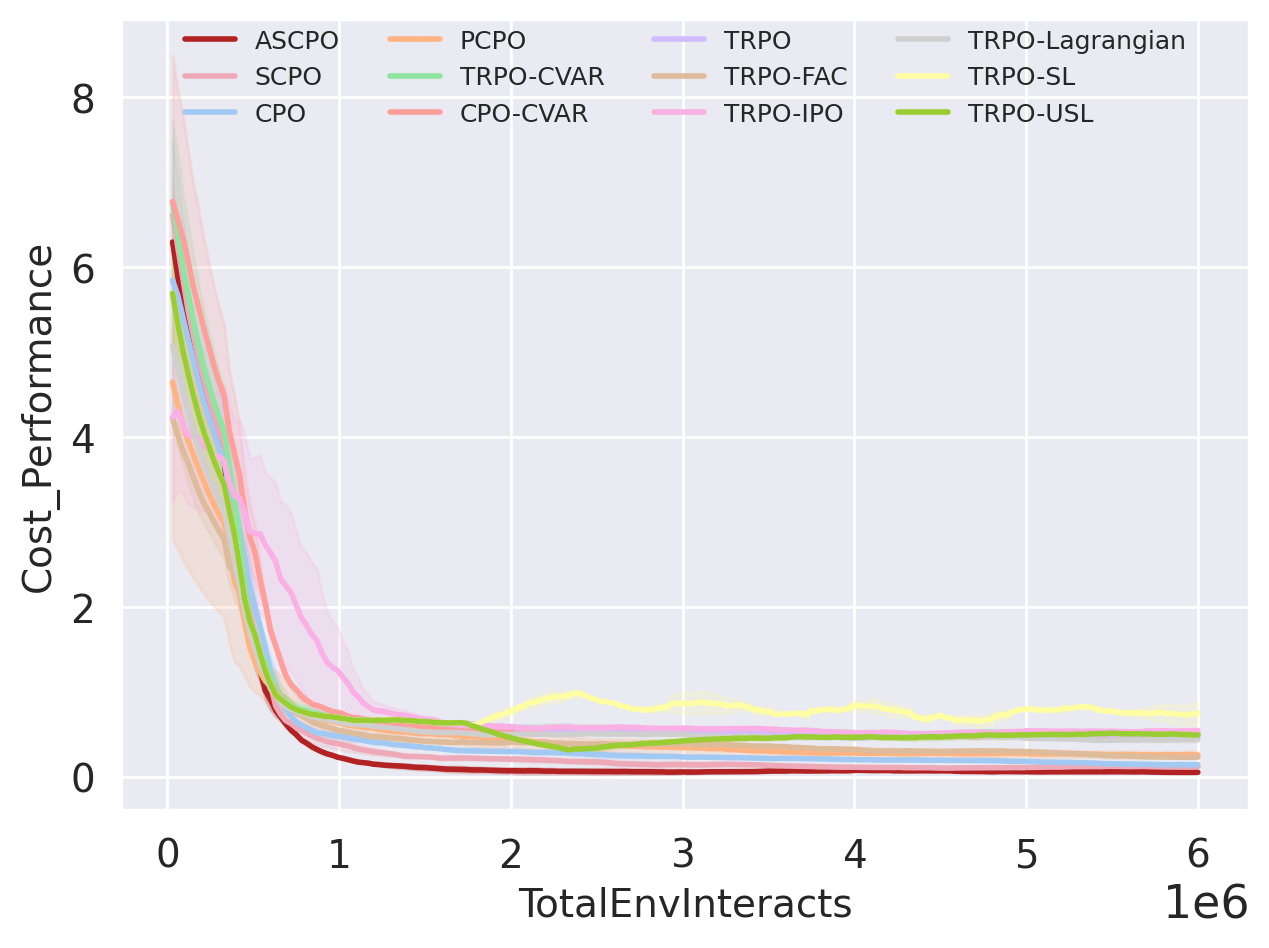}}
    \label{fig:point-ghost-8-AverageEpCost}
    \end{subfigure}
    \hfill
    \begin{subfigure}[t]{1.00\textwidth}
        \raisebox{-\height}{\includegraphics[height=0.7\textwidth]{fig/guard/Goal_Point_8Ghosts/Goal_Point_8Ghosts_Cost_Rate_Performance.png}}
    \label{fig:point-ghost-8-CostRate}
    \end{subfigure}
    \caption{Point-8-Ghost}
    \label{fig:point-ghost-8}
    \end{subfigure}
    \caption{Point-Ghost}
    \label{fig:exp-point-ghost}
\end{figure}

\begin{figure}[p]
    \centering
    \begin{subfigure}[t]{0.32\textwidth}
    \begin{subfigure}[t]{1.00\textwidth}
        \raisebox{-\height}{\includegraphics[height=0.7\textwidth]{fig/guard/Goal_Swimmer_1Hazards/Goal_Swimmer_1Hazards_Reward_Performance.png}}
        \label{fig:swimmer-hazard-1-Performance}
    \end{subfigure}
    \hfill
    \begin{subfigure}[t]{1.00\textwidth}
        \raisebox{-\height}{\includegraphics[height=0.7\textwidth]{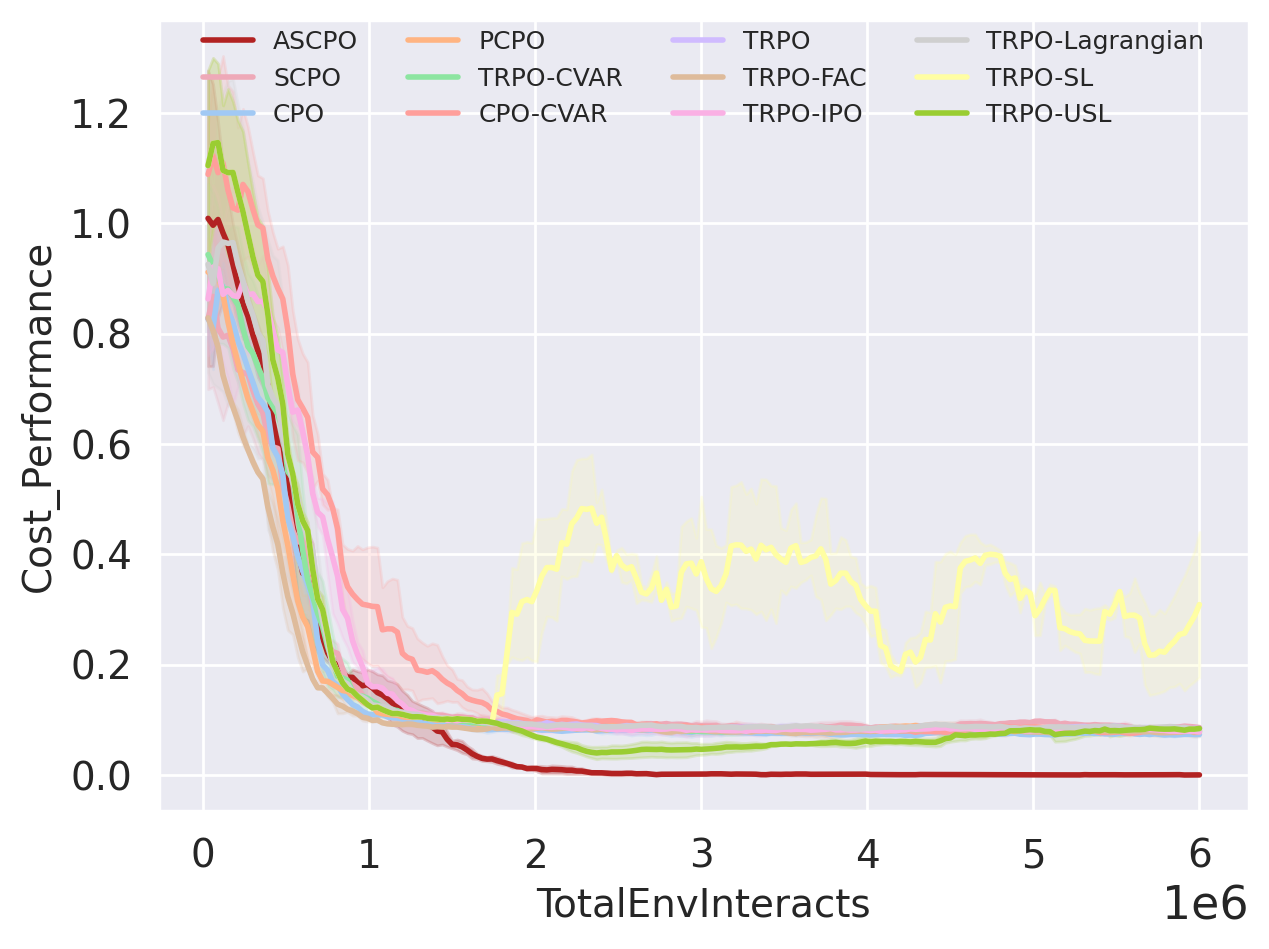}}
    \label{fig:swimmer-hazard-1-AverageEpCost}
    \end{subfigure}
    \hfill
    \begin{subfigure}[t]{1.00\textwidth}
        \raisebox{-\height}{\includegraphics[height=0.7\textwidth]{fig/guard/Goal_Swimmer_1Hazards/Goal_Swimmer_1Hazards_Cost_Rate_Performance.png}}
    \label{fig:swimmer-hazard-1-CostRate}
    \end{subfigure}
    \caption{Swimmer-1-Hazard}
    \label{fig:swimmer-hazard-1}
    \end{subfigure}
   \begin{subfigure}[t]{0.32\textwidth}
    \begin{subfigure}[t]{1.00\textwidth}
        \raisebox{-\height}{\includegraphics[height=0.7\textwidth]{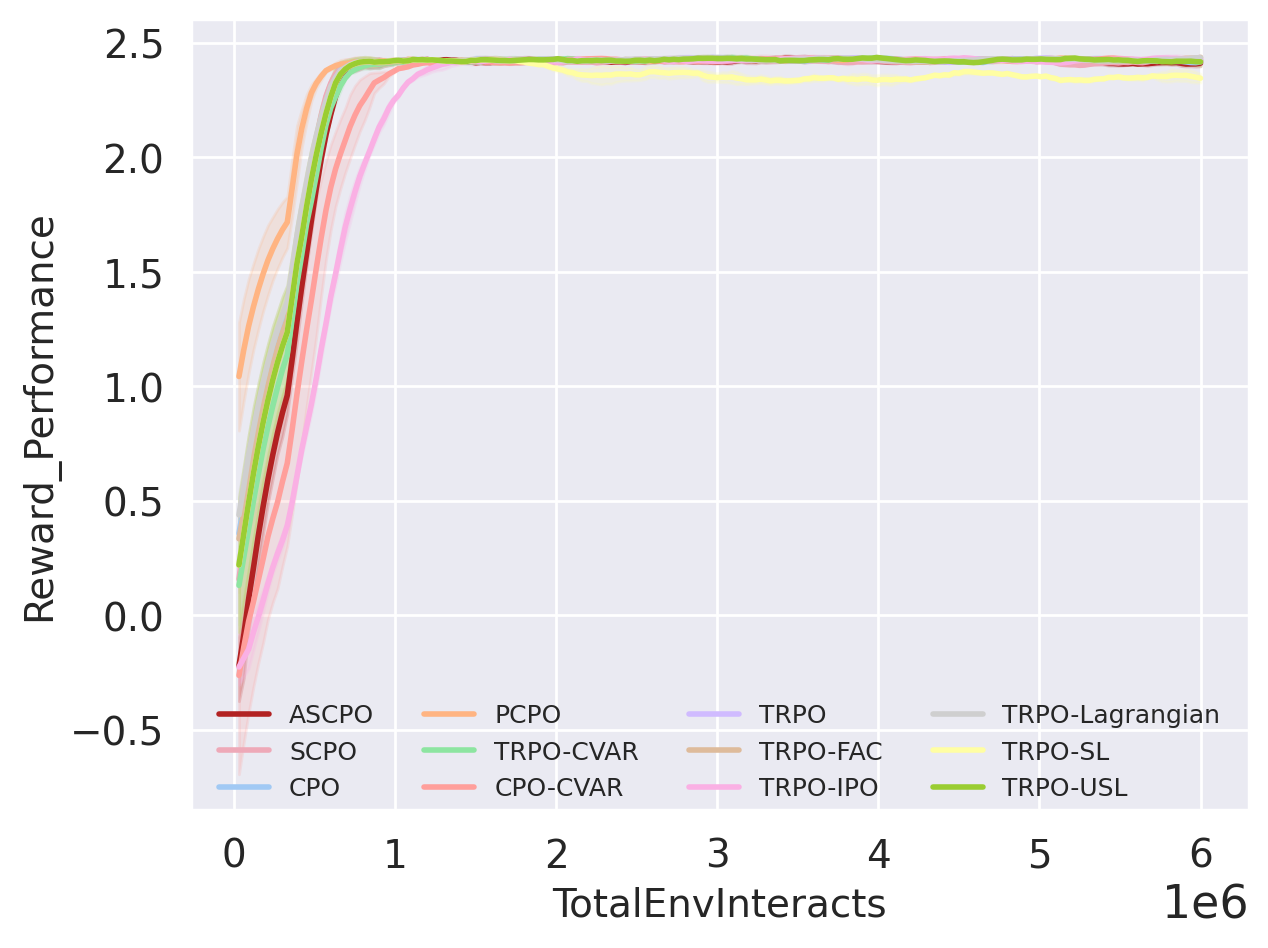}}
        \label{fig:swimmer-hazard-4-Performance}
    \end{subfigure}
    \hfill
    \begin{subfigure}[t]{1.00\textwidth}
        \raisebox{-\height}{\includegraphics[height=0.7\textwidth]{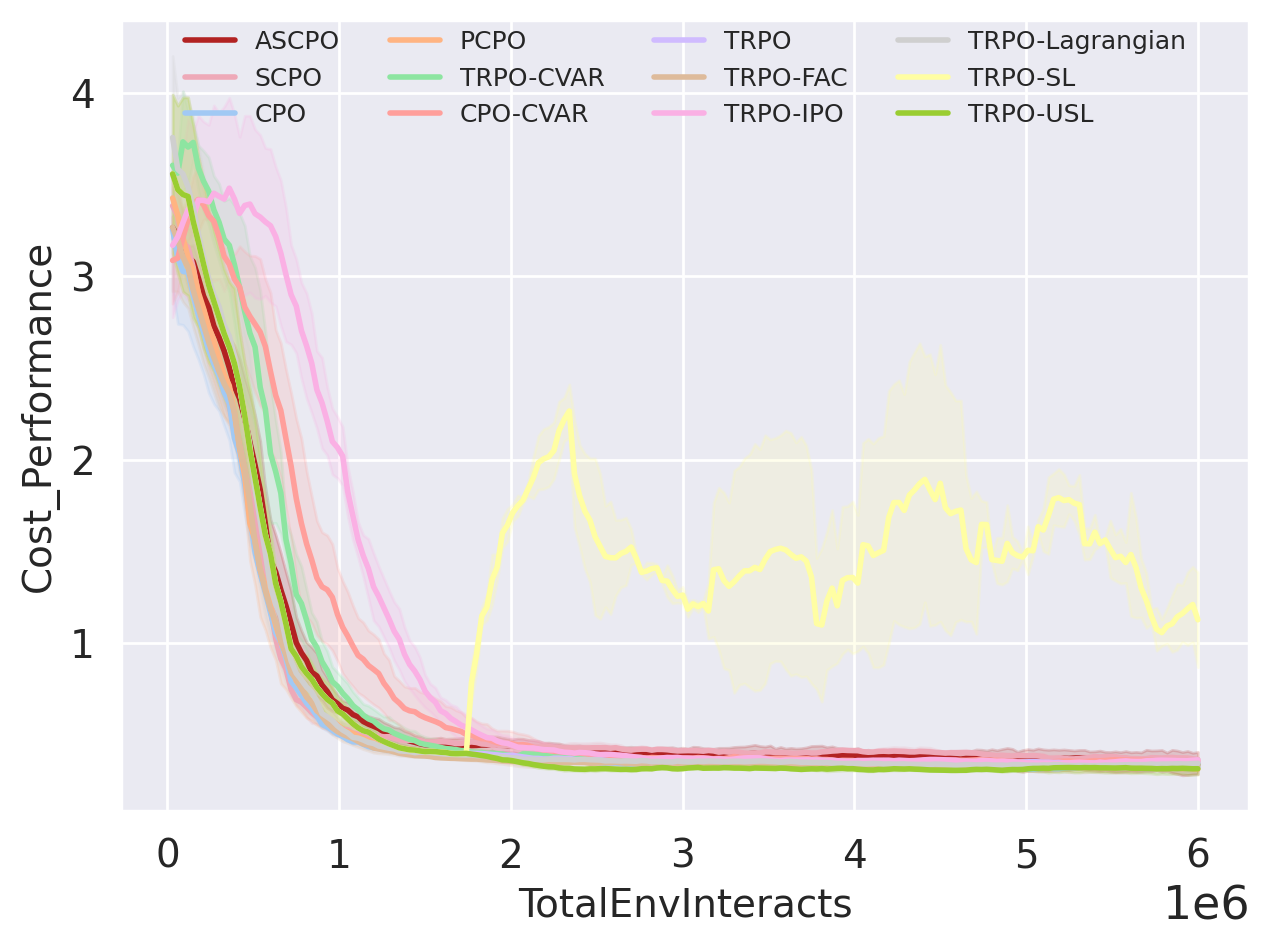}}
    \label{fig:swimmer-hazard-4-AverageEpCost}
    \end{subfigure}
    \hfill
    \begin{subfigure}[t]{1.00\textwidth}
        \raisebox{-\height}{\includegraphics[height=0.7\textwidth]{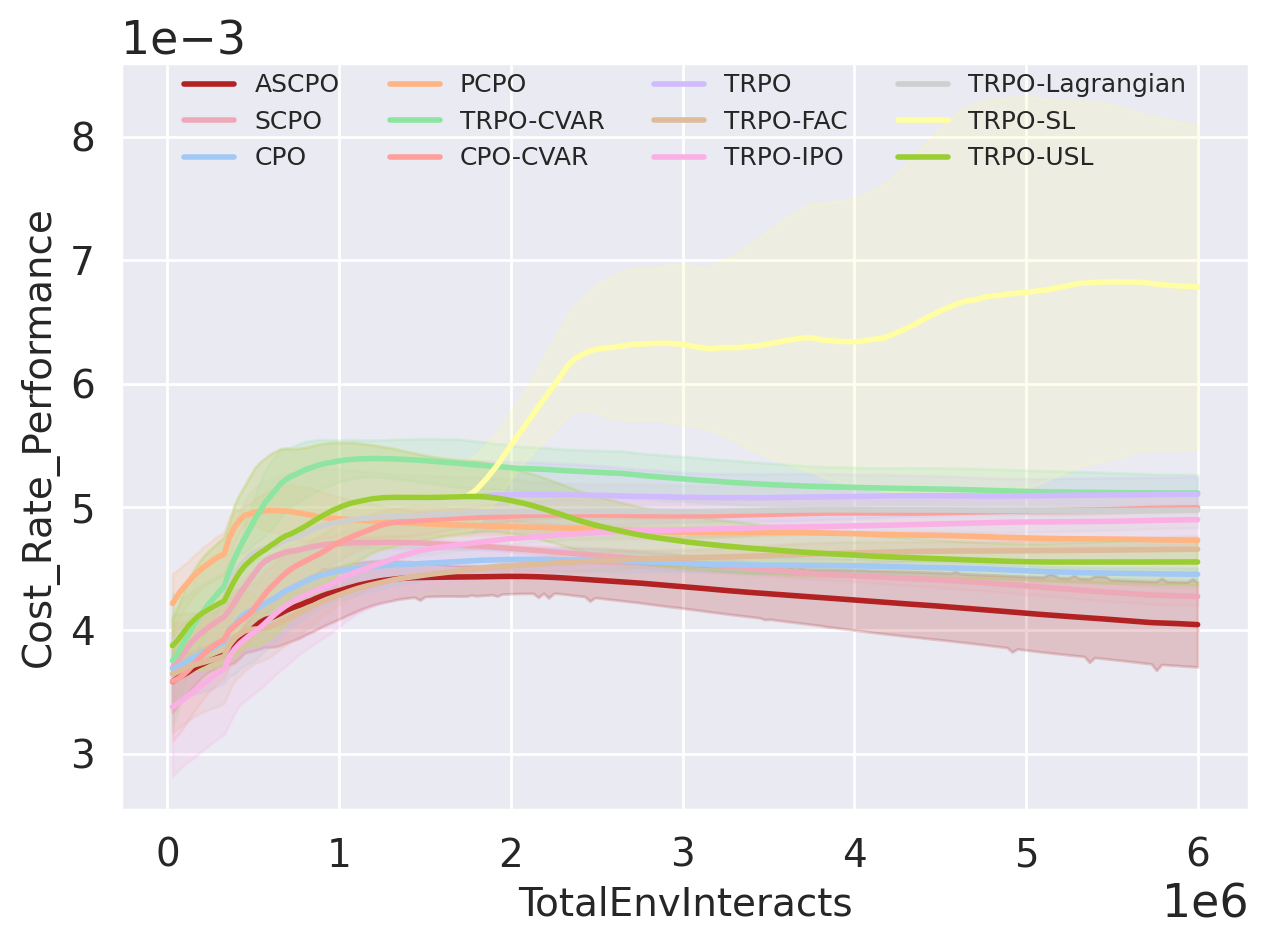}}
    \label{fig:swimmer-hazard-4-CostRate}
    \end{subfigure}
    \caption{Swimmer-4-Hazard}
    \label{fig:swimmer-hazard-4}
    \end{subfigure}
    \begin{subfigure}[t]{0.32\textwidth}
    \begin{subfigure}[t]{1.00\textwidth}
        \raisebox{-\height}{\includegraphics[height=0.7\textwidth]{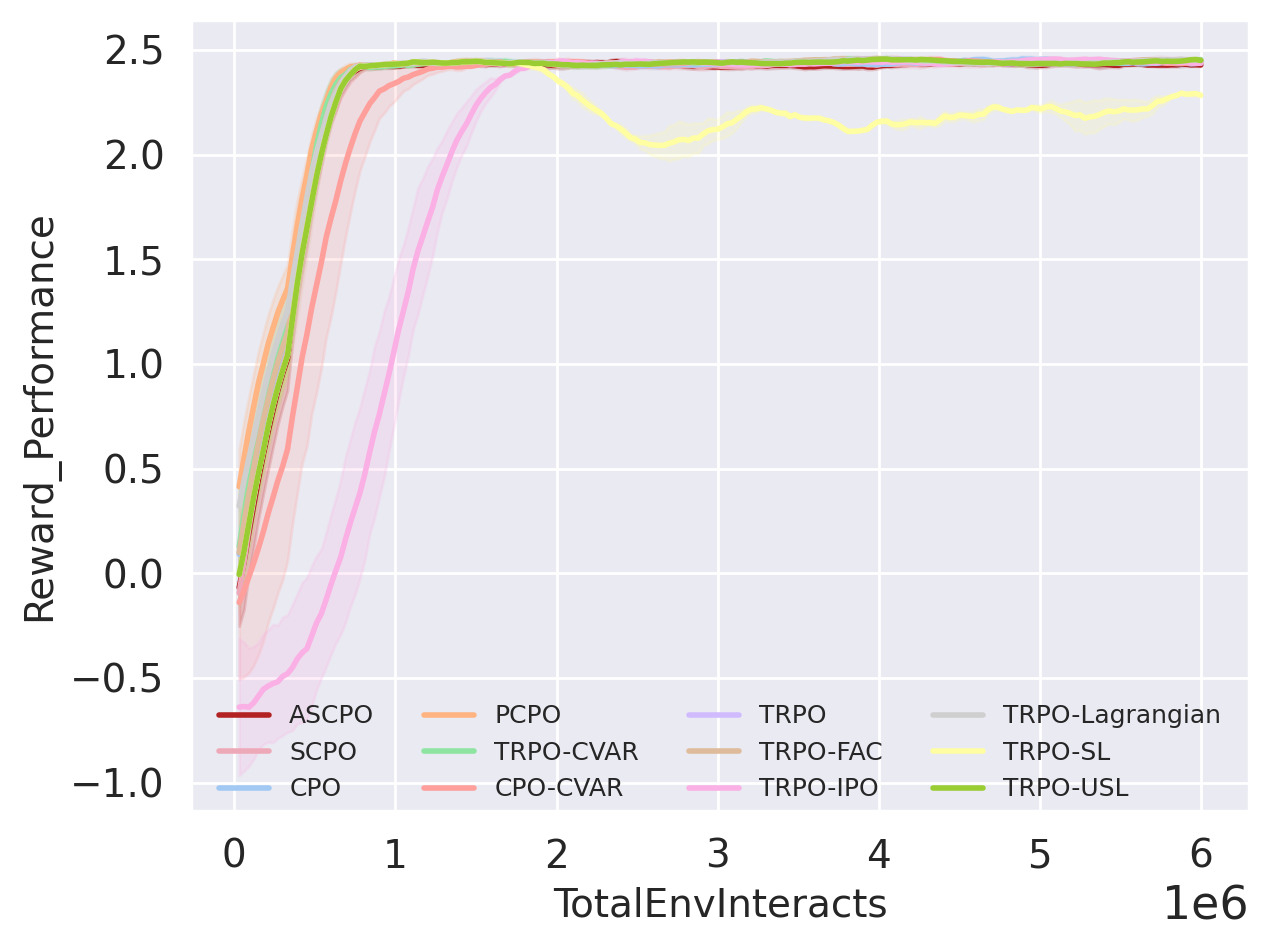}}
        \label{fig:swimmer-hazard-8-Performance}
    \end{subfigure}
    \hfill
    \begin{subfigure}[t]{1.00\textwidth}
        \raisebox{-\height}{\includegraphics[height=0.7\textwidth]{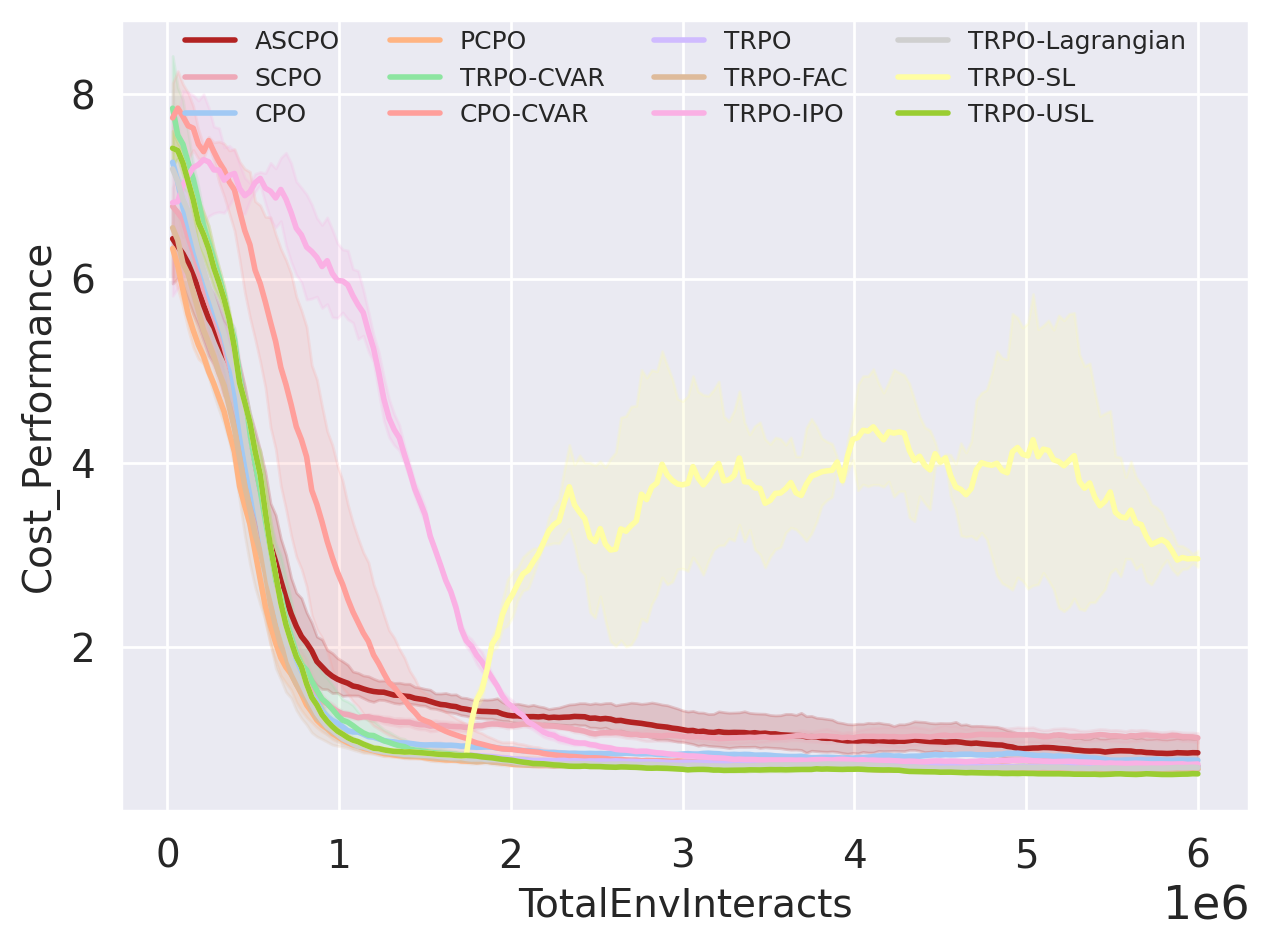}}
    \label{fig:swimmer-hazard-8-AverageEpCost}
    \end{subfigure}
    \hfill
    \begin{subfigure}[t]{1.00\textwidth}
        \raisebox{-\height}{\includegraphics[height=0.7\textwidth]{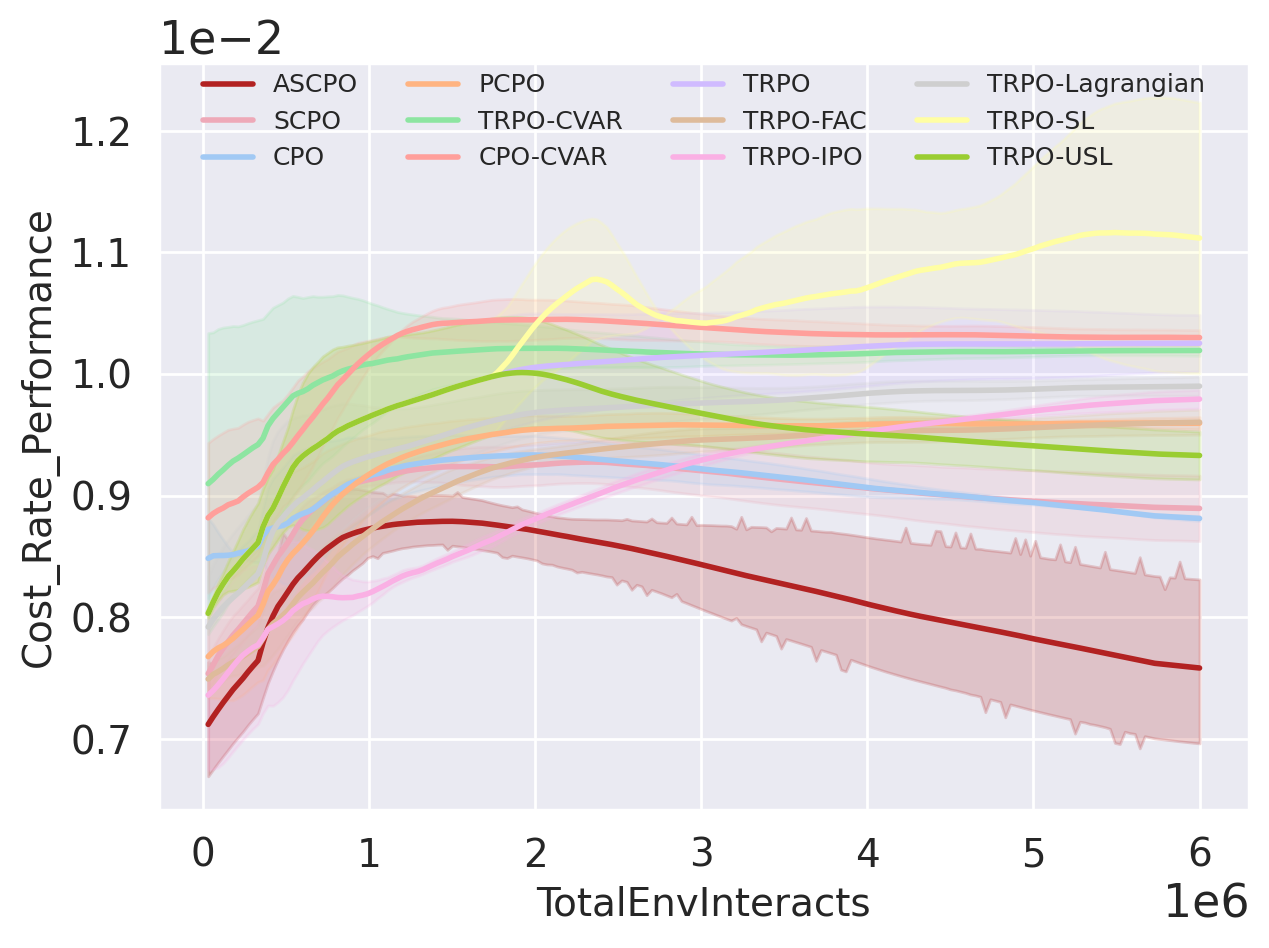}}
    \label{fig:swimmer-hazard-8-CostRate}
    \end{subfigure}
    \caption{Swimmer-8-Hazard}
    \label{fig:swimmer-hazard-8}
    \end{subfigure}
    \caption{Swimmer-Hazard}
    \label{fig:exp-swimmer-hazard}
\end{figure}

\begin{figure}[p]
    \centering
    \begin{subfigure}[t]{0.32\textwidth}
    \begin{subfigure}[t]{1.00\textwidth}
        \raisebox{-\height}{\includegraphics[height=0.7\textwidth]{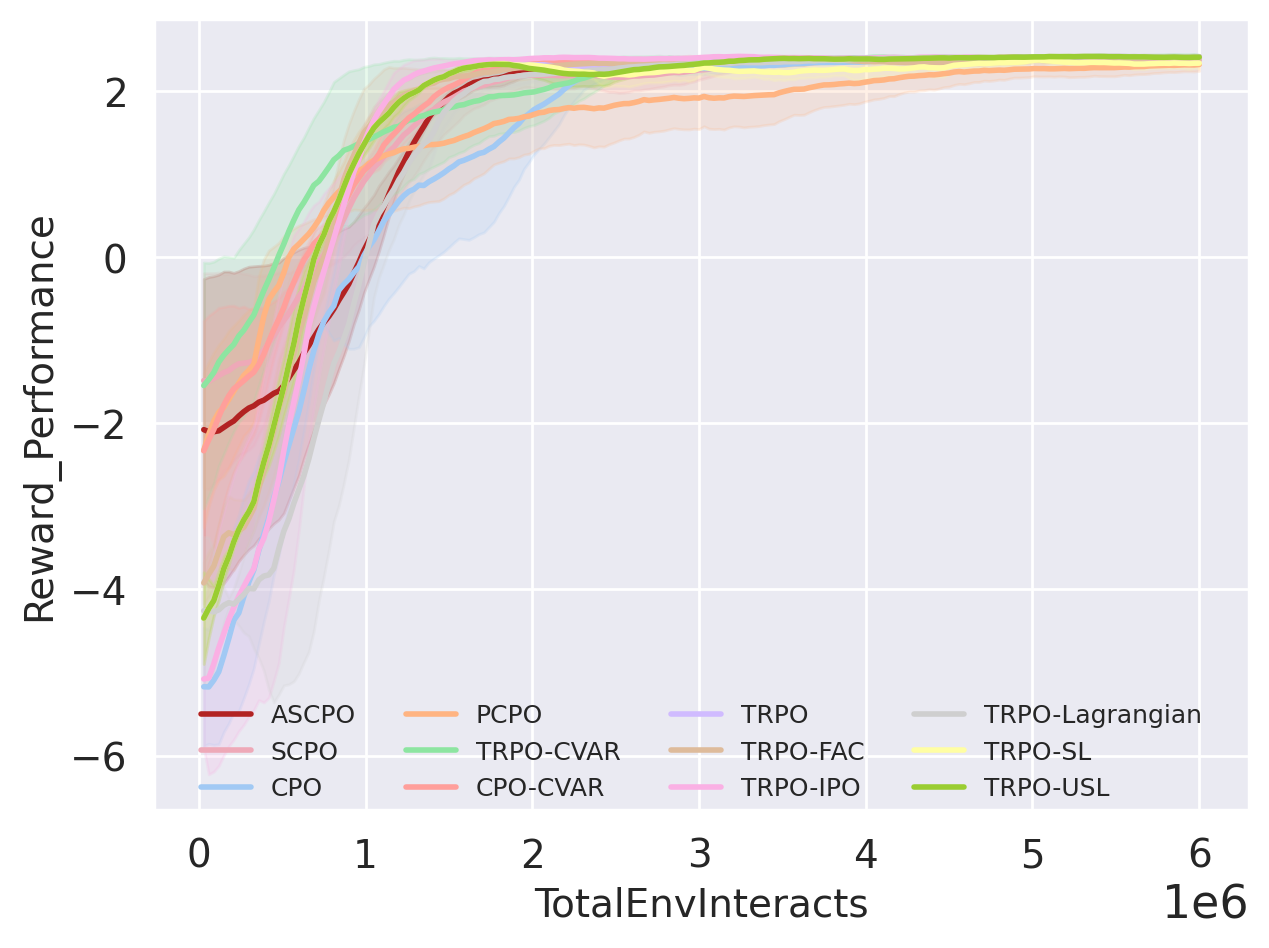}}
        \label{fig:Drone-hazard-1-Performance}
    \end{subfigure}
    \hfill
    \begin{subfigure}[t]{1.00\textwidth}
        \raisebox{-\height}{\includegraphics[height=0.7\textwidth]{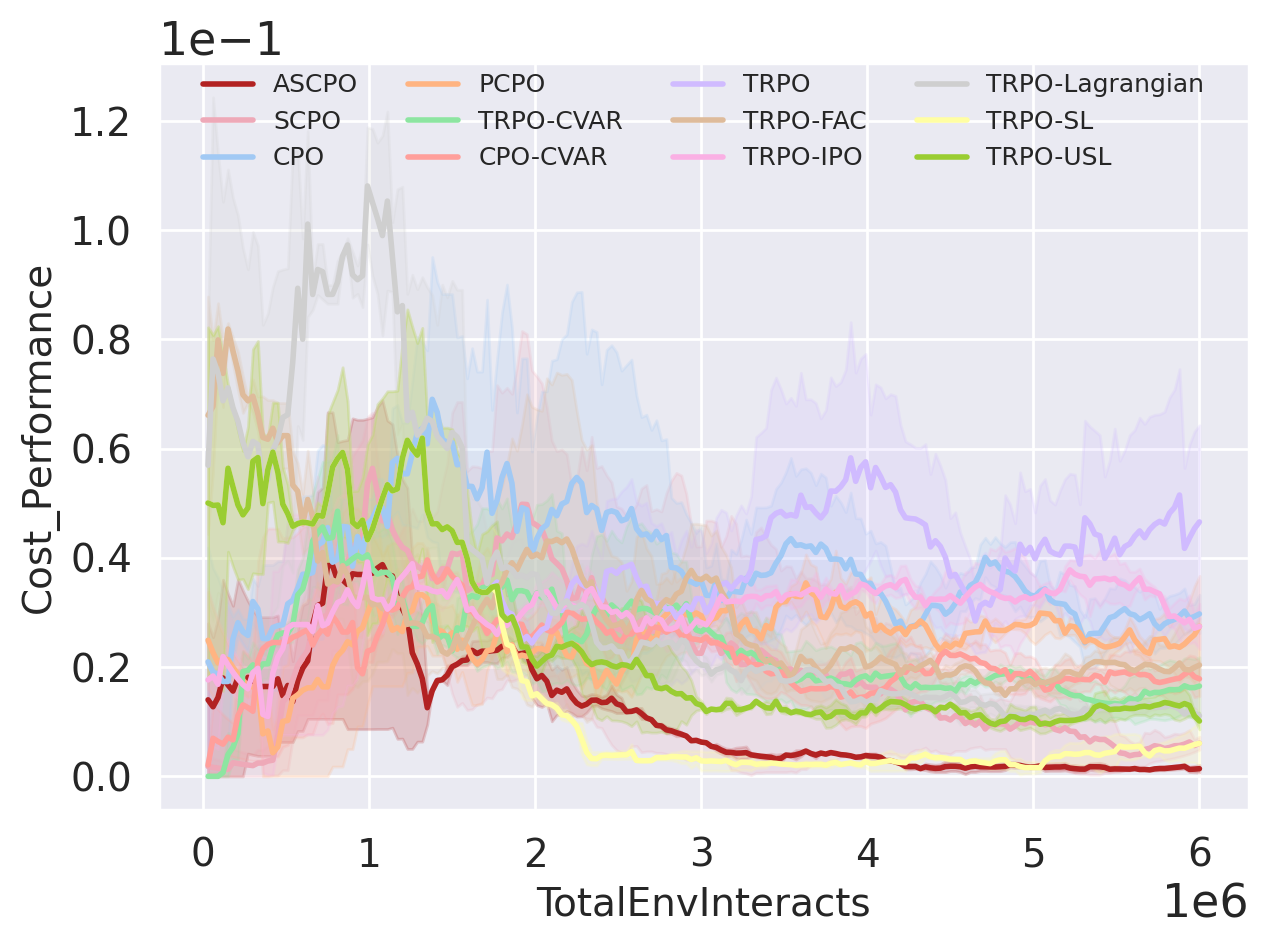}}
    \label{fig:Drone-hazard-1-AverageEpCost}
    \end{subfigure}
    \hfill
    \begin{subfigure}[t]{1.00\textwidth}
        \raisebox{-\height}{\includegraphics[height=0.7\textwidth]{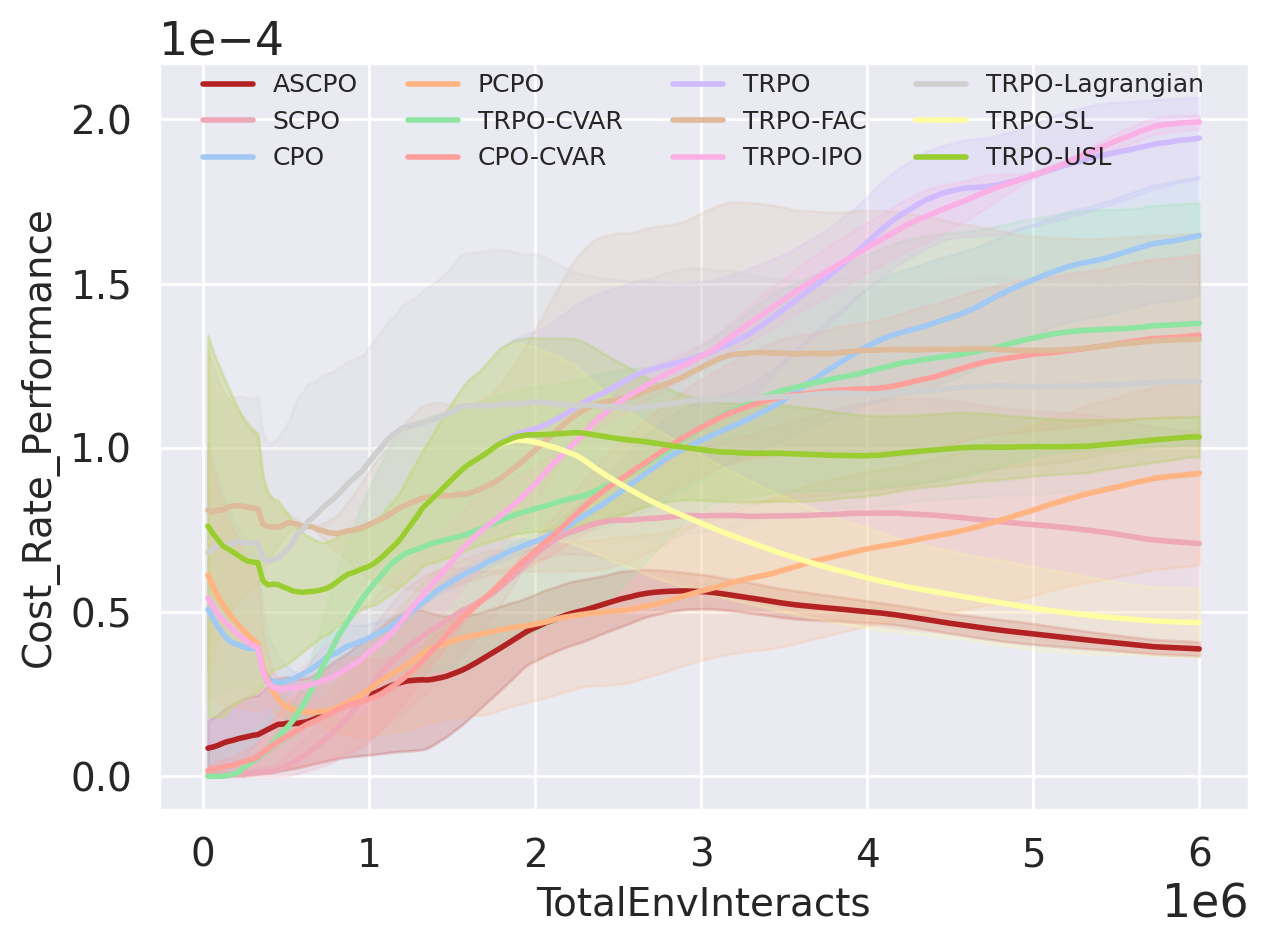}}
    \label{fig:Drone-hazard-1-CostRate}
    \end{subfigure}
    \caption{Drone-1-Hazard}
    \label{fig:Drone-hazard-1}
    \end{subfigure}
   \begin{subfigure}[t]{0.32\textwidth}
    \begin{subfigure}[t]{1.00\textwidth}
        \raisebox{-\height}{\includegraphics[height=0.7\textwidth]{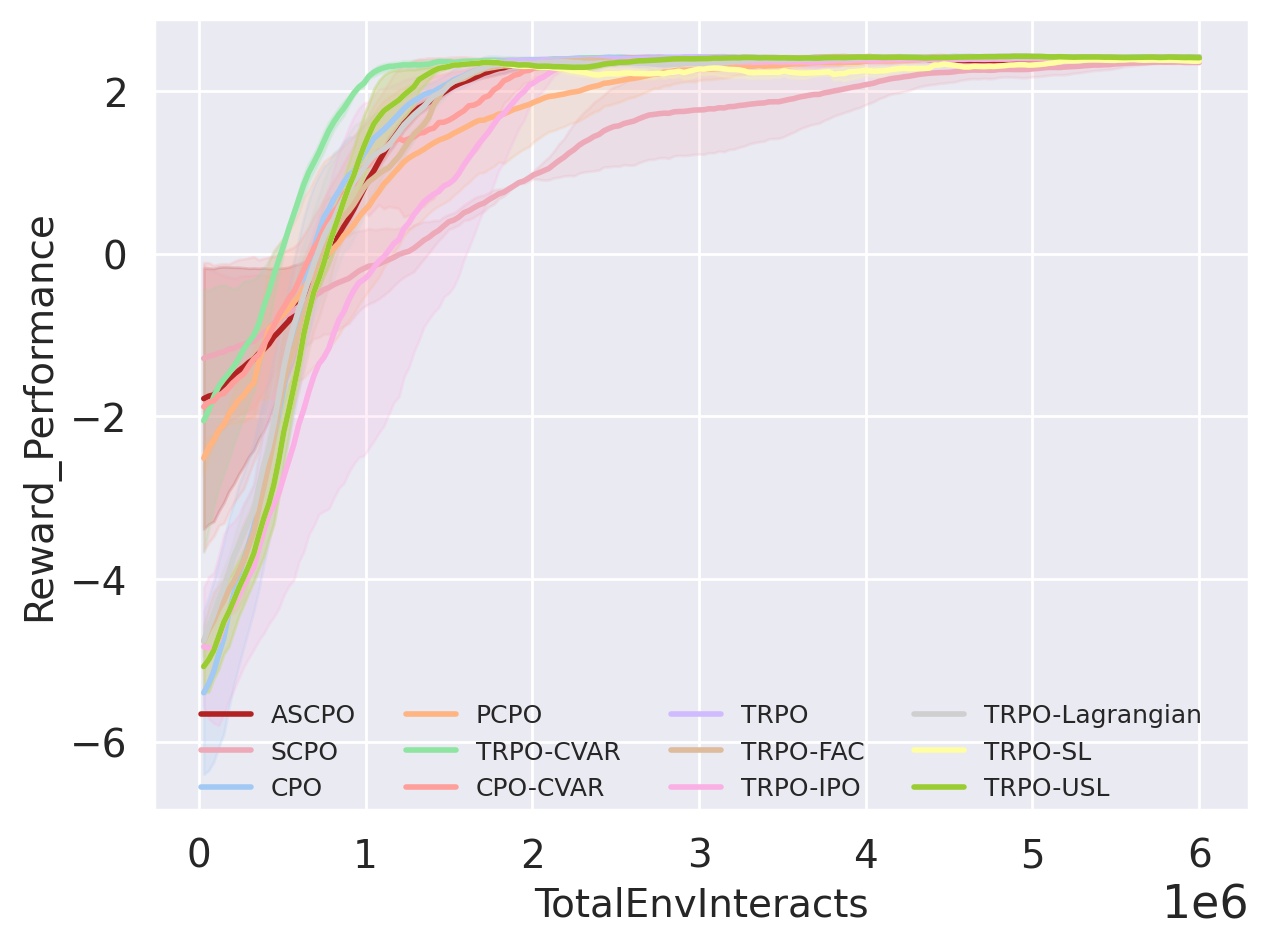}}
        \label{fig:Drone-hazard-4-Performance}
    \end{subfigure}
    \hfill
    \begin{subfigure}[t]{1.00\textwidth}
        \raisebox{-\height}{\includegraphics[height=0.7\textwidth]{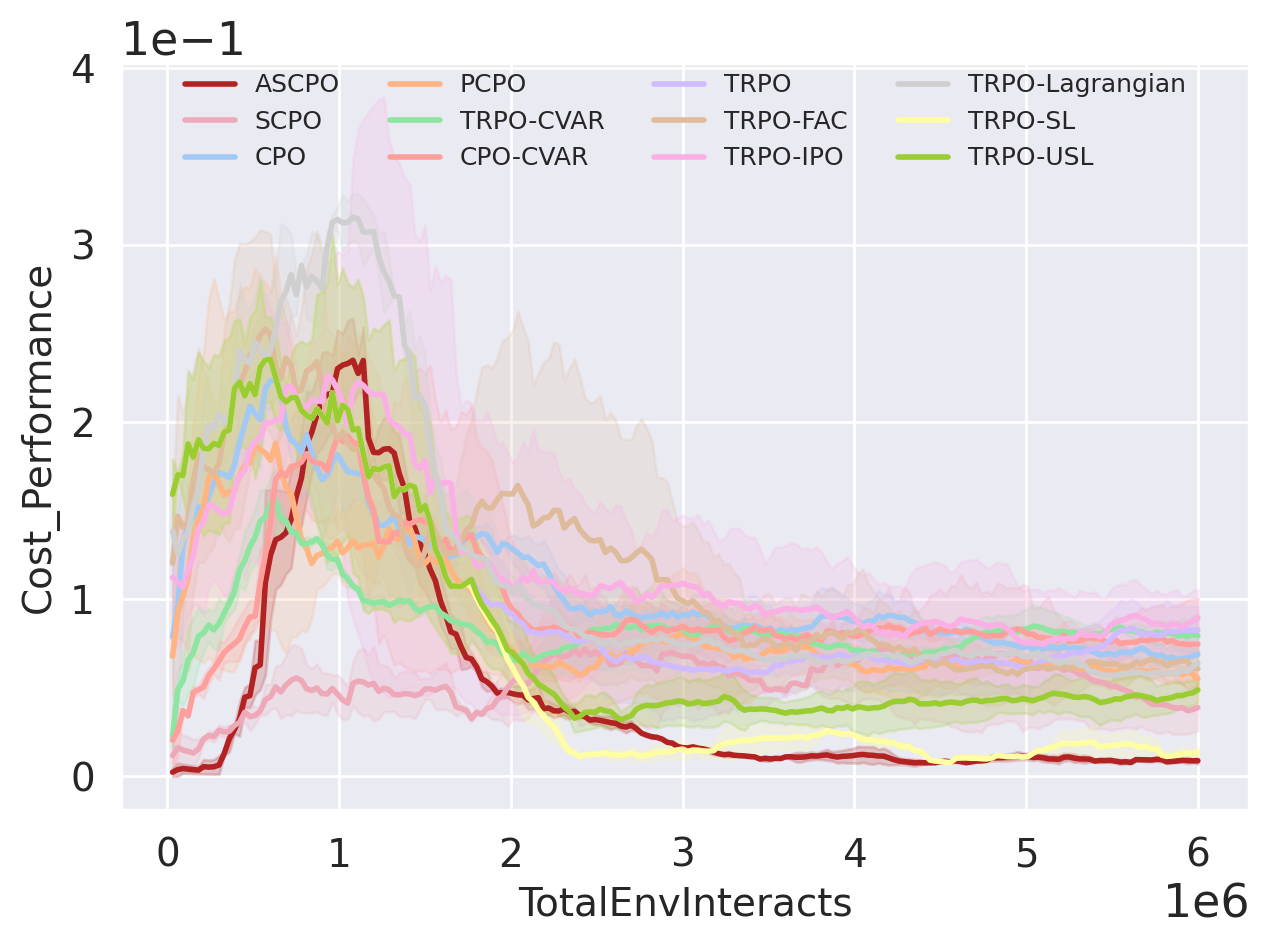}}
    \label{fig:Drone-hazard-4-AverageEpCost}
    \end{subfigure}
    \hfill
    \begin{subfigure}[t]{1.00\textwidth}
        \raisebox{-\height}{\includegraphics[height=0.7\textwidth]{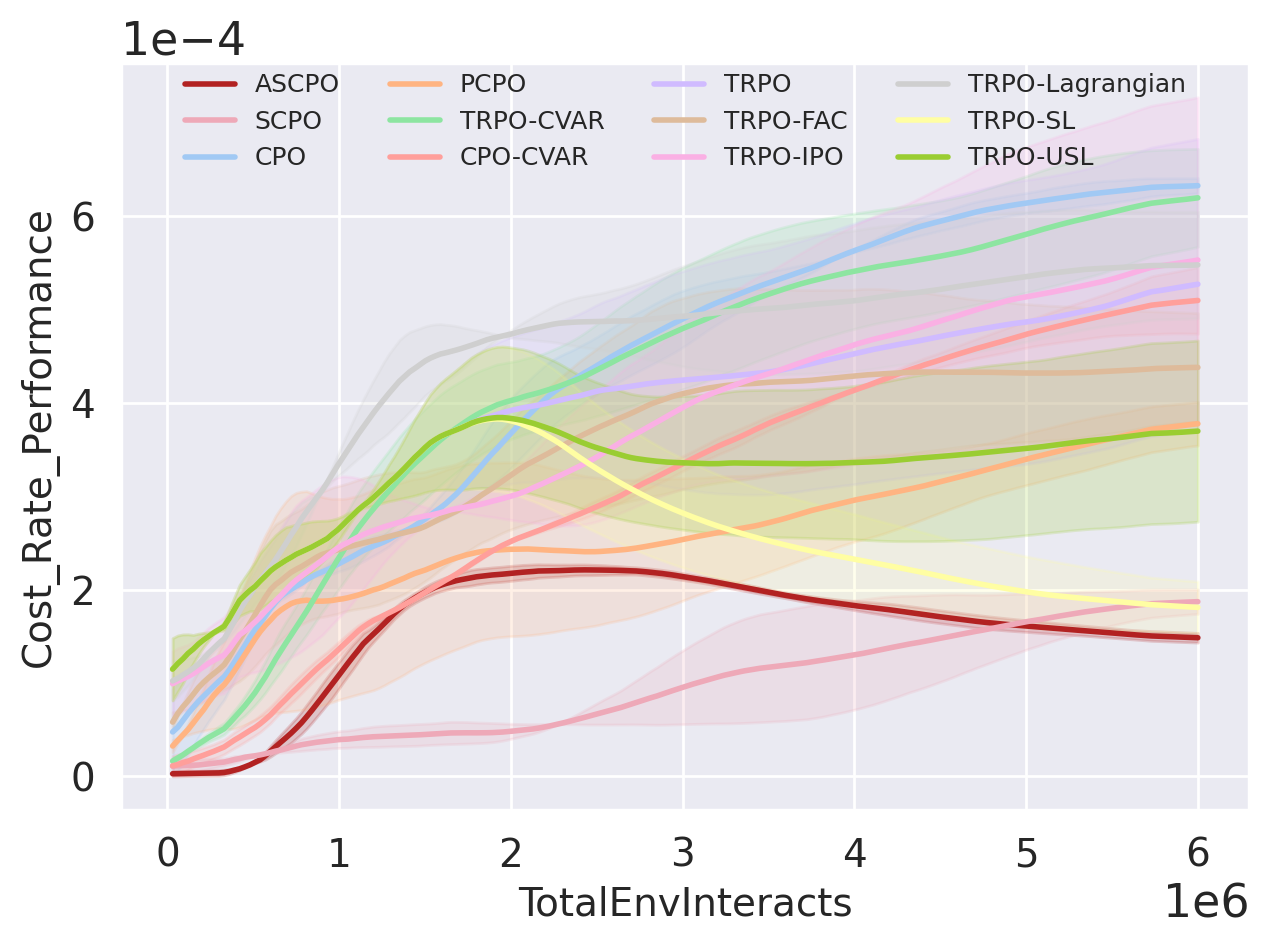}}
    \label{fig:Drone-hazard-4-CostRate}
    \end{subfigure}
    \caption{Drone-4-Hazard}
    \label{fig:Drone-hazard-4}
    \end{subfigure}
    \begin{subfigure}[t]{0.32\textwidth}
    \begin{subfigure}[t]{1.00\textwidth}
        \raisebox{-\height}{\includegraphics[height=0.7\textwidth]{fig/guard/Goal_Drone_8Hazards/Goal_Drone_8Hazards_Reward_Performance.png}}
        \label{fig:Drone-hazard-8-Performance}
    \end{subfigure}
    \hfill
    \begin{subfigure}[t]{1.00\textwidth}
        \raisebox{-\height}{\includegraphics[height=0.7\textwidth]{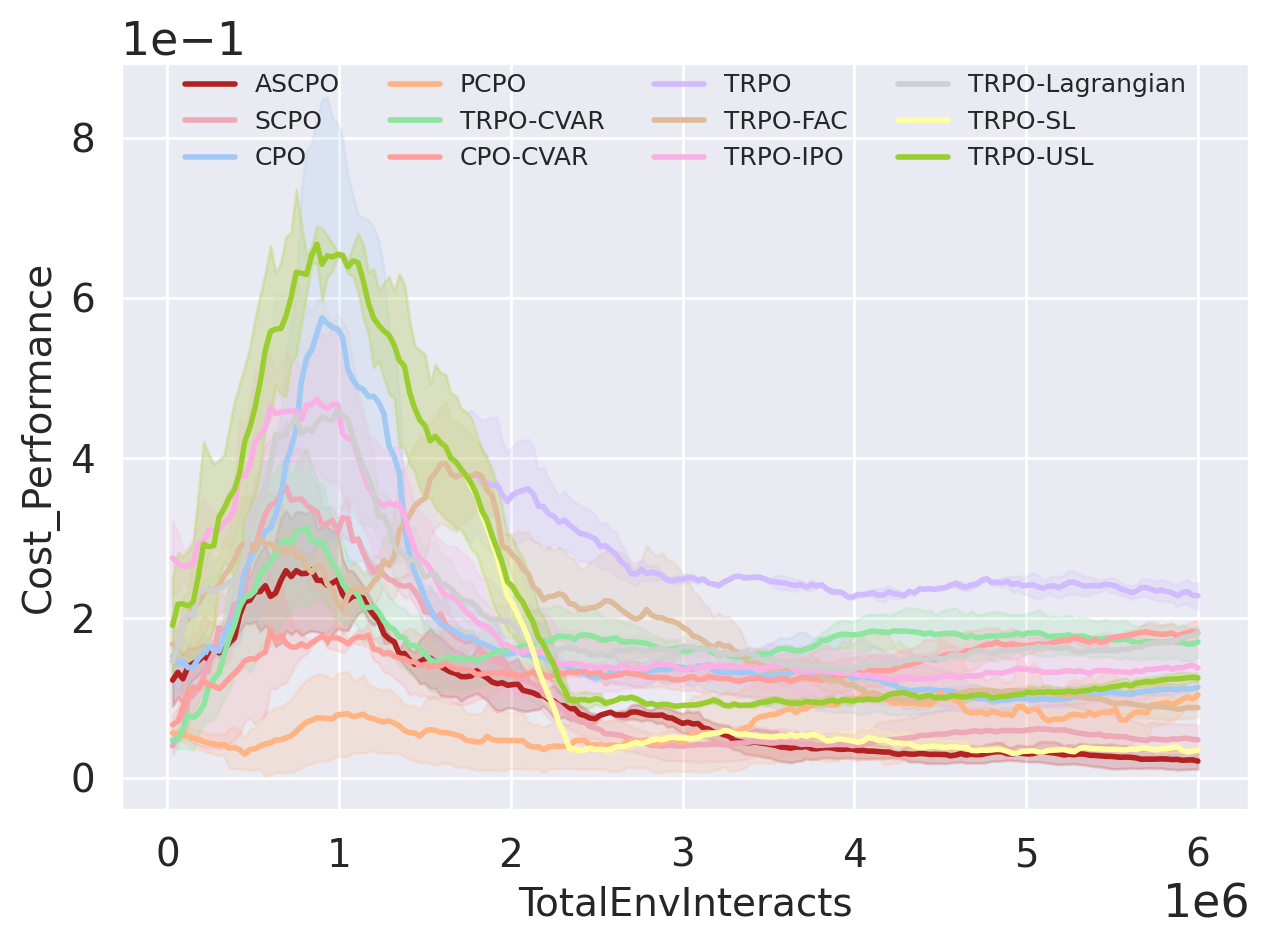}}
    \label{fig:Drone-hazard-8-AverageEpCost}
    \end{subfigure}
    \hfill
    \begin{subfigure}[t]{1.00\textwidth}
        \raisebox{-\height}{\includegraphics[height=0.7\textwidth]{fig/guard/Goal_Drone_8Hazards/Goal_Drone_8Hazards_Cost_Rate_Performance.png}}
    \label{fig:Drone-hazard-8-CostRate}
    \end{subfigure}
    \caption{Drone-8-Hazard}
    \label{fig:drone-hazard-8}
    \end{subfigure}
    \caption{Drone-Hazard}
    \label{fig:exp-drone-hazard}
\end{figure}

\begin{figure}[p]
    \centering
    \begin{subfigure}[t]{0.24\textwidth}
    \begin{subfigure}[t]{1\textwidth}
        \raisebox{-\height}{\includegraphics[height=0.7\textwidth]{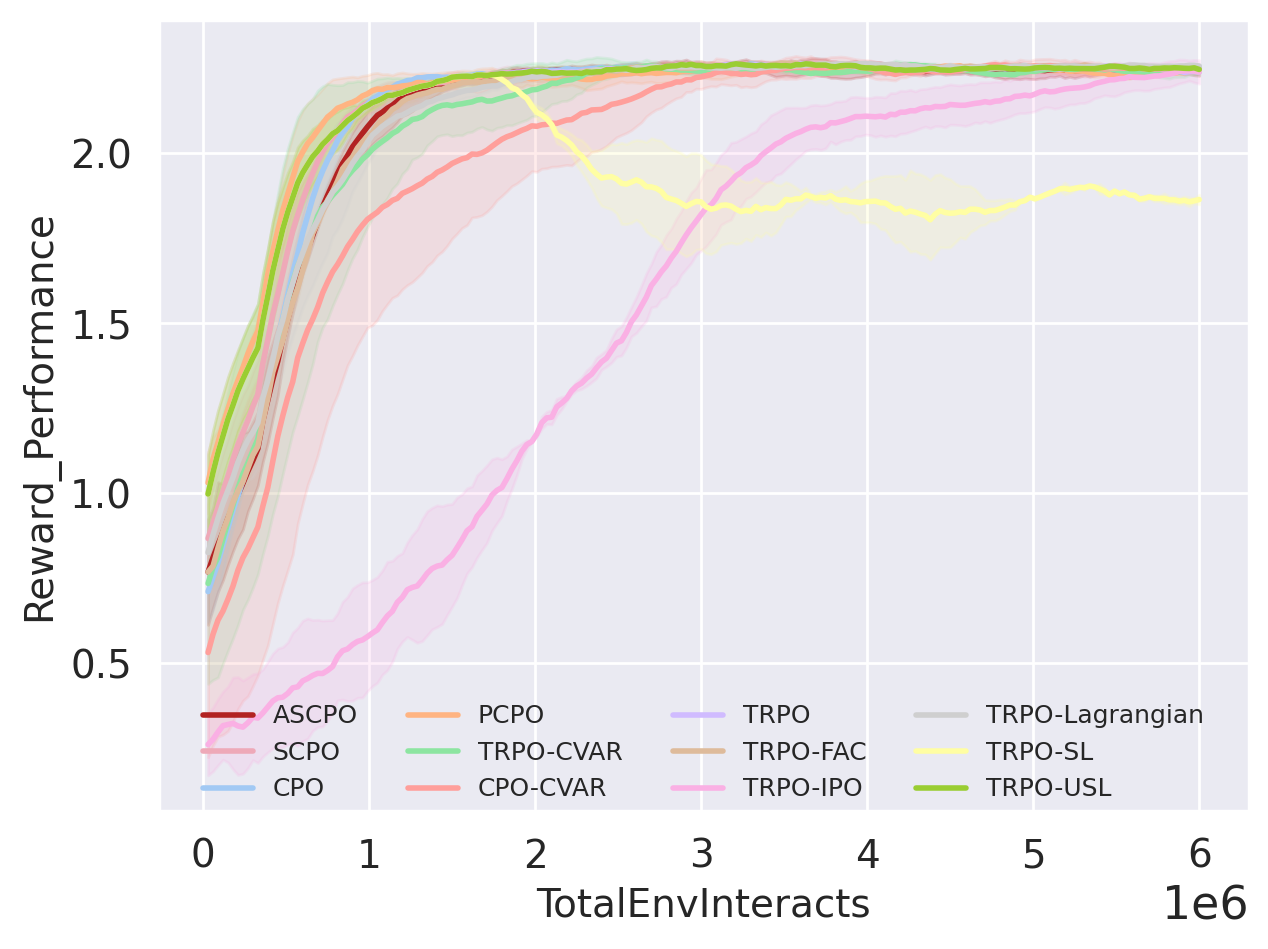}}
        \label{fig:arm3-hazard-8-Performance}
    \end{subfigure}
    \hfill
    \begin{subfigure}[t]{1\textwidth}
        \raisebox{-\height}{\includegraphics[height=0.7\textwidth]{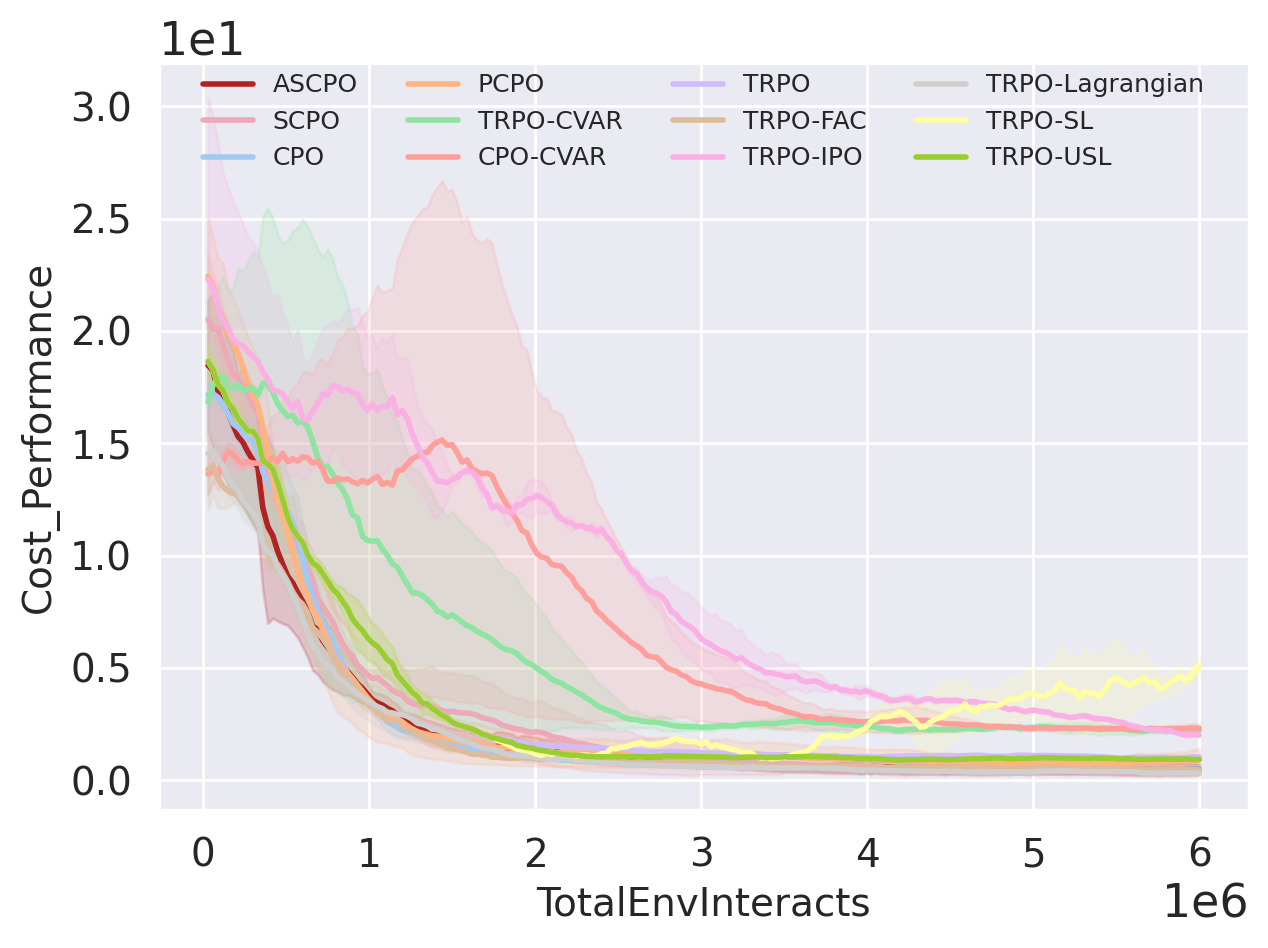}}
    \label{fig:arm3-hazard-8-AverageEpCost}
    \end{subfigure}
    \hfill
    \begin{subfigure}[t]{1\textwidth}
        \raisebox{-\height}{\includegraphics[height=0.7\textwidth]{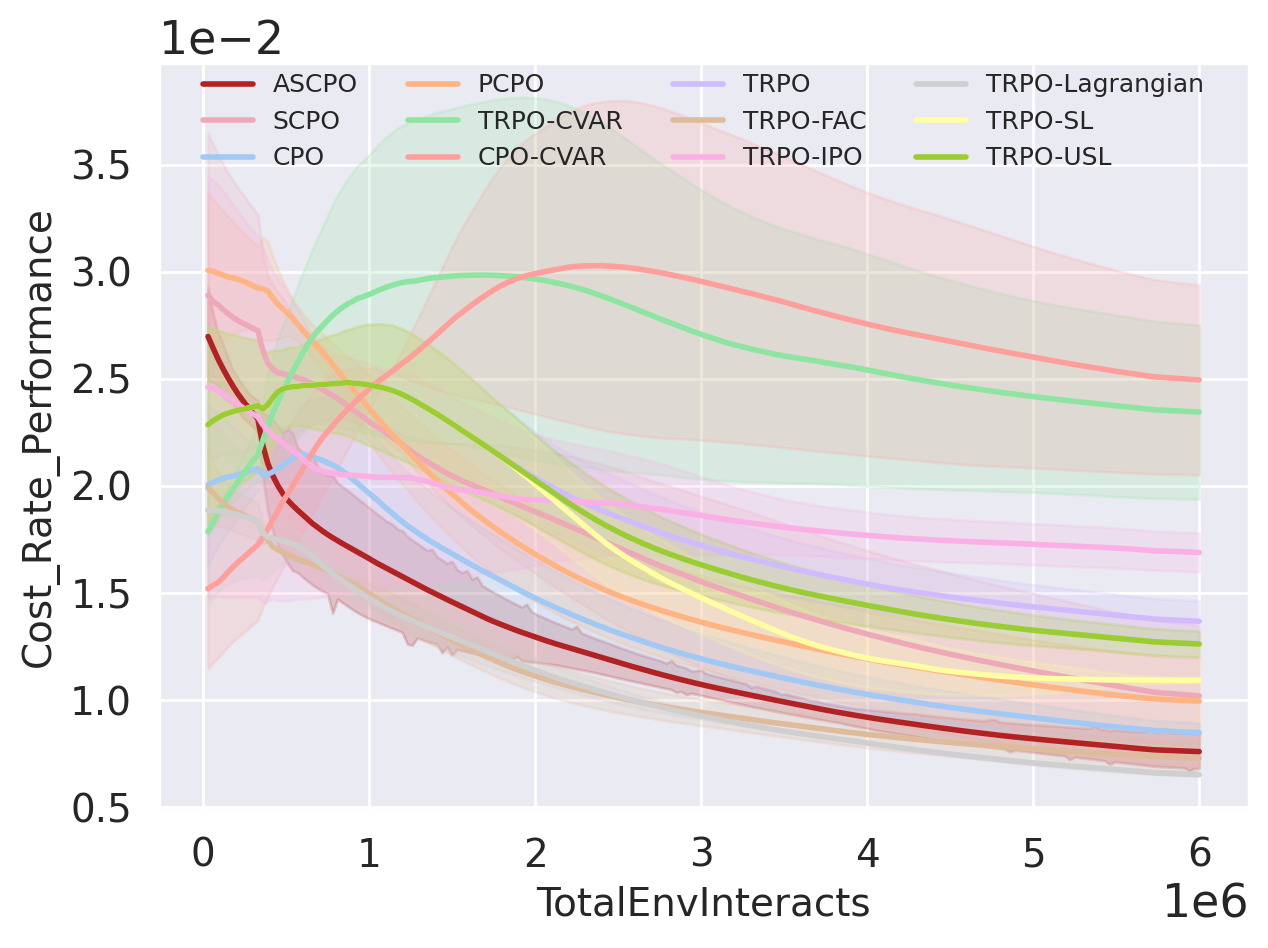}}
    \label{fig:arm3-hazard-8-CostRate}
    \end{subfigure}
    \label{fig:arm3-hazard-8}
    \caption{Arm3-8-Hazard}
    \end{subfigure}
    \begin{subfigure}[t]{0.24\textwidth}
    \begin{subfigure}[t]{1\textwidth}
        \raisebox{-\height}{\includegraphics[height=0.7\textwidth]{fig/guard/Goal_Humanoid/Goal_Humanoid_Reward_Performance.png}}
        \label{fig:humanoid-hazard-8-Performance}
    \end{subfigure}
    \hfill
    \begin{subfigure}[t]{1\textwidth}
        \raisebox{-\height}{\includegraphics[height=0.7\textwidth]{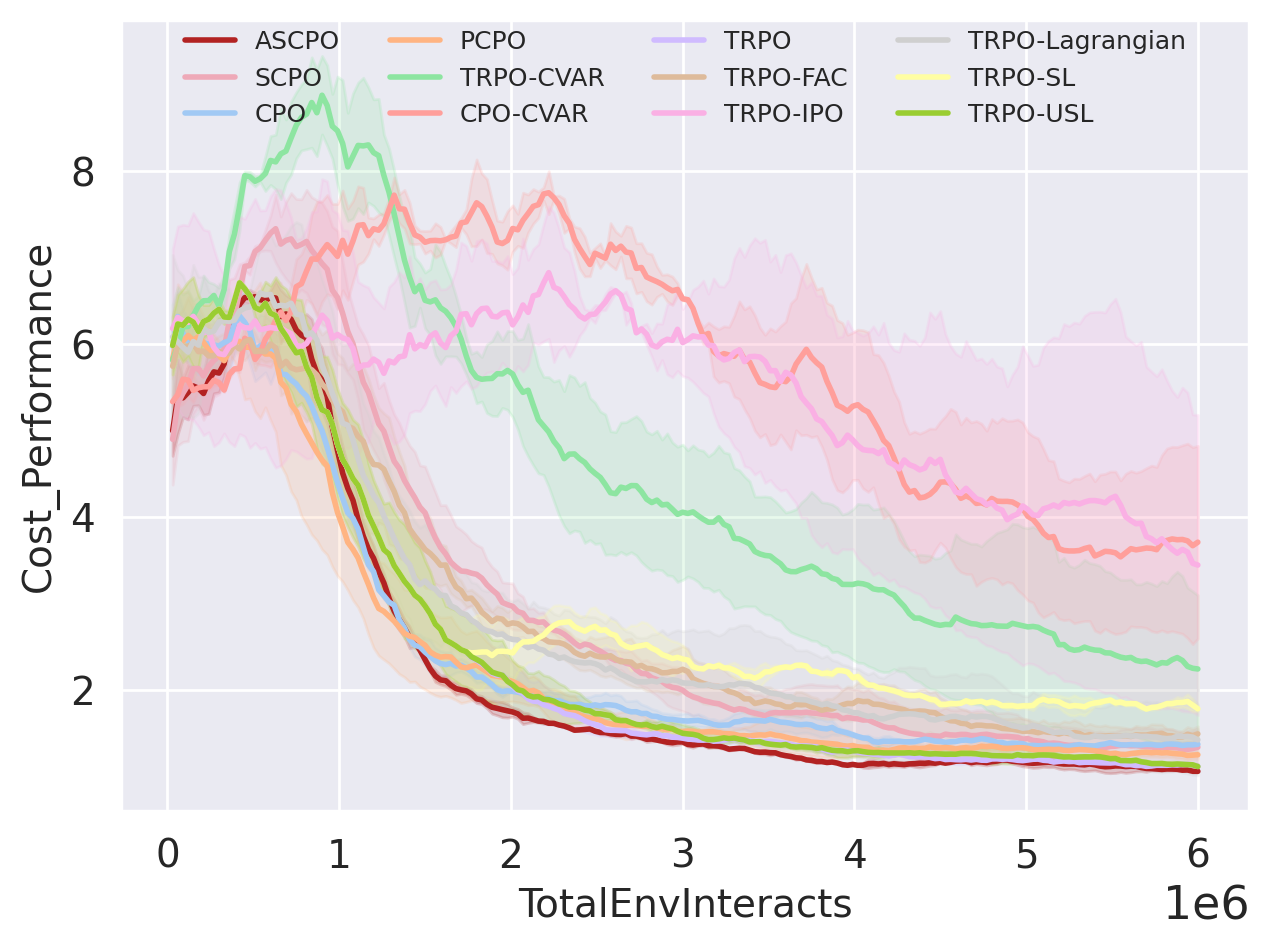}}
    \label{fig:humanoid-hazard-8-AverageEpCost}
    \end{subfigure}
    \hfill
    \begin{subfigure}[t]{1\textwidth}
        \raisebox{-\height}{\includegraphics[height=0.7\textwidth]{fig/guard/Goal_Humanoid/Goal_Humanoid_Cost_Rate_Performance.png}}
    \label{fig:humanoid-hazard-8-CostRate}
    \end{subfigure}
    \label{fig:humanoid-hazard-8}
    \caption{Humanoid-8-Hazard}
    \end{subfigure}
    \begin{subfigure}[t]{0.24\textwidth}
    \begin{subfigure}[t]{1\textwidth}
        \raisebox{-\height}{\includegraphics[height=0.7\textwidth]{fig/guard/Goal_Point_8Hazards/Goal_Point_8Hazards_Reward_Performance.png}}
        \label{fig:ant-hazard-8-Performance}
    \end{subfigure}
    \hfill
    \begin{subfigure}[t]{1\textwidth}
        \raisebox{-\height}{\includegraphics[height=0.7\textwidth]{fig/guard/Goal_Point_8Hazards/Goal_Point_8Hazards_Cost_Performance.png}}
    \label{fig:ant-hazard-8-AverageEpCost}
    \end{subfigure}
    \hfill
    \begin{subfigure}[t]{1\textwidth}
        \raisebox{-\height}{\includegraphics[height=0.7\textwidth]{fig/guard/Goal_Point_8Hazards/Goal_Point_8Hazards_Cost_Rate_Performance.png}}
    \label{fig:ant-hazard-8-CostRate}
    \end{subfigure}
    \label{fig:ant-hazard-8}
    \caption{Ant-8-Hazard}
    \end{subfigure}
    \begin{subfigure}[t]{0.24\textwidth}
    \begin{subfigure}[t]{1\textwidth}
        \raisebox{-\height}{\includegraphics[height=0.7\textwidth]{fig/guard/Goal_Walker/Goal_Walker_Reward_Performance.png}}
        \label{fig:walker-hazard-8-Performance}
    \end{subfigure}
    \hfill
    \begin{subfigure}[t]{1\textwidth}
        \raisebox{-\height}{\includegraphics[height=0.7\textwidth]{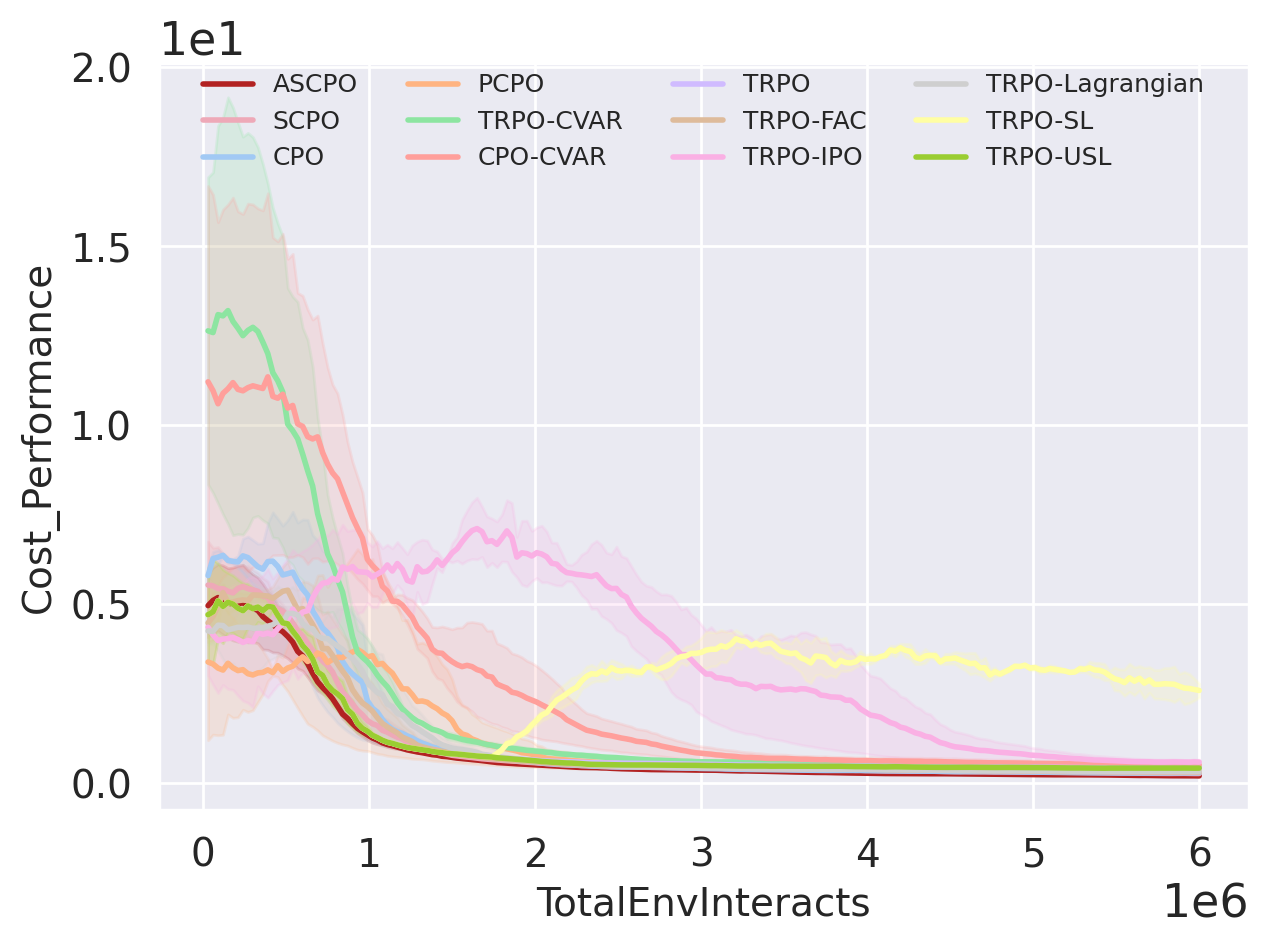}}
    \label{fig:walker-hazard-8-AverageEpCost}
    \end{subfigure}
    \hfill
    \begin{subfigure}[t]{1\textwidth}
        \raisebox{-\height}{\includegraphics[height=0.7\textwidth]{fig/guard/Goal_Walker/Goal_Walker_Cost_Rate_Performance.png}}
    \label{fig:walker-hazard-8-CostRate}
    \end{subfigure}
    \label{fig:walker-hazard-8}
    \caption{Walker-8-Hazard}
    \end{subfigure}
    \caption{Other hazard tasks}
    \label{fig:exp-ant-walker-hazard}
\end{figure}

\newpage
\vskip 0.2in
\bibliography{reference}

\end{document}